%% file: arxivsubmission.tex
\documentclass{article}

\usepackage{PRIMEarxiv}

\usepackage[utf8]{inputenc}
\usepackage[T1]{fontenc}
\usepackage{microtype}

\usepackage{amsmath,amssymb,amsfonts,amsthm}
\usepackage{mathtools}
\usepackage{physics}
\usepackage{cases}
\usepackage{nicefrac}

\usepackage[dvipsnames,table]{xcolor}   % load ONCE, with all options
\usepackage{graphicx}
\graphicspath{{media/}}
\usepackage{pgfplotstable}              % loads pgfplots, tikz and pgf

\usepackage{booktabs}
\usepackage{multirow}
\usepackage{float}
\usepackage{caption}
\usepackage{subcaption}

\usepackage{algorithm}
\usepackage{algorithmic}
\usepackage[inline]{enumitem}
\usepackage[normalem]{ulem}
\usepackage{etoc}                       % table of contents for the appendix only

\usepackage[numbers]{natbib}
\renewcommand{\bibsection}{\section*{References}}

\usepackage{url}
\usepackage{hyperref}

\input{math_commands.tex}

\input{macros}

\newcommand{\figroot}{SOMA_exp2}
\newcommand{\curveroot}{paper_figures_v2}
\newcommand{\curvesroot}{atlas_opt_cpt_curves}
\input{\figroot/common/preamble.tex}

\title{Sharpness-aware Orthogonalized Momentum Algorithms}

\author{%
  Geetika\textsuperscript{\mdseries =,\,1,\,$*$},
  Shivani~Singh~Chauhan\textsuperscript{\mdseries =,\,2}, Gauranshi~Gupta\textsuperscript{\mdseries =,\,1,\,$\dagger$}, Anant~Jain\textsuperscript{\mdseries =,\,1,\,$\dagger$},\\
  \bfseries Kaushik~Kalyanaraman\textsuperscript{\mdseries 2},
  Bapi~Chatterjee\textsuperscript{\mdseries 1,\,$\dagger$}\\[0.5ex]
  \textsuperscript{1}Department of Computer Science and Engineering \quad
  \textsuperscript{2}Department of Mathematics \\
  Indraprastha Institute of Information Technology, Delhi, India.\\
  \texttt{\{geetikai, shivanic, gauranshig, anantj, kaushik, bapi\}@iiitd.ac.in}
}
\begin{document}
\maketitle

%% ---------- Definitions of the superscripts for the author block ----------
\begin{NoHyper}
{\renewcommand{\thefootnote}{=}%
 \footnotetext{~These authors contributed equally to this work.}%
 \renewcommand{\thefootnote}{$*$}%
 \footnotetext{Supported in part by the INSPIRE fellowship No.\ DST/INSPIRE Fellowship/IF220579 offered by the Department of Science \& Technology (DST), Government of India.}%
 \renewcommand{\thefootnote}{$\dagger$}%
 \footnotetext{Supported in part by the Infosys Centre for Artificial Intelligence (CAI) at Indraprastha Institute of Information Technology, Delhi (IIIT-Delhi).}}
\end{NoHyper}

\begin{abstract}
  \input{sections/abstract}
\end{abstract}


\etocdepthtag.toc{mtmain}               % tag: main-text sections (hidden from the appendix TOC)

\section{Introduction}\label{sec:introduction}
\input{sections/intro}

\section{Related Work}\label{sec:relatedwork}
\input{sections/relatedwork}

\section{Problem Formulation}\label{sec:preliminaries}
\input{sections/prelim}

\section{Algorithm}\label{sec:algorithm}
\input{sections/algo}

\section{Convergence Analysis}\label{sec:convergence}
\input{sections/convergence}

\section{Experiments}\label{sec:experiments}
\input{sections/experiments}

\section{Conclusion}\label{sec:conc}
\input{sections/conclusion}

\section*{AI use statement}
In this work, we used a generative AI tool to help refine hypotheses, provide feedback on research methodology or experiments, and implement methods. We have not used generative AI tools for proving theorems, formulating mathematical claims, providing critical ingredients for proving mathematical claims, assisting in the writing of proofs, developing theoretical models or conceptual frameworks, interpreting results, and assisting with translation, cleaning, and reformatting the dataset, supporting qualitative and thematic data analysis; these are not applicable to this work.
Additionally, we used generative AI tools to draft parts of a research paper, summarize or analyze existing literature, discover research topics or identify gaps, brainstorm, source or search for information, and edit a research paper to improve readability and identify relevant literature. We have reviewed all AI-assisted work, including reading and reviewing all the suggested related papers. We take responsibility for the final content of this work,
including text, claims, or artifacts produced with the aid of generative AI.

\bibliographystyle{plainnat}
\bibliography{refs}

%% ================= APPENDIX =================
\clearpage
\appendix
\etocdepthtag.toc{mtappendix}           % tag: appendix sections

\begin{center}
  {\Large\bfseries Sharpness-aware Orthogonalized Momentum Algorithms:\\[2pt] Supplementary Material}
\end{center}

\etocsettagdepth{mtmain}{none}          % hide main-text sections
\etocsettagdepth{mtappendix}{subsection} % show appendix sections and subsections
\renewcommand{\contentsname}{Contents of the Appendix}
\tableofcontents
\clearpage

\section{Variants of \sam + Muon}
\label{app:algo}
\input{sections/variants}

\section{Convergence Analysis under Heavy-Tailed Gradient Noise}
\input{sections/heavy}

\section{Convergence Analysis under Standard Bounded-Variance Assumption}
\input{sections/Non-heavy_tailed}

\section{SOMA Analysis}
\input{sections/SOMA_Lemma}

\section{Additional Heavy-tailed Experiments}\label{app:isotropic}
\input{isotropic_heavy_tailed_appendix}

\section{Additional LLM Experiments}\label{app:llm}
\input{sections/llmexperiment}

\end{document}

%% file: math_commands.tex
\usepackage{amsmath,amsfonts,bm}

\def\eqref#1{equation~\ref{#1}}
\def\1{\bm{1}}

\DeclareMathAlphabet{\mathsfit}{\encodingdefault}{\sfdefault}{m}{sl}
\SetMathAlphabet{\mathsfit}{bold}{\encodingdefault}{\sfdefault}{bx}{n}

\def\gF{{\mathcal{F}}}

\def\gO{{\mathcal{O}}}

\def\gV{{\mathcal{V}}}

\def\sR{{\mathbb{R}}}

\newcommand{\E}{\mathbb{E}}

\newcommand{\R}{\mathbb{R}}

\DeclareMathOperator*{\argmax}{arg\,max}

\usepackage{amsmath}

%% file: macros.tex
\usepackage{xspace,xcolor}
\definecolor{LightCyan}{rgb}{0.88,1,1}
\definecolor{Ao}{rgb}{0.0, 0.5, 0.0}
\definecolor{cadmiumorange}{rgb}{0.93, 0.53, 0.18}
\definecolor{cardinal}{rgb}{0.77, 0.12, 0.23}
\definecolor{byzantium}{rgb}{0.44, 0.16, 0.39}
\definecolor{gamboge}{rgb}{0.89, 0.61, 0.06}
\definecolor{goldenbrown}{rgb}{0.6, 0.4, 0.08}

\newcommand{\sgd}{\textsc{SGD}\xspace}

\newcommand{\adam}{\textsc{Adam}\xspace}

\newcommand{\sam}{\textsc{Sam}\xspace}
\newcommand{\fsam}{\textsc{F{-}Sam}\xspace}

\newcommand{\specsammuon}{\textsc{Ssm}\xspace}
\newcommand{\randsammuon}{\textsc{Rsm}\xspace}
\newcommand{\soma}{\textsc{Soma}\xspace}
\newcommand{\opsoma}{\textsc{OP{-}Soma}\xspace}
\newcommand{\opsomapre}{\textsc{Opspre}\xspace}
\newcommand{\opsomapost}{\textsc{Opspost}\xspace}
\newcommand{\samadamw}{\textsc{Saw}\xspace}
\newcommand{\fsamadamw}{\textsc{Fsaw}\xspace}
\newcommand{\muon}{\textsc{Muon}\xspace}
\newcommand{\fpsoma}{\textsc{FP{-}Soma}\xspace}
\newcommand{\adamw}{\textsc{AdamW}\xspace}
\newcommand{\erm}{\textsc{ERM}\xspace}
\newcommand{\ptq}{\textsc{PTQ}\xspace}
\newcommand{\ema}{\textsc{EMA}\xspace}
\newcommand{\msoma}{\textsc{MSOMA}\xspace}
\newcommand{\msam}{\textsc{MSAM}\xspace}
\newcommand{\saw}{\textsc{SAW}\xspace}
\newcommand{\fsaw}{\textsc{FSAW}\xspace}
\newcommand{\Stwosam}{\textsc{S$^2${-}Sam}\xspace}

\newcommand{\tn}{\stackrel{\sim}{\smash{{\nabla}}\rule{0pt}{1.1ex}}\kern-1ex}

\makeatletter
\newcommand{\nosemic}{\renewcommand{\@endalgocfline}{\relax}}% Drop semi-colon ;
\newcommand{\dosemic}{\renewcommand{\@endalgocfline}{\algocf@endline}}% 
\let\oldnl\nl% Store \nl in \oldnl
\newcommand{\nonl}{\renewcommand{\nl}{\let\nl\oldnl}}% Remove line number for 
\theoremstyle{plain}
\newtheorem{theorem}{Theorem}
\newtheorem{corollary}{Corollary}%[section]
\newtheorem{definition}{Definition}%[section]
\newtheorem{lemma}{Lemma}%[section]
\theoremstyle{definition}
\newtheorem{remark}{Remark}%[section]
\newtheorem{assumption}{Assumption}%[section]
\makeatother

\theoremstyle{plain}
\newtheorem*{theorem*}{Theorem}

\theoremstyle{definition}

\theoremstyle{remark}
\newcommand{\G}{\nabla}

\renewcommand{\abs}[1]{\left\lvert #1 \right\rvert}
\renewcommand{\innerproduct}[2]{\left\langle #1, #2 \right\rangle}

\newcommand{\normF}[1]{\left\lVert #1 \right\rVert_F}
\newcommand{\normop}[1]{\left\lVert #1 \right\rVert_{\mathrm{op}}}
\newcommand{\normnuc}[1]{\left\lVert #1 \right\rVert_*}

\newcommand{\normtwo}[1]{\left\lVert #1 \right\rVert_2}

\newcommand{\innerF}[2]{\left\langle #1, #2 \right\rangle_F}

\usepackage{tikz}
\usetikzlibrary{calc,trees,positioning,arrows,chains,shapes.geometric,%
    decorations.pathreplacing,decorations.pathmorphing,shapes,%
    matrix,shapes.symbols}

\usepackage{tikz}

\tikzstyle{startstop} = [rectangle, draw, rounded corners, align=center, minimum width=3cm, minimum height=1cm,text centered]
\tikzstyle{decision} = [diamond, draw, fill=blue!20,
text width=4.5em, text badly centered, client distance=3cm, inner sep=0pt]
\tikzstyle{block} = [rectangle, draw, fill=blue!10, align=center, rounded corners, minimum width=3cm, minimum height=1cm]
\tikzstyle{blockcast} = [rectangle, draw, fill=red!10, align=center, rounded corners, minimum width=3cm, minimum height=0.45cm]
\tikzstyle{line} = [draw, -latex']
\tikzstyle{cloud} = [draw, ellipse,fill=red!20, client distance=3cm,
minimum height=2em]

\DeclareMathOperator{\ns}{Newton-Schulz}
\DeclareMathOperator{\ortho}{Ortho}

\newcommand{\projperp}[2]{\operatorname{Proj}^{\perp}_{#1}{(#2)}}
\newcommand{\NSQ}{\operatorname{NS}_q}

%% file: atlas_opt_cpt_curves/common/preamble.tex
\usepgfplotslibrary{groupplots}
\pgfplotsset{compat=1.17}
\providecommand{\curvesroot}{atlas_opt_cpt_curves}  % folder that contains data/
\providecommand{\fitwidth}[1]{\resizebox{\ifdim\width>\textwidth\textwidth\else\width\fi}{!}{#1}}  % shrink to the text width only if wider

\definecolor{ccv-fsam-ortho-muon-stale-momentum}{HTML}{0B5D1E}
\tikzset{cv-fsam-ortho-muon-stale-momentum/.style={color=ccv-fsam-ortho-muon-stale-momentum, solid, line width=1.1pt}}
\definecolor{ccv-fsam-ortho-muon}{HTML}{33CC33}
\tikzset{cv-fsam-ortho-muon/.style={color=ccv-fsam-ortho-muon, solid, line width=1.1pt}}
\definecolor{ccv-msoma}{HTML}{8FD694}
\tikzset{cv-msoma/.style={color=ccv-msoma, loosely dashed, line width=0.8pt}}
\definecolor{ccv-soma-prens5}{HTML}{9ACD32}
\tikzset{cv-soma-prens5/.style={color=ccv-soma-prens5, densely dotted, line width=1.1pt}}
\definecolor{ccv-op-soma-postns5}{HTML}{004D40}
\tikzset{cv-op-soma-postns5/.style={color=ccv-op-soma-postns5, dash dot, line width=1.1pt}}
\definecolor{ccv-muon-sam}{HTML}{00B3A4}
\tikzset{cv-muon-sam/.style={color=ccv-muon-sam, solid, line width=0.8pt}}
\definecolor{ccv-randsam-muon}{HTML}{808000}
\tikzset{cv-randsam-muon/.style={color=ccv-randsam-muon, dashed, line width=0.8pt}}
\definecolor{ccv-muon}{HTML}{111111}
\tikzset{cv-muon/.style={color=ccv-muon, dashed, line width=0.7pt}}
\definecolor{ccv-sam}{HTML}{0A3A8A}
\tikzset{cv-sam/.style={color=ccv-sam, solid, line width=0.8pt}}
\definecolor{ccv-fsam}{HTML}{2F80D8}
\tikzset{cv-fsam/.style={color=ccv-fsam, solid, line width=0.8pt}}
\definecolor{ccv-msam}{HTML}{7DB9F2}
\tikzset{cv-msam/.style={color=ccv-msam, solid, line width=0.8pt}}
\definecolor{ccv-sgd}{HTML}{444444}
\tikzset{cv-sgd/.style={color=ccv-sgd, dotted, line width=0.8pt}}
\definecolor{ccv-sam-adam}{HTML}{8E0F1C}
\tikzset{cv-sam-adam/.style={color=ccv-sam-adam, solid, line width=1.1pt}}
\definecolor{ccv-fsam-adam}{HTML}{F4623A}
\tikzset{cv-fsam-adam/.style={color=ccv-fsam-adam, densely dashed, line width=1.4pt}}
\definecolor{ccv-adam}{HTML}{8A8A8A}
\tikzset{cv-adam/.style={color=ccv-adam, dash dot, line width=0.7pt}}
\newcommand{\curveslegend}{%
\begin{tikzpicture}
\begin{axis}[hide axis, scale only axis, width=0.9300\textwidth, height=1pt, xmin=0, xmax=1, ymin=0, ymax=1,
    legend columns=8, legend cell align=left,
    legend style={at={(0.5,0)}, anchor=south, draw=none, fill=none, font=\scriptsize, inner sep=1pt,
      /tikz/every even column/.append style={column sep=5pt}}]
  \addlegendimage{cv-muon}\addlegendentry{\muon}
  \addlegendimage{cv-muon-sam}\addlegendentry{\specsammuon}
  \addlegendimage{cv-fsam-ortho-muon}\addlegendentry{\fpsoma}
  \addlegendimage{cv-fsam-ortho-muon-stale-momentum}\addlegendentry{\soma}
  \addlegendimage{cv-randsam-muon}\addlegendentry{\randsammuon}
  \addlegendimage{cv-soma-prens5}\addlegendentry{\opsomapre}
  \addlegendimage{cv-op-soma-postns5}\addlegendentry{\opsomapost}
  \addlegendimage{cv-msoma}\addlegendentry{\msoma}
  \addlegendimage{cv-adam}\addlegendentry{\adamw}
  \addlegendimage{cv-sam-adam}\addlegendentry{\saw}
  \addlegendimage{cv-fsam-adam}\addlegendentry{\fsaw}
  \addlegendimage{cv-sgd}\addlegendentry{\sgd}
  \addlegendimage{cv-sam}\addlegendentry{\sam}
  \addlegendimage{cv-fsam}\addlegendentry{\fsam}
  \addlegendimage{cv-msam}\addlegendentry{\msam}
\end{axis}
\end{tikzpicture}}

\pgfplotsset{
  curve axis/.style={width=\columnwidth, height=0.72\columnwidth, grid=major, grid style={gray!20}, xlabel={Step}, tick align=outside,
    tick label style={font=\scriptsize}, x tick label style={/pgf/number format/.cd, 1000 sep={}, /tikz/.cd, font=\tiny, rotate=45, anchor=east, inner sep=1pt}, label style={font=\footnotesize}, title style={font=\footnotesize}, scaled y ticks=false, scaled x ticks=false},
  grid curve axis/.style={width=0.34\textwidth, height=0.255\textwidth, grid=major, grid style={gray!20}, tick align=outside,
    tick label style={font=\scriptsize}, x tick label style={/pgf/number format/.cd, 1000 sep={}, /tikz/.cd, font=\tiny, rotate=45, anchor=east, inner sep=1pt}, label style={font=\footnotesize}, title style={font=\footnotesize, yshift=-2pt}, scaled y ticks=false, scaled x ticks=false},
}

%% file: sections/abstract.tex
Adversarial sharpness penalty per iteration of a base optimizer, known as sharpness-aware minimization (\sam), is an accepted technique to improve generalization in deep learning. Recently, \muon has emerged as a dominant optimization algorithm for training large language models (LLMs). It achieves lower sample complexity for heavy-tailed gradient updates -- a phenomenon common in LLM training. However, sharpness penalties with \muon updates are yet largely unexplored in the literature. Moreover, the behavior of \sam under heavy-tailed noise is unknown both theoretically and empirically. A potential reason for this gap is the additional cost of extra gradient computation per iteration for penalty perturbations in \sam, which becomes prohibitive for LLMs.

This paper introduces Sharpness-aware Orthogonalized Momentum Algorithm (\soma) to combine a computationally efficient and theoretically grounded perturbation scheme with the \muon algorithm. The key idea is to use stale (purely) stochastic gradient noise for adversarial perturbations. Theoretically, we show that \soma's sample complexity is invariant to the perturbation direction. We provide explorations and insights into \soma's robustness to heavy-tailed gradient noise. We extensively explore \soma's efficacy in (a) mitigating quality degradation due to post-training quantization, and (b) catastrophic forgetting during fine-tuning of LLMs. The code is available at \url{https://github.com/dcll-iiitd/soma}.

%% file: sections/intro.tex
Consider the problems of empirical risk minimization (\erm) for training a deep learning model:
\begin{align*}
    \min_{W \in \mathbb{B}} \left\{ f(W) \coloneq \frac{1}{N} \sum_{i = 1}^{N} f(W; \xi_i) \right\},
\end{align*}
where $W$ are parameters in a Banach space $\mathbb{B}$ such as $\R^{d}$ or $\R^{m \times n}$ and $\xi_i$ is the $i$-th sample drawn from a training set $\mathcal{S} = \left\{ X_i, Y_i \right\}_{i = 1}^{N}$. The basic approach for solving \erm is stochastic gradient descent (\sgd):
\[
W_{t + 1} = W_t -  \frac{\eta_t}{\abs{\mathcal{B}_t}} \sum_{\xi_i \in \mathcal{B}_t} \nabla f(W_t; \xi_i),
\]
where $\mathcal{B}_t \subset \mathcal{S}$ is a stochastic mini-batch, $\eta_t > 0$ is the learning rate and $W_t$ the parameter both at iteration $t$ \cite{robbins1951a,bottou2008the}.

Beyond \erm, an enduring pursuit in deep learning is to achieve better generalization, that is, performance on unseen data. For large language models (LLMs), where \erm happens in phases of pre-training and fine-tuning, mitigating issues such as catastrophic forgetting due to fine-tuning \cite{springer2025overtrained} and performance degradation after post-training quantization (\ptq) \cite{kumar2025scaling} are both important endeavours beyond simply a basic generalization of the pretrained model, and they have been found to relate to model training dynamics \cite{watts2026sharpness,catalantatjer2026training}.

The loss landscape geometry of training trajectory, especially that of flat minima, is well known to have connections to a model's generalization ability \cite{keskar2017improving,dziugaite2017computing,li2018visualizing}. Motivated by the idea of landing at a flatter minima, \citet{foret2021sharpness} introduced sharpness-aware minimization (\sam) to \sgd updates:
\begin{equation}
W_{t + 1} = W_t - \dfrac{\eta_t}{\abs{\mathcal{B}_t}} \sum_{\xi_i \in \mathcal{B}_t} \nabla f(W_t + E_t; \xi_i),
\end{equation}
where $E_t \coloneq \rho \, G_t/\normtwo{G_t}$ is a sharpness penalty perturbation in which $\rho > 0$ is a penalty parameter and $G_t \coloneq \frac{1}{\abs{\mathcal{B}_t}} \sum_{\xi_i \in \mathcal{B}_t} \nabla f(W_t, \xi_i)$. \sam has demonstrated considerable empirical successes \cite{bahri2021sharpness,chen2021when,liu2022towards}.

Naturally, these successes of \sam attracted both empirical and theoretical research \cite{bahri2021sharpness,chen2023why,wen2023how,schapiro2024towards}, and inspired a large family of variants largely motivated by two obvious goals: (a) finding ${E}_t$ that better captures sharpness \cite{kwon2021asam,zhuang2022surrogate}, and (b) reducing the cost of the additional gradient evaluation required for ${E}_t$ \cite{liu2022towards,du2022efficient}; we survey them with more details in Appendix~\ref{sec:relatedwork}. \fsam \cite{li2024friendly} noted that the full gradient direction degrades generalization and proposed ${E}_t$ along the pure stochastic noise, computed by subtracting a scaled exponentially moving average (\ema) of stochastic gradients from a fresh mini-batch gradient, and observing that \ema estimates the full gradient. Nevertheless, the variants of \sam have largely focused on improving \sgd as the base optimizer empirically \cite{kwon2021asam} as well as theoretically \cite{khanh2024fundamental}, but with very limited inclusion of adaptive preconditioned updates and without any backing convergence theory \cite{singh2025avoiding}. In the meantime, \adam \cite{kingma2014adam} and \adamw \cite{loshchilov2017decoupled} have become the mainstream optimizers for LLMs.

\muon \cite{jordan2024muon}, which orthogonalizes the momentum of matrix-valued parameters (see Section~\ref{sec:preliminaries}), has also quickly become one of the most sought-after optimizers for LLM training: it scales to LLM pre-training \cite{liu2025muon,shah2025practical}, it was used to pre-train open-weight frontier models \cite{kimiteam2025kimi,glmteam2025glm}, and its convergence is now well understood, including under heavy-tailed gradient noise \cite{kim2026convergence,hbler2026free,iiduka2026muon}. In many cases, \muon remains competitive to \adamw when fine-tuning \adamw-pretrained LLMs, while reducing the required resources significantly \cite{liu2025muon,page2025muonall}. A \muon-pretrained model fine-tuned with \muon outperforms its \adamw-fine-tuned counterpart \cite{liu2025muon} and, since switching the optimizer between pre-training and fine-tuning can erode acquired knowledge \cite{qu2026can} due to optimizer mismatch, the growing number of \muon-pretrained open-weight models makes \muon the natural optimizer for fine-tuning them. Yet, to the best of our knowledge, the only work combining \sam with \muon is the concurrent and unpublished work of \citet{zhong2026sharpness}, which pairs a spectral version of vanilla \sam ${E}_t$ with a \muon update and evaluates it empirically, without any convergence theory.

Furthermore, stochastic gradient noise in deep networks is known to be heavy-tailed \cite{simsekli2019a,hodgkinson2021multiplicative}, markedly so in attention-based language models \cite{zhang2020why,ahn2024linear}. In fact, the tail of the noise in \sgd was found to avoid sharp minima and settle in flatter ones \cite{wang2025global}, which is precisely the goal of \sam. Yet existing analyses of \sam assume either deterministic gradients \cite{adly2026stationarity} or stochastic gradients with bounded variance \cite{si2023practical}. To the best of our knowledge, no work has studied the empirical or theoretical performance of \sam under heavy-tailed noise.

From the preceding discussion, we identify three gaps at the intersection of \sam, \muon, and LLM training. First, \sam variants have been designed and analyzed almost exclusively around \sgd; their only pairing with \muon \cite{zhong2026sharpness} uses a vanilla \sam perturbation that doubles the per-iteration gradient cost and comes with no convergence guarantee. Second, although heavy-tailed gradient noise is pervasive in LLM training, it is unknown, empirically or theoretically, whether and when \sam remains beneficial in this regime. Third, it is open whether a sharpness-aware perturbation can be added to \muon at essentially no extra cost while retaining the convergence guarantees of \muon under heavy-tailed noise. We bridge all these gaps in our work.

In this paper, we propose \textbf{S}harp\-ness-aware \textbf{O}r\-tho\-go\-nal\-ized \textbf{M}o\-men\-tum \textbf{A}l\-go\-rithms, which build a spectral-friendly perturbation from quantities that \muon already holds, namely, the perturbed gradient of the previous iteration and the scaled \muon momentum. Effectively, it needs half the number of gradient evaluations than standard \sam approaches: $T+1$ gradient evaluations over $T$ iterations instead of the $2T$ of its fresh perturbation counterpart \fpsoma, a spectral adaptation of \fsam. We also propose \opsoma, which removes the component of the perturbation aligned with the previous \muon step (see Section~\ref{sec:algorithm}). We then prove that \soma reaches an $\epsilon$-stationary point (for some $\epsilon > 0$) in the nuclear norm under appropriate smoothness requirements and in the presence of heavy-tailed noise with a bounded $p$-th central moment, and our analysis leads us to an estimate with an irreducible floor of order $\rho^\nu$ due to the nonzero perturbation radius. Our guarantee holds for any perturbation direction as well as recovers the convergence rate of \muon at $\rho = 0$, and our analysis also shows that the reuse of the stale gradient costs little by way of accuracy (see Section~\ref{sec:convergence}). Empirically, on synthetic problems with a controlled tail index, the spectral variants reduce the optimality gap of \muon by up to $40\%$ under even the heaviest tails, whereas \sam and \fsam on \sgd degrade sharply. Additionally, the spectral perturbations outperform the Frobenius ones. We further evaluate all our proposed methods on $4$-bit \ptq and observe catastrophic forgetting in nanoGPT and Pythia-70M (see Section~\ref{sec:experiments}).

%% file: sections/relatedwork.tex
\paragraph{Sharpness-aware optimization.}
A line of work reduces the overhead of the two gradient evaluations that \sam requires per iteration: LookSAM recomputes ${E}_t$ only periodically \cite{liu2022towards}, ESAM perturbs a random subset of the weights and evaluates the perturbed loss only on sharpness-sensitive samples \cite{du2022efficient}, SSAM perturbs a sparse subset of parameters selected via the Fisher information \cite{mi2022make}, and SAMPa computes the two gradients in parallel \cite{xie2024sampa}. SAF and its memory-efficient version MESA \cite{du2022sharpness} store past iterations' weights to compute ${E}_t$ in a zero-order way. A second line refines how sharpness is measured: ASAM rescales ${E}_t$ by the parameter magnitudes to make sharpness scale-invariant \cite{kwon2021asam}, GSAM additionally minimizes the surrogate gap between the perturbed and the unperturbed loss \cite{zhuang2022surrogate}, and BiSAM recasts the min-max problem of \sam as a bilevel problem built on surrogates of the 0-1 loss \cite{xie2024improving}. Recently, MSAM \cite{becker2025momentum} proposed aligning ${E}_t$ to the negative momentum direction, thereby avoiding the extra gradient evaluation, by observing that the negative momentum direction increases the loss and thus serves as a free adversarial perturbation \cite{becker2025momentum}. MSAM is closer to \fsam on account of utilizing the negative momentum direction; however, they did not empirically or theoretically compare against \fsam \cite{li2024friendly}. A third line stabilizes the \sam trajectory: Lookahead-SAM gathers multi-step landscape information to avoid oscillations around saddle points \cite{yu2024improving} that end up paying a third gradient computation cost per iteration, \cite{tan2025stabilizing} renormalize the descent step so that its gradient norm matches that of the ascent step, and EISAM adopts an extragradient scheme to reduce the sensitivity to the perturbation radius $\rho$ \cite{fu2026leveraging}. This list is surely inexhaustive.

\paragraph{\muon and its variants.}
\muon \citep{jordan2024muon} orthogonalizes the momentum of each matrix-valued parameter. It is steepest descent under the spectral norm \citep{bernstein2024old}. Equivalently, it calls a linear minimization oracle (LMO) over a spectral-norm ball \citep{pethick2025training}. With decoupled weight decay, it implicitly solves a spectral-norm-constrained problem \citep{chen2025muon}. \citet{liu2025muon} scaled \muon to LLM pre-training with weight decay and per-parameter update scaling. MuonClip adds QK-clipping for stability at the trillion-parameter scale \citep{kimiteam2025kimi}. AdaMuon adds element-wise adaptivity \citep{si2025adamuon}. MuonEq equilibrates rows and columns before orthogonalization \citep{chang2026muoneq}. LiMuon cuts memory and variance \citep{huang2026limuon}. MuonAll extends \muon to all parameters for fine-tuning \citep{page2025muonall}. Scion and Gluon generalize \muon to layer-wise norm choices \citep{pethick2025training,riabinin2025gluon}. \citet{zhong2026sharpness} (yet unpublished) pair a layer-wise spectral \sam perturbation with a \muon outer update and report gains in vision classification.

\paragraph{Convergence under heavy-tailed noise.}
Heavy-tailed noise is commonly modeled by a bounded $p$-th central moment ($p$-BCM), $p\in(1,2]$. For \sgd, clipping attains the optimal $\mathcal{O}(\varepsilon^{-(3p-2)/(p-1)})$ complexity in the nonconvex setting \citep{zhang2020why}. High-probability bounds have followed for clipped methods \citep{cutkosky2021high,sadiev2023high,nguyen2023improved}. Normalization can replace clipping \citep{hubler2025from}, and normalized \sgd with momentum is optimal without clipping \citep{liu2025nonconvex}. Even vanilla \sgd converges under $p$-BCM, though at the slower rate $\mathcal{O}(\varepsilon^{-2p/(p-1)})$ in the nonconvex setting \citep{fatkhullin2025can}. Most \muon analyses still assume bounded variance. They cover exact and inexact orthogonalization \citep{li2025a,kim2026convergence,shulgin2025beyond}, variance-reduced momentum \citep{qian2025muon}, and the role of spectral flattening in step-size choice \citep{nguyen2026spectral}. Recent works extend \muon to $p$-BCM noise. \citet{iiduka2026muon} treat H\"older-smooth objectives. \citet{choudhury2026muon} analyze Nesterov momentum with inexact polar decomposition. \citet{hbler2026free} prove nuclear-norm stationarity without clipping and without the extra dimension dependence of Euclidean methods. None of these analyses includes a sharpness-aware perturbation.

%% file: sections/prelim.tex
\paragraph{Notation and setup.}

For a matrix-valued parameter $W \in \mathbb{R}^{m \times n}$, we consider the optimization problem $\min_W \left\{ f(W) \coloneq \E_{\xi_i} \left[ f(W; \xi_i) \right] \right\}$, where $\xi_i$ denotes a training sample. For \erm, the expectation is taken over the empirical distribution of the training set. We denote $r \coloneq \min \{ m, n \}$ to be the maximal possible rank of $W$. Given a minibatch $\mathcal{B}_t$ of size $B$ at the $t$-th iteration, we let $f_{\mathcal{B}_t}(W) \coloneq \frac{1}{B} \sum_{\xi_i \in \mathcal{B}_t} f(W; \xi_i)$ and $G_t \coloneq \nabla f_{\mathcal{B}_t}(W_{t - 1})$. We also denote the Frobenius, operator and nuclear norms of matrices by $\normF{\cdot}$, $\normop{\cdot}$, and $\normnuc{\cdot}$, respectively, and use $\innerF{A}{B} \coloneq \trace(A^\top B)$. We denote by $E_t$ the perturbation for the parameter $W$ and scale it by a radius (size) $\rho > 0$. Finally, we depict using lowercase symbols vectorized quantities, for example, $w_t = \operatorname{vec}(W_t)$ and $g_t = \operatorname{vec}(G_t)$. % The matrix $E_t$ denotes the actual parameter perturbation, with its radius scaling $\rho$.

\paragraph{Sharpness-aware minimization (\sam).}
Low training loss alone does not describe a model's
sensitivity to changes in its parameters. \sam
\citep{foret2021sharpness} addresses this sensitivity by seeking parameters whose neighbourhood also has a uniformly low loss for its vectorized parameters, that is,
\[
\min_w f_\rho^{\sam}(w) \quad \text{with} \quad
f_\rho^{\sam}(w) \coloneq \max_{\abs{\epsilon} \leq \rho} f(w+\epsilon),
\]
where $\epsilon$ is the perturbation constrained to be in a neighbourhood of radius $\rho \geq 0$ about $w$. \citet{foret2021sharpness} use the $\ell_2$-norm for the perturbation of parameters with $e = \rho \frac{g} {\normtwo{g}}$. Thus, \sam descends along a gradient as: %in the worst-case 
\[
w_{t + 1} = w_t - \eta_t \nabla f_{\mathcal{B}_t} \left( w + \rho \dfrac{g}{\normtwo{g}} \right).
\]
This two-step update doubles the per-iteration cost.

\paragraph{Friendly \sam (\fsam).} \fsam perturbs the parameter by decomposing the minibatch gradient into two orthogonal components: a full gradient and stochastic noise, and the perturbation is along the direction of the stochastic noise. For some $\gamma \geq 0$, \fsam \cite{li2024friendly} perturbs the current iterate using the residual as:
\begin{gather*}
v_t = \lambda v_{t - 1} + (1 - \lambda) \, g_t, \,\, \text{with} \,\, v_0 = 0, \,\, \lambda \in [0, 1),\\
r_t = g_t - \gamma \, v_t, \,\, \text{and} \,\, e_t^{\fsam} = \rho \dfrac{r_t}{\normtwo{r_t}}.
\end{gather*}

\paragraph{Momentum-\sam (\msam).} \citet{becker2025momentum} proposed a computationally efficient \sam variant. They avoid the second gradient evaluation of \sam by reusing a momentum buffer of the base optimizer \sgd with momentum. Perturbing in the negative momentum direction increases the loss, making it suitable for estimating sharpness. With $m_t=\beta m_{t-1}+g_t$ and $m_0=\mathbf{0}^{m\times n}$, their perturbation is given by:
\begin{gather*}
e_t^{\msam} = -\rho \, \dfrac{m_t}{\normtwo{m_t}},
\end{gather*}
following which the gradient is evaluated at $w_t + e_t^{\msam}$ and passed to the base update. Since $m_t$ is already computed and stored by the base optimizer, \msam requires only one gradient evaluation per iteration and no additional memory beyond that of SGD or Adam with momentum, and this is in contrast to \sam, which doubles the per-iteration cost.

\paragraph{\muon.}
\muon \citep{jordan2024muon} works for the matrix-structured layers $W \in \mathbb{R}^{m \times n}$ of a deep model as:
\[
M_t = \beta \, M_{t - 1} + G_t, \quad W_{t + 1} = W_t -\eta_t \ortho(M_t),
\]
with $M_0 \coloneq \mathbf{0}^{m \times n}$ and $\beta \in [0, 1)$. Here, $\ortho(A) \in \argmax_{\normop{D} \leq 1} \innerF{A} {D}$ is the steepest-descent direction for the linear model with gradient $A$ under an operator-norm constraint. Practical implementations approximate $\ortho{A}$ using $q$ steps of $\ns$ (denoted hereby as $\NSQ(A)$), thereby avoiding an explicit computation of a singular value decomposition (SVD) at each step. \muon preserves the singular directions but sets all nonzero singular values to be $1$.

\input{sections/main-algo}

\paragraph{Spectral \sam \muon (\specsammuon).}
For matrix blocks, \citet{zhong2026sharpness} proposed the spectral robust perturbed objective:
\[
f_\rho^{\specsammuon}(W) \coloneq \max_{\normop{E} \leq \rho} f(W + E).
\]
Linearizing the minibatch loss yields the perturbation $E_t^{\specsammuon} = \rho \,\ortho(G_t)$, which is approximated by computing $\rho \, \NSQ(G_t)$ instead. For multiple matrix blocks, the layer-wise constraint is taken to be: ${\max_\ell} {\normop{E^{(\ell)}} \leq \rho}$. Each layer receives its own spectral perturbation budget, and the maximum is taken across all layers rather than over individual matrix entries. \specsammuon thus evaluates
$\overline G_t = \nabla f_{\mathcal{B}_t}(W_t + E_t)$
and then uses the following update:
\[
M_t = \beta \, M_{t - 1} + \overline G_t, \quad W_{t + 1} = W_t - \eta_t \NSQ(M_t).
\]
As a result, both inner perturbations and outer updates are aware of the spectral geometry.

\paragraph{Summary and our motivation}
\specsammuon respects the spectral geometry but perturbs along the full mini-batch gradient $G_t$, whose full-gradient component \fsam is found to be unhelpful for generalization. \fsam removes this component, but then it only works in Euclidean geometry. Both methods evaluate a fresh gradient for the perturbation, doubling the per-iteration cost, whereas \msam reuses the momentum buffer to avoid this cost but does not account for the matrix geometry of its layers. Therefore, this motivates an optimizer that forms a low-cost, friendly residual while still working in spectral geometry. As an outcome of such a pursuit, we propose \soma, which we detail next. %in Section~\ref{sec:algorithm}.

%% file: sections/main-algo.tex
\begin{algorithm}[H]
    \caption{Sharpness-aware Orthogonalized Momentum Algorithms (\soma) Framework}
    \label{alg:soma framework}
    \begin{algorithmic}[1]
        \REQUIRE Given $\mathcal{S} = \left\{ X_i, Y_i \right\}_{i = 1}^{N}$; initial weights $W_0 \in \mathbb{R}^{m\times n}$; learning rate $\eta > 0$; momentum coefficient $\beta \in [0,1)$; perturbation radius $\rho \geq 0$; perturbation rule $\mathcal{P}_t$; batch size $B$; iterations $T$; $\ns$ iterations $q$.

        \STATE Initialize momentum $M_0 = \mathbf{0}$ and any auxiliary
        state required by $\mathcal{P}_t$.

        \FOR{$t = 1, \dots, T$}
        \STATE Sample a minibatch
        $\mathcal{B}_t \subset \mathcal S$.

        \STATE Construct perturbation: $E_t = \rho \, \mathcal P_t$.

        \STATE Compute the perturbed gradient:

        $\displaystyle \hfill \overline{G}_t = \dfrac{1}{B} \sum\limits_{\xi_i \in \mathcal B_t} \nabla f(W_{t - 1} + E_t; \xi_{i}). \hfill$

        \STATE Update momentum: $M_t = \beta \, M_{t - 1} + \overline{G}_t$.

        \STATE Orthogonalize momentum: $O_t = \NSQ(M_t)$.

        \STATE Update weights: $W_t = W_{t - 1} - \eta \, O_t$.

        \STATE Store any quantities required by the
        perturbation rule for the next iteration.

        \ENDFOR

        \RETURN $W_T$.
    \end{algorithmic}
\end{algorithm}

%% file: sections/algo.tex
Algorithm~\ref{alg:soma framework} provides a common framework for sharpness-aware \muon with varying perturbation schemes $\mathcal P_t$, which can be as simple as a random direction whose entries are drawn independently and identically distributed (i.i.d.)\ standard normal random variables, that is, %$(\mathcal P_t)_{ij} \overset{\mathrm{i.i.d.}}{\sim} 
$(\mathcal P_t)_{ij} \sim \mathcal{N}(0, 1)$. However, as noted in \cite{xie2024sampa} and also observed in our experiments, this easy choice does not work in practice. This is further underscored by numerous other works on efficient perturbation strategies for \sam.

\paragraph{\fpsoma.}
We first directly adapt the approach of \fsam~\citep{li2024friendly} to spectral orthogonalization and a \muon outer update, which we call freshly-perturbed-\soma (\fpsoma).
\begin{numcases}{\text{\fpsoma}:}
\displaystyle
G_t = \dfrac{1}{B} \sum_{i = 1}^{B} \nabla_W f(W_{t - 1}; \xi_{i}), \nonumber\\
V_t = \beta \, V_{t - 1} + (1 - \beta) \, G_t, \label{eq:fpsoma-perturbation}\\
E_t = \rho \, \NSQ{(G_t - \gamma \, V_t)}, \nonumber
\end{numcases}
where $V_0 \coloneq \mathbf{0}^{m \times n}$, and $\gamma \in (0, 1)$ controls the strength of EMA subtraction. The residual emphasizes deviations from the running gradient estimate, while orthogonalization determines the spectral geometry of this perturbation. The fresh gradient $G_t$ dominates the perturbation, and the subsequent perturbed gradient $\overline G_t$ drives the descent. \fpsoma is costly as it requires two gradient evaluations per iteration and a separate EMA buffer $V_t$.

\paragraph{\soma.}
To reduce the cost of \fpsoma, we construct friendly perturbations by reusing quantities available from the preceding update. Specifically, the last perturbed gradient $\overline G_{t - 1}$ replaces the fresh gradient, and the scaled \muon momentum $(1 - \beta) \, M_{t - 1}$ supplies an EMA of past perturbed gradients. At the first iteration, since no previous gradient is available, \soma uses the exact same initialization and perturbation as \fpsoma. In Lemma \ref{lem:lag}, we argue that \soma's updates are as good as those of \fpsoma. The complete rule for $t > 1$ is:
% \begin{equation}
% \label{eq:soma-perturbation}
\begin{numcases}{\text{\soma}:}
\displaystyle
E_t = \rho \, \NSQ{(\overline G_{t - 1} - \gamma \, V_t)}, \nonumber\\
V_t = (1 - \beta) \, M_{t - 1}, \label{eq:soma-perturbation}\\ %\text{ and }\\
\overline G_{t - 1} = \dfrac{1}{B} \sum_{i = 1}^{B} \nabla_W f(W_{t - 2} + E_{t - 1}; \xi_i). \nonumber
\end{numcases}
% \end{equation}
The factor $(1 - \beta)$ converts the \muon's momentum (see Line 6, Algorithm~\ref{alg:soma framework}) into an EMA. Thus, for $t > 1$, $V_t$ is obtained directly from the existing momentum buffer and a separate EMA is not required. This EMA also tracks gradients at past perturbed iterates, whereas in \fpsoma, it only tracks fresh gradients at unperturbed iterates.

Constructing $E_t$ for \soma thus requires no new gradient evaluations, thereby significantly reducing the cost to $T + 1$ gradient evaluations as compared with $2T$ gradient evaluations for \fpsoma. Nevertheless, $E_t$ of \fsam, \fpsoma, and \soma can all align with the previous descent direction $-O_{t - 1}$. While \sam intends the perturbation to be an adversarial step, this alignment suggests that $E_t$ may reinforce the previous optimization movement. Motivated by this observation, we propose \opsoma, which removes the component of the stale friendly residual along the previous \muon momentum.

\paragraph{Orthogonally Projected \soma (\opsoma) (Pre-$\NSQ$)} 
For each weight matrix, define the Frobenius-orthogonal projection by $\projperp{M}{X} \coloneq X - ( \innerF{X}{M}/\normtwo{M} ) M$. For $t > 1$, the perturbation is:
\begin{equation} \label{eq:osoma-perturbation-prensq}
\begin{aligned}
\opsoma:\begin{cases}
R_t &= \overline G_{t - 1} - \gamma \, (1 - \beta) \, M_{t - 1},\\
P_t &= \projperp{M_{t - 1}}{R_t}\\
&= \projperp{M_{t - 1}}{\overline G_{t - 1}},\\
E_t &= \rho \, \NSQ{(P_t)}.
\end{cases}
\end{aligned}
\end{equation}
The second equality for $P_t$ follows because the friendly correction is parallel to $M_{t-1}$ and is therefore removed by the proposed projection. By construction, $\innerF{P_t}{M_{t - 1}} = 0$. Here we intended to do the alignment correction \emph{before} applying the $\ns$ transformation. However, the nonlinear $\ns$ transformation does not generally preserve this orthogonality, and so $\innerF{E_t}{M_{t - 1}}$ need not vanish.

\paragraph{Orthogonally Projected \soma (\opsoma) (Post-$\NSQ$).} 
Orthogonality to raw momentum $M_{t - 1}$ does not imply orthogonality to the previous descent direction $-O_{t - 1}$. Therefore, we next construct $P_t$ for \opsoma by removing the component of the spectrally transformed friendly residual along the previous \muon update direction $-O_{t - 1}$, and then by normalizing the projected matrix using its Frobenius norm:
\begin{numcases}{\text{\opsoma}:}
R_t = \NSQ{(\overline G_{t - 1} - \gamma \, (1 - \beta) \, M_{t - 1})}, \nonumber\\
P_t = \rho \, \projperp{O_{t - 1}}{R_t}, \label{eq:osoma-perturbation-postnsq}\\
E_t = \rho \dfrac{P_t}{\normF{P_t}}.\nonumber
\end{numcases}
For $P_t = \mathbf{0}^{m \times n}$, we take $E_t = \mathbf{0}^{m \times n}$. The convergence properties of \soma remain invariant to the choice of $E_t$.

\paragraph{Momentum-\soma (\msoma).} \msoma applies the momentum idea of \msam to the \soma framework. Rather than forming a friendly residual from the perturbed gradient $\overline G_{t-1}$ and the scaled momentum, it perturbs along the negative \muon momentum from the previous iteration:
\begin{numcases}{\text{\msoma}:}
E_t^{\msoma} = -\rho \, \dfrac{M_{t - 1}}{\normF{M_{t - 1}}}.
\end{numcases}
Because $M_{t-1}$ is already available before the gradient at $W_t$ is computed, this perturbation requires no additional gradient evaluation, and the cost matches that of \soma. Unlike \opsoma, \msoma does not project out the component of the perturbation along the previous descent direction $-O_{t-1}$, so its perturbation may still align with the previous movement.

The complete algorithms are described in Appendix~\ref{app:algo}, and their empirical performance is discussed in Section~\ref{sec:experiments}.

%% file: sections/convergence.tex
We recall a key definition and state some standard assumptions needed for our analysis of \soma.

\begin{definition}[$\epsilon$-stationary point]
We say $W \in \sR^{m \times n}$ is an $\epsilon$-stationary point if $\normnuc{\G f(W)} \leq \epsilon$. An algorithm attains $\epsilon$-stationarity in $T$ steps if:
\[
\dfrac{1}{T} \sum_{t=1}^{T} \E \left[ \normnuc{\G f(W_{t - 1})} \right] \leq \epsilon.
\]
\end{definition}

In our analysis, we take $f$ to be nonconvex with optimal $f^* \coloneq \inf_{W} f(W) > -\infty$. In all our algorithms, at iteration $t$, a mini-batch $\xi_t = (\xi_{t,1}, \dots, \xi_{t,B})$ with sample size $B$ drawn i.i.d.~from the uniform distribution over the set of all samples $\mathcal{S}$ and for all steps $t = 1, \dots, T$. At $t = 0$, the model parameter $W_0 \in \R^{m \times n}$ is initialized, and we define the initial suboptimality gap to be $D \coloneq f(W_0) - f^*$. Given this setup, we also make the following assumptions for convergence in a heavy-tailed noise setting. Lastly, we use smoothness with respect to the operator norm and (its dual) nuclear norm as appropriate.

\begin{assumption}[$(\nu, L)$-Smoothness] \label{assump:nu_L_smoothness}
A continuously differentiable function $f: \sR^{m \times n} \to \sR$ is said to be $(\nu, L)$-smooth with $\nu \in (0, 1]$ and $L > 0$ if for all $X, Y \in \sR^{m \times n}$, we have that:
\[
\normnuc{\G f(X) - \G f(Y)} \le L \normop{X - Y}^\nu.
\]
\end{assumption}

\begin{assumption}[$p$-Bounded central moment (BCM) condition] \label{assum:pth_bounded_moment}
Let $\xi$ be a random variable independent of $W \in \sR^{m\times n}$ and for a sufficiently smooth $f: \R^{m \times n} \to \R$, let $\G f(W;\xi)$ denote the stochastic gradient of $f$ at $W$. Then, we have the following:
\begin{enumerate}[label=(\roman*)]
\item (Unbiasedness) $\E_\xi \left[ \G f(W; \xi) \right] = \G f(W)$ for all $W \in \sR^{m \times n}$, and

\item ($p$-BCM) There exist $\sigma \ge 0$ and $p \in (1, 2]$ such that for all $W \in \sR^{m \times n}$, $\gV_\xi^p \left[ \G f(W; \xi) \right] \coloneq \E_\xi \left[ \normF{\G f(W; \xi) - \E_\xi \left[\G f(W; \xi) \right] }^p \right] \le \sigma^p$.
\end{enumerate}
For a mini-batch of size $B$, the corresponding $p$-th moment is bounded by $(2^{2 - p} \sigma^p)/B^{p - 1}$ \citep{iiduka2026muon}.
\end{assumption}

\textit{Note:} For $p=2$, Assumption~\ref{assum:pth_bounded_moment} reduces to the bounded-variance condition.

\begin{theorem}[\soma under heavy-tailed noise with $p$-BCM condition]
\label{thm:sam_muon_nu_p_convergence}
Consider Algorithm~\ref{alg:soma framework} initialized with $W_0 \in \sR^{m \times n}$, perturbation radius $\rho > 0$, and step size $\eta$ and momentum parameter $\beta$ to be:
\begin{gather*}
\eta = \left( \dfrac{(1 - \beta) \, D}{TL} \right)^{1/(\nu + 1)}, \quad \beta = 1 - \min \left\{ 1, \left( \left( \dfrac{L \, D^\nu}{T^\nu} \right)^{1/(\nu + 1)}  \dfrac{B^{(1 - 1/p)}}{\sigma \sqrt{r}} \right)^{p \, (\nu + 1)/\delta} \right\}.
\end{gather*}
Then under Assumptions~\ref{assump:nu_L_smoothness} and \ref{assum:pth_bounded_moment}, the following $\epsilon$-stationarity guarantee holds:
\begin{multline*}
\frac{1}{T} \sum_{t = 1}^T \E \left[ \normnuc{\G f(W_{t - 1})} \right] \le \chi_q \, \gO \left( \left( \dfrac{L \, D^\nu}{T^\nu} \right)^{(p - 1)/\delta} \left( \frac{\sigma \sqrt{r}}{B^{1 - 1/p}}  \right)^{p/\delta} + \left( \frac{L \, D^\nu}{T^\nu} \right)^{1/(\nu + 1)} \right.+\\
\left. \frac{L^{(\delta - 2 p)/(2 \delta)} \, \sqrt{r} \, D^{(\delta - 2 p \, \nu)/(2 \delta)}}{T^{(2p - 1 - \nu)/\delta}} \left( \frac{\sigma \sqrt{r}}{B^{1 - 1/p}} \right)^{p (\nu + 1)/\delta} \right)\\
+ \chi_q \, L \, \rho^\nu \gO \left( 1 + \frac{1}{L^{p/\delta} \, D^{p \nu/\delta} \, T^{(2 p - 1 - \nu)/\delta}} \left( \frac{\sigma \sqrt{r}}{B^{1 - 1/p}} \right)^{p (\nu + 1)/\delta} \right).
\end{multline*}
where $\delta \coloneq (p - 1)(\nu + 1) + p$ and $\chi_q \ge 0$ is the error factor associated with using $q$ steps of $\ns$ in \textup{\muon}.
\end{theorem}

\begin{remark} \label{rem:sam_muon_comparison_heavy_tailed}
Theorem~\ref{thm:sam_muon_nu_p_convergence} guarantees that \soma converges to an $\epsilon$-stationary point under heavy-tailed noise with an irreducible error floor $\gO (\chi_q \, L \, \rho^\nu)$. In addition, we can derive the following insights about the various terms in the bound.
\end{remark}

\begin{enumerate}
\item Setting the perturbation radius $\rho = 0$ completely eliminates the last term: this recovers the convergence bound of \muon under heavy-tailed noise. Then, over the range of $\nu$ and $\delta$, the first term simply dominates, giving the convergence rate to be $\gO \left( T^{-\nu (p - 1)/\delta} \right)$.

\item The SAM term:
\[
\chi_q \, L \, \rho^\nu  \gO \left( 1 + \dfrac{\sigma^{p(\nu + 1)/\delta} \, r^{p (\nu + 1)/(2 \delta)}}{B^\tau (L \, D^\nu)^{p/\delta} \, T^\tau}\right),
\]
can be construed to be made up of two components. The first is a reducible term of the form:
\[
\gO \left( \dfrac{\chi_q \,  L \, \rho^\nu \, \sigma^{p (\nu + 1)/\delta} \, r^{p (\nu + 1)/(2 \delta)}}{B^\tau (L \, D^\nu)^{p/\delta} \, T^\tau}\right),
\]
which decays to zero as $T\to\infty$. The second is an irreducible error floor $\gO \left( \chi_q \, L \, \rho^\nu \right)$ which remains nonzero since $\rho \ne 0$, and also this irreducibility is independent of the nature of noise.

\item The factor $\chi_q$ represents the error factor resulting from approximating the polar decomposition using $q$ steps of $\ns$. By \citep[Theorem 2]{kim2026convergence}, we have that as $q \to \infty$, $\chi_q \to 1^+$.
\end{enumerate}

\begin{remark}
The convergence proof of Theorem~\ref{thm:sam_muon_nu_p_convergence} in the Appendix~\ref{proof:Theorem1} works for any variant of \sam when combined with \muon, as it relies on the convergence of \muon at a perturbed point where the underlying perturbation depends only on the \sam variant we are using.
\end{remark}

\begin{corollary} \label{coro:convg_of_bounded_variance}
For $p = 2$ and $\nu = 1$ and under the same assumptions as in Theorem~\ref{alg:soma framework}, choosing the step size $\eta$ and momentum parameter $\beta$ to be:
\[
\eta = \sqrt{\dfrac{(1 - \beta)D}{T \, L}}, \quad \beta = 1 - \min \left\{1, \dfrac{\sqrt{L \, D \, B}}{\sigma \, \sqrt{r \, T}} \right\},
\]
we have the following $\epsilon$-stationarity guarantee:
% TWO COLUMN COMMENT
% \begin{align*}
%     %\label{coro:final_bound_bounded_variance}
%     %\begin{aligned}
%     \dfrac{1}{T} \sum\limits_{t = 1}^T \E \left[ \normnuc{\G f(W_{t - 1})} \right]
%     \leq \!\chi_q \!\cdot \gO\left( \left( \dfrac{r \, \sigma^2 \, L \, D}{B \, T} \right)^{1/4} \right) \\ 
%     + \chi_q\! \cdot\! \gO \left( \sqrt{\frac{L \, D}{T}} + \dfrac{\sigma \, r}{\sqrt{B \, T}} \right)\\
%     + \chi_q L \, \rho \!\cdot\! \gO \left( 1 + \dfrac{\sigma \, \sqrt{r}}{\sqrt{L \, D \, B \, T}} \right),
%     % \end{aligned}
% \end{align*}
\[
%\label{coro:final_bound_bounded_variance}
\dfrac{1}{T} \sum\limits_{t = 1}^T \E \left[ \normnuc{\G f(W_{t - 1})} \right]
\leq \chi_q \, \gO \left( \left( \dfrac{r \, \sigma^2 \, L \, D}{B \, T} \right)^{1/4}  \sqrt{\frac{L \, D}{T}} + \dfrac{\sigma \, r}{\sqrt{B \, T}} \right)
+ \chi_q L \, \rho \, \gO \left( 1 + \dfrac{\sigma \, \sqrt{r}}{\sqrt{L \, D \, B \, T}} \right),
\]
where $\chi_q \ge 0$ is the error factor associated with using $q$ steps of $\ns$ in \textup{\muon}.
\end{corollary}

\begin{remark} \label{rem:sam_muon_comparison}
Corollary~\ref{coro:convg_of_bounded_variance} gives convergence of \sam with \muon under the bounded variance assumption and with an irreducible error floor $\gO(\chi_q L\rho)$ due to the perturbation used in \sam. We can see that the bound here is a particular case of Theorem~\ref{thm:sam_muon_nu_p_convergence} with $\nu=1$ and $p=2$. As before, we can recover the bound for \muon under the bounded variance assumption to be $\gO(T^{-1/4})$ when $\rho = 0$.
\end{remark}

\begin{remark}[Frobenius-norm special case]
If both the SAM perturbation $E$ and \muon update direction $O$ use Frobenius norms, that is, suppose we have:
\[
B \in \argmax_{\normF{B} \le 1} \innerF{A}{B}, \quad B^\star = \dfrac{A}{\normF{A}}.
 \]
Then, this is simply \sam with momentum-based normalized \sgd as the base optimizer. As the $\NSQ$ error factor $\chi_q$ is not present, the convergence is at the rate $\gO (T^{-\nu (p - 1)/\delta})$, $\delta \coloneq (p - 1) \, (\nu + 1) + p$, and with the corresponding SAM perturbation term of the form $\gO(\rho)$.
\end{remark}

\subsection{\soma Analysis}
For our analysis of \soma, in addition to Assumptions~\ref{assump:nu_L_smoothness}
and~\ref{assum:pth_bounded_moment}, we also make use of the following.

\begin{assumption}[Hessian Lipschitzness]\label{assump:hessian_lipschitz}
A twice continuously differentiable function $f:\R^{m \times n} \to \R$ with Hessian $\mathcal H(X)$ at $X \in \R^{m \times n}$ is said to have the stated property if there exists $L_H^{\mathrm{op}} > 0$ such that for all $X, Y, A \in \mathbb R^{m \times n}$, we have:
\[
\abs{\innerF{A}{(\mathcal H(X) - \mathcal H(Y))[A]}} \le L_H^{\mathrm{op}} \normop{X - Y}\normop{A}^2.
\]
\end{assumption}

\begin{lemma}[Temporal lag of \soma]\label{lem:lag}
Under Assumptions~\ref{assump:nu_L_smoothness}, \ref{assum:pth_bounded_moment}
and~\ref{assump:hessian_lipschitz} with $\nu=1$, $p=2$, and $t\ge3$:
\[
\overline G_{t - 1} = \G f(W_{t - 1}) + b_t + \zeta_{t - 1}, \quad \normnuc{b_t} \le L \, (\rho + \eta),
\]
where $\E \left[ \zeta_{t - 1} \mid \mathcal{F}_{t - 2} \right] = 0$. Moreover, we also have:
%TWOCOLUMN
    % \begin{multline*}
    %     \dfrac{\rho^2}{2}\Big|\big\langle\mathcal P_t,\big(\mathcal H(W_{t-1})-\mathcal H(W_{t-2}+E_{t-1})\big)[\mathcal P_t]\big\rangle_F\Big| \\
    %     \le\dfrac{\rho^2}{2}L_H^{\mathrm{op}}(\rho+\eta).
    % \end{multline*}
%ONECOLUMN
\begin{align*}
\dfrac{\rho^2}{2}\Big|\big\langle\mathcal P_t,\big(\mathcal H(W_{t-1})-\mathcal H(W_{t-2}+E_{t-1})\big)[\mathcal P_t]\big\rangle_F\Big|        \le\dfrac{\rho^2}{2}L_H^{\mathrm{op}}(\rho+\eta).
\end{align*}
\end{lemma}

\begin{remark}
Lemma~\ref{lem:lag} shows that reusing the previous iteration's perturbed gradient costs little in terms of accuracy. Its bias $L \, (\rho + \eta)$ matches the usual \sam error when $\eta \lesssim \rho$, its noise is
unchanged, and the curvature seen by the perturbation shifts only by $\gO ( \rho^2 \, (\rho + \eta))$.
\soma thus saves a gradient evaluation without any significant error.
\end{remark}

%% file: sections/experiments.tex
We evaluate \soma in two settings. First, we train on a synthetic problem with simulated heavy-tailed gradient noise, the regime covered by our analysis in Section~\ref{sec:convergence}. Second, we run experiments on language models and a CNN: continued pretraining of nanoGPT and Pythia-70M on FineWeb-Edu, training of the Airbench CNN on CIFAR-10, 4-bit post-training quantization (\ptq) of the two language models, and fine-tuning to measure catastrophic forgetting (Appendix~\ref{sec:catastrophic-forgetting}).

\begin{figure*}[t]
    \centering   
    \input{tikz_for_heavytailed/experiment_1/val_acc_vs_rho.tikz}%
    
    \vspace{0.3cm}
    \input{tikz_for_heavytailed/experiment_1/train_loss_vs_rho.tikz}%

    \caption{\small Performance under \textbf{anisotropic} heavy-tailed noise (Condition no., $\kappa_\varepsilon = 10$). Top: validation accuracy. Bottom: training loss (log scale). \textcolor{red}{Red}, \textcolor{blue}{Blue} and \textcolor{ForestGreen}{Green} lines indicate \adamw-based, \sgd-based and \muon-based methods, respectively.  Lines show the median over seeds, and shaded bands show the interquartile range.}
    \label{fig:exp1_nonlinear_anisotropic}
\end{figure*}
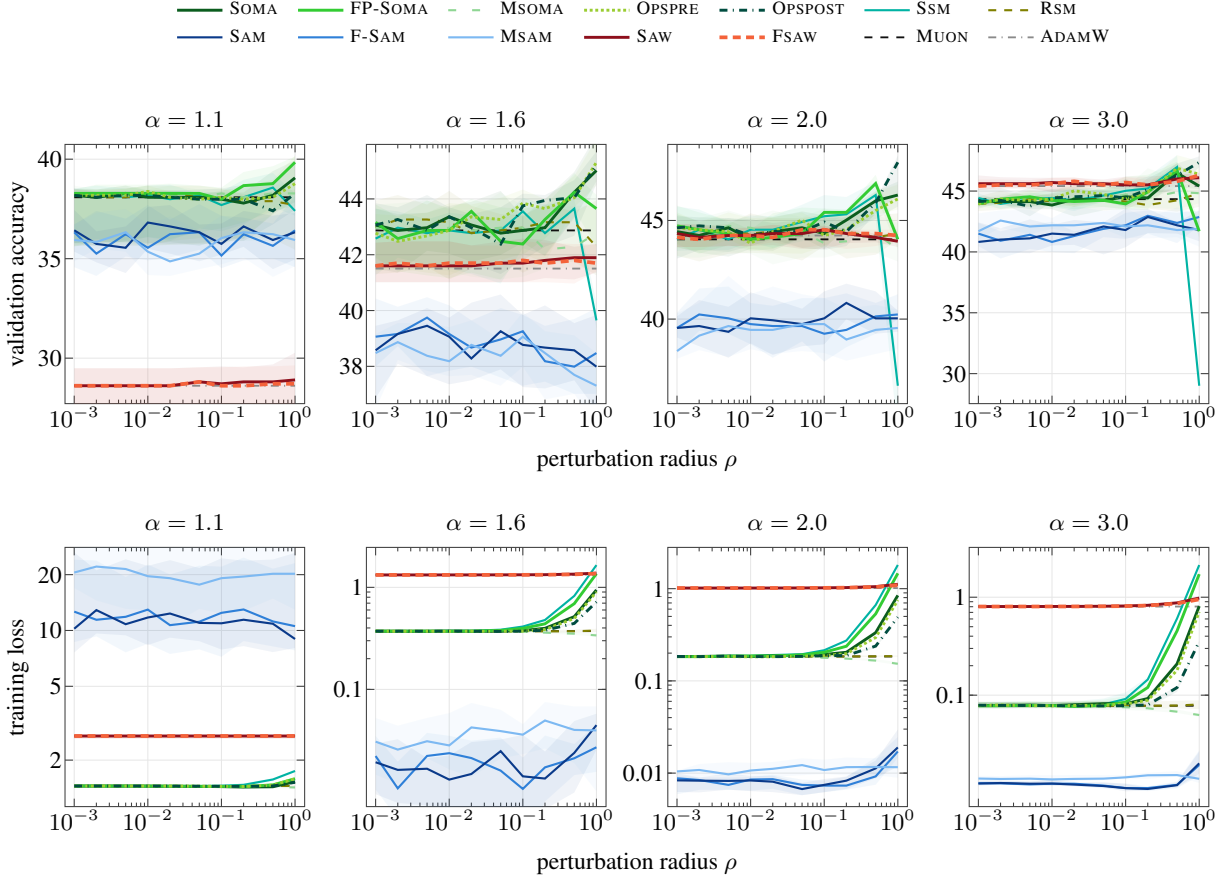

We compare the base optimizers \sgd, \adamw and \muon with sharpness-aware methods built on each of them:
\begin{enumerate*}[label=(\roman*)]
\item on \sgd: \sam, \fsam and \msam;
\item on \adamw: \samadamw (\sam on \adamw) and \fsamadamw (\fsam on \adamw);
\item on \muon: \specsammuon \textsc{(SpecSAM-Muon)} and \randsammuon \textsc{(RandSAM-Muon)}; and
\item our methods: \fpsoma, \soma, \msoma, and the two variants of \opsoma, \opsomapre (Pre-$\NSQ$) and \opsomapost (Post-$\NSQ$).
\end{enumerate*}
For all \muon-based methods, we use $q=5$ Newton--Schulz steps.

\subsection{Synthetic Experiments under Heavy-tailed Noise}
\label{sec:exp-heavy}
\textbf{Setup.} We use a synthetic teacher--student problem to control the gradient noise exactly. The student has two weight matrices $W=(W_1,W_2)$ with $W_1\in\mathbb{R}^{32\times64}$ and $W_2\in\mathbb{R}^{32\times32}$. It maps an input $x\in\mathbb{R}^{64}$ to the logits $W_2\,\mathrm{ReLU}(W_1x)$ and is trained with cross-entropy to predict the 32 class labels assigned by a fixed teacher network of the same shape. The inputs are Gaussian with an ill-conditioned covariance (condition no.\ $100$). At iteration $t$, the optimizer receives the minibatch gradient corrupted by additive noise:
\begin{equation}
\begin{aligned}
    \widehat{\nabla} f(W_t;\xi_t)=\nabla f_{\mathcal{B}_t}(W_t)+\mathcal{E}_t,\\
    \mathcal{E}_t=L\,Z_t\,R^\top,
    \quad
    (Z_t)_{ij}=s_{ij}\,u_{ij}^{-1/\alpha},
\end{aligned}
\end{equation}
where $s_{ij}\sim\mathrm{Unif}(\{-1,1\})$ and $u_{ij}\sim\mathrm{Unif}(0,1)$ are drawn independently. Each entry of $Z_t$ \cite{fatkhullin2025can} follows a two-sided Pareto distribution with tail index $\alpha$, so $\mathcal{E}_t$ has finite moments only of order $p<\alpha$. We test $\alpha\in\{1.1,1.6,2,3\}$: the noise has infinite variance for $\alpha\le2$, and $\alpha=3$ serves as a finite-variance control. The fixed matrices $L$ and $R$ have geometrically decaying singular values with condition number $\kappa_\varepsilon=10$, so the noise is stronger in some directions than in others (anisotropic noise).

\textbf{Implementation.} Every method runs for $500$ iterations. The learning rate is tuned once for each base optimizer and kept fixed across all $\alpha$ and $\rho$. For the sharpness-aware methods, we sweep the perturbation across 10  values of $\rho$ in $\{0.001, 0.002, 0.005, 0.01, 0.02, 0.05, 0.1, 0.2, 0.5, 1\}$. We report validation accuracy on $512$ held-out samples and the training loss, both as the median over $8$ noise seeds. Appendix~\ref{app:isotropic} gives the full setup and also reports results with isotropic noise ($\kappa_\varepsilon=1$) and on a deep linear network with a known optimum.

\textbf{Results.} Figure~\ref{fig:exp1_nonlinear_anisotropic} shows validation accuracy and training loss against $\rho$, and Tables~\ref{tab:nn_aniso_val_acc} and~\ref{tab:nn_aniso_train_loss} in Appendix~\ref{app:isotropic} report the best value of each method. These results indicate that the base optimizer largely determines robustness to heavy-tailed noise. As $\alpha$ approaches $1$ and the tails get heavier, the \adamw-based methods degrade the most, reaching only about $29\%$ accuracy at $\alpha=1.1$, compared with about $36\%$ for \sgd-based and $38$--$40\%$ for \muon-based methods. Adding a perturbation changes a base optimizer's best accuracy by at most $1.7$ points at $\alpha=1.1$, much less than the gap between optimizer families. This is consistent with our convergence theory, Remark~\ref{rem:sam_muon_comparison_heavy_tailed}(2), where the error floor introduced by the perturbation does not depend on the nature of noise. Spectral perturbations help \muon at every tail index: \soma improves on \muon by $1.0$--$2.3$ points, and stays within $\pm0.8$ points of \fpsoma while using half the gradient evaluations. A random perturbation in the Frobenius geometry barely improves on \muon; however, as we will see, random perturbation in the spectral geometry matches \soma (Appendix~\ref{app:isotropic}). \msoma reaches the lowest training loss in the \muon family at every $\alpha$, which suggests that perturbing along the momentum mainly accelerates optimization rather than improving generalization. In the finite-variance case ($\alpha=3$), the \adamw-based methods generalize well despite the highest training loss; only the sharpness-aware \muon methods with a large radius overtake them. \sgd-based methods show the opposite trend, with the lowest training loss but the lowest accuracy, indicating overfitting.

Further observations, the linear problem and isotropic noise are reported in Appendix~\ref{app:isotropic}. We also include ablations on perturbations in Frobenius vs. spectral geometry and on stale vs. fresh gradients in the appendix.

\subsection{LLM Experiments}
We next evaluate the methods on language-model pretraining and on the robustness of the trained models to 4-bit quantization.
We include experiments for fine-tuning to investigate the impact on catastrophic forgetting in Appendix~\ref{app:llm}.
\input{\curvesroot/common/preamble.tex}
\input{atlas_opt_cpt_curves/figure_curves}
% \input{paper_tables/table_pretraining}
\input{SOMA_exp2/figure_train_loss}

\subsubsection{Continued Pretraining}
\label{sec:exp-llm-pretraining}
\textbf{Setup.} We continue pretraining two language models from their released weights, nanoGPT (GPT-2 small, 124M parameters) and Pythia-70M, on FineWeb-Edu \citep{penedo2024the} (\texttt{sample-10BT}). We also train the Airbench CNN, an image benchmark, on CIFAR-10. 

\textbf{Implementation.} Each epoch consists of $100$ optimizer steps with a context length of $512$ tokens (about $2.6\times10^{7}$ tokens per epoch), and we train for $15$ epochs on nanoGPT, $14$ on Pythia-70M and $7$ on the CNN. All results are averaged over three seeds. The detailed hyperparameters are given in Appendix~\ref{app:llm}.

\textbf{Results.} Figure~\ref{fig:curves} shows the training curves; the top row gives the training loss and the bottom row the validation perplexity (validation accuracy for the CNN). Table~\ref{tab:pretraining} in Appendix~\ref{app:llm} reports the final values. On the language models, the base optimizer dominates: the \muon-based methods end with the lowest validation perplexity, ahead of the \adamw- and \sgd-based methods. The two-gradient \specsammuon and \fpsoma end with a slightly higher perplexity than \muon on nanoGPT ($24.55$ versus $24.41$), whereas the one-gradient \soma, \opsomapre and \opsomapost match \muon. This contrasts with \citet{zhong2026sharpness}, who found \specsammuon to outperform \muon on vision models; on CIFAR-10 we observe the same ($87.63\%$ versus $87.27\%$), so the small cost appears specific to language models. In return, \specsammuon and \fpsoma reduce forgetting the most on nanoGPT (Appendix~\ref{sec:catastrophic-forgetting}), and \fpsoma degrades the least under quantization on Pythia-70M (Section~\ref{sec:ptq}). As in the synthetic experiments, perturbing along the negative momentum (\msoma, \msam) speeds up early training but does not improve the final perplexity, and on Pythia-70M \msoma even ends up well behind \muon. On CIFAR-10, the \muon-based methods learn fastest and reach the lowest training loss, but the \adamw-based methods end the most accurate despite a higher training loss, as in the finite-variance synthetic case. Finally, on nanoGPT, \soma costs only $3\%$ more time per epoch than \muon, whereas \specsammuon and \fpsoma take $2.4$--$2.9\times$ as long (Table~\ref{tab:runtime-pretraining}). In short, pretraining quality is set by the base optimizer, and \soma adds sharpness awareness to \muon at almost no cost and without degrading it; whether this pays off is better measured after training, under quantization (Section~\ref{sec:ptq}) and fine-tuning (Appendix~\ref{sec:catastrophic-forgetting}).

\input{paper_tables/table_ptq_exp}
\subsubsection{Post-training Quantization}
\label{sec:ptq}
\textbf{Setup.} Quantizing a model to 4 bits rounds each weight to a nearby value, which perturbs the weights slightly. Sharpness-aware training seeks solutions whose loss changes little under such perturbations, so these models may lose less quality after quantization. To test this, we quantize the language models using each optimizer and measure how much their held-out perplexity increases.

\textbf{Implementation.} We quantize the linear layers of the transformer blocks to 4-bit weights with AWQ \citep{lin2024awq}, using group size $128$ and $128$ calibration sequences of $512$ tokens; the embeddings and the output head stay in full precision. We report perplexity on $2^{20}$ held-out FineWeb-Edu tokens before (FP) and after (INT4) quantization, averaged over three pretraining seeds and three calibration seeds (Table~\ref{tab:ptq-fineweb}).

\textbf{Results.} Quantization raises perplexity for every optimizer, by $3.4$--$4.5\%$ on nanoGPT and $10.2$--$14.6\%$ on Pythia-70M. As in pretraining, the base optimizer matters most: the \sgd-based methods degrade the most on both models, and the \muon-based methods give the lowest perplexity after quantization. The effect of the perturbation is clearest on Pythia-70M, the smaller model, which suffers more from quantization. There, sharpness-aware training reduces the degradation of both \muon and \adamw but not of \sgd: \fpsoma lowers \muon's increase from $10.96\%$ to $10.20\%$ and reaches the lowest INT4 perplexity, \soma lowers it to $10.55\%$, recovering about half of this gain with one gradient evaluation per step, and \sam and \fsam lower \adamw's increase from $11.12\%$ to about $10.4\%$. On nanoGPT, the INT4 perplexity is nearly the same across the \muon family ($23.85$--$23.87$); the smaller increase of \specsammuon and \fpsoma there comes from their slightly higher full-precision perplexity rather than from a better quantized model. Overall, sharpness-aware perturbations make \muon-trained models moderately more robust to quantization, and \soma obtains a substantial part of this benefit at almost the cost of \muon.

%% file: tikz_for_heavytailed/experiment_1/val_acc_vs_rho.tikz
%% Auto-generated by make_heavytailed_figures.py  --  do not edit by hand, re-run the script.
%% Values: line = median over seeds, band = interquartile range (25th-75th percentile).
%% Preamble needed:  \usepackage{pgfplots}  \pgfplotsset{compat=1.16}   (nothing else)
%% Width follows \textwidth, fonts follow the document (so the legend can never overflow).

\definecolor{htSOMA}{HTML}{0B5D1E}
\definecolor{htFPSOMA}{HTML}{33CC33}
\definecolor{htSpec}{HTML}{00B3A4}
\definecolor{htRand}{HTML}{808000}
\definecolor{htFSAM}{HTML}{2F80D8}
\definecolor{htSAM}{HTML}{0A3A8A}
\definecolor{htPreNS5}{HTML}{9ACD32}
\definecolor{htPostNS5}{HTML}{004D40}
\definecolor{htSAMAdam}{HTML}{8E0F1C}
\definecolor{htFSAMAdam}{HTML}{F4623A}
\definecolor{htMSAM}{HTML}{7DB9F2}
\definecolor{htMSOMA}{HTML}{8FD694}
\definecolor{htMuon}{HTML}{111111}
\definecolor{htAdam}{HTML}{8A8A8A}
\pgfplotsset{
  ht panel/.style={
    scale only axis, width=0.1920\textwidth, height=1.35in,
    xmin=0.00075, xmax=1.35,
    xtick={0.001,0.01,0.1,1}, xticklabels={$10^{-3}$,$10^{-2}$,$10^{-1}$,$10^{0}$},
    minor xtick={0.002,0.003,0.004,0.005,0.006,0.007,0.008,0.009,0.02,0.03,0.04,0.05,0.06,0.07,0.08,0.09,0.2,0.3,0.4,0.5,0.6,0.7,0.8,0.9},
    xmajorgrids=true, ymajorgrids=true,
    grid style={line width=0.25pt, draw=black!10},
    axis line style={line width=0.4pt, black!75},
    tick style={line width=0.4pt, black!75}, major tick length=3pt, minor tick length=1.5pt,
    tick label style={font=\footnotesize, /pgf/number format/assume math mode=true},
    label style={font=\small}, title style={font=\small, yshift=-3pt},
    ylabel style={yshift=-1pt, align=center},
    every tick label/.append style={inner sep=1pt},
    clip=true,
  }}

\begin{tikzpicture}
\begin{axis}[name=ax0_r0,
    at={(0.0820\textwidth,0in)},
    anchor=south west,
    ht panel, xmode=log,
    ymin=27.7148, ymax=40.7422,
    ytick={30,35,40},
    yticklabels={30,35,40},
    title={$\alpha=1.1$},
    ylabel={validation accuracy}]
  \addplot[draw=none, fill=htSOMA, fill opacity=0.06] coordinates {(0.001,38.5254) (0.002,38.4766) (0.005,38.6719) (0.01,38.4277) (0.02,38.4277) (0.05,38.3301) (0.1,38.3301) (0.2,38.3301) (0.5,39.4043) (1,39.9414) (1,37.8418) (0.5,37.2559) (0.2,35.9863) (0.1,35.9863) (0.05,35.6445) (0.02,35.8398) (0.01,35.7422) (0.005,35.8887) (0.002,35.6445) (0.001,35.791)} \closedcycle;
  \addplot[draw=none, fill=htFPSOMA, fill opacity=0.06] coordinates {(0.001,38.4766) (0.002,38.7207) (0.005,38.7207) (0.01,38.7207) (0.02,38.7695) (0.05,38.6719) (0.1,38.5742) (0.2,39.3555) (0.5,39.6484) (1,40.1367) (1,37.6465) (0.5,37.4023) (0.2,36.8164) (0.1,36.5234) (0.05,36.3281) (0.02,35.9863) (0.01,36.084) (0.005,35.8398) (0.002,35.8398) (0.001,35.5957)} \closedcycle;
  \addplot[draw=none, fill=htSpec, fill opacity=0.06] coordinates {(0.001,38.5254) (0.002,38.6719) (0.005,38.5254) (0.01,38.6719) (0.02,38.7207) (0.05,38.2812) (0.1,38.3789) (0.2,38.9648) (0.5,39.3066) (1,38.2812) (1,37.1582) (0.5,37.5977) (0.2,36.7188) (0.1,36.377) (0.05,36.5234) (0.02,35.8887) (0.01,36.084) (0.005,35.9375) (0.002,35.9863) (0.001,35.7422)} \closedcycle;
  \addplot[draw=none, fill=htRand, fill opacity=0.06] coordinates {(0.001,38.4766) (0.002,38.4766) (0.005,38.5254) (0.01,38.6719) (0.02,38.6719) (0.05,38.6719) (0.1,38.5254) (0.2,38.2324) (0.5,38.3789) (1,38.4766) (1,35.4492) (0.5,35.9863) (0.2,36.1328) (0.1,35.6934) (0.05,35.7422) (0.02,35.6934) (0.01,35.6445) (0.005,35.6445) (0.002,35.6445) (0.001,35.6445)} \closedcycle;
  \addplot[draw=none, fill=htPreNS5, fill opacity=0.06] coordinates {(0.001,38.4766) (0.002,38.3789) (0.005,38.5254) (0.01,38.6719) (0.02,38.3789) (0.05,38.1836) (0.1,38.4277) (0.2,38.623) (0.5,38.9648) (1,40.1855) (1,38.2812) (0.5,37.207) (0.2,36.0352) (0.1,36.084) (0.05,36.0352) (0.02,35.8887) (0.01,35.6445) (0.005,35.791) (0.002,35.8398) (0.001,35.6445)} \closedcycle;
  \addplot[draw=none, fill=htPostNS5, fill opacity=0.06] coordinates {(0.001,38.5254) (0.002,38.3789) (0.005,38.3789) (0.01,38.4277) (0.02,38.4277) (0.05,38.4766) (0.1,38.5254) (0.2,38.7207) (0.5,38.8672) (1,39.7461) (1,37.4512) (0.5,36.7188) (0.2,35.8887) (0.1,35.8887) (0.05,35.791) (0.02,35.7422) (0.01,35.791) (0.005,35.8398) (0.002,35.8398) (0.001,35.791)} \closedcycle;
  \addplot[draw=none, fill=htMSOMA, fill opacity=0.06] coordinates {(0.001,38.6719) (0.002,38.4766) (0.005,38.6719) (0.01,38.4766) (0.02,38.6719) (0.05,38.4766) (0.1,38.4766) (0.2,38.5254) (0.5,38.6719) (1,38.8672) (1,36.2305) (0.5,35.8887) (0.2,35.8887) (0.1,35.8398) (0.05,35.6934) (0.02,35.6934) (0.01,35.6445) (0.005,35.6934) (0.002,35.6934) (0.001,35.6445)} \closedcycle;
  \addplot[draw=none, fill=htFSAM, fill opacity=0.06] coordinates {(0.001,36.7188) (0.002,37.4023) (0.005,37.0117) (0.01,37.6953) (0.02,36.9629) (0.05,37.7441) (0.1,37.1582) (0.2,37.0605) (0.5,36.7676) (1,38.1348) (1,34.5703) (0.5,35.0098) (0.2,34.9121) (0.1,34.5215) (0.05,35.4492) (0.02,35.0586) (0.01,34.9121) (0.005,35.1074) (0.002,34.1309) (0.001,34.5703)} \closedcycle;
  \addplot[draw=none, fill=htSAM, fill opacity=0.06] coordinates {(0.001,37.4512) (0.002,36.377) (0.005,36.6211) (0.01,37.5977) (0.02,37.6953) (0.05,36.7188) (0.1,36.7676) (0.2,37.4023) (0.5,37.1582) (1,37.3047) (1,34.9609) (0.5,34.4727) (0.2,35.5957) (0.1,35.2051) (0.05,34.8633) (0.02,35.2539) (0.01,35.498) (0.005,34.2773) (0.002,34.8145) (0.001,35.3516)} \closedcycle;
  \addplot[draw=none, fill=htMSAM, fill opacity=0.06] coordinates {(0.001,36.7188) (0.002,37.3047) (0.005,37.5) (0.01,36.7188) (0.02,37.1094) (0.05,36.4258) (0.1,37.1582) (0.2,36.8164) (0.5,37.1582) (1,37.5488) (1,34.6191) (0.5,34.082) (0.2,35.0586) (0.1,34.7168) (0.05,34.4727) (0.02,34.1309) (0.01,34.1797) (0.005,35.3516) (0.002,35.498) (0.001,34.8633)} \closedcycle;
  \addplot[draw=none, fill=htSAMAdam, fill opacity=0.06] coordinates {(0.001,29.4922) (0.002,29.4922) (0.005,29.4922) (0.01,29.4922) (0.02,29.4922) (0.05,29.541) (0.1,29.5898) (0.2,29.541) (0.5,29.6875) (1,30.2734) (1,26.416) (0.5,26.123) (0.2,26.123) (0.1,26.0742) (0.05,26.0742) (0.02,26.0742) (0.01,26.0742) (0.005,26.0742) (0.002,26.0742) (0.001,26.0742)} \closedcycle;
  \addplot[draw=none, fill=htFSAMAdam, fill opacity=0.06] coordinates {(0.001,29.4922) (0.002,29.4922) (0.005,29.4922) (0.01,29.4922) (0.02,29.4922) (0.05,29.541) (0.1,29.541) (0.2,29.5898) (0.5,29.5898) (1,29.6387) (1,26.0742) (0.5,26.123) (0.2,26.123) (0.1,26.0742) (0.05,26.0742) (0.02,26.0742) (0.01,26.0742) (0.005,26.0742) (0.002,26.0742) (0.001,26.0742)} \closedcycle;
  \addplot[color=htMuon, dashed, line width=0.7pt, forget plot] coordinates {(0.001,38.0859) (1,38.0859)};
  \addplot[color=htAdam, dash dot, line width=0.7pt, forget plot] coordinates {(0.001,28.6133) (1,28.6133)};
  \addplot[color=htSpec, solid, line width=0.8pt] coordinates {(0.001,38.2812) (0.002,38.2812) (0.005,38.0859) (0.01,38.1836) (0.02,37.9883) (0.05,38.1836) (0.1,37.6953) (0.2,38.0859) (0.5,38.5742) (1,37.4023)};
  \addplot[color=htRand, dashed, line width=0.8pt] coordinates {(0.001,38.1836) (0.002,38.0859) (0.005,38.0859) (0.01,38.2812) (0.02,38.2812) (0.05,38.0859) (0.1,37.9883) (0.2,37.8906) (0.5,37.8906) (1,37.6953)};
  \addplot[color=htMSOMA, loosely dashed, line width=0.8pt] coordinates {(0.001,38.2812) (0.002,38.0859) (0.005,38.1836) (0.01,38.0859) (0.02,38.1836) (0.05,38.0859) (0.1,38.2812) (0.2,38.0859) (0.5,38.2812) (1,38.2812)};
  \addplot[color=htFSAM, solid, line width=0.8pt] coordinates {(0.001,36.3281) (0.002,35.2539) (0.005,36.2305) (0.01,35.5469) (0.02,36.2305) (0.05,36.3281) (0.1,35.1562) (0.2,36.2305) (0.5,35.6445) (1,36.4258)};
  \addplot[color=htSAM, solid, line width=0.8pt] coordinates {(0.001,36.4258) (0.002,35.7422) (0.005,35.5469) (0.01,36.8164) (0.02,36.6211) (0.05,36.3281) (0.1,35.7422) (0.2,36.6211) (0.5,35.9375) (1,36.3281)};
  \addplot[color=htMSAM, solid, line width=0.8pt] coordinates {(0.001,35.9375) (0.002,35.8398) (0.005,36.3281) (0.01,35.3516) (0.02,34.8633) (0.05,35.2539) (0.1,36.0352) (0.2,36.3281) (0.5,36.2305) (1,35.9375)};
  \addplot[color=htSOMA, solid, line width=1.1pt] coordinates {(0.001,38.1836) (0.002,38.0859) (0.005,38.2812) (0.01,38.0859) (0.02,38.0859) (0.05,38.0859) (0.1,37.9883) (0.2,37.793) (0.5,38.1836) (1,39.0625)};
  \addplot[color=htFPSOMA, solid, line width=1.1pt] coordinates {(0.001,38.2812) (0.002,38.2812) (0.005,38.2812) (0.01,38.2812) (0.02,38.2812) (0.05,38.2812) (0.1,37.9883) (0.2,38.6719) (0.5,38.7695) (1,39.8438)};
  \addplot[color=htPreNS5, densely dotted, line width=1.1pt] coordinates {(0.001,38.1836) (0.002,38.1836) (0.005,38.1836) (0.01,38.3789) (0.02,38.0859) (0.05,37.9883) (0.1,38.0859) (0.2,37.9883) (0.5,38.0859) (1,38.7695)};
  \addplot[color=htPostNS5, dash dot, line width=1.1pt] coordinates {(0.001,38.1836) (0.002,38.0859) (0.005,38.1836) (0.01,38.1836) (0.02,38.0859) (0.05,37.9883) (0.1,37.9883) (0.2,37.9883) (0.5,37.4023) (1,38.2812)};
  \addplot[color=htSAMAdam, solid, line width=1.1pt] coordinates {(0.001,28.6133) (0.002,28.6133) (0.005,28.6133) (0.01,28.6133) (0.02,28.6133) (0.05,28.8086) (0.1,28.7109) (0.2,28.8086) (0.5,28.8086) (1,28.9062)};
  \addplot[color=htFSAMAdam, densely dashed, line width=1.4pt] coordinates {(0.001,28.6133) (0.002,28.6133) (0.005,28.6133) (0.01,28.6133) (0.02,28.6133) (0.05,28.8086) (0.1,28.6133) (0.2,28.6133) (0.5,28.7109) (1,28.7109)};
\end{axis}
\begin{axis}[name=ax1_r0,
    at={(0.3235\textwidth,0in)},
    anchor=south west,
    ht panel, xmode=log,
    ymin=36.6641, ymax=45.9531,
    ytick={38,40,42,44},
    yticklabels={38,40,42,44},
    title={$\alpha=1.6$}]
  \addplot[draw=none, fill=htSOMA, fill opacity=0.06] coordinates {(0.001,43.6523) (0.002,43.457) (0.005,43.5059) (0.01,43.75) (0.02,43.8965) (0.05,44.2871) (0.1,43.3594) (0.2,43.9941) (0.5,44.9707) (1,46.2891) (1,44.6289) (0.5,43.3105) (0.2,41.9434) (0.1,42.4316) (0.05,41.9922) (0.02,41.748) (0.01,42.041) (0.005,42.1387) (0.002,41.3574) (0.001,41.5039)} \closedcycle;
  \addplot[draw=none, fill=htFPSOMA, fill opacity=0.06] coordinates {(0.001,43.7988) (0.002,43.3594) (0.005,43.9453) (0.01,43.6523) (0.02,43.9941) (0.05,43.7012) (0.1,43.6523) (0.2,43.75) (0.5,44.9219) (1,44.1406) (1,42.2852) (0.5,43.6035) (0.2,42.5781) (0.1,41.9922) (0.05,41.1133) (0.02,41.3574) (0.01,41.1621) (0.005,41.2598) (0.002,41.4062) (0.001,41.8945)} \closedcycle;
  \addplot[draw=none, fill=htSpec, fill opacity=0.06] coordinates {(0.001,43.6523) (0.002,43.5547) (0.005,43.2617) (0.01,43.5059) (0.02,43.4082) (0.05,43.2617) (0.1,44.043) (0.2,43.7988) (0.5,43.9941) (1,40.2832) (1,39.3555) (0.5,43.3594) (0.2,41.9922) (0.1,42.2363) (0.05,40.9668) (0.02,41.3086) (0.01,41.5039) (0.005,41.6504) (0.002,41.8945) (0.001,41.2598)} \closedcycle;
  \addplot[draw=none, fill=htRand, fill opacity=0.06] coordinates {(0.001,43.9453) (0.002,44.1406) (0.005,43.7988) (0.01,44.2871) (0.02,43.75) (0.05,43.457) (0.1,43.4082) (0.2,43.7988) (0.5,43.5547) (1,43.9453) (1,41.2598) (0.5,41.748) (0.2,41.6992) (0.1,41.3574) (0.05,41.8457) (0.02,41.5039) (0.01,41.1133) (0.005,41.5527) (0.002,41.3574) (0.001,41.3086)} \closedcycle;
  \addplot[draw=none, fill=htPreNS5, fill opacity=0.06] coordinates {(0.001,43.457) (0.002,43.4082) (0.005,43.5059) (0.01,43.8477) (0.02,43.7988) (0.05,44.043) (0.1,44.3359) (0.2,43.9941) (0.5,44.8242) (1,46.1426) (1,44.4824) (0.5,43.7012) (0.2,41.6992) (0.1,42.041) (0.05,41.7969) (0.02,41.6016) (0.01,41.2598) (0.005,41.3574) (0.002,41.5039) (0.001,41.4062)} \closedcycle;
  \addplot[draw=none, fill=htPostNS5, fill opacity=0.06] coordinates {(0.001,43.8477) (0.002,43.7012) (0.005,44.1406) (0.01,43.7012) (0.02,43.457) (0.05,43.7012) (0.1,44.1895) (0.2,44.1406) (0.5,44.6777) (1,45.7031) (1,44.5312) (0.5,42.6758) (0.2,42.2852) (0.1,41.7969) (0.05,41.3086) (0.02,41.6992) (0.01,41.5527) (0.005,41.5527) (0.002,41.748) (0.001,41.5527)} \closedcycle;
  \addplot[draw=none, fill=htMSOMA, fill opacity=0.06] coordinates {(0.001,43.7988) (0.002,43.8477) (0.005,43.9453) (0.01,43.6035) (0.02,43.6035) (0.05,44.2383) (0.1,43.75) (0.2,43.6523) (0.5,43.75) (1,43.75) (1,41.4551) (0.5,41.5527) (0.2,41.5039) (0.1,41.2109) (0.05,41.1133) (0.02,41.2109) (0.01,41.4062) (0.005,41.5039) (0.002,41.5039) (0.001,41.8457)} \closedcycle;
  \addplot[draw=none, fill=htFSAM, fill opacity=0.06] coordinates {(0.001,40.4297) (0.002,40.1367) (0.005,40.4785) (0.01,40.1855) (0.02,39.6484) (0.05,40.0879) (0.1,39.8438) (0.2,39.4531) (0.5,39.6973) (1,40.0879) (1,37.0117) (0.5,37.0117) (0.2,37.4512) (0.1,37.3047) (0.05,37.5977) (0.02,37.9883) (0.01,37.8418) (0.005,37.6465) (0.002,38.7207) (0.001,37.9883)} \closedcycle;
  \addplot[draw=none, fill=htSAM, fill opacity=0.06] coordinates {(0.001,39.4531) (0.002,39.6484) (0.005,40.0879) (0.01,39.7461) (0.02,40.2832) (0.05,40.5762) (0.1,39.8438) (0.2,40.1855) (0.5,39.5508) (1,39.8438) (1,36.4258) (0.5,37.5) (0.2,37.207) (0.1,37.2559) (0.05,37.8906) (0.02,37.5) (0.01,37.3535) (0.005,37.6465) (0.002,38.0859) (0.001,36.8652)} \closedcycle;
  \addplot[draw=none, fill=htMSAM, fill opacity=0.06] coordinates {(0.001,39.1602) (0.002,39.5508) (0.005,39.5508) (0.01,39.4043) (0.02,39.6484) (0.05,38.9648) (0.1,39.9414) (0.2,39.2578) (0.5,39.1113) (1,38.4766) (1,35.8887) (0.5,37.0117) (0.2,36.8164) (0.1,37.9883) (0.05,37.3047) (0.02,37.6953) (0.01,37.6465) (0.005,37.8418) (0.002,38.2324) (0.001,36.4746)} \closedcycle;
  \addplot[draw=none, fill=htSAMAdam, fill opacity=0.06] coordinates {(0.001,42.4316) (0.002,42.4805) (0.005,42.4805) (0.01,42.4805) (0.02,42.334) (0.05,42.4805) (0.1,42.5781) (0.2,42.5293) (0.5,42.2852) (1,42.3828) (1,41.3574) (0.5,41.1133) (0.2,41.1621) (0.1,41.0156) (0.05,41.0156) (0.02,41.0156) (0.01,41.0156) (0.005,41.0156) (0.002,41.0156) (0.001,41.0156)} \closedcycle;
  \addplot[draw=none, fill=htFSAMAdam, fill opacity=0.06] coordinates {(0.001,42.4316) (0.002,42.4805) (0.005,42.4805) (0.01,42.4805) (0.02,42.334) (0.05,42.334) (0.1,42.4805) (0.2,42.4316) (0.5,42.3828) (1,42.3828) (1,41.4062) (0.5,40.9668) (0.2,41.0156) (0.1,41.0156) (0.05,41.0156) (0.02,41.0156) (0.01,41.0156) (0.005,41.0156) (0.002,41.0156) (0.001,41.0156)} \closedcycle;
  \addplot[color=htMuon, dashed, line width=0.7pt, forget plot] coordinates {(0.001,42.8711) (1,42.8711)};
  \addplot[color=htAdam, dash dot, line width=0.7pt, forget plot] coordinates {(0.001,41.5039) (1,41.5039)};
  \addplot[color=htSpec, solid, line width=0.8pt] coordinates {(0.001,42.5781) (0.002,42.9688) (0.005,42.7734) (0.01,42.8711) (0.02,42.7734) (0.05,42.8711) (0.1,43.5547) (0.2,42.7734) (0.5,43.6523) (1,39.6484)};
  \addplot[color=htRand, dashed, line width=0.8pt] coordinates {(0.001,42.9688) (0.002,43.2617) (0.005,43.2617) (0.01,43.1641) (0.02,42.7734) (0.05,42.8711) (0.1,42.9688) (0.2,43.1641) (0.5,43.1641) (1,42.2852)};
  \addplot[color=htMSOMA, loosely dashed, line width=0.8pt] coordinates {(0.001,42.9688) (0.002,42.8711) (0.005,42.8711) (0.01,42.9688) (0.02,42.8711) (0.05,42.9688) (0.1,42.9688) (0.2,42.1875) (0.5,42.2852) (1,42.7734)};
  \addplot[color=htFSAM, solid, line width=0.8pt] coordinates {(0.001,39.0625) (0.002,39.1602) (0.005,39.7461) (0.01,39.1602) (0.02,38.6719) (0.05,38.9648) (0.1,39.2578) (0.2,38.1836) (0.5,37.9883) (1,38.4766)};
  \addplot[color=htSAM, solid, line width=0.8pt] coordinates {(0.001,38.5742) (0.002,39.1602) (0.005,39.4531) (0.01,39.0625) (0.02,38.2812) (0.05,39.2578) (0.1,38.7695) (0.2,38.6719) (0.5,38.5742) (1,37.9883)};
  \addplot[color=htMSAM, solid, line width=0.8pt] coordinates {(0.001,38.4766) (0.002,38.8672) (0.005,38.3789) (0.01,38.1836) (0.02,38.7695) (0.05,38.3789) (0.1,39.0625) (0.2,38.4766) (0.5,37.6953) (1,37.3047)};
  \addplot[color=htSOMA, solid, line width=1.1pt] coordinates {(0.001,43.0664) (0.002,42.8711) (0.005,42.9688) (0.01,43.3594) (0.02,43.0664) (0.05,42.7734) (0.1,42.8711) (0.2,42.9688) (0.5,44.2383) (1,45.0195)};
  \addplot[color=htFPSOMA, solid, line width=1.1pt] coordinates {(0.001,43.1641) (0.002,42.5781) (0.005,42.8711) (0.01,42.8711) (0.02,43.5547) (0.05,42.4805) (0.1,42.3828) (0.2,43.2617) (0.5,44.2383) (1,43.6523)};
  \addplot[color=htPreNS5, densely dotted, line width=1.1pt] coordinates {(0.001,42.7734) (0.002,42.4805) (0.005,42.6758) (0.01,42.8711) (0.02,43.3594) (0.05,43.2617) (0.1,43.8477) (0.2,43.6523) (0.5,44.043) (1,45.3125)};
  \addplot[color=htPostNS5, dash dot, line width=1.1pt] coordinates {(0.001,43.0664) (0.002,43.2617) (0.005,42.8711) (0.01,43.3594) (0.02,42.9688) (0.05,42.3828) (0.1,43.75) (0.2,43.9453) (0.5,44.043) (1,45.2148)};
  \addplot[color=htSAMAdam, solid, line width=1.1pt] coordinates {(0.001,41.6016) (0.002,41.6016) (0.005,41.6016) (0.01,41.6016) (0.02,41.6016) (0.05,41.6992) (0.1,41.6992) (0.2,41.7969) (0.5,41.8945) (1,41.8945)};
  \addplot[color=htFSAMAdam, densely dashed, line width=1.4pt] coordinates {(0.001,41.6016) (0.002,41.6992) (0.005,41.6016) (0.01,41.6992) (0.02,41.6992) (0.05,41.6992) (0.1,41.7969) (0.2,41.6992) (0.5,41.7969) (1,41.6992)};
\end{axis}
\begin{axis}[name=ax2_r0,
    at={(0.5650\textwidth,0in)},
    anchor=south west,
    ht panel, xmode=log,
    ymin=35.7148, ymax=48.8555,
    ytick={40,45},
    yticklabels={40,45},
    title={$\alpha=2.0$}]
  \addplot[draw=none, fill=htSOMA, fill opacity=0.06] coordinates {(0.001,45.3125) (0.002,44.9707) (0.005,45.1172) (0.01,44.9707) (0.02,45.166) (0.05,45.459) (0.1,45.2637) (0.2,45.5078) (0.5,46.4844) (1,46.9727) (1,45.7031) (0.5,45.0684) (0.2,44.6777) (0.1,43.5547) (0.05,44.1406) (0.02,43.457) (0.01,43.5547) (0.005,43.1641) (0.002,43.3594) (0.001,43.6523)} \closedcycle;
  \addplot[draw=none, fill=htFPSOMA, fill opacity=0.06] coordinates {(0.001,45.752) (0.002,45.4102) (0.005,45.0195) (0.01,44.5312) (0.02,45.3613) (0.05,45.6543) (0.1,45.6055) (0.2,45.5566) (0.5,47.3145) (1,44.3359) (1,43.5059) (0.5,46.3379) (0.2,45.3125) (0.1,44.7266) (0.05,44.1406) (0.02,43.5547) (0.01,43.457) (0.005,43.75) (0.002,43.3105) (0.001,43.7988)} \closedcycle;
  \addplot[draw=none, fill=htSpec, fill opacity=0.06] coordinates {(0.001,45.7031) (0.002,45.5566) (0.005,44.9219) (0.01,45.2637) (0.02,45.2148) (0.05,45.8008) (0.1,46.2891) (0.2,46.1914) (0.5,46.9727) (1,37.4512) (1,34.1309) (0.5,45.0684) (0.2,44.3359) (0.1,44.4336) (0.05,43.9453) (0.02,43.6035) (0.01,43.8965) (0.005,43.5059) (0.002,43.5059) (0.001,43.457)} \closedcycle;
  \addplot[draw=none, fill=htRand, fill opacity=0.06] coordinates {(0.001,44.7754) (0.002,44.7266) (0.005,44.8242) (0.01,45.3613) (0.02,45.1172) (0.05,45.1172) (0.1,45.3125) (0.2,45.3613) (0.5,45.2637) (1,45.3613) (1,43.9941) (0.5,43.2617) (0.2,43.0664) (0.1,43.7012) (0.05,42.9688) (0.02,43.3105) (0.01,43.3594) (0.005,43.3105) (0.002,42.9199) (0.001,43.2129)} \closedcycle;
  \addplot[draw=none, fill=htPreNS5, fill opacity=0.06] coordinates {(0.001,45.4102) (0.002,44.9219) (0.005,45.4102) (0.01,45.1172) (0.02,45.166) (0.05,45.4102) (0.1,44.7754) (0.2,45.2148) (0.5,46.4355) (1,46.8262) (1,45.5566) (0.5,45.2637) (0.2,44.3848) (0.1,43.7012) (0.05,43.3105) (0.02,43.3594) (0.01,43.5059) (0.005,44.043) (0.002,43.457) (0.001,43.6035)} \closedcycle;
  \addplot[draw=none, fill=htPostNS5, fill opacity=0.06] coordinates {(0.001,45.2148) (0.002,45.166) (0.005,45.1172) (0.01,45.0684) (0.02,44.873) (0.05,45.0195) (0.1,45.1172) (0.2,45.166) (0.5,46.0449) (1,48.291) (1,46.2891) (0.5,45.3613) (0.2,43.7012) (0.1,43.5059) (0.05,43.6035) (0.02,44.043) (0.01,43.7012) (0.005,43.4082) (0.002,43.5059) (0.001,43.0664)} \closedcycle;
  \addplot[draw=none, fill=htMSOMA, fill opacity=0.06] coordinates {(0.001,44.9219) (0.002,45.1172) (0.005,45.5078) (0.01,45.166) (0.02,45.5078) (0.05,45.166) (0.1,45.6055) (0.2,45.3125) (0.5,45.3125) (1,44.873) (1,42.9199) (0.5,43.3594) (0.2,43.3594) (0.1,43.1152) (0.05,43.1152) (0.02,42.7734) (0.01,43.3105) (0.005,43.4082) (0.002,43.1152) (0.001,43.1152)} \closedcycle;
  \addplot[draw=none, fill=htFSAM, fill opacity=0.06] coordinates {(0.001,40.332) (0.002,41.5527) (0.005,41.6504) (0.01,40.8203) (0.02,40.6738) (0.05,41.0645) (0.1,41.0645) (0.2,41.3086) (0.5,40.8203) (1,41.2598) (1,39.9902) (0.5,39.2578) (0.2,38.5742) (0.1,37.8418) (0.05,38.8672) (0.02,38.0371) (0.01,38.7207) (0.005,39.2578) (0.002,39.4043) (0.001,38.2812)} \closedcycle;
  \addplot[draw=none, fill=htSAM, fill opacity=0.06] coordinates {(0.001,40.1367) (0.002,40.9668) (0.005,40.625) (0.01,41.3574) (0.02,41.0156) (0.05,40.5273) (0.1,41.4062) (0.2,41.7969) (0.5,41.4062) (1,40.4297) (1,39.3555) (0.5,39.2578) (0.2,39.9902) (0.1,38.7207) (0.05,38.623) (0.02,38.1836) (0.01,38.7695) (0.005,38.0371) (0.002,38.623) (0.001,39.0625)} \closedcycle;
  \addplot[draw=none, fill=htMSAM, fill opacity=0.06] coordinates {(0.001,40.2832) (0.002,40.8203) (0.005,42.1875) (0.01,41.2598) (0.02,41.6504) (0.05,41.2598) (0.1,41.0645) (0.2,40.5273) (0.5,40.7227) (1,40.4785) (1,38.4766) (0.5,38.2324) (0.2,37.8906) (0.1,38.2812) (0.05,38.3789) (0.02,38.2324) (0.01,38.4766) (0.005,38.0859) (0.002,38.4277) (0.001,37.9883)} \closedcycle;
  \addplot[draw=none, fill=htSAMAdam, fill opacity=0.06] coordinates {(0.001,44.8242) (0.002,44.8242) (0.005,44.6289) (0.01,44.6289) (0.02,44.7754) (0.05,44.7266) (0.1,44.7266) (0.2,44.5801) (0.5,44.3359) (1,44.3359) (1,43.6523) (0.5,43.8965) (0.2,43.8477) (0.1,43.5547) (0.05,43.5547) (0.02,43.75) (0.01,43.7012) (0.005,43.7012) (0.002,43.6523) (0.001,43.6523)} \closedcycle;
  \addplot[draw=none, fill=htFSAMAdam, fill opacity=0.06] coordinates {(0.001,44.8242) (0.002,44.7754) (0.005,44.4824) (0.01,44.6777) (0.02,44.7754) (0.05,44.5801) (0.1,44.7266) (0.2,44.5312) (0.5,44.5312) (1,44.5312) (1,43.9453) (0.5,43.7012) (0.2,43.5547) (0.1,43.75) (0.05,43.6035) (0.02,43.8965) (0.01,43.7012) (0.005,43.7012) (0.002,43.6523) (0.001,43.6523)} \closedcycle;
  \addplot[color=htMuon, dashed, line width=0.7pt, forget plot] coordinates {(0.001,44.043) (1,44.043)};
  \addplot[color=htAdam, dash dot, line width=0.7pt, forget plot] coordinates {(0.001,44.2383) (1,44.2383)};
  \addplot[color=htSpec, solid, line width=0.8pt] coordinates {(0.001,44.043) (0.002,44.5312) (0.005,44.043) (0.01,44.5312) (0.02,44.5312) (0.05,44.9219) (0.1,45.2148) (0.2,45.3125) (0.5,46.2891) (1,36.6211)};
  \addplot[color=htRand, dashed, line width=0.8pt] coordinates {(0.001,44.3359) (0.002,44.1406) (0.005,44.3359) (0.01,44.3359) (0.02,44.4336) (0.05,44.3359) (0.1,44.5312) (0.2,44.4336) (0.5,44.1406) (1,44.3359)};
  \addplot[color=htMSOMA, loosely dashed, line width=0.8pt] coordinates {(0.001,44.2383) (0.002,44.2383) (0.005,44.7266) (0.01,44.6289) (0.02,44.3359) (0.05,44.4336) (0.1,43.8477) (0.2,43.9453) (0.5,44.3359) (1,43.5547)};
  \addplot[color=htFSAM, solid, line width=0.8pt] coordinates {(0.001,39.5508) (0.002,40.2344) (0.005,40.0391) (0.01,39.7461) (0.02,39.6484) (0.05,39.6484) (0.1,39.2578) (0.2,39.4531) (0.5,40.1367) (1,40.2344)};
  \addplot[color=htSAM, solid, line width=0.8pt] coordinates {(0.001,39.5508) (0.002,39.6484) (0.005,39.3555) (0.01,40.0391) (0.02,39.9414) (0.05,39.7461) (0.1,40.0391) (0.2,40.8203) (0.5,40.0391) (1,40.0391)};
  \addplot[color=htMSAM, solid, line width=0.8pt] coordinates {(0.001,38.3789) (0.002,39.1602) (0.005,39.6484) (0.01,39.4531) (0.02,39.4531) (0.05,39.7461) (0.1,39.7461) (0.2,38.9648) (0.5,39.4531) (1,39.5508)};
  \addplot[color=htSOMA, solid, line width=1.1pt] coordinates {(0.001,44.4336) (0.002,44.2383) (0.005,44.6289) (0.01,44.043) (0.02,44.4336) (0.05,44.6289) (0.1,44.3359) (0.2,45.0195) (0.5,45.9961) (1,46.2891)};
  \addplot[color=htFPSOMA, solid, line width=1.1pt] coordinates {(0.001,44.7266) (0.002,44.5312) (0.005,44.3359) (0.01,44.043) (0.02,44.1406) (0.05,44.3359) (0.1,45.4102) (0.2,45.4102) (0.5,46.875) (1,44.043)};
  \addplot[color=htPreNS5, densely dotted, line width=1.1pt] coordinates {(0.001,44.6289) (0.002,44.5312) (0.005,44.4336) (0.01,43.8477) (0.02,44.4336) (0.05,45.0195) (0.1,44.2383) (0.2,44.5312) (0.5,45.5078) (1,46.0938)};
  \addplot[color=htPostNS5, dash dot, line width=1.1pt] coordinates {(0.001,44.6289) (0.002,44.7266) (0.005,44.6289) (0.01,44.2383) (0.02,44.4336) (0.05,44.4336) (0.1,44.9219) (0.2,44.4336) (0.5,45.8984) (1,47.9492)};
  \addplot[color=htSAMAdam, solid, line width=1.1pt] coordinates {(0.001,44.3359) (0.002,44.1406) (0.005,44.2383) (0.01,44.2383) (0.02,44.3359) (0.05,44.3359) (0.1,44.5312) (0.2,44.3359) (0.5,44.1406) (1,43.9453)};
  \addplot[color=htFSAMAdam, densely dashed, line width=1.4pt] coordinates {(0.001,44.1406) (0.002,44.043) (0.005,44.2383) (0.01,44.2383) (0.02,44.1406) (0.05,44.5312) (0.1,44.5312) (0.2,44.3359) (0.5,44.3359) (1,44.2383)};
\end{axis}
\begin{axis}[name=ax3_r0,
    at={(0.8065\textwidth,0in)},
    anchor=south west,
    ht panel, xmode=log,
    ymin=27.5352, ymax=48.832,
    ytick={30,35,40,45},
    yticklabels={30,35,40,45},
    title={$\alpha=3.0$}]
  \addplot[draw=none, fill=htSOMA, fill opacity=0.06] coordinates {(0.001,44.5801) (0.002,44.4336) (0.005,45.3125) (0.01,45.166) (0.02,45.3613) (0.05,45.166) (0.1,45.9961) (0.2,45.4102) (0.5,47.3145) (1,45.8008) (1,45.166) (0.5,45.3125) (0.2,44.2383) (0.1,43.5059) (0.05,43.4082) (0.02,42.9199) (0.01,43.5059) (0.005,43.5059) (0.002,43.8965) (0.001,43.9453)} \closedcycle;
  \addplot[draw=none, fill=htFPSOMA, fill opacity=0.06] coordinates {(0.001,44.2383) (0.002,44.5801) (0.005,44.873) (0.01,44.6289) (0.02,44.3848) (0.05,45.4102) (0.1,46.1914) (0.2,46.3379) (0.5,47.4121) (1,42.4316) (1,41.1621) (0.5,46.3867) (0.2,43.457) (0.1,43.5059) (0.05,43.7012) (0.02,43.5059) (0.01,43.75) (0.005,43.75) (0.002,43.8477) (0.001,43.75)} \closedcycle;
  \addplot[draw=none, fill=htSpec, fill opacity=0.06] coordinates {(0.001,44.8242) (0.002,44.1895) (0.005,44.7754) (0.01,45.166) (0.02,45.166) (0.05,45.1172) (0.1,45.6055) (0.2,46.1914) (0.5,47.9492) (1,29.1016) (1,28.9062) (0.5,46.0938) (0.2,44.6777) (0.1,44.1406) (0.05,44.1406) (0.02,43.3594) (0.01,43.7012) (0.005,43.9453) (0.002,43.5059) (0.001,44.0918)} \closedcycle;
  \addplot[draw=none, fill=htRand, fill opacity=0.06] coordinates {(0.001,44.5801) (0.002,45.0195) (0.005,44.5801) (0.01,44.7754) (0.02,44.9707) (0.05,44.7754) (0.1,44.7754) (0.2,45.1172) (0.5,45.166) (1,46.0449) (1,43.7012) (0.5,43.3105) (0.2,43.5547) (0.1,43.75) (0.05,44.0918) (0.02,43.7012) (0.01,44.0918) (0.005,43.7012) (0.002,43.457) (0.001,43.3594)} \closedcycle;
  \addplot[draw=none, fill=htPreNS5, fill opacity=0.06] coordinates {(0.001,44.4824) (0.002,44.4824) (0.005,45.0684) (0.01,44.873) (0.02,45.8008) (0.05,45.8496) (0.1,44.873) (0.2,45.5566) (0.5,47.9492) (1,47.4609) (1,46.2402) (0.5,46.0938) (0.2,44.4824) (0.1,44.0918) (0.05,44.4336) (0.02,43.7012) (0.01,43.8965) (0.005,43.8965) (0.002,44.0918) (0.001,43.7012)} \closedcycle;
  \addplot[draw=none, fill=htPostNS5, fill opacity=0.06] coordinates {(0.001,45.166) (0.002,45.2637) (0.005,44.5312) (0.01,44.8242) (0.02,45.2148) (0.05,45.0195) (0.1,44.7754) (0.2,46.0449) (0.5,47.2656) (1,47.8516) (1,46.9727) (0.5,45.1172) (0.2,44.6777) (0.1,43.6523) (0.05,44.2383) (0.02,43.8477) (0.01,44.0918) (0.005,42.9199) (0.002,44.043) (0.001,43.5547)} \closedcycle;
  \addplot[draw=none, fill=htMSOMA, fill opacity=0.06] coordinates {(0.001,44.4336) (0.002,44.7266) (0.005,44.9707) (0.01,44.6289) (0.02,44.9707) (0.05,45.0684) (0.1,45.2637) (0.2,46.875) (0.5,45.9961) (1,45.4102) (1,43.9941) (0.5,43.5059) (0.2,43.6523) (0.1,43.5059) (0.05,43.7012) (0.02,43.8965) (0.01,43.5547) (0.005,43.7012) (0.002,43.5547) (0.001,43.7012)} \closedcycle;
  \addplot[draw=none, fill=htFSAM, fill opacity=0.06] coordinates {(0.001,42.3828) (0.002,41.6992) (0.005,42.2363) (0.01,42.4805) (0.02,41.6992) (0.05,42.1875) (0.1,43.2129) (0.2,43.4082) (0.5,43.2129) (1,43.4082) (1,42.1387) (0.5,40.7227) (0.2,41.5039) (0.1,41.0645) (0.05,41.2109) (0.02,39.8926) (0.01,39.9902) (0.005,40.2344) (0.002,39.4531) (0.001,40.4785)} \closedcycle;
  \addplot[draw=none, fill=htSAM, fill opacity=0.06] coordinates {(0.001,42.4316) (0.002,41.9434) (0.005,42.6758) (0.01,42.2363) (0.02,42.627) (0.05,42.4316) (0.1,43.5547) (0.2,43.5547) (0.5,42.4316) (1,43.4082) (1,40.8691) (0.5,41.5039) (0.2,41.5527) (0.1,40.3809) (0.05,39.7949) (0.02,40.5762) (0.01,40.5273) (0.005,39.8926) (0.002,39.9414) (0.001,39.8926)} \closedcycle;
  \addplot[draw=none, fill=htMSAM, fill opacity=0.06] coordinates {(0.001,43.1641) (0.002,43.4082) (0.005,43.457) (0.01,42.7734) (0.02,43.7988) (0.05,43.3594) (0.1,43.5547) (0.2,42.8223) (0.5,43.6523) (1,43.0176) (1,40.4297) (0.5,41.2109) (0.2,41.1133) (0.1,41.1621) (0.05,41.8945) (0.02,41.6016) (0.01,41.8945) (0.005,41.4551) (0.002,41.9922) (0.001,41.0156)} \closedcycle;
  \addplot[draw=none, fill=htSAMAdam, fill opacity=0.06] coordinates {(0.001,46.2891) (0.002,46.1426) (0.005,46.0938) (0.01,46.0938) (0.02,46.0938) (0.05,45.9473) (0.1,46.1914) (0.2,46.1914) (0.5,46.3379) (1,46.4844) (1,45.6543) (0.5,45.7031) (0.2,45.1172) (0.1,45.1172) (0.05,45.2637) (0.02,45.0684) (0.01,44.9219) (0.005,45.2148) (0.002,45.2637) (0.001,45.0684)} \closedcycle;
  \addplot[draw=none, fill=htFSAMAdam, fill opacity=0.06] coordinates {(0.001,45.9473) (0.002,45.9473) (0.005,45.8984) (0.01,46.0938) (0.02,46.1426) (0.05,45.752) (0.1,46.1426) (0.2,45.9473) (0.5,46.2891) (1,46.5332) (1,45.5566) (0.5,45.4102) (0.2,45.1172) (0.1,45.3125) (0.05,45.459) (0.02,45.2148) (0.01,44.9219) (0.005,45.2637) (0.002,45.1172) (0.001,45.1172)} \closedcycle;
  \addplot[color=htMuon, dashed, line width=0.7pt, forget plot] coordinates {(0.001,44.3359) (1,44.3359)};
  \addplot[color=htAdam, dash dot, line width=0.7pt, forget plot] coordinates {(0.001,45.4102) (1,45.4102)};
  \addplot[color=htSpec, solid, line width=0.8pt] coordinates {(0.001,44.4336) (0.002,43.9453) (0.005,44.3359) (0.01,44.4336) (0.02,44.2383) (0.05,44.6289) (0.1,45.0195) (0.2,45.2148) (0.5,46.9727) (1,29.0039)};
  \addplot[color=htRand, dashed, line width=0.8pt] coordinates {(0.001,43.9453) (0.002,44.043) (0.005,44.3359) (0.01,44.1406) (0.02,44.3359) (0.05,44.3359) (0.1,44.1406) (0.2,43.8477) (0.5,44.2383) (1,44.8242)};
  \addplot[color=htMSOMA, loosely dashed, line width=0.8pt] coordinates {(0.001,44.1406) (0.002,44.043) (0.005,44.5312) (0.01,44.1406) (0.02,44.4336) (0.05,44.3359) (0.1,44.7266) (0.2,44.4336) (0.5,44.9219) (1,44.8242)};
  \addplot[color=htFSAM, solid, line width=0.8pt] coordinates {(0.001,41.5039) (0.002,40.918) (0.005,41.4062) (0.01,40.8203) (0.02,41.3086) (0.05,41.8945) (0.1,42.1875) (0.2,42.9688) (0.5,42.3828) (1,42.8711)};
  \addplot[color=htSAM, solid, line width=0.8pt] coordinates {(0.001,40.8203) (0.002,41.0156) (0.005,41.1133) (0.01,41.5039) (0.02,41.4062) (0.05,42.0898) (0.1,41.7969) (0.2,42.8711) (0.5,42.1875) (1,41.7969)};
  \addplot[color=htMSAM, solid, line width=0.8pt] coordinates {(0.001,41.6992) (0.002,42.5781) (0.005,42.0898) (0.01,42.1875) (0.02,42.1875) (0.05,42.3828) (0.1,42.0898) (0.2,42.1875) (0.5,41.7969) (1,41.8945)};
  \addplot[color=htSOMA, solid, line width=1.1pt] coordinates {(0.001,44.1406) (0.002,44.1406) (0.005,44.1406) (0.01,43.8477) (0.02,44.3359) (0.05,44.2383) (0.1,44.5312) (0.2,44.4336) (0.5,46.4844) (1,45.4102)};
  \addplot[color=htFPSOMA, solid, line width=1.1pt] coordinates {(0.001,43.9453) (0.002,44.3359) (0.005,44.5312) (0.01,44.2383) (0.02,44.1406) (0.05,44.2383) (0.1,43.9453) (0.2,44.9219) (0.5,46.6797) (1,41.6992)};
  \addplot[color=htPreNS5, densely dotted, line width=1.1pt] coordinates {(0.001,43.9453) (0.002,44.3359) (0.005,44.2383) (0.01,44.3359) (0.02,45.0195) (0.05,44.7266) (0.1,44.3359) (0.2,44.9219) (0.5,46.875) (1,46.3867)};
  \addplot[color=htPostNS5, dash dot, line width=1.1pt] coordinates {(0.001,44.1406) (0.002,44.2383) (0.005,43.75) (0.01,44.4336) (0.02,44.5312) (0.05,44.6289) (0.1,44.4336) (0.2,45.1172) (0.5,46.2891) (1,47.3633)};
  \addplot[color=htSAMAdam, solid, line width=1.1pt] coordinates {(0.001,45.6055) (0.002,45.6055) (0.005,45.6055) (0.01,45.7031) (0.02,45.6055) (0.05,45.6055) (0.1,45.5078) (0.2,45.5078) (0.5,45.9961) (1,46.0938)};
  \addplot[color=htFSAMAdam, densely dashed, line width=1.4pt] coordinates {(0.001,45.4102) (0.002,45.5078) (0.005,45.5078) (0.01,45.6055) (0.02,45.8008) (0.05,45.5078) (0.1,45.7031) (0.2,45.5078) (0.5,45.8008) (1,46.2891)};
\end{axis}
\begin{axis}[hide axis, scale only axis, at={(ax0_r0.north west)}, anchor=south west, yshift=34pt,
    width=0.9300\textwidth, height=1pt, xmin=0, xmax=1, ymin=0, ymax=1,
    legend columns=7, legend cell align=left,
    legend style={at={(0.5,0)}, anchor=south, draw=none, fill=none, font=\scriptsize, inner sep=1pt,
      /tikz/every even column/.append style={column sep=5pt}}]
  \addlegendimage{color=htSOMA, solid, line width=1.1pt}\addlegendentry{\textsc{Soma}}
  \addlegendimage{color=htFPSOMA, solid, line width=1.1pt}\addlegendentry{\textsc{FP{-}Soma}}
  \addlegendimage{color=htMSOMA, loosely dashed, line width=0.8pt}\addlegendentry{\textsc{Msoma}}
  \addlegendimage{color=htPreNS5, densely dotted, line width=1.1pt}\addlegendentry{\textsc{Opspre}}
  \addlegendimage{color=htPostNS5, dash dot, line width=1.1pt}\addlegendentry{\textsc{Opspost}}
  \addlegendimage{color=htSpec, solid, line width=0.8pt}\addlegendentry{\textsc{Ssm}}
  \addlegendimage{color=htRand, dashed, line width=0.8pt}\addlegendentry{\textsc{Rsm}}
  \addlegendimage{color=htSAM, solid, line width=0.8pt}\addlegendentry{\textsc{Sam}}
  \addlegendimage{color=htFSAM, solid, line width=0.8pt}\addlegendentry{\textsc{F{-}Sam}}
  \addlegendimage{color=htMSAM, solid, line width=0.8pt}\addlegendentry{\textsc{Msam}}
  \addlegendimage{color=htSAMAdam, solid, line width=1.1pt}\addlegendentry{\textsc{Saw}}
  \addlegendimage{color=htFSAMAdam, densely dashed, line width=1.4pt}\addlegendentry{\textsc{Fsaw}}
  \addlegendimage{color=htMuon, dashed, line width=0.7pt}\addlegendentry{\textsc{Muon}}
  \addlegendimage{color=htAdam, dash dot, line width=0.7pt}\addlegendentry{\textsc{AdamW}}
\end{axis}
\path (ax0_r0.south west) -- (ax3_r0.south east) coordinate[midway] (xlab);
\node[font=\small, anchor=north] at ([yshift=-14pt]xlab) {perturbation radius $\rho$};
\end{tikzpicture}

%% file: tikz_for_heavytailed/experiment_1/train_loss_vs_rho.tikz
%% Auto-generated by make_heavytailed_figures.py  --  do not edit by hand, re-run the script.
%% Values: line = median over seeds, band = interquartile range (25th-75th percentile).
%% Preamble needed:  \usepackage{pgfplots}  \pgfplotsset{compat=1.16}   (nothing else)
%% Width follows \textwidth, fonts follow the document (so the legend can never overflow).

\definecolor{htSOMA}{HTML}{0B5D1E}
\definecolor{htFPSOMA}{HTML}{33CC33}
\definecolor{htSpec}{HTML}{00B3A4}
\definecolor{htRand}{HTML}{808000}
\definecolor{htFSAM}{HTML}{2F80D8}
\definecolor{htSAM}{HTML}{0A3A8A}
\definecolor{htPreNS5}{HTML}{9ACD32}
\definecolor{htPostNS5}{HTML}{004D40}
\definecolor{htSAMAdam}{HTML}{8E0F1C}
\definecolor{htFSAMAdam}{HTML}{F4623A}
\definecolor{htMSAM}{HTML}{7DB9F2}
\definecolor{htMSOMA}{HTML}{8FD694}
\definecolor{htMuon}{HTML}{111111}
\definecolor{htAdam}{HTML}{8A8A8A}
\pgfplotsset{
  ht panel/.style={
    scale only axis, width=0.1920\textwidth, height=1.35in,
    xmin=0.00075, xmax=1.35,
    xtick={0.001,0.01,0.1,1}, xticklabels={$10^{-3}$,$10^{-2}$,$10^{-1}$,$10^{0}$},
    minor xtick={0.002,0.003,0.004,0.005,0.006,0.007,0.008,0.009,0.02,0.03,0.04,0.05,0.06,0.07,0.08,0.09,0.2,0.3,0.4,0.5,0.6,0.7,0.8,0.9},
    xmajorgrids=true, ymajorgrids=true,
    grid style={line width=0.25pt, draw=black!10},
    axis line style={line width=0.4pt, black!75},
    tick style={line width=0.4pt, black!75}, major tick length=3pt, minor tick length=1.5pt,
    tick label style={font=\footnotesize, /pgf/number format/assume math mode=true},
    label style={font=\small}, title style={font=\small, yshift=-3pt},
    ylabel style={yshift=-1pt, align=center},
    every tick label/.append style={inner sep=1pt},
    clip=true,
  }}

\begin{tikzpicture}
\begin{axis}[name=ax0_r0,
    at={(0.0820\textwidth,0in)},
    anchor=south west,
    ht panel, xmode=log, ymode=log,
    ymin=1.12286, ymax=28.0524,
    ytick={2,5,10,20},
    yticklabels={2,5,10,20},
    title={$\alpha=1.1$},
    ylabel={training loss}]
  \addplot[draw=none, fill=htSOMA, fill opacity=0.06] coordinates {(0.001,1.46828) (0.002,1.46874) (0.005,1.46907) (0.01,1.4678) (0.02,1.46713) (0.05,1.46436) (0.1,1.4596) (0.2,1.45537) (0.5,1.46626) (1,1.56269) (1,1.5165) (0.5,1.42018) (0.2,1.40945) (0.1,1.41553) (0.05,1.41905) (0.02,1.42129) (0.01,1.42207) (0.005,1.42252) (0.002,1.4225) (0.001,1.42271)} \closedcycle;
  \addplot[draw=none, fill=htFPSOMA, fill opacity=0.06] coordinates {(0.001,1.46918) (0.002,1.46899) (0.005,1.46859) (0.01,1.46702) (0.02,1.4668) (0.05,1.46459) (0.1,1.46167) (0.2,1.46298) (0.5,1.50946) (1,1.6124) (1,1.56499) (0.5,1.44529) (0.2,1.41201) (0.1,1.41392) (0.05,1.41804) (0.02,1.42007) (0.01,1.42194) (0.005,1.42222) (0.002,1.42257) (0.001,1.42278)} \closedcycle;
  \addplot[draw=none, fill=htSpec, fill opacity=0.06] coordinates {(0.001,1.46853) (0.002,1.46898) (0.005,1.46781) (0.01,1.46849) (0.02,1.46936) (0.05,1.47034) (0.1,1.47494) (0.2,1.49817) (0.5,1.59787) (1,1.77006) (1,1.72678) (0.5,1.53724) (0.2,1.4459) (0.1,1.42704) (0.05,1.42439) (0.02,1.42304) (0.01,1.42329) (0.005,1.4223) (0.002,1.42293) (0.001,1.42287)} \closedcycle;
  \addplot[draw=none, fill=htRand, fill opacity=0.06] coordinates {(0.001,1.46815) (0.002,1.46845) (0.005,1.46842) (0.01,1.46872) (0.02,1.46862) (0.05,1.46861) (0.1,1.46753) (0.2,1.46834) (0.5,1.46782) (1,1.47227) (1,1.42672) (0.5,1.42504) (0.2,1.42434) (0.1,1.42289) (0.05,1.42308) (0.02,1.42314) (0.01,1.4225) (0.005,1.42277) (0.002,1.42286) (0.001,1.42269)} \closedcycle;
  \addplot[draw=none, fill=htPreNS5, fill opacity=0.06] coordinates {(0.001,1.46809) (0.002,1.46883) (0.005,1.46833) (0.01,1.46888) (0.02,1.46643) (0.05,1.46657) (0.1,1.46467) (0.2,1.46684) (0.5,1.49329) (1,1.60234) (1,1.57115) (0.5,1.45042) (0.2,1.42423) (0.1,1.42045) (0.05,1.42054) (0.02,1.42283) (0.01,1.42265) (0.005,1.4227) (0.002,1.42296) (0.001,1.42284)} \closedcycle;
  \addplot[draw=none, fill=htPostNS5, fill opacity=0.06] coordinates {(0.001,1.46842) (0.002,1.46863) (0.005,1.46939) (0.01,1.46867) (0.02,1.46778) (0.05,1.46686) (0.1,1.46565) (0.2,1.46512) (0.5,1.47945) (1,1.55209) (1,1.49358) (0.5,1.42941) (0.2,1.42135) (0.1,1.42142) (0.05,1.42168) (0.02,1.42232) (0.01,1.42251) (0.005,1.42296) (0.002,1.42274) (0.001,1.42279)} \closedcycle;
  \addplot[draw=none, fill=htMSOMA, fill opacity=0.06] coordinates {(0.001,1.46803) (0.002,1.46813) (0.005,1.46826) (0.01,1.46826) (0.02,1.46776) (0.05,1.46646) (0.1,1.46527) (0.2,1.46293) (0.5,1.45467) (1,1.44286) (1,1.39897) (0.5,1.41195) (0.2,1.41809) (0.1,1.42108) (0.05,1.42206) (0.02,1.42233) (0.01,1.42278) (0.005,1.42278) (0.002,1.42298) (0.001,1.42295)} \closedcycle;
  \addplot[draw=none, fill=htFSAM, fill opacity=0.06] coordinates {(0.001,24.0857) (0.002,23.829) (0.005,19.4409) (0.01,22.6817) (0.02,20.0505) (0.05,22.3688) (0.1,22.0878) (0.2,24.7444) (0.5,21.5509) (1,23.0669) (1,7.803) (0.5,8.20108) (0.2,8.76755) (0.1,7.21221) (0.05,7.90551) (0.02,7.46938) (0.01,8.47635) (0.005,7.61492) (0.002,9.29205) (0.001,8.64854)} \closedcycle;
  \addplot[draw=none, fill=htSAM, fill opacity=0.06] coordinates {(0.001,26.3768) (0.002,21.1101) (0.005,24.2069) (0.01,22.3121) (0.02,23.9592) (0.05,21.7184) (0.1,24.5129) (0.2,23.7792) (0.5,24.1481) (1,26.2087) (1,8.03888) (0.5,7.64199) (0.2,8.23372) (0.1,7.88943) (0.05,8.75366) (0.02,9.14273) (0.01,8.34069) (0.005,9.414) (0.002,9.62899) (0.001,7.60185)} \closedcycle;
  \addplot[draw=none, fill=htMSAM, fill opacity=0.06] coordinates {(0.001,40.4131) (0.002,41.5119) (0.005,36.1859) (0.01,39.9233) (0.02,38.2125) (0.05,45.0523) (0.1,40.1066) (0.2,39.0548) (0.5,43.4648) (1,39.5678) (1,12.9368) (0.5,14.859) (0.2,12.7656) (0.1,12.2972) (0.05,13.0913) (0.02,10.5627) (0.01,12.8447) (0.005,13.0861) (0.002,14.5134) (0.001,16.9513)} \closedcycle;
  \addplot[draw=none, fill=htSAMAdam, fill opacity=0.06] coordinates {(0.001,2.72292) (0.002,2.7229) (0.005,2.72294) (0.01,2.72289) (0.02,2.72292) (0.05,2.7228) (0.1,2.72264) (0.2,2.72283) (0.5,2.72314) (1,2.72418) (1,2.61255) (0.5,2.61216) (0.2,2.61187) (0.1,2.61192) (0.05,2.61193) (0.02,2.61178) (0.01,2.61184) (0.005,2.61183) (0.002,2.61183) (0.001,2.61184)} \closedcycle;
  \addplot[draw=none, fill=htFSAMAdam, fill opacity=0.06] coordinates {(0.001,2.72292) (0.002,2.72291) (0.005,2.7229) (0.01,2.72289) (0.02,2.72283) (0.05,2.72301) (0.1,2.72276) (0.2,2.72302) (0.5,2.72254) (1,2.72349) (1,2.61206) (0.5,2.61217) (0.2,2.61201) (0.1,2.61194) (0.05,2.61199) (0.02,2.61187) (0.01,2.61188) (0.005,2.61185) (0.002,2.61183) (0.001,2.61183)} \closedcycle;
  \addplot[color=htMuon, dashed, line width=0.7pt, forget plot] coordinates {(0.001,1.45012) (1,1.45012)};
  \addplot[color=htAdam, dash dot, line width=0.7pt, forget plot] coordinates {(0.001,2.69767) (1,2.69767)};
  \addplot[color=htSpec, solid, line width=0.8pt] coordinates {(0.001,1.45023) (0.002,1.44972) (0.005,1.44983) (0.01,1.4491) (0.02,1.44908) (0.05,1.4495) (0.1,1.45588) (0.2,1.4712) (0.5,1.57238) (1,1.7475)};
  \addplot[color=htRand, dashed, line width=0.8pt] coordinates {(0.001,1.45017) (0.002,1.45032) (0.005,1.44999) (0.01,1.4499) (0.02,1.44977) (0.05,1.44989) (0.1,1.44943) (0.2,1.45179) (0.5,1.45116) (1,1.45045)};
  \addplot[color=htMSOMA, loosely dashed, line width=0.8pt] coordinates {(0.001,1.4502) (0.002,1.45016) (0.005,1.44975) (0.01,1.44982) (0.02,1.4488) (0.05,1.44814) (0.1,1.44589) (0.2,1.44286) (0.5,1.43513) (1,1.42459)};
  \addplot[color=htFSAM, solid, line width=0.8pt] coordinates {(0.001,12.6354) (0.002,11.4212) (0.005,11.849) (0.01,12.9569) (0.02,10.6817) (0.05,11.149) (0.1,12.4472) (0.2,12.9797) (0.5,11.1882) (1,10.5488)};
  \addplot[color=htSAM, solid, line width=0.8pt] coordinates {(0.001,10.2204) (0.002,12.8772) (0.005,10.811) (0.01,11.7463) (0.02,12.3425) (0.05,10.9896) (0.1,10.9332) (0.2,11.4282) (0.5,10.8719) (1,8.99626)};
  \addplot[color=htMSAM, solid, line width=0.8pt] coordinates {(0.001,20.5321) (0.002,22.1108) (0.005,21.4469) (0.01,19.624) (0.02,19.1378) (0.05,17.6653) (0.1,19.1466) (0.2,19.5373) (0.5,20.2206) (1,20.2227)};
  \addplot[color=htSOMA, solid, line width=1.1pt] coordinates {(0.001,1.45006) (0.002,1.45009) (0.005,1.44953) (0.01,1.44876) (0.02,1.44598) (0.05,1.44274) (0.1,1.43896) (0.2,1.43125) (0.5,1.4365) (1,1.53023)};
  \addplot[color=htFPSOMA, solid, line width=1.1pt] coordinates {(0.001,1.45013) (0.002,1.44956) (0.005,1.44875) (0.01,1.44798) (0.02,1.44712) (0.05,1.44378) (0.1,1.44226) (0.2,1.43942) (0.5,1.46959) (1,1.58744)};
  \addplot[color=htPreNS5, densely dotted, line width=1.1pt] coordinates {(0.001,1.44985) (0.002,1.45049) (0.005,1.4496) (0.01,1.44977) (0.02,1.44915) (0.05,1.44639) (0.1,1.44382) (0.2,1.44302) (0.5,1.46546) (1,1.58753)};
  \addplot[color=htPostNS5, dash dot, line width=1.1pt] coordinates {(0.001,1.44988) (0.002,1.45012) (0.005,1.45007) (0.01,1.44932) (0.02,1.44832) (0.05,1.44593) (0.1,1.44462) (0.2,1.44207) (0.5,1.4506) (1,1.51793)};
  \addplot[color=htSAMAdam, solid, line width=1.1pt] coordinates {(0.001,2.69767) (0.002,2.69767) (0.005,2.69769) (0.01,2.6977) (0.02,2.69771) (0.05,2.69792) (0.1,2.69774) (0.2,2.69771) (0.5,2.69747) (1,2.69785)};
  \addplot[color=htFSAMAdam, densely dashed, line width=1.4pt] coordinates {(0.001,2.69767) (0.002,2.69767) (0.005,2.69767) (0.01,2.6977) (0.02,2.69773) (0.05,2.69801) (0.1,2.69797) (0.2,2.69778) (0.5,2.69806) (1,2.6988)};
\end{axis}
\begin{axis}[name=ax1_r0,
    at={(0.3235\textwidth,0in)},
    anchor=south west,
    ht panel, xmode=log, ymode=log,
    ymin=0.00710258, ymax=2.46313,
    ytick={0.1,1},
    yticklabels={0.1,1},
    title={$\alpha=1.6$}]
  \addplot[draw=none, fill=htSOMA, fill opacity=0.06] coordinates {(0.001,0.383559) (0.002,0.383458) (0.005,0.384551) (0.01,0.38134) (0.02,0.384271) (0.05,0.385161) (0.1,0.3882) (0.2,0.400939) (0.5,0.52251) (1,0.972998) (1,0.933014) (0.5,0.504438) (0.2,0.387369) (0.1,0.372153) (0.05,0.366559) (0.02,0.365046) (0.01,0.365642) (0.005,0.365678) (0.002,0.36625) (0.001,0.366438)} \closedcycle;
  \addplot[draw=none, fill=htFPSOMA, fill opacity=0.06] coordinates {(0.001,0.384862) (0.002,0.383828) (0.005,0.383139) (0.01,0.382503) (0.02,0.382647) (0.05,0.384473) (0.1,0.401324) (0.2,0.448525) (0.5,0.702388) (1,1.36628) (1,1.33451) (0.5,0.680043) (0.2,0.431136) (0.1,0.383389) (0.05,0.369462) (0.02,0.363683) (0.01,0.366549) (0.005,0.365537) (0.002,0.366167) (0.001,0.365525)} \closedcycle;
  \addplot[draw=none, fill=htSpec, fill opacity=0.06] coordinates {(0.001,0.383884) (0.002,0.383095) (0.005,0.384138) (0.01,0.38441) (0.02,0.380305) (0.05,0.390822) (0.1,0.415539) (0.2,0.491302) (0.5,0.83139) (1,1.66589) (1,1.62887) (0.5,0.815247) (0.2,0.464587) (0.1,0.396711) (0.05,0.376699) (0.02,0.368666) (0.01,0.369264) (0.005,0.365963) (0.002,0.365096) (0.001,0.365928)} \closedcycle;
  \addplot[draw=none, fill=htRand, fill opacity=0.06] coordinates {(0.001,0.386351) (0.002,0.387793) (0.005,0.38518) (0.01,0.384158) (0.02,0.382817) (0.05,0.384265) (0.1,0.385141) (0.2,0.38378) (0.5,0.386053) (1,0.393044) (1,0.364994) (0.5,0.360911) (0.2,0.359393) (0.1,0.36633) (0.05,0.365583) (0.02,0.365739) (0.01,0.365624) (0.005,0.364917) (0.002,0.365349) (0.001,0.364964)} \closedcycle;
  \addplot[draw=none, fill=htPreNS5, fill opacity=0.06] coordinates {(0.001,0.383587) (0.002,0.383762) (0.005,0.383041) (0.01,0.383005) (0.02,0.384227) (0.05,0.383032) (0.1,0.388826) (0.2,0.401174) (0.5,0.513511) (1,0.914641) (1,0.887892) (0.5,0.484019) (0.2,0.374258) (0.1,0.363365) (0.05,0.361369) (0.02,0.365495) (0.01,0.364721) (0.005,0.365539) (0.002,0.366634) (0.001,0.367103)} \closedcycle;
  \addplot[draw=none, fill=htPostNS5, fill opacity=0.06] coordinates {(0.001,0.383796) (0.002,0.38334) (0.005,0.384431) (0.01,0.38193) (0.02,0.383264) (0.05,0.382293) (0.1,0.384226) (0.2,0.391506) (0.5,0.451741) (1,0.73024) (1,0.697307) (0.5,0.432761) (0.2,0.368802) (0.1,0.362375) (0.05,0.364927) (0.02,0.365698) (0.01,0.366562) (0.005,0.365572) (0.002,0.365925) (0.001,0.366072)} \closedcycle;
  \addplot[draw=none, fill=htMSOMA, fill opacity=0.06] coordinates {(0.001,0.385096) (0.002,0.384841) (0.005,0.385468) (0.01,0.3833) (0.02,0.383865) (0.05,0.379935) (0.1,0.381588) (0.2,0.376847) (0.5,0.367592) (1,0.352224) (1,0.328804) (0.5,0.341641) (0.2,0.353044) (0.1,0.357688) (0.05,0.362408) (0.02,0.364128) (0.01,0.364498) (0.005,0.365091) (0.002,0.36482) (0.001,0.365176)} \closedcycle;
  \addplot[draw=none, fill=htFSAM, fill opacity=0.06] coordinates {(0.001,0.0350221) (0.002,0.0196077) (0.005,0.0375331) (0.01,0.0351233) (0.02,0.0314499) (0.05,0.0337712) (0.1,0.0214709) (0.2,0.0353248) (0.5,0.0291557) (1,0.0450456) (1,0.0155482) (0.5,0.0144289) (0.2,0.00795835) (0.1,0.0029806) (0.05,0.00257337) (0.02,0.00578396) (0.01,0.00428055) (0.005,0.00402915) (0.002,0.00539989) (0.001,0.0104156)} \closedcycle;
  \addplot[draw=none, fill=htSAM, fill opacity=0.06] coordinates {(0.001,0.0294805) (0.002,0.0325225) (0.005,0.0211754) (0.01,0.0224799) (0.02,0.0304546) (0.05,0.0322795) (0.1,0.0359102) (0.2,0.0246566) (0.5,0.0490353) (1,0.0578108) (1,0.0255456) (0.5,0.0097451) (0.2,0.00668259) (0.1,0.00481037) (0.05,0.00654271) (0.02,0.00459128) (0.01,0.00688368) (0.005,0.00727114) (0.002,0.00771078) (0.001,0.00939225)} \closedcycle;
  \addplot[draw=none, fill=htMSAM, fill opacity=0.06] coordinates {(0.001,0.052435) (0.002,0.0440916) (0.005,0.0525146) (0.01,0.0585514) (0.02,0.07364) (0.05,0.0631572) (0.1,0.0455915) (0.2,0.0751596) (0.5,0.0680271) (1,0.0681914) (1,0.0100145) (0.5,0.0189675) (0.2,0.0263948) (0.1,0.023226) (0.05,0.0210737) (0.02,0.01382) (0.01,0.0164649) (0.005,0.0211347) (0.002,0.0119139) (0.001,0.0146142)} \closedcycle;
  \addplot[draw=none, fill=htSAMAdam, fill opacity=0.06] coordinates {(0.001,1.34628) (0.002,1.34645) (0.005,1.34663) (0.01,1.34657) (0.02,1.34702) (0.05,1.34917) (0.1,1.35023) (0.2,1.35471) (0.5,1.36798) (1,1.40769) (1,1.35718) (0.5,1.32278) (0.2,1.31068) (0.1,1.30633) (0.05,1.30434) (0.02,1.30391) (0.01,1.30367) (0.005,1.30356) (0.002,1.30378) (0.001,1.30372)} \closedcycle;
  \addplot[draw=none, fill=htFSAMAdam, fill opacity=0.06] coordinates {(0.001,1.34627) (0.002,1.34632) (0.005,1.3465) (0.01,1.34635) (0.02,1.34678) (0.05,1.34832) (0.1,1.34966) (0.2,1.35235) (0.5,1.36157) (1,1.38883) (1,1.34437) (0.5,1.31571) (0.2,1.3087) (0.1,1.30547) (0.05,1.30421) (0.02,1.30404) (0.01,1.30352) (0.005,1.3037) (0.002,1.30376) (0.001,1.30372)} \closedcycle;
  \addplot[color=htMuon, dashed, line width=0.7pt, forget plot] coordinates {(0.001,0.372409) (1,0.372409)};
  \addplot[color=htAdam, dash dot, line width=0.7pt, forget plot] coordinates {(0.001,1.32343) (1,1.32343)};
  \addplot[color=htSpec, solid, line width=0.8pt] coordinates {(0.001,0.372791) (0.002,0.372393) (0.005,0.372956) (0.01,0.374821) (0.02,0.375483) (0.05,0.384565) (0.1,0.412963) (0.2,0.478683) (0.5,0.824378) (1,1.65187)};
  \addplot[color=htRand, dashed, line width=0.8pt] coordinates {(0.001,0.371934) (0.002,0.371622) (0.005,0.371695) (0.01,0.37323) (0.02,0.37177) (0.05,0.37271) (0.1,0.372346) (0.2,0.370672) (0.5,0.371057) (1,0.377434)};
  \addplot[color=htMSOMA, loosely dashed, line width=0.8pt] coordinates {(0.001,0.372524) (0.002,0.371246) (0.005,0.37137) (0.01,0.370909) (0.02,0.370774) (0.05,0.369612) (0.1,0.367221) (0.2,0.359987) (0.5,0.351449) (1,0.338231)};
  \addplot[color=htFSAM, solid, line width=0.8pt] coordinates {(0.001,0.0222246) (0.002,0.0107201) (0.005,0.0222793) (0.01,0.0236972) (0.02,0.0212362) (0.05,0.0159922) (0.1,0.0105908) (0.2,0.0172155) (0.5,0.0211426) (1,0.0270847)};
  \addplot[color=htSAM, solid, line width=0.8pt] coordinates {(0.001,0.0193278) (0.002,0.016261) (0.005,0.0166917) (0.01,0.0130619) (0.02,0.0148252) (0.05,0.0248183) (0.1,0.0140635) (0.2,0.013378) (0.5,0.0238686) (1,0.0446523)};
  \addplot[color=htMSAM, solid, line width=0.8pt] coordinates {(0.001,0.0307483) (0.002,0.0257132) (0.005,0.0310482) (0.01,0.0282426) (0.02,0.0424038) (0.05,0.0390334) (0.1,0.0360014) (0.2,0.0495387) (0.5,0.0400759) (1,0.0396523)};
  \addplot[color=htSOMA, solid, line width=1.1pt] coordinates {(0.001,0.37221) (0.002,0.371116) (0.005,0.371276) (0.01,0.372815) (0.02,0.370942) (0.05,0.371425) (0.1,0.377481) (0.2,0.396694) (0.5,0.516042) (1,0.943441)};
  \addplot[color=htFPSOMA, solid, line width=1.1pt] coordinates {(0.001,0.372064) (0.002,0.372207) (0.005,0.372557) (0.01,0.371894) (0.02,0.372516) (0.05,0.378003) (0.1,0.396059) (0.2,0.439059) (0.5,0.692401) (1,1.36038)};
  \addplot[color=htPreNS5, densely dotted, line width=1.1pt] coordinates {(0.001,0.372144) (0.002,0.371078) (0.005,0.371623) (0.01,0.370648) (0.02,0.372113) (0.05,0.370898) (0.1,0.372727) (0.2,0.386237) (0.5,0.495769) (1,0.902918)};
  \addplot[color=htPostNS5, dash dot, line width=1.1pt] coordinates {(0.001,0.371478) (0.002,0.372096) (0.005,0.37159) (0.01,0.37223) (0.02,0.371389) (0.05,0.370117) (0.1,0.372439) (0.2,0.379325) (0.5,0.443325) (1,0.719756)};
  \addplot[color=htSAMAdam, solid, line width=1.1pt] coordinates {(0.001,1.32354) (0.002,1.32348) (0.005,1.32357) (0.01,1.32405) (0.02,1.32458) (0.05,1.32607) (0.1,1.3277) (0.2,1.33113) (0.5,1.34773) (1,1.37781)};
  \addplot[color=htFSAMAdam, densely dashed, line width=1.4pt] coordinates {(0.001,1.32358) (0.002,1.32362) (0.005,1.32362) (0.01,1.32404) (0.02,1.32426) (0.05,1.32571) (0.1,1.32634) (0.2,1.32912) (0.5,1.34098) (1,1.36345)};
\end{axis}
\begin{axis}[name=ax2_r0,
    at={(0.5650\textwidth,0in)},
    anchor=south west,
    ht panel, xmode=log, ymode=log,
    ymin=0.00433304, ymax=2.80192,
    ytick={0.01,0.1,1},
    yticklabels={0.01,0.1,1},
    title={$\alpha=2.0$}]
  \addplot[draw=none, fill=htSOMA, fill opacity=0.06] coordinates {(0.001,0.189728) (0.002,0.188633) (0.005,0.190145) (0.01,0.191769) (0.02,0.193051) (0.05,0.194626) (0.1,0.195597) (0.2,0.211841) (0.5,0.346128) (1,0.854178) (1,0.835127) (0.5,0.316247) (0.2,0.199605) (0.1,0.189299) (0.05,0.184007) (0.02,0.180356) (0.01,0.179257) (0.005,0.178554) (0.002,0.177495) (0.001,0.179564)} \closedcycle;
  \addplot[draw=none, fill=htFPSOMA, fill opacity=0.06] coordinates {(0.001,0.188666) (0.002,0.189111) (0.005,0.190236) (0.01,0.188948) (0.02,0.190961) (0.05,0.192552) (0.1,0.20518) (0.2,0.246987) (0.5,0.543459) (1,1.46994) (1,1.43736) (0.5,0.522866) (0.2,0.235791) (0.1,0.193798) (0.05,0.187694) (0.02,0.181359) (0.01,0.18127) (0.005,0.178818) (0.002,0.179236) (0.001,0.178446)} \closedcycle;
  \addplot[draw=none, fill=htSpec, fill opacity=0.06] coordinates {(0.001,0.189616) (0.002,0.188353) (0.005,0.190633) (0.01,0.191256) (0.02,0.192096) (0.05,0.198743) (0.1,0.218619) (0.2,0.279164) (0.5,0.671004) (1,1.82645) (1,1.79994) (0.5,0.625814) (0.2,0.267343) (0.1,0.210378) (0.05,0.193169) (0.02,0.184589) (0.01,0.181374) (0.005,0.178914) (0.002,0.179035) (0.001,0.179226)} \closedcycle;
  \addplot[draw=none, fill=htRand, fill opacity=0.06] coordinates {(0.001,0.190471) (0.002,0.188837) (0.005,0.188888) (0.01,0.1905) (0.02,0.191161) (0.05,0.190091) (0.1,0.189942) (0.2,0.188859) (0.5,0.188806) (1,0.188677) (1,0.182301) (0.5,0.182388) (0.2,0.180025) (0.1,0.179497) (0.05,0.17761) (0.02,0.177998) (0.01,0.179851) (0.005,0.178602) (0.002,0.17833) (0.001,0.179006)} \closedcycle;
  \addplot[draw=none, fill=htPreNS5, fill opacity=0.06] coordinates {(0.001,0.190528) (0.002,0.189336) (0.005,0.189512) (0.01,0.191436) (0.02,0.19314) (0.05,0.192138) (0.1,0.19123) (0.2,0.201997) (0.5,0.305145) (1,0.781555) (1,0.742775) (0.5,0.287082) (0.2,0.186769) (0.1,0.182455) (0.05,0.180764) (0.02,0.179451) (0.01,0.177752) (0.005,0.17821) (0.002,0.178338) (0.001,0.179336)} \closedcycle;
  \addplot[draw=none, fill=htPostNS5, fill opacity=0.06] coordinates {(0.001,0.189275) (0.002,0.188885) (0.005,0.190821) (0.01,0.191139) (0.02,0.190786) (0.05,0.19217) (0.1,0.18941) (0.2,0.191558) (0.5,0.243665) (1,0.497132) (1,0.478477) (0.5,0.225701) (0.2,0.186336) (0.1,0.183249) (0.05,0.180554) (0.02,0.177316) (0.01,0.178777) (0.005,0.177753) (0.002,0.179742) (0.001,0.179048)} \closedcycle;
  \addplot[draw=none, fill=htMSOMA, fill opacity=0.06] coordinates {(0.001,0.1887) (0.002,0.18951) (0.005,0.190287) (0.01,0.188864) (0.02,0.188629) (0.05,0.188262) (0.1,0.184721) (0.2,0.178162) (0.5,0.172653) (1,0.159719) (1,0.147961) (0.5,0.162076) (0.2,0.171994) (0.1,0.174857) (0.05,0.175384) (0.02,0.176809) (0.01,0.176044) (0.005,0.178464) (0.002,0.178351) (0.001,0.178449)} \closedcycle;
  \addplot[draw=none, fill=htFSAM, fill opacity=0.06] coordinates {(0.001,0.00923743) (0.002,0.00904572) (0.005,0.00898175) (0.01,0.00910411) (0.02,0.00960517) (0.05,0.00851424) (0.1,0.00832968) (0.2,0.00911359) (0.5,0.0148701) (1,0.0246466) (1,0.0146549) (0.5,0.00747115) (0.2,0.0066092) (0.1,0.00619938) (0.05,0.00576005) (0.02,0.00675139) (0.01,0.00645136) (0.005,0.00658414) (0.002,0.0065737) (0.001,0.00589257)} \closedcycle;
  \addplot[draw=none, fill=htSAM, fill opacity=0.06] coordinates {(0.001,0.00999456) (0.002,0.00896467) (0.005,0.0099606) (0.01,0.00928302) (0.02,0.00934671) (0.05,0.00783572) (0.1,0.00863336) (0.2,0.00920703) (0.5,0.0136739) (1,0.0297361) (1,0.0117949) (0.5,0.00878865) (0.2,0.00673767) (0.1,0.00620737) (0.05,0.00575615) (0.02,0.00621956) (0.01,0.00636867) (0.005,0.00660519) (0.002,0.00617404) (0.001,0.00624468)} \closedcycle;
  \addplot[draw=none, fill=htMSAM, fill opacity=0.06] coordinates {(0.001,0.0106587) (0.002,0.0125484) (0.005,0.0131364) (0.01,0.0129151) (0.02,0.0131946) (0.05,0.01445) (0.1,0.0139904) (0.2,0.0147531) (0.5,0.0144812) (1,0.0168289) (1,0.00889453) (0.5,0.00948951) (0.2,0.0094201) (0.1,0.00976456) (0.05,0.00949393) (0.02,0.00910793) (0.01,0.009187) (0.005,0.00936425) (0.002,0.00913103) (0.001,0.00801938)} \closedcycle;
  \addplot[draw=none, fill=htSAMAdam, fill opacity=0.06] coordinates {(0.001,1.01939) (0.002,1.01922) (0.005,1.01924) (0.01,1.01947) (0.02,1.02112) (0.05,1.02431) (0.1,1.02683) (0.2,1.03373) (0.5,1.06394) (1,1.144) (1,1.08731) (0.5,1.04201) (0.2,1.01538) (0.1,1.00862) (0.05,1.00489) (0.02,1.0034) (0.01,1.00269) (0.005,1.00264) (0.002,1.00258) (0.001,1.00246)} \closedcycle;
  \addplot[draw=none, fill=htFSAMAdam, fill opacity=0.06] coordinates {(0.001,1.01932) (0.002,1.01928) (0.005,1.01947) (0.01,1.01972) (0.02,1.0206) (0.05,1.02365) (0.1,1.02627) (0.2,1.03079) (0.5,1.05199) (1,1.11644) (1,1.07159) (0.5,1.03291) (0.2,1.01267) (0.1,1.00738) (0.05,1.00424) (0.02,1.00342) (0.01,1.00283) (0.005,1.00265) (0.002,1.00255) (0.001,1.00246)} \closedcycle;
  \addplot[color=htMuon, dashed, line width=0.7pt, forget plot] coordinates {(0.001,0.184003) (1,0.184003)};
  \addplot[color=htAdam, dash dot, line width=0.7pt, forget plot] coordinates {(0.001,1.01513) (1,1.01513)};
  \addplot[color=htSpec, solid, line width=0.8pt] coordinates {(0.001,0.183239) (0.002,0.183276) (0.005,0.186917) (0.01,0.186546) (0.02,0.190641) (0.05,0.19458) (0.1,0.215017) (0.2,0.273819) (0.5,0.665164) (1,1.80854)};
  \addplot[color=htRand, dashed, line width=0.8pt] coordinates {(0.001,0.183514) (0.002,0.183777) (0.005,0.184069) (0.01,0.184792) (0.02,0.185396) (0.05,0.184406) (0.1,0.184323) (0.2,0.183693) (0.5,0.184667) (1,0.185431)};
  \addplot[color=htMSOMA, loosely dashed, line width=0.8pt] coordinates {(0.001,0.183673) (0.002,0.183787) (0.005,0.183508) (0.01,0.182975) (0.02,0.182244) (0.05,0.181921) (0.1,0.178244) (0.2,0.174401) (0.5,0.165986) (1,0.153249)};
  \addplot[color=htFSAM, solid, line width=0.8pt] coordinates {(0.001,0.00878639) (0.002,0.00842056) (0.005,0.00745359) (0.01,0.00852114) (0.02,0.00859533) (0.05,0.00745336) (0.1,0.00731904) (0.2,0.00732143) (0.5,0.00921682) (1,0.0171622)};
  \addplot[color=htSAM, solid, line width=0.8pt] coordinates {(0.001,0.00834465) (0.002,0.00827862) (0.005,0.00822785) (0.01,0.00831798) (0.02,0.00806984) (0.05,0.00671307) (0.1,0.00745616) (0.2,0.00828454) (0.5,0.0112454) (1,0.0189921)};
  \addplot[color=htMSAM, solid, line width=0.8pt] coordinates {(0.001,0.0104284) (0.002,0.0108148) (0.005,0.0096936) (0.01,0.0106723) (0.02,0.0111288) (0.05,0.0122139) (0.1,0.0107952) (0.2,0.0116008) (0.5,0.0116581) (1,0.0115915)};
  \addplot[color=htSOMA, solid, line width=1.1pt] coordinates {(0.001,0.183645) (0.002,0.184148) (0.005,0.185526) (0.01,0.184181) (0.02,0.187219) (0.05,0.189555) (0.1,0.193372) (0.2,0.204207) (0.5,0.333566) (1,0.845456)};
  \addplot[color=htFPSOMA, solid, line width=1.1pt] coordinates {(0.001,0.182926) (0.002,0.183316) (0.005,0.185579) (0.01,0.186475) (0.02,0.186384) (0.05,0.191358) (0.1,0.204177) (0.2,0.237456) (0.5,0.532779) (1,1.467)};
  \addplot[color=htPreNS5, densely dotted, line width=1.1pt] coordinates {(0.001,0.184011) (0.002,0.183065) (0.005,0.183562) (0.01,0.182498) (0.02,0.182802) (0.05,0.185297) (0.1,0.187856) (0.2,0.197293) (0.5,0.295419) (1,0.75224)};
  \addplot[color=htPostNS5, dash dot, line width=1.1pt] coordinates {(0.001,0.184054) (0.002,0.183417) (0.005,0.185778) (0.01,0.183679) (0.02,0.183167) (0.05,0.184441) (0.1,0.187542) (0.2,0.188364) (0.5,0.23837) (1,0.490886)};
  \addplot[color=htSAMAdam, solid, line width=1.1pt] coordinates {(0.001,1.01517) (0.002,1.01524) (0.005,1.01542) (0.01,1.01559) (0.02,1.01625) (0.05,1.01785) (0.1,1.02142) (0.2,1.02762) (0.5,1.052) (1,1.11562)};
  \addplot[color=htFSAMAdam, densely dashed, line width=1.4pt] coordinates {(0.001,1.01516) (0.002,1.01518) (0.005,1.01544) (0.01,1.01553) (0.02,1.01613) (0.05,1.0173) (0.1,1.02065) (0.2,1.02525) (0.5,1.04335) (1,1.08986)};
\end{axis}
\begin{axis}[name=ax3_r0,
    at={(0.8065\textwidth,0in)},
    anchor=south west,
    ht panel, xmode=log, ymode=log,
    ymin=0.00728619, ymax=3.22549,
    ytick={0.1,1},
    yticklabels={0.1,1},
    title={$\alpha=3.0$}]
  \addplot[draw=none, fill=htSOMA, fill opacity=0.06] coordinates {(0.001,0.0848723) (0.002,0.0848865) (0.005,0.0846226) (0.01,0.0855025) (0.02,0.0862455) (0.05,0.0867343) (0.1,0.0896383) (0.2,0.101256) (0.5,0.217119) (1,0.829154) (1,0.781895) (0.5,0.201268) (0.2,0.0853377) (0.1,0.075064) (0.05,0.0751378) (0.02,0.0752412) (0.01,0.0749904) (0.005,0.0746779) (0.002,0.0736066) (0.001,0.0739823)} \closedcycle;
  \addplot[draw=none, fill=htFPSOMA, fill opacity=0.06] coordinates {(0.001,0.085624) (0.002,0.0852048) (0.005,0.0866036) (0.01,0.0854976) (0.02,0.085798) (0.05,0.0899423) (0.1,0.097841) (0.2,0.132406) (0.5,0.463482) (1,1.72883) (1,1.71029) (0.5,0.434605) (0.2,0.112143) (0.1,0.079949) (0.05,0.0744945) (0.02,0.0743793) (0.01,0.0739086) (0.005,0.0740912) (0.002,0.0736394) (0.001,0.0740006)} \closedcycle;
  \addplot[draw=none, fill=htSpec, fill opacity=0.06] coordinates {(0.001,0.0852222) (0.002,0.0863334) (0.005,0.0861548) (0.01,0.0847918) (0.02,0.0869833) (0.05,0.0924591) (0.1,0.107166) (0.2,0.153453) (0.5,0.629275) (1,2.14442) (1,2.11831) (0.5,0.592387) (0.2,0.134977) (0.1,0.0877383) (0.05,0.0764008) (0.02,0.0743351) (0.01,0.0749286) (0.005,0.0745185) (0.002,0.074418) (0.001,0.0737185)} \closedcycle;
  \addplot[draw=none, fill=htRand, fill opacity=0.06] coordinates {(0.001,0.0846816) (0.002,0.0849048) (0.005,0.085873) (0.01,0.0854256) (0.02,0.0856523) (0.05,0.0868977) (0.1,0.0857332) (0.2,0.0864265) (0.5,0.087494) (1,0.0886641) (1,0.0719631) (0.5,0.0734228) (0.2,0.0732339) (0.1,0.0739117) (0.05,0.0739453) (0.02,0.0739105) (0.01,0.0735082) (0.005,0.0737602) (0.002,0.0735581) (0.001,0.0738)} \closedcycle;
  \addplot[draw=none, fill=htPreNS5, fill opacity=0.06] coordinates {(0.001,0.0849611) (0.002,0.0848034) (0.005,0.0848006) (0.01,0.0857902) (0.02,0.0860209) (0.05,0.0860979) (0.1,0.0868039) (0.2,0.0918789) (0.5,0.182866) (1,0.725313) (1,0.696658) (0.5,0.172737) (0.2,0.0822535) (0.1,0.0759059) (0.05,0.0740462) (0.02,0.0747464) (0.01,0.0743614) (0.005,0.0745346) (0.002,0.0740689) (0.001,0.0734851)} \closedcycle;
  \addplot[draw=none, fill=htPostNS5, fill opacity=0.06] coordinates {(0.001,0.0856397) (0.002,0.0845396) (0.005,0.0844462) (0.01,0.0852) (0.02,0.0849843) (0.05,0.0834527) (0.1,0.0841043) (0.2,0.0847541) (0.5,0.123208) (1,0.359593) (1,0.346391) (0.5,0.10794) (0.2,0.0742924) (0.1,0.0739378) (0.05,0.0732066) (0.02,0.0728613) (0.01,0.0746943) (0.005,0.0743245) (0.002,0.073343) (0.001,0.0734717)} \closedcycle;
  \addplot[draw=none, fill=htMSOMA, fill opacity=0.06] coordinates {(0.001,0.0852836) (0.002,0.0856162) (0.005,0.0860231) (0.01,0.0848191) (0.02,0.084496) (0.05,0.0843265) (0.1,0.081906) (0.2,0.0792615) (0.5,0.0715716) (1,0.0647906) (1,0.0600804) (0.5,0.065914) (0.2,0.0690968) (0.1,0.070411) (0.05,0.0705709) (0.02,0.0734348) (0.01,0.0731025) (0.005,0.0729382) (0.002,0.0734472) (0.001,0.0728087)} \closedcycle;
  \addplot[draw=none, fill=htFSAM, fill opacity=0.06] coordinates {(0.001,0.0132737) (0.002,0.0135044) (0.005,0.0131912) (0.01,0.013751) (0.02,0.0130213) (0.05,0.012738) (0.1,0.0125989) (0.2,0.0117297) (0.5,0.0131573) (1,0.0274684) (1,0.01501) (0.5,0.0111868) (0.2,0.0104675) (0.1,0.0108774) (0.05,0.0111412) (0.02,0.0118496) (0.01,0.011976) (0.005,0.0119008) (0.002,0.0122214) (0.001,0.0119733)} \closedcycle;
  \addplot[draw=none, fill=htSAM, fill opacity=0.06] coordinates {(0.001,0.0135034) (0.002,0.0133713) (0.005,0.0131235) (0.01,0.0136345) (0.02,0.0126461) (0.05,0.0129509) (0.1,0.0124348) (0.2,0.0113871) (0.5,0.0129004) (1,0.0265692) (1,0.0157343) (0.5,0.0114792) (0.2,0.0106691) (0.1,0.0104187) (0.05,0.011425) (0.02,0.0117373) (0.01,0.0119036) (0.005,0.0119241) (0.002,0.0122126) (0.001,0.0119675)} \closedcycle;
  \addplot[draw=none, fill=htMSAM, fill opacity=0.06] coordinates {(0.001,0.0150503) (0.002,0.0149508) (0.005,0.0152371) (0.01,0.0152181) (0.02,0.0153321) (0.05,0.0152866) (0.1,0.0152813) (0.2,0.0161766) (0.5,0.0167511) (1,0.0149481) (1,0.0126894) (0.5,0.0148412) (0.2,0.0140055) (0.1,0.0141369) (0.05,0.013602) (0.02,0.013252) (0.01,0.013325) (0.005,0.0134545) (0.002,0.0135246) (0.001,0.0135188)} \closedcycle;
  \addplot[draw=none, fill=htSAMAdam, fill opacity=0.06] coordinates {(0.001,0.807891) (0.002,0.807901) (0.005,0.808271) (0.01,0.80868) (0.02,0.810547) (0.05,0.815615) (0.1,0.820058) (0.2,0.832864) (0.5,0.891474) (1,1.03389) (1,0.932163) (0.5,0.857019) (0.2,0.821075) (0.1,0.81132) (0.05,0.804623) (0.02,0.802315) (0.01,0.800932) (0.005,0.80088) (0.002,0.800545) (0.001,0.800248)} \closedcycle;
  \addplot[draw=none, fill=htFSAMAdam, fill opacity=0.06] coordinates {(0.001,0.807789) (0.002,0.807967) (0.005,0.808181) (0.01,0.808758) (0.02,0.810039) (0.05,0.815072) (0.1,0.817948) (0.2,0.82767) (0.5,0.875883) (1,0.995643) (1,0.911052) (0.5,0.8467) (0.2,0.817012) (0.1,0.808513) (0.05,0.804221) (0.02,0.801655) (0.01,0.800937) (0.005,0.80055) (0.002,0.800707) (0.001,0.800246)} \closedcycle;
  \addplot[color=htMuon, dashed, line width=0.7pt, forget plot] coordinates {(0.001,0.0779868) (1,0.0779868)};
  \addplot[color=htAdam, dash dot, line width=0.7pt, forget plot] coordinates {(0.001,0.803203) (1,0.803203)};
  \addplot[color=htSpec, solid, line width=0.8pt] coordinates {(0.001,0.0779238) (0.002,0.07891) (0.005,0.0781108) (0.01,0.0786208) (0.02,0.0794866) (0.05,0.0805033) (0.1,0.0915399) (0.2,0.144541) (0.5,0.606629) (1,2.13094)};
  \addplot[color=htRand, dashed, line width=0.8pt] coordinates {(0.001,0.078488) (0.002,0.0781353) (0.005,0.0783746) (0.01,0.0783423) (0.02,0.0781033) (0.05,0.0778813) (0.1,0.0780532) (0.2,0.078021) (0.5,0.0784176) (1,0.0803797)};
  \addplot[color=htMSOMA, loosely dashed, line width=0.8pt] coordinates {(0.001,0.0786128) (0.002,0.0784238) (0.005,0.0782561) (0.01,0.0780137) (0.02,0.0776298) (0.05,0.07645) (0.1,0.0753218) (0.2,0.0735082) (0.5,0.0678233) (1,0.0627334)};
  \addplot[color=htFSAM, solid, line width=0.8pt] coordinates {(0.001,0.0126247) (0.002,0.0126014) (0.005,0.0122851) (0.01,0.0124862) (0.02,0.0124761) (0.05,0.0119501) (0.1,0.0112689) (0.2,0.011267) (0.5,0.0121523) (1,0.0193723)};
  \addplot[color=htSAM, solid, line width=0.8pt] coordinates {(0.001,0.0125396) (0.002,0.0126887) (0.005,0.0124803) (0.01,0.0125248) (0.02,0.0122286) (0.05,0.0118176) (0.1,0.0112439) (0.2,0.0110287) (0.5,0.0120635) (1,0.0201404)};
  \addplot[color=htMSAM, solid, line width=0.8pt] coordinates {(0.001,0.014091) (0.002,0.0139541) (0.005,0.0140515) (0.01,0.0138111) (0.02,0.0139685) (0.05,0.0140621) (0.1,0.0144924) (0.2,0.0152105) (0.5,0.0152889) (1,0.0140011)};
  \addplot[color=htSOMA, solid, line width=1.1pt] coordinates {(0.001,0.0783882) (0.002,0.0787621) (0.005,0.0793132) (0.01,0.0782836) (0.02,0.0798929) (0.05,0.0819503) (0.1,0.0817467) (0.2,0.0923929) (0.5,0.208818) (1,0.81104)};
  \addplot[color=htFPSOMA, solid, line width=1.1pt] coordinates {(0.001,0.0780177) (0.002,0.0781934) (0.005,0.0786535) (0.01,0.0784604) (0.02,0.0767857) (0.05,0.0785327) (0.1,0.0849766) (0.2,0.120001) (0.5,0.453476) (1,1.72031)};
  \addplot[color=htPreNS5, densely dotted, line width=1.1pt] coordinates {(0.001,0.0784506) (0.002,0.079424) (0.005,0.0782403) (0.01,0.0778536) (0.02,0.0785056) (0.05,0.0790012) (0.1,0.0811734) (0.2,0.0888161) (0.5,0.179426) (1,0.709696)};
  \addplot[color=htPostNS5, dash dot, line width=1.1pt] coordinates {(0.001,0.0789479) (0.002,0.0786452) (0.005,0.0784461) (0.01,0.078358) (0.02,0.0778901) (0.05,0.0778262) (0.1,0.0780506) (0.2,0.0795567) (0.5,0.119628) (1,0.353556)};
  \addplot[color=htSAMAdam, solid, line width=1.1pt] coordinates {(0.001,0.803461) (0.002,0.803535) (0.005,0.803698) (0.01,0.804596) (0.02,0.805678) (0.05,0.808713) (0.1,0.814257) (0.2,0.827003) (0.5,0.875871) (1,0.979826)};
  \addplot[color=htFSAMAdam, densely dashed, line width=1.4pt] coordinates {(0.001,0.803347) (0.002,0.803337) (0.005,0.803696) (0.01,0.804247) (0.02,0.804671) (0.05,0.807359) (0.1,0.811651) (0.2,0.822643) (0.5,0.862827) (1,0.951221)};
\end{axis}
\path (ax0_r0.south west) -- (ax3_r0.south east) coordinate[midway] (xlab);
\node[font=\small, anchor=north] at ([yshift=-14pt]xlab) {perturbation radius $\rho$};
\end{tikzpicture}

%% file: atlas_opt_cpt_curves/figure_curves.tex
% AUTO-GENERATED
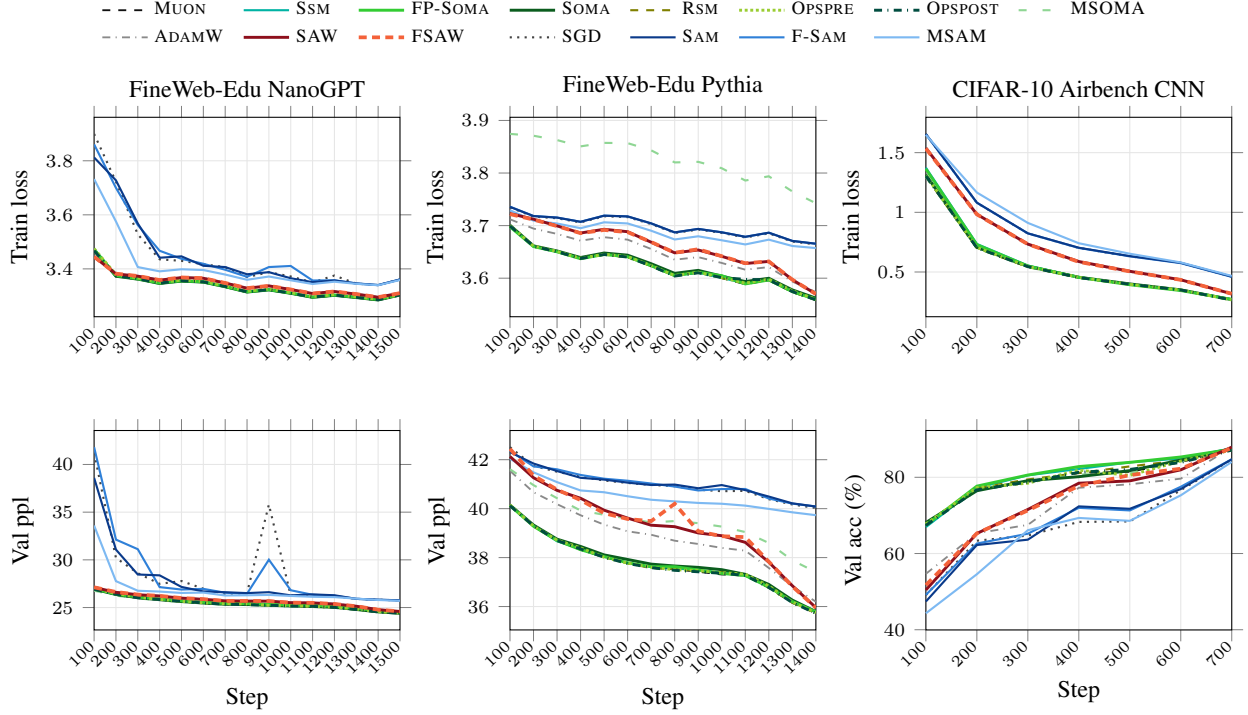
\begin{figure*}[!t]
  \centering
  \curveslegend
  \vspace{4pt}

  \resizebox{\textwidth}{!}{\input{\curvesroot/plots/grid_all.tex}}
  \caption{Training curves of the fifteen optimizers (mean over 3 seeds) on FineWeb-Edu NanoGPT, FineWeb-Edu Pythia and the CIFAR-10 Airbench CNN: training loss, and validation perplexity (language models) or validation accuracy (CNN), against the optimizer step (100 steps per epoch). \textcolor{red}{Red}, \textcolor{blue}{Blue} and \textcolor{ForestGreen}{Green} lines indicate \adamw-based, \sgd-based and \muon-based methods, respectively.}
  \label{fig:curves}
\end{figure*}

%% file: atlas_opt_cpt_curves/plots/grid_all.tex
% AUTO-GENERATED: columns = datasets (NanoGPT | Pythia | CIFAR-10), row 1 = train loss, row 2 = Val ppl (language models) or Val acc (CNN)
\begin{tikzpicture}
\begin{groupplot}[grid curve axis, group style={group size=3 by 2, horizontal sep=1.45cm, vertical sep=1.5cm}]
\nextgroupplot[xmin=100, xmax=1500, xtick={100,200,300,400,500,600,700,800,900,1000,1100,1200,1300,1400,1500}, ylabel={Train loss}, title={FineWeb-Edu NanoGPT}]
\addplot[cv-muon] table[col sep=comma, x=step, y=train_loss]{\curvesroot/data/nanogpt/muon.csv};
\addplot[cv-adam] table[col sep=comma, x=step, y=train_loss]{\curvesroot/data/nanogpt/adam.csv};
\addplot[cv-sgd] table[col sep=comma, x=step, y=train_loss]{\curvesroot/data/nanogpt/sgd.csv};
\addplot[cv-muon-sam] table[col sep=comma, x=step, y=train_loss]{\curvesroot/data/nanogpt/muon_sam.csv};
\addplot[cv-randsam-muon] table[col sep=comma, x=step, y=train_loss]{\curvesroot/data/nanogpt/randsam_muon.csv};
\addplot[cv-msoma] table[col sep=comma, x=step, y=train_loss]{\curvesroot/data/nanogpt/msoma.csv};
\addplot[cv-fsam] table[col sep=comma, x=step, y=train_loss]{\curvesroot/data/nanogpt/fsam.csv};
\addplot[cv-sam] table[col sep=comma, x=step, y=train_loss]{\curvesroot/data/nanogpt/sam.csv};
\addplot[cv-msam] table[col sep=comma, x=step, y=train_loss]{\curvesroot/data/nanogpt/msam.csv};
\addplot[cv-fsam-ortho-muon-stale-momentum] table[col sep=comma, x=step, y=train_loss]{\curvesroot/data/nanogpt/fsam_ortho_muon_stale_momentum.csv};
\addplot[cv-fsam-ortho-muon] table[col sep=comma, x=step, y=train_loss]{\curvesroot/data/nanogpt/fsam_ortho_muon.csv};
\addplot[cv-soma-prens5] table[col sep=comma, x=step, y=train_loss]{\curvesroot/data/nanogpt/soma_prens5.csv};
\addplot[cv-op-soma-postns5] table[col sep=comma, x=step, y=train_loss]{\curvesroot/data/nanogpt/op_soma_postns5.csv};
\addplot[cv-sam-adam] table[col sep=comma, x=step, y=train_loss]{\curvesroot/data/nanogpt/sam_adam.csv};
\addplot[cv-fsam-adam] table[col sep=comma, x=step, y=train_loss]{\curvesroot/data/nanogpt/fsam_adam.csv};
\nextgroupplot[xmin=100, xmax=1400, xtick={100,200,300,400,500,600,700,800,900,1000,1100,1200,1300,1400}, ylabel={Train loss}, title={FineWeb-Edu Pythia}]
\addplot[cv-muon] table[col sep=comma, x=step, y=train_loss]{\curvesroot/data/pythia70m/muon.csv};
\addplot[cv-adam] table[col sep=comma, x=step, y=train_loss]{\curvesroot/data/pythia70m/adam.csv};
\addplot[cv-sgd] table[col sep=comma, x=step, y=train_loss]{\curvesroot/data/pythia70m/sgd.csv};
\addplot[cv-muon-sam] table[col sep=comma, x=step, y=train_loss]{\curvesroot/data/pythia70m/muon_sam.csv};
\addplot[cv-randsam-muon] table[col sep=comma, x=step, y=train_loss]{\curvesroot/data/pythia70m/randsam_muon.csv};
\addplot[cv-msoma] table[col sep=comma, x=step, y=train_loss]{\curvesroot/data/pythia70m/msoma.csv};
\addplot[cv-fsam] table[col sep=comma, x=step, y=train_loss]{\curvesroot/data/pythia70m/fsam.csv};
\addplot[cv-sam] table[col sep=comma, x=step, y=train_loss]{\curvesroot/data/pythia70m/sam.csv};
\addplot[cv-msam] table[col sep=comma, x=step, y=train_loss]{\curvesroot/data/pythia70m/msam.csv};
\addplot[cv-fsam-ortho-muon-stale-momentum] table[col sep=comma, x=step, y=train_loss]{\curvesroot/data/pythia70m/fsam_ortho_muon_stale_momentum.csv};
\addplot[cv-fsam-ortho-muon] table[col sep=comma, x=step, y=train_loss]{\curvesroot/data/pythia70m/fsam_ortho_muon.csv};
\addplot[cv-soma-prens5] table[col sep=comma, x=step, y=train_loss]{\curvesroot/data/pythia70m/soma_prens5.csv};
\addplot[cv-op-soma-postns5] table[col sep=comma, x=step, y=train_loss]{\curvesroot/data/pythia70m/op_soma_postns5.csv};
\addplot[cv-sam-adam] table[col sep=comma, x=step, y=train_loss]{\curvesroot/data/pythia70m/sam_adam.csv};
\addplot[cv-fsam-adam] table[col sep=comma, x=step, y=train_loss]{\curvesroot/data/pythia70m/fsam_adam.csv};
\nextgroupplot[xmin=100, xmax=700, xtick={100,200,300,400,500,600,700}, ylabel={Train loss}, title={CIFAR-10 Airbench CNN}]
\addplot[cv-muon] table[col sep=comma, x=step, y=train_loss]{\curvesroot/data/cifar10/muon.csv};
\addplot[cv-adam] table[col sep=comma, x=step, y=train_loss]{\curvesroot/data/cifar10/adam.csv};
\addplot[cv-sgd] table[col sep=comma, x=step, y=train_loss]{\curvesroot/data/cifar10/sgd.csv};
\addplot[cv-muon-sam] table[col sep=comma, x=step, y=train_loss]{\curvesroot/data/cifar10/muon_sam.csv};
\addplot[cv-randsam-muon] table[col sep=comma, x=step, y=train_loss]{\curvesroot/data/cifar10/randsam_muon.csv};
\addplot[cv-msoma] table[col sep=comma, x=step, y=train_loss]{\curvesroot/data/cifar10/msoma.csv};
\addplot[cv-fsam] table[col sep=comma, x=step, y=train_loss]{\curvesroot/data/cifar10/fsam.csv};
\addplot[cv-sam] table[col sep=comma, x=step, y=train_loss]{\curvesroot/data/cifar10/sam.csv};
\addplot[cv-msam] table[col sep=comma, x=step, y=train_loss]{\curvesroot/data/cifar10/msam.csv};
\addplot[cv-fsam-ortho-muon-stale-momentum] table[col sep=comma, x=step, y=train_loss]{\curvesroot/data/cifar10/fsam_ortho_muon_stale_momentum.csv};
\addplot[cv-fsam-ortho-muon] table[col sep=comma, x=step, y=train_loss]{\curvesroot/data/cifar10/fsam_ortho_muon.csv};
\addplot[cv-soma-prens5] table[col sep=comma, x=step, y=train_loss]{\curvesroot/data/cifar10/soma_prens5.csv};
\addplot[cv-op-soma-postns5] table[col sep=comma, x=step, y=train_loss]{\curvesroot/data/cifar10/op_soma_postns5.csv};
\addplot[cv-sam-adam] table[col sep=comma, x=step, y=train_loss]{\curvesroot/data/cifar10/sam_adam.csv};
\addplot[cv-fsam-adam] table[col sep=comma, x=step, y=train_loss]{\curvesroot/data/cifar10/fsam_adam.csv};
\nextgroupplot[xmin=100, xmax=1500, xtick={100,200,300,400,500,600,700,800,900,1000,1100,1200,1300,1400,1500}, ylabel={Val ppl}, xlabel={Step}]
\addplot[cv-muon] table[col sep=comma, x=step, y=val_ppl]{\curvesroot/data/nanogpt/muon.csv};
\addplot[cv-adam] table[col sep=comma, x=step, y=val_ppl]{\curvesroot/data/nanogpt/adam.csv};
\addplot[cv-sgd] table[col sep=comma, x=step, y=val_ppl]{\curvesroot/data/nanogpt/sgd.csv};
\addplot[cv-muon-sam] table[col sep=comma, x=step, y=val_ppl]{\curvesroot/data/nanogpt/muon_sam.csv};
\addplot[cv-randsam-muon] table[col sep=comma, x=step, y=val_ppl]{\curvesroot/data/nanogpt/randsam_muon.csv};
\addplot[cv-msoma] table[col sep=comma, x=step, y=val_ppl]{\curvesroot/data/nanogpt/msoma.csv};
\addplot[cv-fsam] table[col sep=comma, x=step, y=val_ppl]{\curvesroot/data/nanogpt/fsam.csv};
\addplot[cv-sam] table[col sep=comma, x=step, y=val_ppl]{\curvesroot/data/nanogpt/sam.csv};
\addplot[cv-msam] table[col sep=comma, x=step, y=val_ppl]{\curvesroot/data/nanogpt/msam.csv};
\addplot[cv-fsam-ortho-muon-stale-momentum] table[col sep=comma, x=step, y=val_ppl]{\curvesroot/data/nanogpt/fsam_ortho_muon_stale_momentum.csv};
\addplot[cv-fsam-ortho-muon] table[col sep=comma, x=step, y=val_ppl]{\curvesroot/data/nanogpt/fsam_ortho_muon.csv};
\addplot[cv-soma-prens5] table[col sep=comma, x=step, y=val_ppl]{\curvesroot/data/nanogpt/soma_prens5.csv};
\addplot[cv-op-soma-postns5] table[col sep=comma, x=step, y=val_ppl]{\curvesroot/data/nanogpt/op_soma_postns5.csv};
\addplot[cv-sam-adam] table[col sep=comma, x=step, y=val_ppl]{\curvesroot/data/nanogpt/sam_adam.csv};
\addplot[cv-fsam-adam] table[col sep=comma, x=step, y=val_ppl]{\curvesroot/data/nanogpt/fsam_adam.csv};
\nextgroupplot[xmin=100, xmax=1400, xtick={100,200,300,400,500,600,700,800,900,1000,1100,1200,1300,1400}, ylabel={Val ppl}, xlabel={Step}]
\addplot[cv-muon] table[col sep=comma, x=step, y=val_ppl]{\curvesroot/data/pythia70m/muon.csv};
\addplot[cv-adam] table[col sep=comma, x=step, y=val_ppl]{\curvesroot/data/pythia70m/adam.csv};
\addplot[cv-sgd] table[col sep=comma, x=step, y=val_ppl]{\curvesroot/data/pythia70m/sgd.csv};
\addplot[cv-muon-sam] table[col sep=comma, x=step, y=val_ppl]{\curvesroot/data/pythia70m/muon_sam.csv};
\addplot[cv-randsam-muon] table[col sep=comma, x=step, y=val_ppl]{\curvesroot/data/pythia70m/randsam_muon.csv};
\addplot[cv-msoma] table[col sep=comma, x=step, y=val_ppl]{\curvesroot/data/pythia70m/msoma.csv};
\addplot[cv-fsam] table[col sep=comma, x=step, y=val_ppl]{\curvesroot/data/pythia70m/fsam.csv};
\addplot[cv-sam] table[col sep=comma, x=step, y=val_ppl]{\curvesroot/data/pythia70m/sam.csv};
\addplot[cv-msam] table[col sep=comma, x=step, y=val_ppl]{\curvesroot/data/pythia70m/msam.csv};
\addplot[cv-fsam-ortho-muon-stale-momentum] table[col sep=comma, x=step, y=val_ppl]{\curvesroot/data/pythia70m/fsam_ortho_muon_stale_momentum.csv};
\addplot[cv-fsam-ortho-muon] table[col sep=comma, x=step, y=val_ppl]{\curvesroot/data/pythia70m/fsam_ortho_muon.csv};
\addplot[cv-soma-prens5] table[col sep=comma, x=step, y=val_ppl]{\curvesroot/data/pythia70m/soma_prens5.csv};
\addplot[cv-op-soma-postns5] table[col sep=comma, x=step, y=val_ppl]{\curvesroot/data/pythia70m/op_soma_postns5.csv};
\addplot[cv-sam-adam] table[col sep=comma, x=step, y=val_ppl]{\curvesroot/data/pythia70m/sam_adam.csv};
\addplot[cv-fsam-adam] table[col sep=comma, x=step, y=val_ppl]{\curvesroot/data/pythia70m/fsam_adam.csv};
\nextgroupplot[xmin=100, xmax=700, xtick={100,200,300,400,500,600,700}, ylabel={Val acc (\%)}, xlabel={Step}]
\addplot[cv-muon] table[col sep=comma, x=step, y=val_acc]{\curvesroot/data/cifar10/muon.csv};
\addplot[cv-adam] table[col sep=comma, x=step, y=val_acc]{\curvesroot/data/cifar10/adam.csv};
\addplot[cv-sgd] table[col sep=comma, x=step, y=val_acc]{\curvesroot/data/cifar10/sgd.csv};
\addplot[cv-muon-sam] table[col sep=comma, x=step, y=val_acc]{\curvesroot/data/cifar10/muon_sam.csv};
\addplot[cv-randsam-muon] table[col sep=comma, x=step, y=val_acc]{\curvesroot/data/cifar10/randsam_muon.csv};
\addplot[cv-msoma] table[col sep=comma, x=step, y=val_acc]{\curvesroot/data/cifar10/msoma.csv};
\addplot[cv-fsam] table[col sep=comma, x=step, y=val_acc]{\curvesroot/data/cifar10/fsam.csv};
\addplot[cv-sam] table[col sep=comma, x=step, y=val_acc]{\curvesroot/data/cifar10/sam.csv};
\addplot[cv-msam] table[col sep=comma, x=step, y=val_acc]{\curvesroot/data/cifar10/msam.csv};
\addplot[cv-fsam-ortho-muon-stale-momentum] table[col sep=comma, x=step, y=val_acc]{\curvesroot/data/cifar10/fsam_ortho_muon_stale_momentum.csv};
\addplot[cv-fsam-ortho-muon] table[col sep=comma, x=step, y=val_acc]{\curvesroot/data/cifar10/fsam_ortho_muon.csv};
\addplot[cv-soma-prens5] table[col sep=comma, x=step, y=val_acc]{\curvesroot/data/cifar10/soma_prens5.csv};
\addplot[cv-op-soma-postns5] table[col sep=comma, x=step, y=val_acc]{\curvesroot/data/cifar10/op_soma_postns5.csv};
\addplot[cv-sam-adam] table[col sep=comma, x=step, y=val_acc]{\curvesroot/data/cifar10/sam_adam.csv};
\addplot[cv-fsam-adam] table[col sep=comma, x=step, y=val_acc]{\curvesroot/data/cifar10/fsam_adam.csv};
\end{groupplot}
\end{tikzpicture}

%% file: SOMA_exp2/figure_train_loss.tex
% % AUTO-GENERATED: \input this in the body (needs common/preamble.tex in the preamble)
% \begin{figure*}[t]
%   \centering
%   \input{\figroot/train_loss_grid.tex}

%   \vspace{2pt}
%   \ref{trainlosslegend}
%   \caption{Fine-tuning training loss (\adamw, learning rate $0.000224$) of the models pretrained with each optimizer. Rows: nanoGPT (GPT-2 small) and Pythia-70M. Columns: fine-tuning dataset.}
%   \label{fig:trainloss}
% \end{figure*}

%% file: paper_tables/table_ptq_exp.tex
\makeatletter
\@ifundefined{rowcolor}{\PackageError{table}{Row colours need the colortbl package}{Add \string\usepackage{colortbl} (after xcolor) to the preamble.}}{}
\makeatother
\definecolor{adamrow}{RGB}{255,214,214}   % AdamW family (red)
\definecolor{sgdrow}{RGB}{208,226,247}    % SGD family   (blue)
\definecolor{muonrow}{RGB}{208,236,214}   % Muon family  (green)

\begin{table*}[t]
  \centering
  \caption{INT4 post-training quantization results on held-out FineWeb-Edu perplexity (lower is better), comparing full precision (FP) and 4-bit AWQ (INT4). Results show mean $\pm$ std over 3 pretraining seeds (INT4 averaged over 3 calibration seeds); bold indicates the lowest value per column. Row colour marks the base optimizer family: \muon-based (\colorbox{muonrow}{Green}), \adamw-based (\colorbox{adamrow}{Red}), \sgd-based (\colorbox{sgdrow}{Blue}).}
  \label{tab:ptq-fineweb}
  \vspace{2pt}
  \resizebox{\ifdim\width>\textwidth\textwidth\else\width\fi}{!}{%
  \footnotesize\setlength{\tabcolsep}{5pt}
  \begin{tabular}{lcccccc}
    \toprule
    & \multicolumn{3}{c}{FineWeb-Edu NanoGPT} & \multicolumn{3}{c}{FineWeb-Edu Pythia} \\
    \cmidrule(lr){2-4}\cmidrule(lr){5-7}
    Optimizer & FP & INT4 & $\Delta\%$ & FP & INT4 & $\Delta\%$ \\
    \midrule
    \rowcolor{muonrow} \muon & 22.96{\scriptsize$\pm$0.01} & 23.87{\scriptsize$\pm$0.01} & +3.98{\scriptsize$\pm$0.06} & \textbf{33.60}{\scriptsize$\pm$0.01} & 37.28{\scriptsize$\pm$0.10} & +10.96{\scriptsize$\pm$0.27} \\
    \rowcolor{muonrow} \specsammuon & 23.09{\scriptsize$\pm$0.00} & 23.87{\scriptsize$\pm$0.03} & \textbf{+3.38}{\scriptsize$\pm$0.11} & 33.65{\scriptsize$\pm$0.02} & 37.22{\scriptsize$\pm$0.01} & +10.61{\scriptsize$\pm$0.08} \\
    \rowcolor{muonrow} \fpsoma & 23.07{\scriptsize$\pm$0.00} & 23.86{\scriptsize$\pm$0.03} & +3.42{\scriptsize$\pm$0.09} & 33.65{\scriptsize$\pm$0.02} & \textbf{37.08}{\scriptsize$\pm$0.10} & \textbf{+10.20}{\scriptsize$\pm$0.26} \\
    \rowcolor{muonrow} \soma & \textbf{22.95}{\scriptsize$\pm$0.01} & 23.86{\scriptsize$\pm$0.03} & +3.96{\scriptsize$\pm$0.12} & 33.66{\scriptsize$\pm$0.01} & 37.21{\scriptsize$\pm$0.08} & +10.55{\scriptsize$\pm$0.19} \\
    \rowcolor{muonrow} \randsammuon & 22.96{\scriptsize$\pm$0.01} & 23.87{\scriptsize$\pm$0.01} & +3.98{\scriptsize$\pm$0.02} & 33.61{\scriptsize$\pm$0.02} & 37.34{\scriptsize$\pm$0.10} & +11.13{\scriptsize$\pm$0.26} \\
    \rowcolor{muonrow} \opsomapre & 22.97{\scriptsize$\pm$0.01} & \textbf{23.85}{\scriptsize$\pm$0.03} & +3.79{\scriptsize$\pm$0.11} & 33.62{\scriptsize$\pm$0.01} & 37.21{\scriptsize$\pm$0.07} & +10.66{\scriptsize$\pm$0.23} \\
    \rowcolor{muonrow} \opsomapost & 22.96{\scriptsize$\pm$0.01} & 23.86{\scriptsize$\pm$0.04} & +3.92{\scriptsize$\pm$0.15} & \textbf{33.60}{\scriptsize$\pm$0.02} & 37.29{\scriptsize$\pm$0.02} & +10.97{\scriptsize$\pm$0.10} \\
    \rowcolor{muonrow} \msoma & \textbf{22.95}{\scriptsize$\pm$0.01} & 23.87{\scriptsize$\pm$0.02} & +4.00{\scriptsize$\pm$0.06} & 35.11{\scriptsize$\pm$0.04} & 38.70{\scriptsize$\pm$0.20} & +10.21{\scriptsize$\pm$0.56} \\
    \rowcolor{adamrow} \adamw & 23.12{\scriptsize$\pm$0.01} & 23.94{\scriptsize$\pm$0.01} & +3.55{\scriptsize$\pm$0.02} & 34.03{\scriptsize$\pm$0.02} & 37.81{\scriptsize$\pm$0.01} & +11.12{\scriptsize$\pm$0.04} \\
    \rowcolor{adamrow} \saw & 23.12{\scriptsize$\pm$0.01} & 23.94{\scriptsize$\pm$0.01} & +3.58{\scriptsize$\pm$0.06} & 33.80{\scriptsize$\pm$0.01} & 37.34{\scriptsize$\pm$0.03} & +10.46{\scriptsize$\pm$0.06} \\
    \rowcolor{adamrow} \fsaw & 23.12{\scriptsize$\pm$0.01} & 23.94{\scriptsize$\pm$0.01} & +3.56{\scriptsize$\pm$0.01} & 33.80{\scriptsize$\pm$0.01} & 37.32{\scriptsize$\pm$0.07} & +10.39{\scriptsize$\pm$0.19} \\
    \rowcolor{sgdrow} \sgd & 24.17{\scriptsize$\pm$0.01} & 25.25{\scriptsize$\pm$0.01} & +4.46{\scriptsize$\pm$0.02} & 37.77{\scriptsize$\pm$0.04} & 43.24{\scriptsize$\pm$0.08} & +14.46{\scriptsize$\pm$0.14} \\
    \rowcolor{sgdrow} \sam & 24.17{\scriptsize$\pm$0.00} & 25.27{\scriptsize$\pm$0.01} & +4.53{\scriptsize$\pm$0.05} & 37.82{\scriptsize$\pm$0.02} & 43.33{\scriptsize$\pm$0.05} & +14.56{\scriptsize$\pm$0.08} \\
    \rowcolor{sgdrow} \fsam & 24.18{\scriptsize$\pm$0.00} & 25.26{\scriptsize$\pm$0.00} & +4.48{\scriptsize$\pm$0.02} & 37.82{\scriptsize$\pm$0.02} & 43.33{\scriptsize$\pm$0.03} & +14.56{\scriptsize$\pm$0.10} \\
    \rowcolor{sgdrow} \msam & 24.13{\scriptsize$\pm$0.00} & 25.21{\scriptsize$\pm$0.01} & +4.46{\scriptsize$\pm$0.05} & 37.43{\scriptsize$\pm$0.02} & 42.85{\scriptsize$\pm$0.07} & +14.49{\scriptsize$\pm$0.20} \\
    \bottomrule
  \end{tabular}}
\end{table*}

%% file: sections/conclusion.tex
We introduced \soma, a sharpness-aware variant of \muon that perturbs weights by reusing quantities \muon already computes. We proved that sharpness-aware \muon algorithms converge under heavy-tailed noise with an error floor that grows with the radius but not with the noise. Our experiments confirm that the base optimizer largely determines robustness to heavy-tailed noise. For \muon-based methods, perturbing in the spectral geometry consistently helps, more than the choice of direction. We observe that \soma matches its two-gradient counterpart in accuracy despite using a stale gradient, is more robust to a large radius in accuracy, and that its best radius shrinks as the noise gets lighter, as our bound predicts. In language-model training, \soma maintains the quality of \muon at about a $3\%$ extra run-time cost, moderately improves robustness to quantization, and reduces forgetting only slightly. Testing \soma at a larger scale and fine-tuning with the same optimizer used in pretraining are natural next steps.

%% file: sections/variants.tex
\definecolor{htSOMA}{HTML}{0B5D1E}
\definecolor{htFPSOMA}{HTML}{33CC33}
\definecolor{htSpec}{HTML}{00B3A4}
\definecolor{htRand}{HTML}{808000}
\definecolor{htFSAM}{HTML}{2F80D8}
\definecolor{htSAM}{HTML}{0A3A8A}
\definecolor{htPreNS5}{HTML}{9ACD32}
\definecolor{htPostNS5}{HTML}{004D40}
\definecolor{htSAMAdam}{HTML}{8E0F1C}
\definecolor{htFSAMAdam}{HTML}{F4623A}
\definecolor{htMSAM}{HTML}{7DB9F2}
\definecolor{htMSOMA}{HTML}{8FD694}
\definecolor{htMuon}{HTML}{111111}
\definecolor{htAdam}{HTML}{8A8A8A}

\begin{algorithm}[H]
    \caption{\textsc{SpecSamMuon }(\specsammuon)\citep{zhong2026sharpness}}
    \label{alg:specsam}
    \begin{algorithmic}[1]
        \REQUIRE Initial weights $W_0 \in \mathbb{R}^{m \times n}$; momentum $M_0 = \mathbf{0}$; learning rate $\eta > 0$; momentum coefficient $\beta \in [0, 1)$; perturbation radius $\rho > 0$; batch size $B$; iterations $T$; $\ns$ iterations $q$.

        \FOR {$t = 1, \ldots, T$}
        \STATE Sample minibatch $\mathcal B_t\subset \mathcal S$; $|\mathcal B_t|=B$
        \STATE $\displaystyle G_{t, 0} = \dfrac{1}{B} \sum\limits_{i = 1}^{B} \nabla_W f(W_{t - 1}; \mathcal{E}_{t, i})$\hfill\COMMENT{gradient computation}
        
        \STATE \colorbox{htSpec}{$E_t = \rho \NSQ(G_{t, 0})$}
        
        \STATE $\displaystyle \overline{G}_t = \dfrac{1}{B} \sum\limits_{i = 1}^{B} \nabla_W f(W_{t - 1} + E_t; \mathcal{E}_{t, i})$\hfill\COMMENT{gradient computation}
        
        \STATE $M_t = \beta \, M_{t - 1} + \overline{G}_t$
        
        \STATE $O_t = \NSQ(M_t)$
        
        \STATE $W_t = W_{t - 1} - \eta \, O_t$
        
        \ENDFOR        
    \end{algorithmic}
\end{algorithm}

\begin{algorithm}[H]
    \caption{\sam + \muon with Frobenius perturbation \citep{zhong2026sharpness}}
    \label{alg:sam_muon_fro}
    \begin{algorithmic}[1]
        \REQUIRE Initial weights $W_0 \in \mathbb{R}^{m \times n}$; momentum $M_0 = \mathbf{0}$; learning rate $\eta > 0$; momentum coefficient $\beta \in [0, 1)$; perturbation radius $\rho > 0$; batch size $B$; iterations $T$; $\ns$ iterations $q$.

        \FOR {$t = 1, \ldots, T$}
        \STATE Sample minibatch $\mathcal B_t\subset \mathcal S$; $|\mathcal B_t|=B$
        \STATE $\displaystyle G_t = \dfrac{1}{B} \sum\limits_{i = 1}^{B} \nabla_W f(W_{t - 1}; \mathcal{E}_{t, i})$\hfill\COMMENT{gradient computation}
        
        \STATE \colorbox{htSpec!30}{$E_t = \rho \dfrac{G_t}{\normF{G_t}}$}
        
        \STATE $\displaystyle \overline{G}_t = \dfrac{1}{B} \sum\limits_{i = 1}^{B} \nabla_W f(W_{t - 1} + E_t; \mathcal{E}_{t, i})$\hfill\COMMENT{gradient computation}
        
        \STATE $M_t = \beta \, M_{t - 1} + \overline{G}_t$
        
        \STATE $O_t = \NSQ(M_t)$
        
        \STATE $W_t = W_{t - 1} - \eta \, O_t$
        
        \ENDFOR
    \end{algorithmic}
\end{algorithm}

\begin{algorithm}[H]
    \caption{Freshly-Perturbed-SOMA (\fpsoma)}
    \label{alg:fpsoma}
    \begin{algorithmic}[1]
        \REQUIRE Initialize weights $W_0 \in \mathbb{R}^{m \times n}$; \ema $V_0 = \mathbf{0}$; momentum $M_0 = \mathbf{0}$; learning rate $\eta > 0$; momentum coefficient $\beta \in [0, 1)$; \ema subtraction parameter $\gamma \in (0, 1)$; perturbation radius $\rho > 0$; batch size $B$; iterations $T$; $\ns$ iterations $q$.

        \FOR {$t = 1, \dots, T$}
        \STATE Sample minibatch $\mathcal B_t\subset \mathcal S$; $|\mathcal B_t|=B$
        \STATE $\displaystyle G_t = \dfrac{1}{B} \sum\limits_{i = 1}^{B} \nabla_W f(W_{t - 1}; \mathcal{E}_{t, i})$\hfill\COMMENT{gradient computation}
        
        \STATE \colorbox{htFPSOMA}{ $V_t = \beta \, V_{t - 1} + (1 - \beta) \, G_t$}
        
        \STATE \colorbox{htFPSOMA}{$X_{t, 0} = G_t - \gamma \, V_t$}
        
        \STATE \colorbox{htFPSOMA}{$E_t = \rho \NSQ(X_{t, 0})$}
        
        \STATE $\displaystyle \overline{G}_t = \dfrac{1}{B} \sum\limits_{i = 1}^{B} \nabla_W f(W_{t - 1} + E_t; \mathcal{E}_{t, i})$\hfill\COMMENT{gradient computation}
        
        \STATE $M_t = \beta \, M_{t - 1} + \overline{G}_t$
        
        \STATE $O_t = \NSQ(M_t)$
        
        \STATE $W_t = W_{t - 1} - \eta \, O_t$
        
        \ENDFOR
    \end{algorithmic}
\end{algorithm}

\begin{algorithm}[H]
    \caption{\soma}
    \label{alg:soma}
    \begin{algorithmic}[1]
        \REQUIRE Initialize weights $W_0 \in \mathbb{R}^{m \times n}$; momentum $M_0 = \mathbf{0}$; learning rate $\eta > 0$; momentum coefficient $\beta \in [0, 1)$; \ema subtraction parameter $\gamma \in (0, 1)$; perturbation radius $\rho > 0$; batch size $B$; iterations $T$; $\ns$ iterations $q$.
        
        \FOR {$t = 1, \dots, T$}
        \STATE Sample minibatch $\mathcal B_t\subset \mathcal S$; $|\mathcal B_t|=B$
        \IF {$t = 1$}

        \STATE {$\displaystyle G_1 = \dfrac{1}{B} \sum\limits_{i = 1}^{B} \nabla_W f(W_0; \mathcal{E}_{1, i})$}
        
        \STATE {$V_1 = (1 - \beta) \, G_1$}
        
        \STATE {$X_{1, 0} = G_1 - \gamma \, V_1$}
        
        \STATE {$E_1 = \rho \NSQ(X_{1, 0})$}
        
        \ELSE
        
        \STATE \colorbox{htSOMA!55}{$V_t = (1 - \beta) \, M_{t - 1}$}
        
        \STATE \colorbox{htSOMA!55}{$X_{t, 0} = \overline{G}_{t - 1} - \gamma \, V_t$}
        
        \STATE \colorbox{htSOMA!55}{$E_t = \rho \NSQ(X_{t, 0})$}
        
        \ENDIF
        
        \STATE $\displaystyle \overline{G}_t = \dfrac{1}{B} \sum\limits_{i = 1}^{B} \nabla_W f(W_{t - 1} + E_t; \mathcal{E}_{t, i})$ \quad \colorbox{htSOMA!40}{\COMMENT{Save it in buffer}}\hfill\COMMENT{gradient computation}
        
        \STATE $M_t = \beta \, M_{t - 1} + \overline{G}_t$ \quad \colorbox{htSOMA!40}{\COMMENT{Save it in buffer}}
        
        \STATE $O_t = \NSQ(M_t)$
        
        \STATE $W_t = W_{t - 1} - \eta \, O_t$
        
        \ENDFOR
    \end{algorithmic}
\end{algorithm}

\begin{algorithm}[H]
    \caption{\opsoma (Pre-$\NSQ$)}
    \label{alg:opsoma}
    \begin{algorithmic}[1]
        \REQUIRE Initialize weights $W_0 \in \mathbb{R}^{m \times n}$; momentum $M_0 = \mathbf{0}$; learning rate $\eta > 0$; momentum coefficient $\beta \in [0, 1)$; \ema subtraction parameter $\gamma \in (0, 1)$; perturbation radius $\rho > 0$; batch size $B$; iterations $T$; $\ns$ iterations $q$.

        \FOR {$t = 1, \dots, T$}
        \STATE Sample minibatch $\mathcal B_t\subset \mathcal S$; $|\mathcal B_t|=B$
        \IF {$t = 1$}
        
        \STATE $\displaystyle G_1 = \dfrac{1}{B} \sum\limits_{i = 1}^{B} \nabla_W f(W_0; \mathcal{E}_{1, i})$
        
        \STATE $V_1 = (1 - \beta) \, G_1$
        
        \STATE $X_{1, 0} = G_1 - \gamma \, V_1$
        
        \STATE $E_1 = \rho \NSQ(X_{1, 0})$
        
        \ELSE
        
        \STATE \colorbox{htPreNS5}{$R_t = \overline{G}_{t - 1} - \gamma \, (1 - \beta) \, M_{t - 1}$}
        
        \STATE \colorbox{htPreNS5}{$P_t = \projperp{M_{t - 1}} {R_t}$}
        
        \STATE \colorbox{htPreNS5}{$E_t = \rho \NSQ(P_t)$}
        
        \ENDIF
        
        \STATE $\displaystyle \overline{G}_t = \dfrac{1}{B} \sum\limits_{i = 1}^{B} \nabla_W f(W_{t - 1} + E_t; \mathcal{E}_{t, i})$ \quad \colorbox{htPreNS5!50}{\COMMENT{Save it in buffer}}\hfill\COMMENT{gradient computation}
        
        \STATE $M_t = \beta M_{t - 1} + \overline{G}_t$ \quad \colorbox{htPreNS5!50}{\COMMENT{Save it in buffer}}
        
        \STATE $O_t = \NSQ(M_t)$
        
        \STATE $W_t = W_{t - 1} - \eta \, O_t$
        
        \ENDFOR
    \end{algorithmic}
\end{algorithm}

\begin{algorithm}[H]
    \caption{\opsoma (Post-$\NSQ$)}
    \label{alg:opsoma-post}
    \begin{algorithmic}[1]
        \REQUIRE Initialize weights $W_0 \in \mathbb{R}^{m \times n}$; momentum $M_0 = \mathbf{0}$; learning rate $\eta > 0$; momentum coefficient $\beta \in [0, 1)$; \ema subtraction parameter $\gamma \in (0, 1)$; perturbation radius $\rho > 0$; batch size $B$; iterations $T$; $\ns$ iterations $q$.

        \FOR {$t = 1, \dots, T$}
        \STATE Sample minibatch $\mathcal B_t\subset \mathcal S$; $|\mathcal B_t|=B$
        \IF {$t = 1$}
        
        \STATE $\displaystyle G_1 = \dfrac{1}{B} \sum\limits_{i = 1}^{B} \nabla_W f(W_0; \mathcal{E}_{1, i})$
        
        \STATE $V_1 = (1 - \beta) \, G_1$
        
        \STATE $X_{1, 0 } = G_1 - \gamma \, V_1$
        
        \STATE $E_1 = \rho \NSQ(X_{1, 0})$
        
        \ELSE
        
        \STATE \colorbox{htPostNS5!50}{$R_t = \NSQ \left( \overline{G}_{t - 1} - \gamma \, (1 - \beta) \, M_{t - 1} \right)$}
        
        \STATE  \colorbox{htPostNS5!50}{$P_t = \projperp{O_{t - 1}}{R_t}$}
        
        \STATE  \colorbox{htPostNS5!50}{$E_t = \rho \dfrac{P_t}{\normF{P_t}}$}
        
        \ENDIF
        
        \STATE $\displaystyle \overline{G}_t = \dfrac{1}{B} \sum\limits_{i = 1}^{B} \nabla_W f(W_{t - 1} + E_t; \mathcal{E}_{t, i})$ \quad  \colorbox{htPostNS5!25}{\COMMENT{Save it in buffer}}\hfill\COMMENT{gradient computation}
        
        \STATE $M_t = \beta \, M_{t - 1} + \overline{G}_t$ \quad  \colorbox{htPostNS5!25}{\COMMENT{Save it in buffer}}
        
        \STATE $O_t = \NSQ(M_t)$
        
        \STATE $W_t = W_{t - 1} - \eta \, O_t$
        
        \ENDFOR
    \end{algorithmic}
\end{algorithm}

\begin{algorithm}[H]
    \caption{\msoma}
    \label{alg:msoma}
    \begin{algorithmic}[1]
        \REQUIRE Initialize weights $\widetilde{W}_0 \in \mathbb{R}^{m \times n}$; momentum $M_0 = \mathbf{0}$; learning rate $\eta > 0$; momentum coefficient $\beta \in [0, 1)$; perturbation radius $\rho > 0$; batch size $B$; iterations $T$; $\ns$ iterations $q$.
        
        \FOR {$t = 1, \dots, T$}
        
        \STATE Sample minibatch $\mathcal B_t\subset \mathcal S$; $|\mathcal B_t|=B$
        
        \STATE $\displaystyle G_t = \dfrac{1}{B} \sum\limits_{i = 1}^{B} \nabla_W f(\widetilde{W}_{t - 1}; \mathcal{E}_{t, i})$\hfill\COMMENT{gradient computation}
        
        \STATE \colorbox{htMSOMA}{$\displaystyle W_{t - 1}= \widetilde{W}_{t - 1} + \rho \dfrac{M_{t - 1}}{\normF{M_{t - 1}}}$; $\dfrac{M_{0}}{\normF{M_{0}}} \coloneq 0$} \hfill \COMMENT{removing perturbation to obtain unperturbed state}
        
        \STATE \colorbox{htMSOMA}{$M_{t} = \beta \, M_{t-1} + G_t$}
        
        \STATE$O_t = \NSQ(M_t)$
        
        \STATE $W_t = W_{t - 1} - \eta \, O_t$
        
        \STATE \colorbox{htMSOMA}{Perturbed Step: $\displaystyle \widetilde{W}_t = W_{t} - \rho \dfrac{M_{t}}{\normF{M_{t}}}$} \hfill \COMMENT{negative perturbation}
        
        \ENDFOR
        
        \STATE $\displaystyle W_T = \widetilde{W}_T + \rho \, \dfrac{M_{T}}{\normF{M_{T}}}$
    \end{algorithmic}
\end{algorithm}

%% file: sections/heavy.tex
We shall first provide some comments on the development of our proofs for demonstrating our main result which is our Theorem~\ref{thm:sam_muon_nu_p_convergence}. Our motivation for this result is by generalizing the work of \citep{kim2026convergence} especially their descent lemma and their proof of convergence of \muon under finite $\ns$ iterations. We adopt these strategies for \soma and broaden the scope and applicability by using a more general $(\nu, L)$-Smoothness as in our Assumption~\ref{assump:nu_L_smoothness}. Combined with the more realistic framework of bounded $p$-central moment as in our Assumption~\ref{assum:pth_bounded_moment}, our proof strategy, while in essence inspired by \citep[Theorem 1]{kim2026convergence} for their establishment of convergence of \muon with $\ns$, in fact goes beyond in helping us establish a more general and broader result. This is that our developed algorithm \soma while being a more efficient algorithm also converges under heavy-tailed noise. 

We lay out the exhaustive details of the proof of our Theorem~\ref{thm:sam_muon_nu_p_convergence} by first proving our version of a descent lemma under our assumed notion of smoothness.

\begin{lemma}[Descent Lemma under $(\nu, L)$-Smoothness] \label{lem:descent_holder_smooth}
Let $f : \R^{m \times n} \to \R$ be $(\nu, L)$-smooth with $\nu \in (0, 1]$ and $L > 0$. Then, for all $X, Y \in \R^{m \times n}$, we have that:
\[
f(Y) \le f(X) + \innerF{\G f(X)}{Y - X} + \dfrac{L}{\nu + 1} \normop{Y - X}^{\nu + 1}.
\]
\end{lemma}

\begin{proof}
Define an auxiliary scalar function $g : [0, 1] \to \R$ along the line segment from $X$ to $Y$ by $g(s) \coloneq f(X + s \, (Y - X))$. By the fundamental theorem of calculus:
\[
f(Y) - f(X) = g(1) - g(0) = \int\limits_{0}^{1} g'(s) \, ds.
\]
Applying the chain rule, we get: $g'(s) = \innerF{\G f(X + s \, (Y - X))}{Y - X}$. Substituting this back into the integral gives:
\[
f(Y) - f(X) = \int\limits_{0}^{1} \innerF{\G f(X + s \, (Y - X))}{Y - X} \, ds.
\]
Adding and subtracting $\displaystyle \innerF{\G f(X)}{Y - X} = \int\limits_{0}^{1} \innerF{\G f(X)}{Y - X} \, ds$, we obtain:
\[
f(Y) - f(X) = \innerF{\G f(X)}{Y - X} + \int\limits_{0}^{1} \innerF{\G f(X + s \, (Y - X)) - \G f(X)}{Y - X} \, ds.
\]
By Hölder's inequality: $\abs{\innerF{A}{B}} \le \normnuc{A} \normop{B}$, we bound the integrand:
\[
\innerF{\G f(X + s \, (Y - X)) - \G f(X)}{Y - X} \le \normnuc{\G f(X + s \, (Y - X)) - \G f(X)} \normop{Y - X}.
\]
Then, using Assumption~\ref{assump:nu_L_smoothness}, we have that:
\[
\normnuc{\G f(X + s \, (Y - X)) - \G f(X)} \le L \normop{s (Y - X)}^\nu = L \, s^\nu \normop{Y - X}^\nu.
\]
Combining these inequalities yields:
\begin{align*}
f(Y) - f(X) &\le \innerF{\G f(X)}{Y - X} + \int\limits_{0}^{1} L \, s^\nu \normop{Y - X}^{\nu + 1} \, ds\\
&= \innerF{\G f(X)}{Y - X} + L \normop{Y - X}^{\nu + 1} \int\limits_{0}^{1} s^\nu \, ds.
\end{align*}
Evaluating $\displaystyle\int\limits_{0}^{1} s^\nu \, ds = \left[ \dfrac{s^{\nu + 1}}{\nu + 1} \right]_{0}^{1} = \dfrac{1}{\nu + 1}$ gives the final result:
\[
f(Y) \le f(X) + \innerF{\G f(X)}{Y - X} + \dfrac{L}{\nu + 1} \normop{Y - X}^{\nu + 1}. \qedhere
\]
\end{proof}

\subsection{Proof of Theorem~\ref{thm:sam_muon_nu_p_convergence}}

\begin{proof} \label{proof:Theorem1}
Introduce the scaled \ema momentum: $N_t \coloneq (1 -\beta) \, M_t$. Then, we obtain $N_t = \beta \, N_{t - 1} + (1 - \beta) \, \overline{G}_t$. Let $P_t$ denote the polar factor of $M_t$, that is, $P_t = \ortho{M_t}$. Since the polar factor is scale-invariant, we have $P_t \coloneq \ortho{N_t} = \ortho{M_t}$. Using Assumption~\ref{assump:nu_L_smoothness} and applying the descent lemma (Lemma~\ref{lem:descent_holder_smooth}) to $W_t = W_{t - 1} - \eta \, O_t$ yields:
\begin{align*}
f(W_t) &\le f(W_{t - 1}) + \innerF{\G f(W_{t - 1})}{W_t - W_{t - 1}} + \dfrac{L}{\nu + 1} \normop{W_t - W_{t - 1}}^{\nu+1}\\
&= f(W_{t - 1}) - \eta \innerF{\G f(W_{t - 1})}{O_t} + \dfrac{L \, \eta^{\nu + 1}}{\nu + 1} \normop{O_t}^{\nu + 1}\\
&= f(W_{t - 1}) - \eta \innerF{N_t}{O_t} + \eta \innerF{N_t - \G f(W_{t - 1})}{O_t} + \dfrac{L \, \eta^{\nu + 1}}{\nu + 1} \normop{O_t}^{\nu + 1}\\
&\stackrel{(a)}{\le} f(W_{t - 1}) - \eta \innerF{N_t}{O_t} + \eta \normnuc{N_t - \G f(W_{t - 1})} \normop{O_t} + \dfrac{L \, \eta^{\nu + 1}}{\nu + 1} \normop{O_t}^{\nu + 1}\\
&= f(W_{t - 1}) - \eta \innerF{N_t}{P_t} + \eta \innerF{N_t}{P_t - O_t} + \eta \normnuc{N_t - \G f(W_{t - 1})} \normop{O_t} + \dfrac{L \, \eta^{\nu + 1}}{\nu + 1} \normop{O_t}^{\nu + 1}\\
&\stackrel{(b)}{=} f(W_{t - 1}) - \eta \normnuc{N_t} + \eta \innerF{N_t}{P_t - O_t} + \eta \normnuc{N_t - \G f(W_{t - 1})} \normop{O_t} + \dfrac{L \, \eta^{\nu + 1}}{\nu + 1} \normop{O_t}^{\nu + 1}\\
&\stackrel{(c)}{\le} f(W_{t - 1}) - \eta \normnuc{N_t} + \eta \normnuc{N_t}\normop{P_t - O_t }+ \eta \normnuc{N_t - \G f(W_{t - 1})} \normop{O_t} + \dfrac{L \, \eta^{\nu + 1}}{\nu + 1} \normop{O_t}^{\nu + 1},
\end{align*}
where steps $(a)$ and $(c)$ follow from Hölder's inequality, and equality $(b)$ follows from the identity: $\innerF{N_t}{P_t} = \normnuc{N_t}$. We now define the polar approximation error, that is, the error between the exact polar factor $P_t = \ortho{N_t} = \ortho{M_t}$ and the actual step $O_t = \ns(M_t, q)$ generated by $q$ steps of $\ns$ iterations:
\[
\varepsilon_{t,q} \coloneq \normop{O_t - P_t} \quad \text{and} \quad \varepsilon_q \coloneq \sup_t \varepsilon_{t, q}.
\]
Using $\normop{O_t} \le 1 + \varepsilon_q$ and $\normop{P_t - O_t} \le \varepsilon_q$, we have:
\[
f(W_t) \le f(W_{t - 1}) - \eta(1 - \varepsilon_q) \normnuc{N_t} + \eta(1 + \varepsilon_q) \normnuc{N_t - \G f(W_{t - 1})} + \dfrac{L \, \eta^{\nu + 1}}{\nu + 1}(1 + \varepsilon_q)^{\nu + 1}.
\]
By the reverse triangle inequality, that is, $\normnuc{N_t} \ge \normnuc{\G f(W_{t - 1})} - \normnuc{N_t - \G f(W_{t - 1})}$, we then obtain:
\begin{equation} \label{eq:2_descent_step3}
f(W_t) \le f(W_{t - 1}) - \eta (1 - \varepsilon_q) \normnuc{\G f(W_{t - 1})} + 2 \eta \normnuc{N_t - \G f(W_{t - 1})} + \dfrac{L \, \eta^{\nu + 1}}{\nu + 1}(1 + \varepsilon_q)^{\nu + 1}. 
\end{equation}
Since $E_t = \rho \, \ortho{O_t}$ and $\normop{O_t} \le 1$, we have $\normop{E_t} = \normop{\rho \, \ortho{O_t}} \le \rho$. Even with the use of $\ns$ for $\ortho{O_t}$, there will only be a constant error term of the order of $(1 + \varepsilon_q)$ which scales $\rho$. In order to avoid unnecessarily complicated expressions that result from this being factored in, we simply leave it out of our discussions for convenience, but one needs to be cognizant of this minor additional error factor. 

Next, applying $(\nu, L)$-smoothness with $\normop{E_t} \le \rho$, we get that:
\[
\normnuc{\G f(W_{t - 1} + E_t) - \G f(W_{t - 1})} \le L \normop{E_t}^\nu \le L \, \rho^\nu.
\]
By triangle inequality, we have that: $\normnuc{N_t - \G f(W_{t - 1})} \le \normnuc{N_t - \G f(W_{t - 1} + E_t)} + L \, \rho^\nu$. Substituting this into Eq.~\ref{eq:2_descent_step3}, we get that:
\[
f(W_t) \le f(W_{t - 1}) - \eta \, (1 - \varepsilon_q) \normnuc{\G f(W_{t - 1})} + 2 \eta \normnuc{N_t - \G f(W_{t - 1} + E_t)} + 2 \eta \, L \, \rho^\nu + \dfrac{L \, \eta^{\nu + 1}}{\nu + 1}(1 + \varepsilon_q)^{\nu + 1}.
\]
Rearranging the terms and taking an expectation on both sides yields:
\begin{multline} \label{eq:2_expectation_bound}
\E \left[ \normnuc{\G f(W_{t - 1})} \right] \le \dfrac{\E [f(W_{t - 1}) - f(W_t)]}{\eta \, (1 - \varepsilon_q)} + \dfrac{2}{1 - \varepsilon_q} \E \left[ \normnuc{N_t - \G f(W_{t - 1} + E_t)} \right]\\
+ \dfrac{2 L \, \rho^\nu}{1 - \varepsilon_q} + \dfrac{L \, \eta^\nu (1 + \varepsilon_q)^{\nu + 1}}{(\nu + 1)(1 - \varepsilon_q)}.
\end{multline}
In Eq.~\ref{eq:2_expectation_bound}, we shall first obtain a bound for the second term on the right hand side, and then combine it with a characterization using the gradient-gap inequality as in \citep[Lemma 7]{kim2026convergence} to guide our closing arguments which then yields the required $\epsilon$-stationarity result.

\noindent \textbf{Bound for the term $\E \left[ \normnuc{N_t - \G f(W_{t - 1} + E_t)} \right]$.} We begin by first recalling our $\ema$ momentum that we introduced at the beginning of this proof:
\[
N_t = \beta \, N_{t - 1} + (1 - \beta) \, \overline{G}_t, \quad \text{for~} t \ge 0 \text{~with~} N_0 = 0,
\]
where $\overline{G}_i = \frac{1}{B}\sum_{j = 1}^B \G_W f(W_{i - 1} + E_i)$. Unrolling the recursion yields $N_t = (1 - \beta )\sum_{i = 1}^t \beta^{t - i} \, \overline{G}_i$. So, we have that:
\begin{equation} \label{eq:2_tracking_error_expanded}
N_t - \G f(W_{t - 1} + E_t) = (1 - \beta)\sum_{i = 1}^t \beta^{t - i} \Big( \overline{G}_i - \G f(W_{t - 1} + E_t) \Big) - \beta^t \, \G f(W_{t - 1} + E_t).
\end{equation}
Adding and subtracting $\G f(W_{i-1} + E_i)$ to this expression then provides:
\[
\overline{G}_i - \G f(W_{t - 1} + E_t) = \underbrace{\Big( \overline{G}_i - \G f(W_{i - 1} + E_i) \Big)}_{V_i} + \underbrace{\Big( \G f(W_{i - 1} + E_i) - \G f(W_{t - 1} + E_t) \Big)}_{D_{i, t}},
\]
where $V_i \coloneq \overline{G}_i - \G f(W_{i - 1} + E_i)$ and $D_{i, t} \coloneq \G f(W_{i - 1} + E_i) - \G f(W_{t - 1} + E_t)$.
Substituting this decomposition into Eq.~\ref{eq:2_tracking_error_expanded} yields:
\[
N_t - \G f(W_{t - 1} + E_t) = (1 - \beta) \sum\limits_{i = 1}^t \beta^{t - i} V_i + (1 - \beta) \sum\limits_{i = 1}^t \beta^{t - i} D_{i, t} - \beta^t \G f(W_{t - 1} + E_t).
\]
Taking the nuclear norm $\normnuc{\cdot}$, and using that $\normnuc{A}\le \sqrt{r}\normF{A}$, we obtain:
\begin{multline} \label{eq:2_tracking_error_split}
\E \left[ \normnuc{N_t - \G f(W_{t - 1} + E_t)} \right] \le \underbrace{(1 - \beta) \sqrt{r} \, \E \left[ \normF{ \sum\limits_{i = 1}^t \beta^{t - i} V_i } \right]}_{\text{Term A}} + \underbrace{(1 - \beta) \sum\limits_{i = 1}^t \beta^{t - i} \E \left[ \normnuc{D_{i, t}} \right]}_{\text{Term B}} \\
+ \underbrace{\vphantom{\sum\limits_{i = 1}^t} \beta^t \E \left[ \normnuc{\G f(W_{t - 1} + E_t)} \right]}_{\text{Term C}}.
\end{multline}
We next show the bounds for each of the above identified three terms.

\noindent \textbf{Bound for the Term A.} By Jensen's inequality $\E[X] \le \left( \E[X^p] \right)^{1/p}$ for $p \in (1, 2]$, so we obtain the bound:
\begin{equation} \label{eq:TermA}
\text{Term A} \le (1 - \beta) \sqrt{r} \left( \E \left[ \normF{\sum\limits_{i = 1}^t \beta^{t - i} V_i}^p \right] \right)^{1/p}.
\end{equation}
Next, let $\gF_{i - 1} = \sigma(\xi_1, \xi_2, \dots, \xi_{i - 1})$ denote the $\sigma$-algebra generated by all random mini-batch samplings up to step $i - 1$. Since $\E [ V_i \mid \gF_{i - 1} ] = \mathbf{0}$, the sequence $\{ V_i \}_{i = 1}^t$ forms a martingale difference sequence with respect to $\{ \mathcal{F}_i \}_{i \ge 0}$. Then, applying the von Bahr--Esseen inequality with constant $2^{2 - p}$ for the martingale difference sequences in $p$-th moment ($p \in (1, 2]$), we get that:
\[
\E \left[ \normF{ \sum\limits_{i = 1}^t \beta^{t - i} V_i}^p \right] \le 2^{2 - p} \sum\limits_{i = 1}^t \beta^{p(t - i)} \E \left[ \normF{V_i}^p \right].
\]
Using Assumption~\ref{assum:pth_bounded_moment}, we then have that:
\[
\E \left[ \normF{\sum\limits_{i = 1}^t \beta^{t - i} V_i}^p \right] \le 2^{2 - p} \left( \dfrac{2^{2 - p} \sigma^p}{B^{p - 1}} \right) \sum\limits_{i = 1}^t \beta^{p (t - i)} \le \dfrac{2^{4 - 2 p} \sigma^p}{B^{p - 1}(1 - \beta^p)}.
\]
Substituting back into Eq.~\ref{eq:TermA} yields the final bound for Term A as:
\begin{equation} \label{eq:2_term_A_bound_p}
\text{Term A} \le (1 - \beta) \sqrt{r} \left( \dfrac{2^{4 - 2 p} \sigma^p}{B^{p - 1}(1 - \beta^p)} \right)^{1/p} = \dfrac{2^{(4 - 2 p)/p} \sqrt{r} \sigma}{B^{1 - 1/p}} \cdot \dfrac{1 - \beta}{(1 - \beta^p)^{1/p}}.
\end{equation}

\noindent \textbf{Bound for the Term B.} Recall from Eq.~\ref{eq:2_tracking_error_split} that:
\begin{equation} \label{eq:2_term_B_def}
\text{Term B} = (1 - \beta)\sum\limits_{i = 1}^t \beta^{t - i} \E \left[ \normnuc{D_{i, t}} \right],
\end{equation}
where $D_{i,t} \coloneq \G f(W_{i - 1} + E_i) - \G f(W_{t - 1} + E_t)$. Using $(\nu, L)$-smoothness for $f$, that is, $\normnuc{\G f(X) - \G f(Y)} \le L \normop{X - Y}^\nu$ for all $X, Y \in \R^{m \times n}$, we have that:
\[
\normnuc{D_{i, t}} = \normnuc{\G f(W_{i - 1} + E_i) - \G f(W_{t - 1} + E_t)} \le L \normop{(W_{i - 1} + E_i) - (W_{t - 1} + E_t)}^\nu.
\]
Applying $(a + b)^\nu \le a^\nu + b^\nu$ for $\nu \in (0, 1]$, we next obtain:
\begin{align*}
\normop{(W_{i - 1} + E_i) - (W_{t - 1} + E_t)}^\nu &\le \left( \normop{W_{t - 1} - W_{i - 1}} + \normop{E_t - E_i} \right)^\nu\\
&\le \normop{W_{t - 1} - W_{i - 1}}^\nu + \normop{E_t - E_i}^\nu.
\end{align*}
Using that $\normop{A} \le \normF{A}$ and that $\normF{E_k} \le \rho$, we next get that:
\begin{equation} \label{eq:2_sam_bound}
\normop{E_t - E_i}^\nu \le \left( \normop{E_t} + \normop{E_i} \right)^\nu \le \left( \normF{E_t} + \normF{E_i} \right)^\nu \le (2 \rho)^\nu = 2^\nu \rho^\nu.
\end{equation}
Since the updates are $W_k - W_{k - 1} = -\eta O_k$, this then yields:
\begin{equation} \label{eq:2_iterate_bound}
\normop{W_{t - 1} - W_{i - 1}}^\nu \le \left( \sum\limits_{k = i}^{t - 1} \normop{W_k - W_{k - 1}} \right)^\nu \le \left( \eta \sum\limits_{k = i}^{t - 1} \normop{O_k} \right)^\nu \le \eta^\nu (t - i)^\nu (1 + \varepsilon_q)^\nu.
\end{equation}
Combining Eqs.~\ref{eq:2_sam_bound} and \ref{eq:2_iterate_bound} gives us that
$\E \left[ \normnuc{D_{i, t}} \right] \le L \left( 2^\nu \rho^\nu + \eta^\nu (t - i)^\nu (1 + \varepsilon_q)^\nu \right)$. Substituting this back into Eq.~\ref{eq:2_term_B_def} yields:
\begin{align*}
\text{Term B} &\le (1 - \beta)\sum\limits_{i = 1}^t \beta^{t - i} L \left( 2^\nu \rho^\nu + \eta^\nu \, (t - i)^\nu \, (1 + \varepsilon_q)^\nu \right)\\
&= 2^\nu L \, \rho^\nu \, (1 - \beta) \sum\limits_{i = 1}^t \beta^{t - i} + \eta^\nu L \, (1 + \varepsilon_q)^\nu \, (1 - \beta) \sum\limits_{i = 1}^t \beta^{t - i} \, (t - i)^\nu.
\end{align*}
Using Jensen's inequality for the concave function $\phi(x) = x^\nu$ for $\nu \in (0, 1]$, we next have that:
\[
(1 - \beta) \sum_{i = 1}^t \beta^{t - i}(t - i)^\nu
\le \left( (1 - \beta) \sum_{k = 0}^{\infty} \beta^k k \right)^{\!\nu}
= \left( \dfrac{\beta}{1 - \beta} \right)^{\!\nu}.
\]
Finally, we obtain the bound for Term B to be the following:
\begin{equation} \label{eq:2_term_B_bound}
\text{Term B} \le 2^\nu L \, \rho^\nu + \left( \dfrac{\beta}{1 - \beta} \right)^\nu \eta^\nu L \, (1 + \varepsilon_q)^\nu.
\end{equation}

\noindent \textbf{Bound for Term C.} Recall from Eq.~\ref{eq:2_tracking_error_split} that:
\[ %\label{eq:2_term_C_def}
\text{Term C} = \beta^t \, \E \left[ \normnuc{\G f(W_{t - 1} + E_t)}\right].
\]
Applying triangle inequality, we obtain:
\[
\normnuc{\G f(W_{t - 1} + E_t)} \le \normnuc{\G f(W_0)} + \normnuc{\G f(W_{t - 1} + E_t) - \G f(W_0)}.
\]
Next, using $(\nu, L)$-smoothness for $f$, we get that:
\[
\normnuc{\G f(W_{t - 1} + E_t) - \G f(W_0)} \le L \normop{W_{t - 1} + E_t - W_0}^\nu.
\]
Applying $(a + b)^\nu \le a^\nu + b^\nu$ for $\nu \in (0, 1]$ we next have that:
\begin{equation} \label{eq:2_term_C_displacement}
\normop{(W_{t - 1} + E_t) - W_0}^\nu \le \normop{W_{t - 1} - W_0}^\nu + \normop{E_t}^\nu.
\end{equation}
We have the weight updates to be $W_k - W_{k-1} = -\eta O_k$ for steps $1$ to $t-1$, and therefore, we can write:
\begin{equation} \label{eq:w0_drift}
\normop{W_{t - 1} - W_0}^\nu \le \left( \sum_{k = 1}^{t - 1} \normop{W_k - W_{k - 1}} \right)^\nu \le \left( \eta \sum_{k = 1}^{t - 1} \normop{O_k} \right)^\nu \le \eta^\nu \, (t - 1)^\nu \, (1 + \varepsilon_q)^\nu.
\end{equation}
Combining Eqs.~\ref{eq:2_term_C_displacement} and \ref{eq:w0_drift} gives $\normnuc{\G f(W_{t - 1} + E_t)} \le \normnuc{\G f(W_0)} + L \, \rho^\nu + \eta^\nu L \, (t - 1)^\nu \, (1 + \varepsilon_q)^\nu$, and so we can finally bound Term C as:
\begin{equation} \label{eq:2_term_C_bound}
\text{Term C} \le \beta^t \normnuc{\G f(W_0)} + \beta^t L \, \rho^\nu + \beta^t \, \eta^\nu L \, (t - 1)^\nu \, (1 + \varepsilon_q)^\nu.
\end{equation}
Folding the bounds that we have obtained for Terms A, B and C in Eqs.~\ref{eq:2_term_A_bound_p}, \ref{eq:2_term_B_bound} and \ref{eq:2_term_C_bound}, respectively, into Eq.~\ref{eq:2_tracking_error_split} provides us with:
\begin{multline} \label{eq:2_tracking_error_grouped}
\E \left[ \normnuc{N_t - \G f(W_{t - 1} + E_t)} \right] \le \dfrac{2^{\frac{4 - 2p}{p}} \sqrt{r} \sigma}{B^{1 - \frac{1}{p}}} \cdot \frac{1 - \beta}{(1 - \beta^p)^{1/p}} + (2^\nu + \beta^t) \, L \, \rho^\nu + \beta^t \normnuc{\G f(W_0)} \\
+ \left[ \left( \dfrac{\beta}{1 - \beta} \right)^\nu \!\! + \beta^t \, (t - 1)^\nu \right] \eta^\nu L \, (1 + \varepsilon_q)^\nu.
\end{multline}

\noindent \textbf{Closing Arguments.} Substituting Eq.~\ref{eq:2_tracking_error_grouped} into Eq.~\ref{eq:2_expectation_bound} yields:
\begin{multline*}
\E \left[\normnuc{\G f(W_{t - 1})} \right] \le \dfrac{\E [f(W_{t - 1}) - f(W_t)]}{\eta(1 - \varepsilon_q)} + \frac{2^{(4 - p)/p} \sqrt{r} \, \sigma}{B^{1 - 1/p}(1 - \varepsilon_q)} \cdot \dfrac{1 - \beta}{(1 - \beta^p)^{1/p}} + \dfrac{2 \beta^t \normnuc{\G f(W_0)}}{1 - \varepsilon_q} \\
+ \dfrac{2 (1 + 2^\nu + \beta^t) \, L \, \rho^\nu}{1 - \varepsilon_q} + \dfrac{\eta^\nu L \, (1 + \varepsilon_q)^\nu}{1 - \varepsilon_q} \left[ 2 \left( \dfrac{\beta}{1 - \beta} \right)^\nu + 2 \beta^t \, (t - 1)^\nu + \dfrac{1 + \varepsilon_q}{\nu + 1} \right].
\end{multline*}
Taking the average over $t = 1, 2, \ldots, T$ then gives us that:
\begin{multline} \label{eq:averaged_bound}
\dfrac{1}{T}\sum_{t = 1}^T \E \left[ \normnuc{\G f(W_{t - 1})} \right] \le \dfrac{1}{\eta \, T (1 - \varepsilon_q)}\sum_{t = 1}^T \E \left[ f(W_{t - 1}) - f(W_t) \right] + 2^{(4 - p)/p} \frac{\sqrt{r} \, \sigma}{B^{1 - 1/p}(1 - \varepsilon_q)} \cdot \dfrac{1 - \beta}{(1 - \beta^p)^{1/p}} \\
+ \dfrac{2L \, \rho^\nu}{T (1 - \varepsilon_q)} \sum_{t = 1}^T \left( 1 + 2^\nu + \beta^t \right) + \dfrac{2 \normnuc{\G f(W_0)}}{T(1 - \varepsilon_q)}\sum_{t = 1}^T \beta^t \\
+ \dfrac{2 \eta^\nu \, L \, (1 + \varepsilon_q)^\nu}{1 - \varepsilon_q} \left( \left( \dfrac{\beta}{1 - \beta} \right)^\nu + \dfrac{1}{T} \sum_{t = 1}^T \beta^t \, (t-1)^\nu \right) + \dfrac{\eta^\nu L \,(1 + \varepsilon_q)^{\nu + 1}}{(\nu + 1)(1 - \varepsilon_q)}.
\end{multline}
Using $D = f(W_0) - f^*$ as the initial suboptimality  gap, we have:
\[
\dfrac{1}{\eta T (1 - \varepsilon_q)} \sum_{t = 1}^T \E \left[ f(W_{t - 1}) - f(W_t) \right] \le \dfrac{D}{\eta T (1 - \varepsilon_q)}.
\]
By the norm identity $\normnuc{\G f(W_0)} \le \sqrt{r} \, \normF{\G f(W_0)}$, and the gradient gap inequality for $L$-smooth functions bounded from below by $f^*$ (that is, $\normF{\G f(W)}^2 \le 2 L \, (f(W) - f^*)$) as in \citep[Lemma 7]{kim2026convergence}, we obtain that:
\[
\normnuc{\G f(W_0)} \le \sqrt{r} \, \normF{\G f(W_0)} \le \sqrt{r} \sqrt{2L \, \left( f(W_0) - f^* \right)} =  \sqrt{2r \, L \, D},
\]
and that:
\[
\dfrac{2 \normnuc{\G f(W_0)}}{T(1 - \varepsilon_q)} \sum_{t = 1}^T \beta^t \le \dfrac{2 \beta \sqrt{2 r \, L \, D}}{(1 - \varepsilon_q) \, T \, (1 - \beta)}.
\]
Using that:
\[
\dfrac{2 L \, \rho^\nu}{T (1 - \varepsilon_q)} \sum_{t = 1}^T \left(1 + 2^\nu + \beta^t \right) \le \dfrac{2 (1 + 2^\nu) \, L \, \rho^\nu}{1 - \varepsilon_q} + \dfrac{2 L \, \rho^\nu \, \beta}{T (1 - \varepsilon_q) \, (1 - \beta)},
\]
for $\nu \in (0, 1]$, we then obtain:
\[
\dfrac{2 \eta^\nu \, L \, (1 + \varepsilon_q)^\nu}{1 - \varepsilon_q} \left[ \left( \dfrac{\beta}{1 - \beta} \right)^\nu + \dfrac{1}{T} \sum_{t = 1}^T \beta^t \, (t - 1)^\nu \right] \le \dfrac{2 \eta^\nu \, L \, (1 + \varepsilon_q)^\nu}{1 - \varepsilon_q} \left[ \left( \dfrac{\beta}{1 - \beta}\right)^\nu + \dfrac{\beta}{1 - \beta} \right].
\]
Substituting this into Eq.~\ref{eq:averaged_bound} yields:
\begin{multline*}
\dfrac{1}{T} \sum_{t = 1}^T
\E \left[ \normnuc{\G f(W_{t - 1})} \right] \le \dfrac{D}{\eta \, T (1 - \varepsilon_q)} + 2^{(4 - p)/p} \dfrac{\sqrt{r} \, \sigma}{B^{1 - 1/p} (1 - \varepsilon_q)} \cdot \dfrac{1 - \beta}{(1  \beta^p)^{1/p}} \\
+ \dfrac{2 \eta^\nu L \, (1 + \varepsilon_q)^\nu}{1 - \varepsilon_q} \left[ \left( \dfrac{\beta}{1 - \beta} \right)^\nu + \dfrac{\beta}{1 - \beta} \right] + \dfrac{\eta^\nu L \, (1 + \varepsilon_q)^{\nu + 1}}{(\nu + 1) \, (1 - \varepsilon_q)} \\
+ \dfrac{2 \beta \sqrt{2 r \, L \, D}}{(1 - \varepsilon_q) \, T \, (1 - \beta)} + \dfrac{2 (1 + 2^\nu)\,  L \, \rho^\nu}{1 - \varepsilon_q} + \dfrac{2 L \, \rho^\nu \, \beta}{T \, (1 - \varepsilon_q) \, (1 - \beta)}.
\end{multline*}
Setting $\displaystyle\eta = \left( \dfrac{(1 - \beta) D}{T L} \right)^{1/(\nu + 1)}$, we get that:
\begin{multline*}
\dfrac{1}{T} \sum_{t = 1}^T \E \normnuc{\G f (W_{t - 1})}
\le \dfrac{1}{1 - \varepsilon_q} \left[ \left( \dfrac{L}{1 - \beta} \right)^{1/(\nu + 1)} \left( \dfrac{D}{T} \right)^{\nu/(\nu + 1)} \left( 1 + 2 (1 + \varepsilon_q)^\nu \left[ \dfrac{\beta^\nu}{(1 - \beta)^{\nu - 1}} + \beta \right] \right) \right. \\
+ \dfrac{(1 + \varepsilon_q)^{\nu + 1}}{\nu + 1} L^{1/(\nu + 1)} \left( \dfrac{D}{T} \right)^{\nu/(\nu + 1)} (1 - \beta)^{\nu/(\nu + 1)} \\
\left. + \dfrac{2^{(4 - p)/p} \sqrt{r} \, \sigma \, (1 - \beta)}{B^{1 - 1/p}(1 - \beta^p)^{1/p}} + \dfrac{2 \beta \sqrt{2 r \, L \, D}}{T (1 - \beta)} + L \, \rho^\nu \left( 2 (1 + 2^\nu) + \dfrac{2 \beta}{T (1 - \beta)} \right) \right].
\end{multline*}
Define $\chi_q \coloneq \dfrac{1}{1 - \varepsilon_q}$, $\delta \coloneq (p - 1)(\nu + 1) + p$ and also setting:
\[
\beta \coloneq 1 - \min \left\{1, \left( \left( \dfrac{L \, D^\nu}{T^\nu} \right)^{1/(\nu + 1)} \dfrac{B^{1 - 1/p}}{\sigma \sqrt{r}} \right)^{p (\nu + 1)/\delta}  \right\},
\]
we get that:
\begin{multline*}
\dfrac{1}{T} \sum_{t = 1}^T \E \left[ \normnuc{\G f(W_{t - 1})} \right] \le \chi_q \left[ \left(1 + 4 (1 + \varepsilon_q)^\nu \right) \left( \dfrac{L \, D^\nu}{T^\nu} \right)^{(p - 1)/\delta} \left( \dfrac{\sigma \sqrt{r}}{B^{1 - 1/p}} \right)^{p/\delta} \right.\\
+ \dfrac{(1 + \varepsilon_q)^{\nu + 1} L^{1/(\nu + 1)}}{\nu + 1} \left( \dfrac{D}{T} \right)^{\nu/(\nu + 1)} 
+ \dfrac{2 \beta \sqrt{2r} \, L^{(\delta - 2 p)/(2 \delta)} D^{(\delta - 2 p \, \nu)/(2 \delta)}}{T^{(2 p - 1 - \nu)/\delta}} \left( \dfrac{\sigma \sqrt{r}}{B^{1 - 1/p}} \right)^{p (\nu + 1)/\delta}\\
+ \left. 2^{(4 - p)/p} \left( \dfrac{L \, D^\nu}{T^\nu} \right)^{(p - 1)/\delta} \left( \dfrac{\sigma \sqrt{r}}{B^{1 - 1/p}} \right)^{p/\delta} \right]\\
+ \chi_q L \, \rho^\nu \left[ 2 (1 + 2^\nu) + \dfrac{2}{L^{p/\delta} D^{p \nu/\delta} T^{(2 p - 1 - \nu)/\delta}} \left( \dfrac{\sigma \sqrt{r}}{B^{1 - 1/p}} \right)^{p (\nu + 1)/\delta} \right]. 
\end{multline*}
Thus, we can find an $\epsilon$-nuclear norm stationary point of $f$ with a complexity of:
\begin{multline*}
\dfrac{1}{T} \sum_{t = 1}^T \E \left[ \normnuc{\G f(W_{t - 1})} \right] \le \chi_q \cdot \gO \left( \left( \dfrac{L \, D^\nu}{T^\nu} \right)^{(p - 1)/\delta} \left( \dfrac{\sigma \sqrt{r}}{B^{1 - 1/p}}  \right)^{p/\delta} + \left( \frac{L \, D^\nu}{T^\nu} \right)^{1/(\nu + 1)} \right.\\
+ \left. \dfrac{L^{(\delta - 2 p)/(2 \delta)}\sqrt{r} D^{(\delta - 2 p \, \nu)/ (2 \delta)}}{T^{(2p - 1 - \nu)/\delta}} \left( \dfrac{\sigma \sqrt{r}}{B^{1 - 1/p}} \right)^{p (\nu + 1)/\delta} \right)\\
+ \chi_q L \rho^\nu \cdot \gO \left( 1 + \dfrac{1}{L^{p/\delta} D^{p \nu/\delta} T^{(2 p - 1 - \nu)/\delta}} \left( \dfrac{\sigma \sqrt{r}}{B^{1 - 1/p}} \right)^{p (\nu + 1)/\delta} \right). \tag*{\qedhere}
\end{multline*}
\end{proof}

%% file: sections/Non-heavy_tailed.tex
\subsection{Proof of Corollary~\ref{coro:convg_of_bounded_variance}}

\begin{proof} \label{Proof:Proof_coro}
Introduce the scaled \ema momentum: $N_t \coloneq (1 - \beta) \, M_t$. Then, we obtain $N_t = \beta \, N_{t - 1} + (1 - \beta) \, \overline{G}_t$. Let $P_t$ denote the polar factor of $M_t$, that is, $P_t = \ortho{M_t}$. Since the polar factor is scale-invariant, we have $P_t \coloneq \ortho{N_t} = \ortho{M_t}$. Setting $p = 2$ in Assumption~\ref{assum:pth_bounded_moment}, we have for a mini-batch of size $B$ that: $\E_\xi \left[ \normF{\G f(W; \xi) - \E_\xi \left[ \G f(W; \xi) \right]}^2 \right] \le \sigma^2/B$. Using Assumption~\ref{assump:nu_L_smoothness} with $\nu = 1$ and applying the descent lemma (Lemma~\ref{lem:descent_holder_smooth}) to $W_t = W_{t - 1} - \eta \, O_t$ yields:
\begin{align*}
f(W_t) &\le f(W_{t - 1}) + \innerF{\G f(W_{t - 1})}{W_t - W_{t - 1}} + \dfrac{L}{2} \normop{W_t - W_{t - 1}}^2\\
&= f(W_{t - 1}) - \eta \innerF{\G f(W_{t - 1})}{O_t} + \dfrac{L \, \eta^2}{2} \normop{O_t}^2\\
&= f(W_{t - 1}) - \eta \innerF{N_t}{O_t} + \eta \innerF{N_t - \G f(W_{t - 1})}{O_t} + \dfrac{L \, \eta^2}{2} \normop{O_t}^2\\
&\stackrel{(a)}{\le} f(W_{t - 1}) - \eta \innerF{N_t}{O_t} + \eta \normnuc{N_t - \G f(W_{t - 1})} \normop{O_t} + \dfrac{L \, \eta^2}{2} \normop{O_t}^2\\
&= f(W_{t - 1}) - \eta \innerF{N_t}{P_t} + \eta \innerF{N_t}{P_t - O_t} + \eta \normnuc{N_t - \G f(W_{t - 1})} \normop{O_t} + \dfrac{L \, \eta^2}{2} \normop{O_t}^2\\
&\stackrel{(b)}{=} f(W_{t - 1}) - \eta \normnuc{N_t} + \eta \innerF{N_t}{P_t - O_t} + \eta \normnuc{N_t - \G f(W_{t - 1})} \normop{O_t} + \dfrac{L \, \eta^2}{2} \normop{O_t}^2\\
&\stackrel{(c)}{\le} f(W_{t - 1}) - \eta \normnuc{N_t} + \eta \normnuc{N_t} \normop{P_t - O_t} + \eta \normnuc{N_t - \G f(W_{t - 1})} \normop{O_t} + \dfrac{L \, \eta^2}{2} \normop{O_t}^2,
\end{align*}
where steps $(a)$ and $(c)$ follow from H\"older's inequality, and equality $(b)$ follows from the dual norm identity: $\innerF{N_t}{P_t} = \normnuc{N_t}$. We now define the polar approximation error, that is, the error between the exact polar factor $P_t = \ortho{N_t} = \ortho{M_t}$ and the actual step $O_t$ generated by $q$ steps of $\ns$ iterations:
\[
\varepsilon_{t, q} \coloneq \normop{O_t - P_t} \quad \text{and} \quad \varepsilon_q \coloneq \sup_t \varepsilon_{t, q}.
\]
Since $\normop{P_t} \le 1$, we have $\normop{O_t} \le 1 + \varepsilon_{t, q} \le 1 + \varepsilon_q$, and combining this with $\normop{P_t - O_t} \le \varepsilon_q$, we have that:
\[
f(W_t) \le f(W_{t - 1}) - \eta \, (1 - \varepsilon_q) \normnuc{N_t} + \eta \, (1 + \varepsilon_q) \normnuc{N_t - \G f(W_{t - 1})} + \dfrac{L \, \eta^2}{2} (1 + \varepsilon_q)^2.
\]
By the reverse triangle inequality $-\normnuc{A} \le -\normnuc{B} + \normnuc{A - B}$ applied to $\normnuc{N_t}$, that is, $\normnuc{N_t} \ge \normnuc{\G f(W_{t - 1})} - \normnuc{N_t - \G f(W_{t - 1})}$, we then obtain:
\begin{equation} \label{eq:1_descent_step3}
f(W_t) \le f(W_{t - 1}) - \eta \, (1 - \varepsilon_q) \normnuc{\G f(W_{t - 1})} + 2 \eta \normnuc{N_t - \G f(W_{t - 1})} + \dfrac{L \, \eta^2}{2} (1 + \varepsilon_q)^2.
\end{equation}
Next, we apply $(\nu, L)$-smoothness to $f$ by setting $\nu = 1$ and using that $\normop{E_t} \le \rho$. This then yields $\normnuc{\G f(W_{t - 1} + E_t) - \G f(W_{t - 1})} \le L \, \rho$, and so by the triangle inequality, we get that:
\[
\normnuc{N_t - \G f(W_{t - 1})} \le \normnuc{N_t - \G f(W_{t - 1} + E_t)} + L \, \rho.
\]
Substituting this into Eq.~\ref{eq:1_descent_step3} yields:
\[
f(W_t) \le f(W_{t - 1}) - \eta \, (1 - \varepsilon_q) \normnuc{\G f(W_{t - 1})} + 2 \eta \normnuc{N_t - \G f(W_{t - 1} + E_t)} + 2 \eta \, L \, \rho + \dfrac{L \, \eta^2}{2} (1 + \varepsilon_q)^2.
\]
Rearranging the terms and taking an expectation on both sides then gives us:
\begin{multline}
\label{eq:1_expectation_descent}
\E \left[ \normnuc{\G f(W_{t - 1})} \right] \le \dfrac{\E [f(W_{t - 1}) - f(W_t)]}{\eta \, (1 - \varepsilon_q)} + \dfrac{2}{1 - \varepsilon_q} \E \left[ \normnuc{N_t - \G f(W_{t - 1} + E_t)} \right]\\
+ \dfrac{2 L \, \rho}{1 - \varepsilon_q} + \dfrac{L \, \eta \, (1 + \varepsilon_q)^2}{2 (1 - \varepsilon_q)}.
\end{multline}
In Eq.~\ref{eq:1_expectation_descent}, we shall first obtain a bound for the second term on the right hand side, and then combine it with the gradient-gap inequality as in \citep[Lemma 7]{kim2026convergence} in our closing arguments, which then yields the required $\epsilon$-stationarity result.

\noindent \textbf{Bound for the term $\E \left[ \normnuc{N_t - \G f(W_{t - 1} + E_t)} \right]$.} We begin by first recalling our \ema momentum that we introduced at the beginning of this proof:
\[
N_t = \beta \, N_{t - 1} + (1 - \beta) \, \overline{G}_t, \quad \text{for~} t \ge 0 \text{~with~} N_0 = 0,
\]
where $\overline{G}_i = \frac{1}{B} \sum_{j = 1}^B \G f(W_{i - 1} + E_i; \xi_{i, j})$ denotes the mini-batch stochastic gradient evaluated at the perturbed iterate $W_{i - 1} + E_i$. Unrolling the recursion yields $N_t = (1 - \beta) \sum_{i = 1}^t \beta^{t - i} \, \overline{G}_i$. So, we have that:
\begin{equation}
\label{eq:1_tracking_error_expanded}
N_t - \G f(W_{t - 1} + E_t) = (1 - \beta) \sum\limits_{i = 1}^t \beta^{t - i} \Big( \overline{G}_i - \G f(W_{t - 1} + E_t) \Big) - \beta^t \, \G f(W_{t - 1} + E_t).
\end{equation}
Adding and subtracting $\G f(W_{i - 1} + E_i)$ to this expression then provides:
\[
\overline{G}_i - \G f(W_{t - 1} + E_t) = \underbrace{\Big( \overline{G}_i - \G f(W_{i - 1} + E_i) \Big)}_{V_i} + \underbrace{\Big( \G f(W_{i - 1} + E_i) - \G f(W_{t - 1} + E_t) \Big)}_{D_{i, t}},
\]
where $V_i \coloneq \overline{G}_i - \G f(W_{i - 1} + E_i)$ and $D_{i, t} \coloneq \G f(W_{i - 1} + E_i) - \G f(W_{t - 1} + E_t)$. Substituting this decomposition into Eq.~\ref{eq:1_tracking_error_expanded} yields:
\[
N_t - \G f(W_{t - 1} + E_t) = (1 - \beta) \sum\limits_{i = 1}^t \beta^{t - i} \, V_i + (1 - \beta) \sum\limits_{i = 1}^t \beta^{t - i} D_{i, t} - \beta^t \, \G f(W_{t - 1} + E_t).
\]
Taking the nuclear norm $\normnuc{\cdot}$, and using the matrix rank inequality $\normnuc{A} \le \sqrt{r} \normF{A}$, we obtain:
\begin{multline} \label{eq:1_tracking_error_split}
\E \left[ \normnuc{N_t - \G f(W_{t - 1} + E_t)} \right] \le \underbrace{(1 - \beta) \sqrt{r} \, \E \left[ \normF{ \sum\limits_{i = 1}^t \beta^{t - i} V_i } \right]}_{\text{Term A}} + \underbrace{(1 - \beta) \sum\limits_{i = 1}^t \beta^{t - i} \E \left[ \normnuc{D_{i, t}} \right]}_{\text{Term B}} \\
+ \underbrace{\vphantom{\sum\limits_{i = 1}^t} \beta^t \E \left[ \normnuc{\G f(W_{t - 1} + E_t)} \right]}_{\text{Term C}}.
\end{multline}
We next show the bounds for each of the above identified three terms.

\noindent \textbf{Bound for Term A.} By Jensen's inequality $\E[X] \le \sqrt{\E[X^2]}$, so we obtain the bound:
\begin{equation}
\label{eq:term_A_jensen}
\text{Term A} \le (1 - \beta) \sqrt{r} \sqrt{\E \left[ \normF{\sum\limits_{i = 1}^t \beta^{t - i} V_i}^2 \right]}.
\end{equation}
To evaluate the squared term $\E \left[ \normF{\sum_{i = 1}^t \beta^{t - i} V_i}^2 \right]$, let $\mathcal{F}_{i - 1} = \sigma(\xi_1, \xi_2, \dots, \xi_{i - 1})$ denote the $\sigma$-algebra generated by all random mini-batch samplings up to step $i - 1$. By Assumption~\ref{assum:pth_bounded_moment}, $\E [ V_i \mid \mathcal{F}_{i - 1} ] = \mathbf{0}$, and so the sequence $\{ V_i \}_{i = 1}^t$ forms a martingale difference sequence with respect to $\{ \mathcal{F}_i \}_{i \ge 0}$. Therefore, we have that:
\begin{align*}
\E \left[ \normF{\sum\limits_{i = 1}^t \beta^{t - i} V_i}^2 \right] &= \E \left[ \innerF{\sum\limits_{i = 1}^t \beta^{t - i} V_i}{\sum\limits_{j = 1}^t \beta^{t - j} V_j} \right]\\
&= \sum\limits_{i = 1}^t \beta^{2 (t - i)} \E \left[ \normF{V_i}^2 \right] + 2 \sum\limits_{1 \le i < j \le t} \beta^{2 t - i - j} \, \E \left[ \innerF{V_i}{V_j} \right].
\end{align*}
By the law of total expectation, for any cross-term with $i < j$, we have:
\[
\E \left[ \innerF{V_i}{V_j} \right] = \E \left[ \E \left[ \innerF{V_i}{V_j} \mid \mathcal{F}_{j - 1} \right] \right].
\]
Since $V_i$ is measurable with respect to $\mathcal{F}_{j - 1}$ for $i < j$, we next get that:
\[
\E \left[ \innerF{V_i}{V_j} \mid \mathcal{F}_{j - 1} \right] = \innerF{V_i}{\E [ V_j \mid \mathcal{F}_{j - 1} ]} = \innerF{V_i}{\mathbf{0}} = 0.
\]
Thus, $\E \left[ \innerF{V_i}{V_j} \right] = 0$ for $i \ne j$, and so we have that:
\[
\E \left[ \normF{\sum\limits_{i = 1}^t \beta^{t - i} V_i}^2 \right] = \sum\limits_{i = 1}^t \beta^{2 (t - i)} \E \left[ \normF{V_i}^2 \right].
\]
Using Assumption~\ref{assum:pth_bounded_moment} for $p = 2$ and a mini-batch of size $B$, we then have that:
\[
\E \left[ \normF{\sum\limits_{i = 1}^t \beta^{t - i} V_i}^2 \right] \le \dfrac{\sigma^2}{B} \sum\limits_{i = 1}^t \beta^{2 (t - i)} \le \dfrac{\sigma^2}{B (1 - \beta^2)}.
\]
Substituting back into Eq.~\ref{eq:term_A_jensen} yields the final bound for Term A as:
\begin{equation} \label{eq:1_term_A_bound}
\text{Term A} \le (1 - \beta) \sqrt{r} \sqrt{\dfrac{\sigma^2}{B (1 - \beta^2)}} = \dfrac{\sqrt{r} \, \sigma}{\sqrt{B}} \sqrt{\dfrac{1 - \beta}{1 + \beta}}.
\end{equation}

\noindent \textbf{Bound for Term B.} Recall from Eq.~\ref{eq:1_tracking_error_split} that:
\begin{equation}
\label{eq:1_term_B_def}
\text{Term B} = (1 - \beta) \sum\limits_{i = 1}^t \beta^{t - i} \E \left[ \normnuc{D_{i, t}} \right],
\end{equation}
where $D_{i, t} \coloneq \G f(W_{i - 1} + E_i) - \G f(W_{t - 1} + E_t)$. Using $(\nu, L)$-smoothness for $f$ with $\nu = 1$, that is, $\normnuc{\G f(X) - \G f(Y)} \le L \normop{X - Y}$ for all $X, Y \in \R^{m \times n}$, we have that:
\[
\normnuc{D_{i, t}} = \normnuc{\G f(W_{i - 1} + E_i) - \G f(W_{t - 1} + E_t)} \le L \normop{(W_{i - 1} + E_i) - (W_{t - 1} + E_t)}.
\]
Applying the triangle inequality, we next obtain:
\[
\normop{(W_{i - 1} + E_i) - (W_{t - 1} + E_t)} \le \normop{W_{t - 1} - W_{i - 1}} + \normop{E_t - E_i}.
\]
Using that $\normop{A} \le \normF{A}$ and that $\normF{E} \le \rho$, we next get:
\[
\normop{E_t - E_i} \le \normop{E_t} + \normop{E_i} \le \normF{E_t} + \normF{E_i} \le 2 \rho.
\]
Since the updates are $W_k - W_{k - 1} = -\eta \, O_k$ with $\normop{O_k} \le 1 + \varepsilon_q$, this then yields:
\[
\normop{W_{t - 1} - W_{i - 1}} \le \sum\limits_{k = i}^{t - 1} \normop{W_k - W_{k - 1}} \le \eta \sum\limits_{k = i}^{t - 1} \normop{O_k} \le \eta \, (t - i) \, (1 + \varepsilon_q).
\]
Combining the above inequalities gives us that:
\[
\E \left[ \normnuc{D_{i, t}} \right] \le L \left( \eta \, (t - i) \, (1 + \varepsilon_q) + 2 \rho \right).
\]
Substituting this back into Eq.~\ref{eq:1_term_B_def} yields:
\begin{align*}
\text{Term B} &\le (1 - \beta) \sum\limits_{i = 1}^t \beta^{t - i} \cdot L \left( \eta \, (t - i) \, (1 + \varepsilon_q) + 2 \rho \right)\\
&= 2 L \, \rho \, (1 - \beta) \sum\limits_{i = 1}^t \beta^{t - i} + \eta \, L \, (1 + \varepsilon_q) \, (1 - \beta) \sum\limits_{i = 1}^t \beta^{t - i} (t - i).
\end{align*}
Bounding the two sums above by their corresponding infinite series, we finally obtain the bound for Term B to be the following:
\begin{equation} \label{eq:1_term_B_bound}
\text{Term B} \le 2 L \, \rho + \dfrac{\beta}{1 - \beta} \, \eta \, L \, (1 + \varepsilon_q).
\end{equation}

\noindent \textbf{Bound for Term C.} Recall from Eq.~\ref{eq:1_tracking_error_split} that:
\[ % \label{eq:1_term_C_def}
\text{Term C} = \beta^t \E \left[ \normnuc{\G f(W_{t - 1} + E_t)} \right].
\]
Applying the triangle inequality, we obtain:
\[
\normnuc{\G f(W_{t - 1} + E_t)} \le \normnuc{\G f(W_0)} + \normnuc{\G f(W_{t - 1} + E_t) - \G f(W_0)}.
\]
Next, using $(\nu, L)$-smoothness for $f$ with $\nu = 1$, we get that:
\[
\normnuc{\G f(W_{t - 1} + E_t) - \G f(W_0)} \le L \normop{(W_{t - 1} + E_t) - W_0}.
\]
Applying the triangle inequality again, we have:
\[
\normop{(W_{t - 1} + E_t) - W_0} \le \normop{W_{t - 1} - W_0} + \normop{E_t}.
\]
Unrolling the weight updates $W_k - W_{k - 1} = -\eta \, O_k$ for steps $1$ to $t - 1$, we can write:
\[
\normop{W_{t - 1} - W_0} \le \sum\limits_{k = 1}^{t - 1} \normop{W_k - W_{k - 1}} \le \eta \sum\limits_{k = 1}^{t - 1} \normop{O_k} \le \eta \, (t - 1) \, (1 + \varepsilon_q).
\]
Combining the above inequalities, we get next:
\[
\normnuc{\G f(W_{t - 1} + E_t)} \le \normnuc{\G f(W_0)} + L \, \rho + \eta \, L \, (t - 1) \, (1 + \varepsilon_q).
\]
Taking an expectation then yields the bound for Term C to be:
\begin{equation} \label{eq:1_term_C_bound}
\text{Term C} \le \beta^t \normnuc{\G f(W_0)} + \beta^t L \, \rho + \beta^t \, \eta \, L \, (t - 1) (1 + \varepsilon_q).
\end{equation}
Folding the bounds that we have obtained for Terms A, B and C in Eqns.~\ref{eq:1_term_A_bound}, \ref{eq:1_term_B_bound} and \ref{eq:1_term_C_bound}, respectively, into Eq.~\ref{eq:1_tracking_error_split} provides us with:
% \begin{equation} \label{eq:1_tracking_error_bound}
% \begin{split}
% \E \left[ \normnuc{N_t - \G f(W_{t - 1} + E_t)} \right] &\le \dfrac{\sqrt{r} \, \sigma}{\sqrt{B}} \sqrt{\dfrac{1 - \beta}{1 + \beta}} + 2 L \, \rho + \dfrac{\beta}{1 - \beta} \, \eta \, L \, (1 + \varepsilon_q)\\
% &\quad + \beta^t \normnuc{\G f(W_0)} + \beta^t L \, \rho + \beta^t \, \eta \, L \, (t - 1) \, (1 + \varepsilon_q),\\
% &= \dfrac{\sqrt{r} \, \sigma}{\sqrt{B}} \sqrt{\dfrac{1 - \beta}{1 + \beta}} + L \, \rho \, (2 + \beta^t) + \beta^t \normnuc{\G f(W_0)}\\
% &\quad + \eta \, L \, (1 + \varepsilon_q) \left( \dfrac{\beta}{1 - \beta} + \beta^t (t - 1) \right).
% \end{split}
% \end{equation}
\begin{align}
\E \left[ \normnuc{N_t - \G f(W_{t - 1} + E_t)} \right]
&\le \begin{multlined}[t]
\dfrac{\sqrt{r} \, \sigma}{\sqrt{B}} \sqrt{\dfrac{1 - \beta}{1 + \beta}}
+ 2 L \, \rho + \dfrac{\beta}{1 - \beta} \, \eta \, L \, (1 + \varepsilon_q) \\
+ \beta^t \normnuc{\G f(W_0)} + \beta^t L \, \rho
+ \beta^t \, \eta \, L \, (t - 1) \, (1 + \varepsilon_q)
\end{multlined} \nonumber \\
&= \begin{multlined}[t]
\dfrac{\sqrt{r} \, \sigma}{\sqrt{B}} \sqrt{\dfrac{1 - \beta}{1 + \beta}}
+ L \, \rho \, (2 + \beta^t) + \beta^t \normnuc{\G f(W_0)} \\
+ \eta \, L \, (1 + \varepsilon_q)
  \left( \dfrac{\beta}{1 - \beta} + \beta^t (t - 1) \right).
\end{multlined} \label{eq:1_tracking_error_bound}
\end{align}

\noindent \textbf{Closing Arguments.} Substituting Eq.~\ref{eq:1_tracking_error_bound} into Eq.~\ref{eq:1_expectation_descent} yields:
% \begin{align*}
% \E \left[ \normnuc{\G f(W_{t - 1})} \right] &\le \dfrac{\E [f(W_{t - 1}) - f(W_t)]}{\eta \, (1 - \varepsilon_q)} + \dfrac{2}{1 - \varepsilon_q} \Bigg[ \dfrac{\sqrt{r} \, \sigma}{\sqrt{B}} \sqrt{\dfrac{1 - \beta}{1 + \beta}} + L \, \rho \, (2 + \beta^t) + \beta^t \normnuc{\G f(W_0)}\\
% &\quad + \eta \, L \, (1 + \varepsilon_q) \left( \dfrac{\beta}{1 - \beta} + \beta^t (t - 1) \right) \Bigg] + \dfrac{2 L \, \rho}{1 - \varepsilon_q} + \dfrac{L \, \eta \, (1 + \varepsilon_q)^2}{2 (1 - \varepsilon_q)}\\
% &= \dfrac{\E [f(W_{t - 1}) - f(W_t)]}{\eta \, (1 - \varepsilon_q)} + \dfrac{2 \sqrt{r} \, \sigma}{(1 - \varepsilon_q) \sqrt{B}} \sqrt{\dfrac{1 - \beta}{1 + \beta}} + \dfrac{2 L \, \rho \, (3 + \beta^t)}{1 - \varepsilon_q} + \dfrac{2 \beta^t \normnuc{\G f(W_0)}}{1 - \varepsilon_q}\\
% &\quad + \dfrac{2 \eta \, L \, (1 + \varepsilon_q)}{1 - \varepsilon_q} \left( \dfrac{\beta}{1 - \beta} + \beta^t (t - 1) \right) + \dfrac{L \, \eta \, (1 + \varepsilon_q)^2}{2 (1 - \varepsilon_q)}.
% \end{align*}
\begin{align*}
\E \left[ \normnuc{\G f(W_{t - 1})} \right]
&\le \begin{multlined}[t]
\dfrac{\E [f(W_{t - 1}) - f(W_t)]}{\eta \, (1 - \varepsilon_q)}
+ \dfrac{2}{1 - \varepsilon_q} \Bigg[ \dfrac{\sqrt{r} \, \sigma}{\sqrt{B}} \sqrt{\dfrac{1 - \beta}{1 + \beta}}
+ L \, \rho \, (2 + \beta^t) + \beta^t \normnuc{\G f(W_0)} \\
+ \eta \, L \, (1 + \varepsilon_q) \left( \dfrac{\beta}{1 - \beta} + \beta^t (t - 1) \right) \Bigg]
+ \dfrac{2 L \, \rho}{1 - \varepsilon_q}
+ \dfrac{L \, \eta \, (1 + \varepsilon_q)^2}{2 (1 - \varepsilon_q)}
\end{multlined} \\
&= \begin{multlined}[t]
\dfrac{\E [f(W_{t - 1}) - f(W_t)]}{\eta \, (1 - \varepsilon_q)}
+ \dfrac{2 \sqrt{r} \, \sigma}{(1 - \varepsilon_q) \sqrt{B}} \sqrt{\dfrac{1 - \beta}{1 + \beta}}
+ \dfrac{2 L \, \rho \, (3 + \beta^t)}{1 - \varepsilon_q}
+ \dfrac{2 \beta^t \normnuc{\G f(W_0)}}{1 - \varepsilon_q} \\
+ \dfrac{2 \eta \, L \, (1 + \varepsilon_q)}{1 - \varepsilon_q} \left( \dfrac{\beta}{1 - \beta} + \beta^t (t - 1) \right)
+ \dfrac{L \, \eta \, (1 + \varepsilon_q)^2}{2 (1 - \varepsilon_q)}.
\end{multlined}
\end{align*}
Taking the average over $t = 1, 2, \ldots, T$ then gives us that:
\begin{multline}
\label{eq:1_average_step_unsimplified}
\dfrac{1}{T} \sum\limits_{t = 1}^T \E \left[ \normnuc{\G f(W_{t - 1})} \right] \le \dfrac{1}{\eta \, T (1 - \varepsilon_q)} \sum\limits_{t = 1}^T \E \left[ f(W_{t - 1}) - f(W_t) \right] + \dfrac{2 \sqrt{r} \, \sigma}{(1 - \varepsilon_q) \sqrt{B}} \sqrt{\dfrac{1 - \beta}{1 + \beta}}\\
+ \dfrac{2 L \, \rho}{T (1 - \varepsilon_q)} \sum\limits_{t = 1}^T (3 + \beta^t) + \dfrac{2 \normnuc{\G f(W_0)}}{T (1 - \varepsilon_q)} \sum\limits_{t = 1}^T \beta^t\\
+ \dfrac{2 \eta \, L \, (1 + \varepsilon_q)}{1 - \varepsilon_q} \left( \dfrac{\beta}{1 - \beta} + \dfrac{1}{T} \sum\limits_{t = 1}^T \beta^t (t - 1) \right) + \dfrac{L \, \eta \, (1 + \varepsilon_q)^2}{2 (1 - \varepsilon_q)}.
\end{multline}
Using $D = f(W_0) - f^*$ as the initial suboptimality gap, we have:
\[
\dfrac{1}{\eta T (1 - \varepsilon_q)} \sum\limits_{t = 1}^T \E \left[ f(W_{t - 1}) - f(W_t) \right] \le \dfrac{D}{\eta T (1 - \varepsilon_q)}.
\]
By the norm identity $\normnuc{\G f(W_0)} \le \sqrt{r} \, \normF{\G f(W_0)}$, and the gradient gap inequality for $L$-smooth functions bounded from below by $f^*$ (that is, $\normF{\G f(W)}^2 \le 2 L \, (f(W) - f^*)$) as in \citep[Lemma 7]{kim2026convergence}, we obtain that:
\[
\normnuc{\G f(W_0)} \le \sqrt{r} \, \normF{\G f(W_0)} \le \sqrt{r} \sqrt{2 L \, \left( f(W_0) - f^* \right)} = \sqrt{2 r \, L \, D},
\]
and that:
\[
\dfrac{2 \normnuc{\G f(W_0)}}{T (1 - \varepsilon_q)} \sum\limits_{t = 1}^T \beta^t \le \dfrac{2 \beta \sqrt{2 r \, L \, D}}{(1 - \varepsilon_q) \, T \, (1 - \beta)}.
\]
Similarly, we also have that:
\[
\dfrac{2 \eta \, L \, (1 + \varepsilon_q)}{1 - \varepsilon_q} \left( \dfrac{\beta}{1 - \beta} + \dfrac{1}{T} \sum\limits_{t = 1}^T \beta^t (t - 1) \right) \le \dfrac{4 \beta L \, \eta \, (1 + \varepsilon_q)}{(1 - \varepsilon_q) \, (1 - \beta)}.
\]
Substituting these back into Eq.~\ref{eq:1_average_step_unsimplified} yields:
\begin{multline*}
\dfrac{1}{T} \sum\limits_{t = 1}^T \E \left[ \normnuc{\G f(W_{t - 1})} \right] \le \dfrac{D}{\eta T (1 - \varepsilon_q)} + \dfrac{2}{1 - \varepsilon_q} \left( \dfrac{2 \beta L \, \eta \, (1 + \varepsilon_q)}{1 - \beta} + \sqrt{\dfrac{1 - \beta}{1 + \beta}} \, \dfrac{\sqrt{r} \, \sigma}{\sqrt{B}} \right) + \dfrac{L \, \eta \, (1 + \varepsilon_q)^2}{2 (1 - \varepsilon_q)}\\
+ \dfrac{2 \beta \sqrt{2 r \, L \, D}}{(1 - \varepsilon_q) T (1 - \beta)} + \dfrac{6 L \, \rho}{1 - \varepsilon_q} + \dfrac{2 L \, \rho \, \beta}{T (1 - \varepsilon_q) \, (1 - \beta)}.
\end{multline*}
Setting $\displaystyle \eta = \sqrt{\dfrac{(1 - \beta) D}{T L}}$, we get that:
\begin{multline*}
\dfrac{1}{T} \sum\limits_{t = 1}^T \E \left[ \normnuc{\G f(W_{t - 1})} \right] \le \dfrac{1}{1 - \varepsilon_q} \left[ \sqrt{\dfrac{L \, D}{T}} \left( \dfrac{1 + 4 \beta (1 + \varepsilon_q)}{\sqrt{1 - \beta}} + \dfrac{(1 + \varepsilon_q)^2 \sqrt{1 - \beta}}{2} \right) \right.\\
\left. + \dfrac{2 \beta \sqrt{2 r \, L \, D}}{T (1 - \beta)} + \dfrac{2 \sigma \sqrt{r}}{\sqrt{B}} \sqrt{\dfrac{1 - \beta}{1 + \beta}} + L \, \rho \left( 6 + \dfrac{2 \beta}{T (1 - \beta)} \right) \right].
\end{multline*}
Defining $\chi_q \coloneq \dfrac{1}{1 - \varepsilon_q}$ and $\beta \coloneq 1 - \min \left\{1, \dfrac{\sqrt{L \, D \, B}}{\sigma \sqrt{r \, T}} \right\}$, we get that:
\begin{multline*}    
\dfrac{1}{T} \sum_{t = 1}^T \mathbb{E} \left[ \normnuc{\nabla f(W_{t - 1})} \right] 
\le \chi_q \left[ (7 + 4 \varepsilon_q) \left( \dfrac{r 
\, \sigma^2 \, L \, D}{B \, T} \right)^{1/4} \right.\\ 
+ \left.\dfrac{(1 + \varepsilon_q)^2}{2} \sqrt{\dfrac{L \, D}{T}} + \dfrac{2\sqrt{2} \sigma \, r}{\sqrt{B \, T}} 
+L \rho \left( 6 + \dfrac{2 \sigma \sqrt{r}}{\sqrt{L \, D  \, B \, T}} \right) \right].
\end{multline*}
Thus, we can find an $\epsilon$-nuclear norm stationary point of $f$ with a complexity of:
\[
\dfrac{1}{T} \sum\limits_{t = 1}^T \E \left[ \normnuc{\G f(W_{t - 1})} \right] \le \chi_q \cdot \gO \left( \left( \dfrac{r \, \sigma^2 L \, D}{B \, T} \right)^{\!1/4} + \sqrt{\dfrac{L \, D}{T}} + \dfrac{\sigma r}{\sqrt{B \, T}} \right) + \chi_q L \, \rho \cdot \gO \left( 1 + \dfrac{\sigma \sqrt{r}}{\sqrt{L \, D \, B \, T}} \right). \qedhere
\]
\end{proof}

%% file: sections/SOMA_Lemma.tex
\subsection{Proof of Lemma~\ref{lem:lag}}
\begin{proof}
For $t \ge 2$, \soma uses $E_t = \rho \, \mathcal{P}_t$ with:
\[
\mathcal{P}_t = \ortho{\left( \overline G_{t - 1} - \gamma \, (1 - \beta) \, M_{t - 1} \right)}, \quad \overline G_{t - 1 } = \G f_{\mathcal B_{t - 1}} \left( W_{t - 2} + E_{t - 1} \right),
\]
and so $\mathcal{P}_t$ is $\mathcal F_{t - 1}$-measurable and $\normop{P_t}\le1$. The gradient $\overline G_{t - 1}$ used by \soma is evaluated at $Y_t = W_{t - 2} + E_{t - 1}$ whereas the current iterate is $X_t = W_{t - 1}$. Since $W_{t - 1} = W_{t - 2} - \eta \, O_{t - 1}$, we have $Y_t - X_t = E_{t - 1} + \eta \, O_{t - 1}$. Because $\normop{E_{t - 1}} \leq \rho$ and $\normop{O_{t - 1}} \leq 1$, $\normop{Y_t - X_t} \leq \rho + \eta$. Therefore, decompose the reused gradient as:
\[
\overline G_{t - 1} = \G f(X_t) + \underbrace{ \left( \G f(Y_t) - \G f(X_t) \right)}_{b_t} + \underbrace{ \left( \overline G_{t - 1} - \G f(Y_t) \right)}_{\zeta_{t - 1}}.
\]
By $L$-smoothness of $f$, we have $\normnuc{b_t} \leq L \, \normop{Y_t - X_t} \leq L \, (\rho + \eta)$. For $t \geq 3$, $Y_t$ depends only on the past while the mini-batch used to compute $\overline G_{t - 1}$ is fresh. As a result, we have that:
\[
\E[ \zeta_{t - 1} \mid \mathcal F_{t - 2}] = 0, \quad\text{and}\quad \E[ \normF{\zeta_{t - 1}}^2 \mid \mathcal F_{t - 2}] \leq \dfrac{\sigma^2}{B}.
\]
By the Hessian Lipschitzness property, we also have that:
\[
\dfrac{\rho^2}{2} \abs{\innerproduct{\mathcal{P}_t}{\left( H(X_t) - H(Y_t) \right)[P_t]}} \leq \dfrac{\rho^2}{2} L_H \normop{X_t - Y_t} \normop{\mathcal{P}_t}^2 \leq \dfrac{\rho^2}{2} L_H \, (\rho + \eta). \qedhere
\]
\end{proof}

%% file: isotropic_heavy_tailed_appendix.tex
 All synthetic heavy-tailed experiments were run on CPU only, on a single laptop with an Intel Core i7-13620H processor (10 cores: 6 performance and 4 efficient, 16 threads) and 32 GB of RAM, running Windows 11 (64-bit). No GPU was used. Runs were parallelized across 14 worker processes, with the 8 seeds vectorized in parallel. The models are small, so a single seed run takes ~4–5 seconds (or ~30–40 seconds for the full 8-seed batch).

\subsection{Experiment 1. Synthetic experiment with teacher-student model}
\textbf{Setup of the synthetic experiments}
\label{app:synthetic-setup}
\textbf{Data and model.} Inputs are drawn from a zero-mean Gaussian whose covariance has condition number $100$, with $n_{\text{train}}=512$ training and $512$ held-out samples. The data, the teacher and the initialization are generated once and shared by all runs. In the main text (Experiment~1), the student and the teacher are two-layer ReLU networks of the same shape, and the student is trained with cross-entropy over $k=32$ classes. In Experiment~2, both are two-matrix-deep linear networks trained with squared loss (see below).

\textbf{Noise.} Each minibatch gradient is computed on $|\mathcal{B}_t|=64$ samples. For $(Z_t)_{ij}=s_{ij}u_{ij}^{-1/\alpha}$, we have $\Pr(|(Z_t)_{ij}|>z)=z^{-\alpha}$ for $z\ge1$, so $\mathbb{E}|(Z_t)_{ij}|^p<\infty$ if and only if $p<\alpha$. In particular, $\mathbb{E}\normF{\mathcal{E}_t}^2=\infty$ for $\alpha\le2$. The matrices $L$ and $R$ are fixed across iterations and have geometrically decaying singular values with condition number $\kappa_\varepsilon$: $\kappa_\varepsilon=1$ gives isotropic noise and $\kappa_\varepsilon=10$ gives anisotropic noise.

\textbf{Optimization.} Every method runs for $T=500$ iterations with a cosine learning-rate schedule decayed to $1\%$ of its initial value. \muon-based methods use momentum $\beta=0.95$, and \adamw-based methods use $(\beta_1,\beta_2)=(0.9,0.999)$ with $\epsilon=10^{-8}$ and no weight decay; SGD-based methods use no momentum, except \textsc{Msam}, whose perturbation is the momentum buffer and which uses momentum $0.9$. Learning rates were chosen per base optimizer by grid search (on $16$ tuning seeds disjoint from the evaluation seeds), over $\eta\in\{0.01,0.03,0.1,0.3,1,3,10,30\}$ for SGD, $\eta\in\{10^{-4},3\!\cdot\!10^{-4},10^{-3},3\!\cdot\!10^{-3},10^{-2},3\!\cdot\!10^{-2}\}$ for \adamw, and $\eta\in\{0.001,0.002,0.005,0.01,0.02,0.05,0.1\}$ for \muon. The selected values are $\eta=1.0$ for SGD-based methods ($0.1$ for \textsc{Msam}), $\eta=0.003$ for \adamw-based methods, and $\eta=0.02$ for \muon-based methods, and are kept fixed across all $\alpha$ and $\rho$. For the sharpness-aware methods, we sweep $\rho\in\{0.001,0.002,0.005,0.01,0.02,0.05,0.1,0.2,0.5,1\}$. All results are the median over $8$ independent noise seeds.

\input{dataforsyntheticplots/nn_aniso_val_acc}
\input{dataforsyntheticplots/nn_aniso_train_loss}

\textbf{Detailed observations from Experiment 1.} Two observations complement the main text. First, the geometry of the perturbation matters more than how its direction is obtained. In the spectral geometry, every perturbation we tested improves on \muon, including a random one, while in the Frobenius geometry neither a random perturbation (\randsammuon) nor the Euclidean methods \sam, \fsam and \msam do so (Table~\ref{tab:nonlinear_aniso_acc}). Second, stale perturbations are more robust to the choice of radius. At $\rho=1$, the two-gradient methods degrade as the tails get lighter, with \specsammuon dropping from $47.0\%$ to $29.0\%$ and \fpsoma from $46.7\%$ to $41.7\%$ at $\alpha=3$, whereas \soma, \opsomapre and \opsomapost lose at most $1.1$ points from their best. Finally, at $\alpha=1.1$ the differences among the \muon-based methods are smaller than the spread over seeds. Adding \sam or \fsam to \adamw changes its accuracy by less than one point at every $\alpha$.

\subsection{Experiment 2. Linear two-matrix network under heavy-tailed noise}
We also perform experiments in a setting where the global optimum is known. The student is a deep linear network \begin{align*}
    P(W)=W_2W_1
\end{align*}
trained to recover a rank-8 teacher matrix $P^\star$ (condition number 10) from ill-conditioned Gaussian inputs, with objective 
\begin{align*}
  F(W)=\frac{1}{2nk}\|XP(W)^\top-Y\|_F^2+\frac{\mu}{2}\sum_{l}\|W_l\|_F^2  
\end{align*}
under a spectral-norm constraint. We report the final objective gap $F(W_T)-F^\star$ which uses anisotropic noise ($\kappa_\varepsilon=10$).

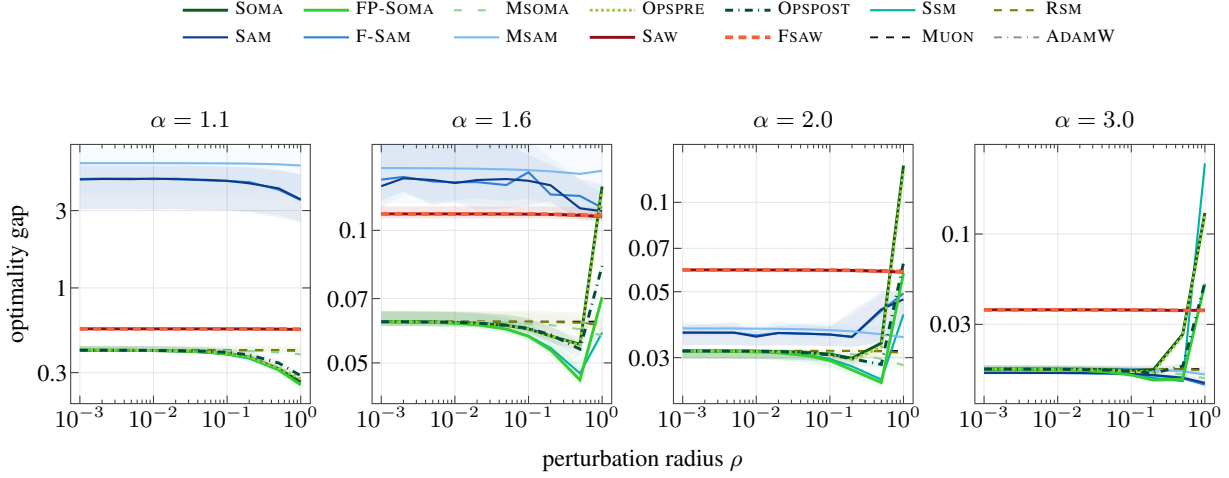
\begin{figure*}[!t]
    \centering
\input{tikz_for_heavytailed/experiment_2/gap_vs_rho.tikz}
    
    \caption{Performance under \textbf{anisotropic} heavy-tailed noise (Condition no., $\kappa_\varepsilon = 10$). Optimality gap (log scale) vs $\rho$. \textcolor{red}{Red}, \textcolor{blue}{Blue} and \textcolor{ForestGreen}{Green} lines indicate \adamw-based, \sgd-based and \muon-based methods, respectively.}
    \label{fig:exp2_linear_anisotropic}
\end{figure*}
\input{dataforsyntheticplots/mse_aniso_gap}

\textbf{Results.} Since the optimum is known, Figure~\ref{fig:exp2_linear_anisotropic} measures how well each method optimizes under heavy-tailed noise rather than how it generalizes, which makes it a direct check of our convergence analysis. Two findings stand out. First, spectral perturbations reduce \muon's optimality gap most under the heaviest tails: at $\alpha=1.1$, \soma lowers it by about $36\%$, as much as the two-gradient \specsammuon and close to \fpsoma, while the \sgd-based methods fail to converge (a gap about $8$--$14\times$ that of \muon) and the \adamw-based methods remain $1.3$--$2.2\times$ above \muon at every $\alpha$. As the tails get lighter, the gain of the one-gradient methods shrinks, and at $\alpha=3$ only the two-gradient methods still improve on \muon noticeably; a fresh gradient thus helps optimization when the noise is mild. Second, the best radius of \soma decreases as the tails get lighter, from $\rho^\star=1$ at $\alpha=1.1$ to $\rho^\star=0.1$ at $\alpha=3$, and for $\alpha\ge1.6$ a radius of $\rho=1$ increases the gap of most spectral variants well above that of \muon. This matches the bound of Theorem~\ref{thm:sam_muon_nu_p_convergence}: the error floor $\gO(\chi_q L\rho^\nu)$ is negligible when the noise term dominates, but it limits accuracy once the noise is light, so a moderate radius is the safer choice. \msoma reduces \muon's gap at every $\alpha$, by $6$--$11\%$, even though it barely improved validation accuracy in Experiment~1; this supports our earlier observation that perturbing along the momentum mainly aids optimization rather than generalization. As in Experiment~1, a Frobenius random perturbation (\randsammuon) performs like \muon, and adding \sam or \fsam to \adamw changes its gap by about $1\%$ (Table~\ref{tab:mse_aniso_gap}).

\subsection{Isotropic Noise for Heavy-tailed experiments}
 In this section we repeat Experiments~1 and~2 with \textbf{isotropic} heavy-tailed noise. The optimizers, learning rates, $\rho$ grid, tail indices $\alpha\in\{1.1,1.6,2,3\}$ and the 8 seeds are unchanged. Lines show the median over seeds, and shaded bands show the interquartile range.

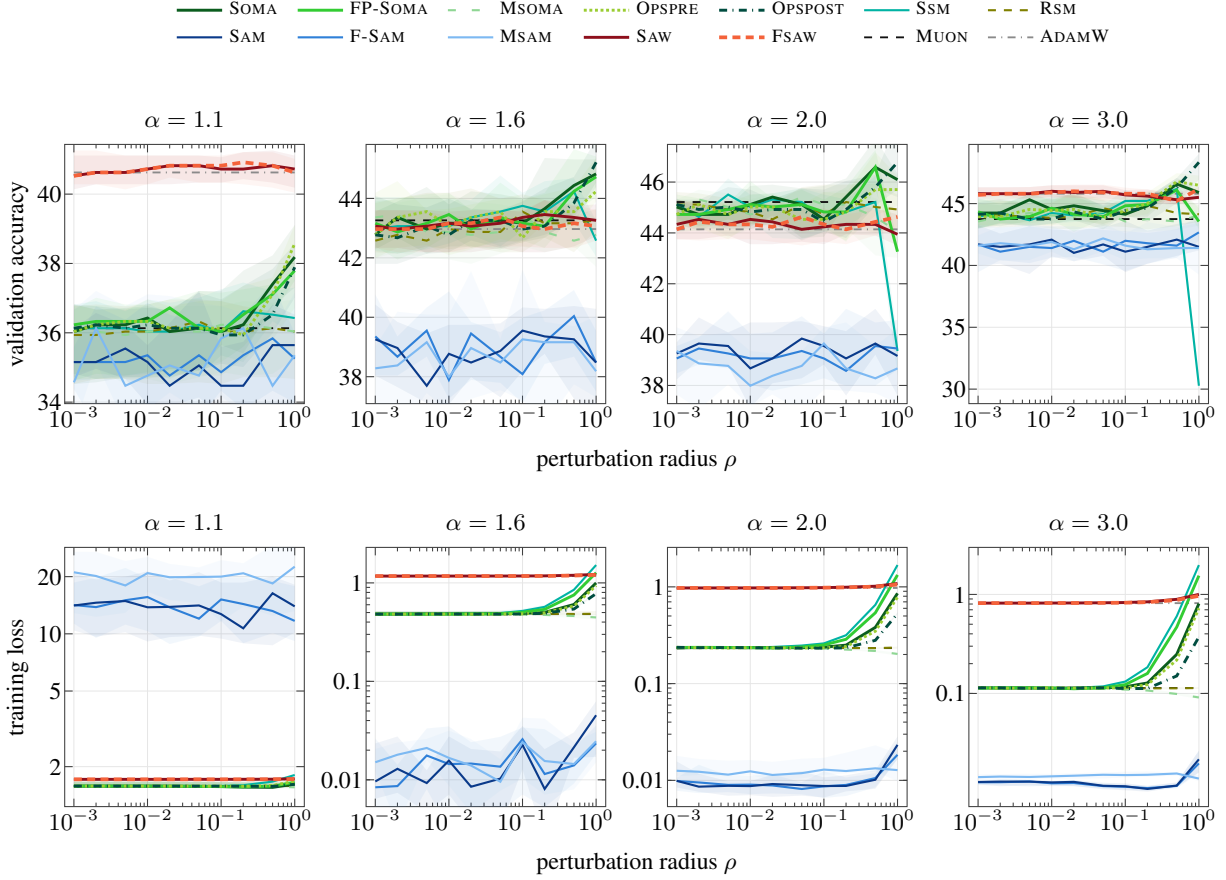
\begin{figure*}[!t]
    \centering
    % Top row: validation accuracy
    \input{tikz_for_heavytailed/appendex_isotropic/val_acc_vs_rho.tikz}

    \vspace{0.3cm}

    % Bottom row: training loss
    \input{tikz_for_heavytailed/appendex_isotropic/train_loss_vs_rho.tikz}

    \caption{Performance under \textbf{isotropic} heavy-tailed noise (nonlinear setting)(Condition no., $\kappa_\epsilon=1$). Top: validation accuracy. Bottom: training loss (log scale). \textcolor{red}{Red}, \textcolor{blue}{Blue} and \textcolor{ForestGreen}{Green} lines indicate \adamw-based, \sgd-based and \muon-based methods, respectively.}
    \label{fig:iso_nonlinear}
\end{figure*}
\input{dataforsyntheticplots/nn_iso_val_acc}
\input{dataforsyntheticplots/nn_iso_train_loss}

\textbf{Results for Experiment 1.} Figure~\ref{fig:iso_nonlinear} shows validation accuracy and training loss against the perturbation radius $\rho$ for nonlinear two-layer classification. Tables~\ref{tab:nn_iso_val_acc} and~\ref{tab:nn_iso_train_loss} report the best value over $\rho$ and the radius $\rho^\star$ at which it is attained. With isotropic noise, the spectral perturbations still improve on \muon at every tail index: \soma raises the best validation accuracy of \muon by $1.4$--$2.8$ points and matches or exceeds \fpsoma at every $\alpha$ while using half the gradient evaluations. The main difference from the anisotropic case concerns \adamw. Under the heaviest tails ($\alpha=1.1$), the \adamw-based methods are now the most accurate ($40.6$--$40.9\%$, versus $36.1\%$ for \muon and $38.2\%$ for \soma), whereas with anisotropic noise they were the least accurate. The robustness of \adamw to heavy-tailed noise therefore depends on the structure of the noise, while \muon and \soma change much less between the two cases. We attribute this to the preconditioning: when the noise has the same scale in every direction, \adamw normalizes each coordinate by the same noise level, so coordinates where the true gradient is strong still receive larger steps, while heavy-tailed outliers are damped. \muon, in contrast, sets all singular values to one, so directions dominated by noise receive the same step as directions carrying a strong gradient, and this equalization brings no benefit when the noise has no spectral structure to correct. Under anisotropic noise, the noise concentrates along a few directions that are mixed across coordinates, so the per-coordinate scaling of \adamw no longer matches it, while the spectral normalization of \muon does.
For $\alpha\ge1.6$, the pattern of the anisotropic case returns: the \sgd-based methods reach the lowest training loss but the lowest accuracy, and the random (\randsammuon) and momentum (\msoma) perturbations perform like \muon. Stale perturbations are again more robust to a large radius: at $\rho=1$ and $\alpha=3$, \specsammuon drops to $30.3\%$ and \fpsoma to $43.6\%$, while \soma loses less than one point from its best.

\textbf{Results for Experiment 2.} Figure~\ref{fig:iso_linear} shows the optimality gap against $\rho$ for the linear two-matrix network. Table~\ref{tab:mse_iso_gap} reports the best gap over $\rho$ and the radius at which it is attained. The optimality gap shows the same trends as with anisotropic noise. The spectral perturbations reduce \muon's gap at every $\alpha$, by about $38\%$ for \soma at $\alpha=1.1$ and by $12\%$ at $\alpha=3$, while the two-gradient \specsammuon and \fpsoma keep a larger reduction at light tails (about $25\%$ at $\alpha=3$). The \sgd-based methods again fail to converge under the heaviest tails, with a gap about $6$--$9\times$ that of \muon at $\alpha=1.1$. As in Experiment~1, \adamw benefits from isotropic noise: it has the smallest gap at $\alpha=1.1$, but it falls behind \muon for $\alpha\ge2$. The best radius of \soma decreases from $\rho^\star=1$ at $\alpha=1.1$ to $\rho^\star=0.5$ for $\alpha\ge1.6$, and at $\alpha=3$ a radius of $\rho=1$ makes its gap about $2.7\times$ larger than at $\rho^\star$, which is again consistent with the $\rho$-dependent error floor of Theorem~\ref{thm:sam_muon_nu_p_convergence}. The random and momentum perturbations and \sam or \fsam on \adamw change the gap of their base optimizer by at most about $5\%$.
\begin{figure*}[!t]
    \centering
    \input{tikz_for_heavytailed/appendex_isotropic/gap_vs_rho.tikz}
    \caption{Performance under \textbf{isotropic} heavy-tailed noise (linear setting)(Condition no., $\kappa_\epsilon=1$). Optimality gap (log scale) vs $\rho$. \textcolor{red}{Red}, \textcolor{blue}{Blue} and \textcolor{ForestGreen}{Green} lines indicate \adamw-based, \sgd-based and \muon-based methods, respectively.}
    \label{fig:iso_linear}
\end{figure*}
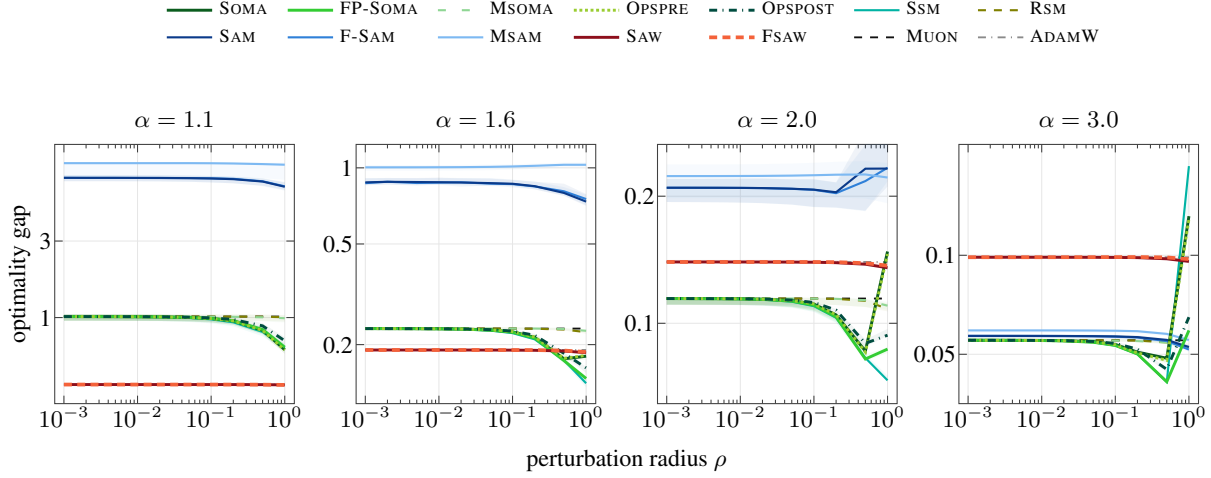
\input{dataforsyntheticplots/mse_iso_gap}

\subsection{Ablation: Frobenius versus Spectral Geometry}
\label{app:ablation-geometry}
\textbf{Motivation and setup.} \muon is steepest descent under the spectral norm, so a perturbation inside a spectral-norm ball matches its geometry \citep{zhong2026sharpness}, whereas \sam and its variants perturb inside a Frobenius ball. This ablation isolates the effect of this choice. We consider four ways of obtaining the perturbation direction: a random direction, the fresh gradient (\sam-\muon), the friendly residual (\soma), and the stale gradient (\Stwosam-\muon). \Stwosam-\muon adapts S$^2$-\sam \citep{ji2024single} to \muon: it perturbs along the gradient of the previous iteration $\overline G_{t-1}$, so it needs one gradient evaluation per step and corresponds to \soma with $\gamma=0$. For each direction, we keep the direction fixed and change only the ball's geometry, and we compare the two versions for the nonlinear (Table~\ref{tab:nonlinear_aniso_acc}) and linear (Table~\ref{tab:linear_aniso_gap}) problems with anisotropic noise.
\begin{table*}[!t]
    \centering
    \caption{Nonlinear anisotropic setting: best validation accuracy (\%, $\uparrow$) with a Frobenius or a spectral perturbation, for each tail index $\alpha$. The better geometry for each method and $\alpha$ is in bold; ties are not marked.}
    \label{tab:nonlinear_aniso_acc}
    \begin{tabular}{@{}llcccc@{}}
        \toprule
        & & \multicolumn{4}{c}{Tail index $\alpha$} \\ \cmidrule(l){3-6}
        Direction & Geometry & $1.1$ & $1.6$ & $2.0$ & $3.0$ \\
        \midrule
        \multirow{2}{*}{Random} & Frobenius & 38.3 & 43.4 & 44.5 & 44.7 \\
                                & Spectral  & \textbf{38.5} & \textbf{44.7} & \textbf{46.5} & \textbf{47.2} \\
        \midrule
        \multirow{2}{*}{\sam-\muon} & Frobenius & 38.3 & 43.2 & 45.0 & 45.6 \\
                                    & Spectral  & \textbf{38.6} & \textbf{43.7} & \textbf{46.3} & \textbf{47.0} \\
        \midrule
        \multirow{2}{*}{\soma} & Frobenius & 38.2 & 43.3 & 45.0 & 45.6 \\
                               & Spectral  & \textbf{39.1} & \textbf{45.0} & \textbf{46.3} & \textbf{46.5} \\
        \midrule
        \multirow{2}{*}{\Stwosam-\muon} & Frobenius & 38.2 & 43.3 & 45.0 & 44.9 \\
                                    & Spectral  & \textbf{38.9} & \textbf{44.4} & \textbf{46.8} & \textbf{45.7} \\
        \bottomrule
    \end{tabular}
\end{table*}

\begin{table*}[!t]
    \centering
    \caption{Linear anisotropic setting: best optimality gap ($\downarrow$) with a Frobenius or a spectral perturbation, for each tail index $\alpha$. The better geometry for each method and $\alpha$ is in bold; ties are not marked.}
    \label{tab:linear_aniso_gap}
    \begin{tabular}{@{}llcccc@{}}
        \toprule
        & & \multicolumn{4}{c}{Tail index $\alpha$} \\ \cmidrule(l){3-6}
        Direction & Geometry & $1.1$ & $1.6$ & $2.0$ & $3.0$ \\
        \midrule
        \multirow{2}{*}{Random} & Frobenius & 0.410 & 0.0612 & 0.0312 & 0.0163 \\
                                & Spectral  & \textbf{0.347} & \textbf{0.0576} & \textbf{0.0301} & 0.0163 \\
        \midrule
        \multirow{2}{*}{\sam-\muon} & Frobenius & 0.404 & 0.0581 & 0.0294 & 0.0149 \\
                                    & Spectral  & \textbf{0.265} & \textbf{0.0473} & \textbf{0.0253} & \textbf{0.0141} \\
        \midrule
        \multirow{2}{*}{\soma} & Frobenius & 0.399 & 0.0592 & 0.0303 & 0.0163 \\
                               & Spectral  & \textbf{0.263} & \textbf{0.0551} & \textbf{0.0299} & 0.0163 \\
        \midrule
        \multirow{2}{*}{\Stwosam-\muon} & Frobenius & 0.407 & 0.0601 & 0.0312 & 0.0165 \\
                                    & Spectral  & \textbf{0.285} & \textbf{0.0588} & \textbf{0.0309} & 0.0165 \\
        \bottomrule
    \end{tabular}
\end{table*}

\textbf{Results.} The spectral geometry is never worse than the Frobenius one: it gives higher accuracy in all 16 (direction, $\alpha$) pairs and an equal or lower optimality gap in all of them. Its effect, however, depends on the tail index in opposite ways for the two metrics. For the optimality gap, the benefit is largest under the heaviest tails: at $\alpha=1.1$, the spectral geometry lowers the gap of \sam-\muon and \soma by about a third, and at $\alpha=3$ the two geometries are nearly identical. For validation accuracy, the benefit grows as the tails get lighter, from at most $0.9$ points at $\alpha=1.1$ to $0.8$--$2.4$ points at $\alpha=3$. Thus, the spectral geometry mainly aids optimization when the noise is heavy and mainly aids generalization when it is light. The direction of the perturbation also interacts with the geometry. For the optimality gap, an informed direction profits far more from the spectral ball than a random one (a reduction of about $34\%$ for \sam-\muon and \soma versus $15\%$ for the random direction at $\alpha=1.1$). For accuracy, even a random direction yields a substantial gain (up to $2.4$ points), indicating that much of the accuracy improvement of the spectral variants comes from the geometry itself.

\subsection{Ablation: Fresh versus Stale Perturbations}
\label{app:ablation-fresh-stale}
\textbf{Motivation and setup.} A fresh perturbation direction costs an extra gradient evaluation at every step, whereas a stale direction reuses quantities from the previous step at no extra cost but may be outdated. This ablation measures what staleness costs. We compare three pairs of methods that differ only in whether the direction is fresh or stale: \sam-\muon and \Stwosam-\muon in the Frobenius ball, and \specsammuon and \Stwosam-\muon, as well as \fpsoma and \soma in the spectral ball. We evaluate them on the nonlinear (Table~\ref{tab:nonlinear_aniso_acc_fresh_stale}) and linear (Table~\ref{tab:linear_aniso_gap_fresh_stale}) problems with anisotropic noise.

\begin{table*}[!t]
    \centering
    \caption{Nonlinear anisotropic setting: best validation accuracy (\%, $\uparrow$) with fresh or stale perturbation direction, for each tail index $\alpha$. The better direction for each geometry and $\alpha$ is in bold; ties are not marked.}
    \label{tab:nonlinear_aniso_acc_fresh_stale}
    \begin{tabular}{@{}lllcccc@{}}
        \toprule
        & & & \multicolumn{4}{c}{Tail index $\alpha$} \\ \cmidrule(l){4-7}
        Geometry & Direction & Method & $1.1$ & $1.6$ & $2.0$ & $3.0$ \\
        \midrule
        \multirow{2}{*}{Frobenius} & Fresh & \sam-\muon & \textbf{0.3828}          & 0.4316          & 0.4502          & \textbf{0.4561}     \\
         & Stale & \Stwosam-\muon&
                 0.3818          & \textbf{0.4326}          & 0.4502          & 0.4492     \\
        \midrule
        \multirow{2}{*}{Spectral}  & Fresh & \specsammuon & 0.3857          & 0.4365          & 0.4629          & \textbf{0.4697}  \\
        & Stale & \Stwosam-\muon& \textbf{0.3887} & \textbf{0.4443} & \textbf{0.4678} & 0.4570        \\
        \midrule
                \multirow{2}{*}{Spectral}  & Fresh & \fpsoma & \textbf{0.3984}			 & 0.4424& \textbf{0.4688}& \textbf{0.4668} \\
        & Stale & \soma & 0.3906 & 	\textbf{0.4502}& 0.4629 & 0.4648   \\
        \bottomrule
    \end{tabular}
\end{table*}

\begin{table*}[!t]
    \centering
    \caption{Linear anisotropic setting: best optimality gap ($\downarrow$) with a fresh or a stale perturbation direction, for each tail index $\alpha$. The better direction for each geometry and $\alpha$ is in bold; ties are not marked.}
    \label{tab:linear_aniso_gap_fresh_stale}
    \begin{tabular}{@{}lllcccc@{}}
        \toprule
        & & & \multicolumn{4}{c}{Tail index $\alpha$} \\ \cmidrule(l){4-7}
        Geometry & Direction & Method & $1.1$ & $1.6$ & $2.0$ & $3.0$ \\
        \midrule
        \multirow{2}{*}{Frobenius} & Fresh & \sam-\muon & \textbf{0.4037}          & \textbf{0.0581}          & \textbf{0.0294}          & \textbf{0.0149} \\
                                   & Stale & \Stwosam-\muon&  0.4068          & 0.0601          & 0.0312          & 0.0165       \\
        \midrule
        \multirow{2}{*}{Spectral}  & Fresh & \specsammuon & \textbf{0.2648} & \textbf{0.0473} & \textbf{0.0253} & \textbf{0.0141}\\
                                   & Stale & \Stwosam-\muon&0.2846          & 0.0588          & 0.0309          & 0.0165        \\
        \midrule
          \multirow{2}{*}{Spectral}  & Fresh & \fpsoma & \textbf{0.2522}	 & \textbf{0.0457}& \textbf{0.0247}& \textbf{0.0143} \\
        & Stale & \soma & 0.2631 & 0.0551& 0.0299 & 0.0163     \\
        \bottomrule
    \end{tabular}
\end{table*}

\textbf{Results.} Staleness costs little in accuracy but more in optimization. In validation accuracy, every fresh--stale pair stays within $1.3$ points, and the stale direction is often the more accurate one; in particular, \soma is within $0.8$ points of \fpsoma at every $\alpha$ and ahead of it at $\alpha=1.6$. In the optimality gap, the fresh direction is better in every case, and the difference is larger in the spectral ball (a stale gap $4$--$24\%$ higher) than in the Frobenius ball ($1$--$11\%$), consistent with Experiment~2, where the two-gradient methods kept an advantage in optimization. The friendly residual reduces this cost under the heaviest tails: at $\alpha=1.1$, the gap of \soma is only $4\%$ above that of \fpsoma, compared with $7\%$ for \Stwosam-\muon relative to \specsammuon, and it matches the gap of the fresh \specsammuon ($0.263$ versus $0.265$). For lighter tails, the two spectral pairs lose about the same, so the advantage of the friendly residual is specific to the heavy-tailed regime that our analysis targets.

%% file: dataforsyntheticplots/nn_aniso_val_acc.tex
% ---- nn_aniso_val_acc
\providecolor{muonrow}{RGB}{208,236,214}\providecolor{adamrow}{RGB}{255,214,214}\providecolor{sgdrow}{RGB}{208,226,247}

\begin{table*}[!t]
\centering
\caption{Validation accuracy (\%): best value over the perturbation radius $\rho$ and the radius $\rho^\star$ where it is attained (median of 8 seeds; Frobenius-ball methods pool the global and per-layer radius modes). Row colour marks the base optimizer family: \muon-based (\colorbox{muonrow}{Green}), \adamw-based (\colorbox{adamrow}{Red}), \sgd-based (\colorbox{sgdrow}{Blue}); baselines have no $\rho^\star$. Best value per $\alpha$ in bold.}
\label{tab:nn_aniso_val_acc}
\vspace{2pt}
\resizebox{\ifdim\width>\textwidth\textwidth\else\width\fi}{!}{%
\footnotesize\setlength{\tabcolsep}{5pt}
\begin{tabular}{lcccccccc}
\toprule
Algorithm & \multicolumn{2}{c}{$\alpha=1.1$} & \multicolumn{2}{c}{$\alpha=1.6$} & \multicolumn{2}{c}{$\alpha=2.0$} & \multicolumn{2}{c}{$\alpha=3.0$} \\
& best acc. & $\rho^\star$ & best acc. & $\rho^\star$ & best acc. & $\rho^\star$ & best acc. & $\rho^\star$ \\
\midrule
\rowcolor{muonrow} \muon & 38.1 & -- & 42.9 & -- & 44 & -- & 44.3 & -- \\
\rowcolor{muonrow} \specsammuon & 38.6 & 0.5 & 43.7 & 0.5 & 46.3 & 0.5 & 47 & 0.5 \\
\rowcolor{muonrow} \fpsoma & \textbf{39.8} & 1 & 44.2 & 0.5 & 46.9 & 0.5 & 46.7 & 0.5 \\
\rowcolor{muonrow} \soma & 39.1 & 1 & 45 & 1 & 46.3 & 1 & 46.5 & 0.5 \\
\rowcolor{muonrow} \randsammuon & 38.3 & 0.01 & 43.3 & 0.002 & 44.5 & 0.1 & 44.8 & 1 \\
\rowcolor{muonrow} \opsomapre & 38.8 & 1 & \textbf{45.3} & 1 & 46.1 & 1 & 46.9 & 0.5 \\
\rowcolor{muonrow} \opsomapost & 38.3 & 1 & 45.2 & 1 & \textbf{47.9} & 1 & \textbf{47.4} & 1 \\
\rowcolor{muonrow} \textsc{Msoma} & 38.3 & 0.001 & 43 & 0.001 & 44.7 & 0.005 & 44.9 & 0.5 \\
\rowcolor{adamrow} \adamw & 28.6 & -- & 41.5 & -- & 44.2 & -- & 45.4 & -- \\
\rowcolor{adamrow} \samadamw & 28.9 & 1 & 41.9 & 0.5 & 44.5 & 0.1 & 46.1 & 1 \\
\rowcolor{adamrow} \fsamadamw & 28.8 & 0.05 & 41.8 & 0.1 & 44.5 & 0.05 & 46.3 & 1 \\
\rowcolor{sgdrow} \sgd & 35.6 & -- & 39.7 & -- & 39.2 & -- & 41.7 & -- \\
\rowcolor{sgdrow} \sam & 36.8 & 0.01 & 39.5 & 0.005 & 40.8 & 0.2 & 42.9 & 0.2 \\
\rowcolor{sgdrow} \fsam & 36.4 & 1 & 39.7 & 0.005 & 40.2 & 0.002 & 43 & 0.2 \\
\rowcolor{sgdrow} \textsc{Msam} & 36.3 & 0.005 & 39.1 & 0.1 & 39.7 & 0.05 & 42.6 & 0.002 \\
\bottomrule
\end{tabular}}
\end{table*}

%% file: dataforsyntheticplots/nn_aniso_train_loss.tex
% ---- nn_aniso_train_loss
\providecolor{muonrow}{RGB}{208,236,214}\providecolor{adamrow}{RGB}{255,214,214}\providecolor{sgdrow}{RGB}{208,226,247}
\begin{table*}[!t]\centering\footnotesize\setlength{\tabcolsep}{5pt}
\caption{Training loss: best value over the perturbation radius $\rho$ and the radius $\rho^\star$ where it is attained (median of 8 seeds; Frobenius-ball methods pool the global and per-layer radius modes). Row colour marks the base optimizer family: \muon-based (\colorbox{muonrow}{Green}), \adamw-based (\colorbox{adamrow}{Red}), \sgd-based (\colorbox{sgdrow}{Blue}); baselines have no $\rho^\star$. Best value per $\alpha$ in bold.}\label{tab:nn_aniso_train_loss}
\vspace{2pt}
\resizebox{\ifdim\width>\textwidth\textwidth\else\width\fi}{!}{%
\begin{tabular}{lcccccccc}\toprule
Algorithm & \multicolumn{2}{c}{$\alpha=1.1$} & \multicolumn{2}{c}{$\alpha=1.6$} & \multicolumn{2}{c}{$\alpha=2.0$} & \multicolumn{2}{c}{$\alpha=3.0$} \\
& best loss & $\rho^\star$ & best loss & $\rho^\star$ & best loss & $\rho^\star$ & best loss & $\rho^\star$ \\\midrule
\rowcolor{muonrow} \muon & 1.45 & -- & 0.372 & -- & 0.184 & -- & 0.078 & -- \\
\rowcolor{muonrow} \specsammuon & 1.45 & 0.02 & 0.372 & 0.002 & 0.183 & 0.001 & 0.0779 & 0.001 \\
\rowcolor{muonrow} \fpsoma & 1.44 & 0.2 & 0.372 & 0.01 & 0.183 & 0.001 & 0.0768 & 0.02 \\
\rowcolor{muonrow} \soma & 1.43 & 0.2 & 0.371 & 0.02 & 0.184 & 0.001 & 0.0783 & 0.01 \\
\rowcolor{muonrow} \randsammuon & 1.45 & 0.1 & 0.371 & 0.2 & 0.184 & 0.001 & 0.0779 & 0.05 \\
\rowcolor{muonrow} \opsomapre & 1.44 & 0.2 & 0.371 & 0.01 & 0.182 & 0.01 & 0.0779 & 0.01 \\
\rowcolor{muonrow} \opsomapost & 1.44 & 0.2 & 0.37 & 0.05 & 0.183 & 0.02 & 0.0778 & 0.05 \\
\rowcolor{muonrow} \textsc{Msoma} & \textbf{1.42} & 1 & 0.338 & 1 & 0.153 & 1 & 0.0627 & 1 \\
\rowcolor{adamrow} \adamw & 2.7 & -- & 1.32 & -- & 1.02 & -- & 0.803 & -- \\
\rowcolor{adamrow} \samadamw & 2.7 & 0.5 & 1.32 & 0.002 & 1.02 & 0.001 & 0.803 & 0.001 \\
\rowcolor{adamrow} \fsamadamw & 2.7 & 0.001 & 1.32 & 0.001 & 1.02 & 0.001 & 0.803 & 0.002 \\
\rowcolor{sgdrow} \sgd & 15 & -- & 0.0278 & -- & 0.00892 & -- & 0.0126 & -- \\
\rowcolor{sgdrow} \sam & 9 & 1 & 0.0131 & 0.01 & \textbf{0.00671} & 0.05 & \textbf{0.011} & 0.2 \\
\rowcolor{sgdrow} \fsam & 10.5 & 1 & \textbf{0.0106} & 0.1 & 0.00732 & 0.1 & 0.0113 & 0.2 \\
\rowcolor{sgdrow} \textsc{Msam} & 17.7 & 0.05 & 0.0257 & 0.002 & 0.00969 & 0.005 & 0.0138 & 0.01 \\
\bottomrule\end{tabular}}\end{table*}

%% file: tikz_for_heavytailed/experiment_2/gap_vs_rho.tikz
%% Auto-generated by make_heavytailed_figures.py  --  do not edit by hand, re-run the script.
%% Values: line = median over seeds, band = interquartile range (25th-75th percentile).
%% Preamble needed:  \usepackage{pgfplots}  \pgfplotsset{compat=1.16}   (nothing else)
%% Width follows \textwidth, fonts follow the document (so the legend can never overflow).

\definecolor{htSOMA}{HTML}{0B5D1E}
\definecolor{htFPSOMA}{HTML}{33CC33}
\definecolor{htSpec}{HTML}{00B3A4}
\definecolor{htRand}{HTML}{808000}
\definecolor{htFSAM}{HTML}{2F80D8}
\definecolor{htSAM}{HTML}{0A3A8A}
\definecolor{htPreNS5}{HTML}{9ACD32}
\definecolor{htPostNS5}{HTML}{004D40}
\definecolor{htSAMAdam}{HTML}{8E0F1C}
\definecolor{htFSAMAdam}{HTML}{F4623A}
\definecolor{htMSAM}{HTML}{7DB9F2}
\definecolor{htMSOMA}{HTML}{8FD694}
\definecolor{htMuon}{HTML}{111111}
\definecolor{htAdam}{HTML}{8A8A8A}
\pgfplotsset{
  ht panel/.style={
    scale only axis, width=0.1920\textwidth, height=1.35in,
    xmin=0.00075, xmax=1.35,
    xtick={0.001,0.01,0.1,1}, xticklabels={$10^{-3}$,$10^{-2}$,$10^{-1}$,$10^{0}$},
    minor xtick={0.002,0.003,0.004,0.005,0.006,0.007,0.008,0.009,0.02,0.03,0.04,0.05,0.06,0.07,0.08,0.09,0.2,0.3,0.4,0.5,0.6,0.7,0.8,0.9},
    xmajorgrids=true, ymajorgrids=true,
    grid style={line width=0.25pt, draw=black!10},
    axis line style={line width=0.4pt, black!75},
    tick style={line width=0.4pt, black!75}, major tick length=3pt, minor tick length=1.5pt,
    tick label style={font=\footnotesize, /pgf/number format/assume math mode=true},
    label style={font=\small}, title style={font=\small, yshift=-3pt},
    ylabel style={yshift=-1pt, align=center},
    every tick label/.append style={inner sep=1pt},
    clip=true,
  }}

\begin{tikzpicture}
\begin{axis}[name=ax0_r0,
    at={(0.0820\textwidth,0in)},
    anchor=south west,
    ht panel, xmode=log, ymode=log,
    ymin=0.193208, ymax=7.68265,
    ytick={0.3,1,3},
    yticklabels={0.3,1,3},
    title={$\alpha=1.1$},
    ylabel={optimality gap}]
  \addplot[draw=none, fill=htSOMA, fill opacity=0.06] coordinates {(0.001,0.438293) (0.002,0.438098) (0.005,0.437515) (0.01,0.436538) (0.02,0.434571) (0.05,0.428568) (0.1,0.418271) (0.2,0.396909) (0.5,0.331601) (1,0.268585) (1,0.250426) (0.5,0.301303) (0.2,0.362362) (0.1,0.383433) (0.05,0.393632) (0.02,0.399584) (0.01,0.401536) (0.005,0.402505) (0.002,0.403084) (0.001,0.403277)} \closedcycle;
  \addplot[draw=none, fill=htFPSOMA, fill opacity=0.06] coordinates {(0.001,0.438266) (0.002,0.438046) (0.005,0.437383) (0.01,0.436276) (0.02,0.434049) (0.05,0.427284) (0.1,0.415777) (0.2,0.392263) (0.5,0.324106) (1,0.25659) (1,0.232993) (0.5,0.293968) (0.2,0.359439) (0.1,0.381907) (0.05,0.392853) (0.02,0.399269) (0.01,0.401377) (0.005,0.402425) (0.002,0.403052) (0.001,0.403261)} \closedcycle;
  \addplot[draw=none, fill=htSpec, fill opacity=0.06] coordinates {(0.001,0.438185) (0.002,0.437884) (0.005,0.436983) (0.01,0.435488) (0.02,0.432518) (0.05,0.423774) (0.1,0.409724) (0.2,0.383446) (0.5,0.318893) (1,0.267594) (1,0.254018) (0.5,0.295465) (0.2,0.352344) (0.1,0.376539) (0.05,0.389691) (0.02,0.397874) (0.01,0.400657) (0.005,0.402059) (0.002,0.402904) (0.001,0.403187)} \closedcycle;
  \addplot[draw=none, fill=htRand, fill opacity=0.06] coordinates {(0.001,0.438487) (0.002,0.438487) (0.005,0.438488) (0.01,0.43849) (0.02,0.438492) (0.05,0.43849) (0.1,0.43846) (0.2,0.43834) (0.5,0.437437) (1,0.434155) (1,0.397934) (0.5,0.402122) (0.2,0.403283) (0.1,0.403435) (0.05,0.403467) (0.02,0.403471) (0.01,0.403471) (0.005,0.40347) (0.002,0.40347) (0.001,0.403469)} \closedcycle;
  \addplot[draw=none, fill=htPreNS5, fill opacity=0.06] coordinates {(0.001,0.438308) (0.002,0.43813) (0.005,0.437593) (0.01,0.436692) (0.02,0.434872) (0.05,0.429269) (0.1,0.41953) (0.2,0.399014) (0.5,0.335181) (1,0.277367) (1,0.258115) (0.5,0.305183) (0.2,0.364619) (0.1,0.384815) (0.05,0.394394) (0.02,0.399907) (0.01,0.4017) (0.005,0.402588) (0.002,0.403117) (0.001,0.403293)} \closedcycle;
  \addplot[draw=none, fill=htPostNS5, fill opacity=0.06] coordinates {(0.001,0.438357) (0.002,0.438227) (0.005,0.437837) (0.01,0.437183) (0.02,0.435865) (0.05,0.431831) (0.1,0.424858) (0.2,0.410125) (0.5,0.362199) (1,0.291476) (1,0.268937) (0.5,0.328795) (0.2,0.375462) (0.1,0.389985) (0.05,0.396878) (0.02,0.400872) (0.01,0.402177) (0.005,0.402825) (0.002,0.403212) (0.001,0.403341)} \closedcycle;
  \addplot[draw=none, fill=htMSOMA, fill opacity=0.06] coordinates {(0.001,0.438466) (0.002,0.438446) (0.005,0.438384) (0.01,0.438281) (0.02,0.438076) (0.05,0.437455) (0.1,0.436409) (0.2,0.434278) (0.5,0.427648) (1,0.415364) (1,0.376949) (0.5,0.390392) (0.2,0.398419) (0.1,0.400978) (0.05,0.402233) (0.02,0.402977) (0.01,0.403223) (0.005,0.403346) (0.002,0.40342) (0.001,0.403445)} \closedcycle;
  \addplot[draw=none, fill=htFSAM, fill opacity=0.06] coordinates {(0.001,5.68622) (0.002,5.68271) (0.005,5.67011) (0.01,5.66233) (0.02,5.6776) (0.05,5.64978) (0.1,5.60655) (0.2,5.60462) (0.5,5.37457) (1,5.03233) (1,2.55024) (0.5,2.81222) (0.2,2.95539) (0.1,3.00862) (0.05,3.01514) (0.02,3.04667) (0.01,3.05119) (0.005,3.06524) (0.002,3.06149) (0.001,3.05087)} \closedcycle;
  \addplot[draw=none, fill=htSAM, fill opacity=0.06] coordinates {(0.001,5.70179) (0.002,5.68871) (0.005,5.687) (0.01,5.69583) (0.02,5.69504) (0.05,5.64647) (0.1,5.61516) (0.2,5.53401) (0.5,5.32029) (1,5.01077) (1,2.53352) (0.5,2.71105) (0.2,2.91861) (0.1,3.02021) (0.05,3.02345) (0.02,3.06829) (0.01,3.03619) (0.005,3.05243) (0.002,3.08012) (0.001,3.06301)} \closedcycle;
  \addplot[draw=none, fill=htMSAM, fill opacity=0.06] coordinates {(0.001,7.31749) (0.002,7.31733) (0.005,7.31685) (0.01,7.31603) (0.02,7.3143) (0.05,7.3087) (0.1,7.30054) (0.2,7.28733) (0.5,7.20634) (1,7.10147) (1,3.85505) (0.5,3.87345) (0.2,3.85175) (0.1,3.84332) (0.05,3.84201) (0.02,3.84345) (0.01,3.84438) (0.005,3.84494) (0.002,3.8453) (0.001,3.84543)} \closedcycle;
  \addplot[draw=none, fill=htSAMAdam, fill opacity=0.06] coordinates {(0.001,0.567313) (0.002,0.567309) (0.005,0.567298) (0.01,0.567278) (0.02,0.567239) (0.05,0.567096) (0.1,0.566897) (0.2,0.566509) (0.5,0.56517) (1,0.562381) (1,0.538654) (0.5,0.542239) (0.2,0.544102) (0.1,0.544562) (0.05,0.544817) (0.02,0.544994) (0.01,0.545044) (0.005,0.545069) (0.002,0.545084) (0.001,0.545082)} \closedcycle;
  \addplot[draw=none, fill=htFSAMAdam, fill opacity=0.06] coordinates {(0.001,0.567316) (0.002,0.567316) (0.005,0.567313) (0.01,0.567309) (0.02,0.567301) (0.05,0.567251) (0.1,0.567204) (0.2,0.56711) (0.5,0.56658) (1,0.56491) (1,0.540884) (0.5,0.543466) (0.2,0.544622) (0.1,0.544827) (0.05,0.544951) (0.02,0.545048) (0.01,0.545071) (0.005,0.545082) (0.002,0.545089) (0.001,0.545084)} \closedcycle;
  \addplot[color=htMuon, dashed, line width=0.7pt, forget plot] coordinates {(0.001,0.412668) (1,0.412668)};
  \addplot[color=htAdam, dash dot, line width=0.7pt, forget plot] coordinates {(0.001,0.559269) (1,0.559269)};
  \addplot[color=htSpec, solid, line width=0.8pt] coordinates {(0.001,0.412411) (0.002,0.412154) (0.005,0.411387) (0.01,0.410114) (0.02,0.407591) (0.05,0.400198) (0.1,0.388404) (0.2,0.366513) (0.5,0.312439) (1,0.264796)};
  \addplot[color=htRand, dashed, line width=0.8pt] coordinates {(0.001,0.412668) (0.002,0.412668) (0.005,0.412669) (0.01,0.41267) (0.02,0.412671) (0.05,0.412669) (0.1,0.412646) (0.2,0.412535) (0.5,0.411666) (1,0.408486)};
  \addplot[color=htMSOMA, loosely dashed, line width=0.8pt] coordinates {(0.001,0.412643) (0.002,0.412619) (0.005,0.412546) (0.01,0.412424) (0.02,0.412179) (0.05,0.411442) (0.1,0.410203) (0.2,0.407692) (0.5,0.399976) (1,0.389062)};
  \addplot[color=htFSAM, solid, line width=0.8pt] coordinates {(0.001,4.70385) (0.002,4.72241) (0.005,4.72489) (0.01,4.68901) (0.02,4.72294) (0.05,4.61136) (0.1,4.56783) (0.2,4.49255) (0.5,4.04435) (1,3.48186)};
  \addplot[color=htSAM, solid, line width=0.8pt] coordinates {(0.001,4.66774) (0.002,4.70216) (0.005,4.68863) (0.01,4.72635) (0.02,4.68514) (0.05,4.64489) (0.1,4.5808) (0.2,4.42095) (0.5,4.11464) (1,3.51858)};
  \addplot[color=htMSAM, solid, line width=0.8pt] coordinates {(0.001,5.88509) (0.002,5.88482) (0.005,5.88404) (0.01,5.88277) (0.02,5.88038) (0.05,5.87401) (0.1,5.86473) (0.2,5.84692) (0.5,5.78806) (1,5.68878)};
  \addplot[color=htSOMA, solid, line width=1.1pt] coordinates {(0.001,0.412482) (0.002,0.412297) (0.005,0.411739) (0.01,0.410808) (0.02,0.408938) (0.05,0.403275) (0.1,0.393699) (0.2,0.374247) (0.5,0.317064) (1,0.26315)};
  \addplot[color=htFPSOMA, solid, line width=1.1pt] coordinates {(0.001,0.412469) (0.002,0.412271) (0.005,0.411675) (0.01,0.410679) (0.02,0.408681) (0.05,0.402627) (0.1,0.39239) (0.2,0.371641) (0.5,0.312522) (1,0.252222)};
  \addplot[color=htPreNS5, densely dotted, line width=1.1pt] coordinates {(0.001,0.412502) (0.002,0.412337) (0.005,0.411839) (0.01,0.411006) (0.02,0.409326) (0.05,0.404182) (0.1,0.395331) (0.2,0.37694) (0.5,0.321347) (1,0.274249)};
  \addplot[color=htPostNS5, dash dot, line width=1.1pt] coordinates {(0.001,0.412545) (0.002,0.412423) (0.005,0.412056) (0.01,0.411442) (0.02,0.410206) (0.05,0.406443) (0.1,0.400003) (0.2,0.386622) (0.5,0.344282) (1,0.287357)};
  \addplot[color=htSAMAdam, solid, line width=1.1pt] coordinates {(0.001,0.559265) (0.002,0.559261) (0.005,0.559242) (0.01,0.559221) (0.02,0.55918) (0.05,0.559036) (0.1,0.558794) (0.2,0.558405) (0.5,0.556988) (1,0.554234)};
  \addplot[color=htFSAMAdam, densely dashed, line width=1.4pt] coordinates {(0.001,0.559267) (0.002,0.559265) (0.005,0.559252) (0.01,0.559242) (0.02,0.559222) (0.05,0.559139) (0.1,0.558999) (0.2,0.558808) (0.5,0.55795) (1,0.556035)};
\end{axis}
\begin{axis}[name=ax1_r0,
    at={(0.3235\textwidth,0in)},
    anchor=south west,
    ht panel, xmode=log, ymode=log,
    ymin=0.0404089, ymax=0.156741,
    ytick={0.05,0.07,0.1},
    yticklabels={0.05,0.07,0.1},
    title={$\alpha=1.6$}]
  \addplot[draw=none, fill=htSOMA, fill opacity=0.06] coordinates {(0.001,0.0655314) (0.002,0.0655073) (0.005,0.0654348) (0.01,0.0653134) (0.02,0.0650689) (0.05,0.0643215) (0.1,0.0630285) (0.2,0.0603512) (0.5,0.0571274) (1,0.126741) (1,0.121686) (0.5,0.0533606) (0.2,0.0554496) (0.1,0.0582192) (0.05,0.0595738) (0.02,0.0602979) (0.01,0.0604968) (0.005,0.0605947) (0.002,0.060653) (0.001,0.0606723)} \closedcycle;
  \addplot[draw=none, fill=htFPSOMA, fill opacity=0.06] coordinates {(0.001,0.0655089) (0.002,0.0654624) (0.005,0.0653229) (0.01,0.0650912) (0.02,0.0646304) (0.05,0.0632748) (0.1,0.0611119) (0.2,0.0571233) (0.5,0.049213) (1,0.0714992) (1,0.0696917) (0.5,0.0443417) (0.2,0.0516728) (0.1,0.0560507) (0.05,0.0584262) (0.02,0.0597956) (0.01,0.0602428) (0.005,0.0604671) (0.002,0.0606018) (0.001,0.0606467)} \closedcycle;
  \addplot[draw=none, fill=htSpec, fill opacity=0.06] coordinates {(0.001,0.0655102) (0.002,0.0654649) (0.005,0.0653299) (0.01,0.0651067) (0.02,0.0646675) (0.05,0.0634017) (0.1,0.0614427) (0.2,0.0579896) (0.5,0.0509273) (1,0.0607376) (1,0.0568571) (0.5,0.0458623) (0.2,0.0526163) (0.1,0.0564432) (0.05,0.0584989) (0.02,0.0597895) (0.01,0.0602364) (0.005,0.060463) (0.002,0.0605999) (0.001,0.0606457)} \closedcycle;
  \addplot[draw=none, fill=htRand, fill opacity=0.06] coordinates {(0.001,0.0655552) (0.002,0.065555) (0.005,0.0655542) (0.01,0.0655529) (0.02,0.0655501) (0.05,0.0655404) (0.1,0.0655198) (0.2,0.0654624) (0.5,0.065171) (1,0.0643397) (1,0.0590875) (0.5,0.0603481) (0.2,0.0606984) (0.1,0.0607046) (0.05,0.0607002) (0.02,0.0606955) (0.01,0.0606937) (0.005,0.0606927) (0.002,0.060692) (0.001,0.0606918)} \closedcycle;
  \addplot[draw=none, fill=htPreNS5, fill opacity=0.06] coordinates {(0.001,0.065532) (0.002,0.0655084) (0.005,0.0654376) (0.01,0.0653189) (0.02,0.0650791) (0.05,0.0643406) (0.1,0.0630452) (0.2,0.060313) (0.5,0.0566755) (1,0.125846) (1,0.120951) (0.5,0.0530943) (0.2,0.0556903) (0.1,0.0584206) (0.05,0.0596714) (0.02,0.060287) (0.01,0.0604901) (0.005,0.0605911) (0.002,0.0606515) (0.001,0.0606716)} \closedcycle;
  \addplot[draw=none, fill=htPostNS5, fill opacity=0.06] coordinates {(0.001,0.0655324) (0.002,0.0655093) (0.005,0.06544) (0.01,0.0653242) (0.02,0.0650918) (0.05,0.0643871) (0.1,0.0631825) (0.2,0.0607091) (0.5,0.0556272) (1,0.0838864) (1,0.0825128) (0.5,0.0513634) (0.2,0.0560077) (0.1,0.0585059) (0.05,0.0596997) (0.02,0.060308) (0.01,0.0605004) (0.005,0.0605961) (0.002,0.0606535) (0.001,0.0606725)} \closedcycle;
  \addplot[draw=none, fill=htMSOMA, fill opacity=0.06] coordinates {(0.001,0.06555) (0.002,0.0655446) (0.005,0.0655282) (0.01,0.065501) (0.02,0.0654468) (0.05,0.0652861) (0.1,0.0650239) (0.2,0.0645177) (0.5,0.0629539) (1,0.0602218) (1,0.0556025) (0.5,0.0578263) (0.2,0.059475) (0.1,0.060117) (0.05,0.0604091) (0.02,0.0605762) (0.01,0.0606335) (0.005,0.0606625) (0.002,0.0606799) (0.001,0.0606858)} \closedcycle;
  \addplot[draw=none, fill=htFSAM, fill opacity=0.06] coordinates {(0.001,0.157931) (0.002,0.164155) (0.005,0.155482) (0.01,0.15379) (0.02,0.164997) (0.05,0.157115) (0.1,0.161076) (0.2,0.146294) (0.5,0.129902) (1,0.121888) (1,0.105024) (0.5,0.112115) (0.2,0.114085) (0.1,0.118134) (0.05,0.11638) (0.02,0.118459) (0.01,0.116911) (0.005,0.116253) (0.002,0.121132) (0.001,0.111563)} \closedcycle;
  \addplot[draw=none, fill=htSAM, fill opacity=0.06] coordinates {(0.001,0.157355) (0.002,0.165778) (0.005,0.15428) (0.01,0.165425) (0.02,0.158508) (0.05,0.157491) (0.1,0.148214) (0.2,0.141075) (0.5,0.13105) (1,0.127364) (1,0.105315) (0.5,0.104903) (0.2,0.116356) (0.1,0.114043) (0.05,0.112216) (0.02,0.11109) (0.01,0.116273) (0.005,0.113907) (0.002,0.122641) (0.001,0.116842)} \closedcycle;
  \addplot[draw=none, fill=htMSAM, fill opacity=0.06] coordinates {(0.001,0.181001) (0.002,0.180959) (0.005,0.180837) (0.01,0.180646) (0.02,0.180319) (0.05,0.17955) (0.1,0.178587) (0.2,0.177471) (0.5,0.179052) (1,0.195897) (1,0.106084) (0.5,0.108881) (0.2,0.110731) (0.1,0.112648) (0.05,0.113947) (0.02,0.114934) (0.01,0.115285) (0.005,0.115459) (0.002,0.115562) (0.001,0.115597)} \closedcycle;
  \addplot[draw=none, fill=htSAMAdam, fill opacity=0.06] coordinates {(0.001,0.113363) (0.002,0.113362) (0.005,0.113378) (0.01,0.113381) (0.02,0.113385) (0.05,0.113329) (0.1,0.113298) (0.2,0.113207) (0.5,0.113041) (1,0.112719) (1,0.105302) (0.5,0.105852) (0.2,0.10619) (0.1,0.106316) (0.05,0.10635) (0.02,0.106388) (0.01,0.106401) (0.005,0.106412) (0.002,0.106418) (0.001,0.10642)} \closedcycle;
  \addplot[draw=none, fill=htFSAMAdam, fill opacity=0.06] coordinates {(0.001,0.113363) (0.002,0.113363) (0.005,0.11338) (0.01,0.113387) (0.02,0.113393) (0.05,0.113348) (0.1,0.113335) (0.2,0.113275) (0.5,0.11316) (1,0.112801) (1,0.10565) (0.5,0.106023) (0.2,0.10629) (0.1,0.10637) (0.05,0.106379) (0.02,0.1064) (0.01,0.106407) (0.005,0.106415) (0.002,0.106419) (0.001,0.10642)} \closedcycle;
  \addplot[color=htMuon, dashed, line width=0.7pt, forget plot] coordinates {(0.001,0.0620412) (1,0.0620412)};
  \addplot[color=htAdam, dash dot, line width=0.7pt, forget plot] coordinates {(0.001,0.108994) (1,0.108994)};
  \addplot[color=htSpec, solid, line width=0.8pt] coordinates {(0.001,0.0619921) (0.002,0.0619432) (0.005,0.0617969) (0.01,0.0615549) (0.02,0.0610778) (0.05,0.0596975) (0.1,0.05759) (0.2,0.0540694) (0.5,0.0473417) (1,0.0585832)};
  \addplot[color=htRand, dashed, line width=0.8pt] coordinates {(0.001,0.062041) (0.002,0.0620409) (0.005,0.0620404) (0.01,0.0620396) (0.02,0.0620378) (0.05,0.0620306) (0.1,0.0620129) (0.2,0.0619559) (0.5,0.0616688) (1,0.0608928)};
  \addplot[color=htMSOMA, loosely dashed, line width=0.8pt] coordinates {(0.001,0.062035) (0.002,0.0620288) (0.005,0.0620104) (0.01,0.0619797) (0.02,0.0619188) (0.05,0.061739) (0.1,0.0614484) (0.2,0.0609517) (0.5,0.059633) (1,0.057632)};
  \addplot[color=htFSAM, solid, line width=0.8pt] coordinates {(0.001,0.130403) (0.002,0.13225) (0.005,0.129136) (0.01,0.128631) (0.02,0.128719) (0.05,0.126726) (0.1,0.135552) (0.2,0.120645) (0.5,0.119727) (1,0.11255)};
  \addplot[color=htSAM, solid, line width=0.8pt] coordinates {(0.001,0.125991) (0.002,0.131279) (0.005,0.130114) (0.01,0.128092) (0.02,0.130108) (0.05,0.130776) (0.1,0.129594) (0.2,0.126568) (0.5,0.112233) (1,0.110512)};
  \addplot[color=htMSAM, solid, line width=0.8pt] coordinates {(0.001,0.138513) (0.002,0.138481) (0.005,0.138384) (0.01,0.138228) (0.02,0.138021) (0.05,0.137553) (0.1,0.137051) (0.2,0.136159) (0.5,0.134366) (1,0.136546)};
  \addplot[color=htSOMA, solid, line width=1.1pt] coordinates {(0.001,0.0620163) (0.002,0.0619914) (0.005,0.0619165) (0.01,0.0617911) (0.02,0.0615381) (0.05,0.0608405) (0.1,0.0597296) (0.2,0.0575499) (0.5,0.0551189) (1,0.125747)};
  \addplot[color=htFPSOMA, solid, line width=1.1pt] coordinates {(0.001,0.0619927) (0.002,0.0619443) (0.005,0.061799) (0.01,0.0615572) (0.02,0.0610744) (0.05,0.0596635) (0.1,0.057466) (0.2,0.0534715) (0.5,0.0457266) (1,0.0705073)};
  \addplot[color=htPreNS5, densely dotted, line width=1.1pt] coordinates {(0.001,0.0620177) (0.002,0.0619941) (0.005,0.0619233) (0.01,0.0618044) (0.02,0.0615644) (0.05,0.06083) (0.1,0.0596918) (0.2,0.0575453) (0.5,0.0548495) (1,0.123146)};
  \addplot[color=htPostNS5, dash dot, line width=1.1pt] coordinates {(0.001,0.0620182) (0.002,0.0619953) (0.005,0.0619264) (0.01,0.0618115) (0.02,0.0615813) (0.05,0.0609082) (0.1,0.0598943) (0.2,0.0579518) (0.5,0.0537373) (1,0.0830557)};
  \addplot[color=htSAMAdam, solid, line width=1.1pt] coordinates {(0.001,0.108993) (0.002,0.108992) (0.005,0.108987) (0.01,0.10898) (0.02,0.108961) (0.05,0.108911) (0.1,0.108839) (0.2,0.108675) (0.5,0.108177) (1,0.107431)};
  \addplot[color=htFSAMAdam, densely dashed, line width=1.4pt] coordinates {(0.001,0.108993) (0.002,0.108993) (0.005,0.108991) (0.01,0.108986) (0.02,0.108975) (0.05,0.108945) (0.1,0.108904) (0.2,0.108799) (0.5,0.108436) (1,0.10779)};
\end{axis}
\begin{axis}[name=ax2_r0,
    at={(0.5650\textwidth,0in)},
    anchor=south west,
    ht panel, xmode=log, ymode=log,
    ymin=0.0210022, ymax=0.156658,
    ytick={0.03,0.05,0.07,0.1},
    yticklabels={0.03,0.05,0.07,0.1},
    title={$\alpha=2.0$}]
  \addplot[draw=none, fill=htSOMA, fill opacity=0.06] coordinates {(0.001,0.0319976) (0.002,0.0319898) (0.005,0.0319657) (0.01,0.0319244) (0.02,0.0318373) (0.05,0.0315477) (0.1,0.031016) (0.2,0.0301944) (0.5,0.0341405) (1,0.133765) (1,0.131498) (0.5,0.0321564) (0.2,0.0285277) (0.1,0.0290925) (0.05,0.0294608) (0.02,0.0296788) (0.01,0.029749) (0.005,0.0297837) (0.002,0.0298043) (0.001,0.0298111)} \closedcycle;
  \addplot[draw=none, fill=htFPSOMA, fill opacity=0.06] coordinates {(0.001,0.0319812) (0.002,0.0319569) (0.005,0.0318841) (0.01,0.0317629) (0.02,0.0315184) (0.05,0.0307372) (0.1,0.0295262) (0.2,0.0276667) (0.5,0.0253291) (1,0.0584622) (1,0.0575517) (0.5,0.0235482) (0.2,0.0258567) (0.1,0.027661) (0.05,0.0287076) (0.02,0.0293661) (0.01,0.0295904) (0.005,0.0297037) (0.002,0.0297722) (0.001,0.0297951)} \closedcycle;
  \addplot[draw=none, fill=htSpec, fill opacity=0.06] coordinates {(0.001,0.0319836) (0.002,0.0319617) (0.005,0.0318964) (0.01,0.0317881) (0.02,0.0315507) (0.05,0.0308451) (0.1,0.0298922) (0.2,0.0284386) (0.5,0.0258372) (1,0.0426151) (1,0.041306) (0.5,0.0240706) (0.2,0.0265706) (0.1,0.0280668) (0.05,0.028913) (0.02,0.0294451) (0.01,0.0296294) (0.005,0.0297231) (0.002,0.0297799) (0.001,0.0297989)} \closedcycle;
  \addplot[draw=none, fill=htRand, fill opacity=0.06] coordinates {(0.001,0.0320055) (0.002,0.0320055) (0.005,0.0320055) (0.01,0.0320055) (0.02,0.0320054) (0.05,0.0320041) (0.1,0.0319987) (0.2,0.0319764) (0.5,0.0318264) (1,0.0313053) (1,0.0294847) (0.5,0.0297895) (0.2,0.0298308) (0.1,0.0298286) (0.05,0.0298243) (0.02,0.0298207) (0.01,0.0298194) (0.005,0.0298187) (0.002,0.0298183) (0.001,0.0298181)} \closedcycle;
  \addplot[draw=none, fill=htPreNS5, fill opacity=0.06] coordinates {(0.001,0.0319955) (0.002,0.0319855) (0.005,0.0319552) (0.01,0.0319039) (0.02,0.031798) (0.05,0.0314524) (0.1,0.0308248) (0.2,0.0298186) (0.5,0.0328944) (1,0.129551) (1,0.127076) (0.5,0.0313687) (0.2,0.0282463) (0.1,0.0289835) (0.05,0.0293876) (0.02,0.0296452) (0.01,0.0297314) (0.005,0.0297746) (0.002,0.0298006) (0.001,0.0298093)} \closedcycle;
  \addplot[draw=none, fill=htPostNS5, fill opacity=0.06] coordinates {(0.001,0.0319955) (0.002,0.0319856) (0.005,0.0319558) (0.01,0.0319056) (0.02,0.0318043) (0.05,0.031487) (0.1,0.0309147) (0.2,0.0298561) (0.5,0.0290374) (1,0.0637106) (1,0.062201) (0.5,0.0276485) (0.2,0.0281412) (0.1,0.0289104) (0.05,0.029347) (0.02,0.0296271) (0.01,0.029722) (0.005,0.0297699) (0.002,0.0297987) (0.001,0.0298083)} \closedcycle;
  \addplot[draw=none, fill=htMSOMA, fill opacity=0.06] coordinates {(0.001,0.0320012) (0.002,0.031997) (0.005,0.0319844) (0.01,0.0319637) (0.02,0.0319229) (0.05,0.0318057) (0.1,0.0316039) (0.2,0.0311951) (0.5,0.0303152) (1,0.0291249) (1,0.0268556) (0.5,0.0280614) (0.2,0.0289745) (0.1,0.0293616) (0.05,0.0295806) (0.02,0.0297206) (0.01,0.0297688) (0.005,0.0297933) (0.002,0.0298081) (0.001,0.029813)} \closedcycle;
  \addplot[draw=none, fill=htFSAM, fill opacity=0.06] coordinates {(0.001,0.0394334) (0.002,0.0388338) (0.005,0.0383997) (0.01,0.0368325) (0.02,0.0399364) (0.05,0.0384262) (0.1,0.0384229) (0.2,0.041837) (0.5,0.0491187) (1,0.0507441) (1,0.0394616) (0.5,0.0338075) (0.2,0.0325241) (0.1,0.0327793) (0.05,0.0329189) (0.02,0.0330132) (0.01,0.0330471) (0.005,0.0330646) (0.002,0.0330753) (0.001,0.0330789)} \closedcycle;
  \addplot[draw=none, fill=htSAM, fill opacity=0.06] coordinates {(0.001,0.0383832) (0.002,0.0389245) (0.005,0.0394276) (0.01,0.0368374) (0.02,0.0384652) (0.05,0.039094) (0.1,0.0381796) (0.2,0.0432265) (0.5,0.0500523) (1,0.0493714) (1,0.040028) (0.5,0.033598) (0.2,0.032635) (0.1,0.0328341) (0.05,0.0329456) (0.02,0.0330237) (0.01,0.0327697) (0.005,0.0330672) (0.002,0.0330763) (0.001,0.0330794)} \closedcycle;
  \addplot[draw=none, fill=htMSAM, fill opacity=0.06] coordinates {(0.001,0.0393418) (0.002,0.0393378) (0.005,0.0393259) (0.01,0.0393104) (0.02,0.0392803) (0.05,0.0391963) (0.1,0.0390712) (0.2,0.0388473) (0.5,0.0382909) (1,0.0379368) (1,0.0332796) (0.5,0.0340223) (0.2,0.0346694) (0.1,0.0350093) (0.05,0.0352211) (0.02,0.0353652) (0.01,0.0354165) (0.005,0.0354428) (0.002,0.0354588) (0.001,0.0354641)} \closedcycle;
  \addplot[draw=none, fill=htSAMAdam, fill opacity=0.06] coordinates {(0.001,0.0600573) (0.002,0.0600561) (0.005,0.0600527) (0.01,0.0600444) (0.02,0.0600311) (0.05,0.0599909) (0.1,0.0599375) (0.2,0.0598487) (0.5,0.059488) (1,0.0589472) (1,0.0571797) (0.5,0.0575214) (0.2,0.0579205) (0.1,0.0580355) (0.05,0.058108) (0.02,0.0581533) (0.01,0.0581796) (0.005,0.0581872) (0.002,0.0581941) (0.001,0.0581957)} \closedcycle;
  \addplot[draw=none, fill=htFSAMAdam, fill opacity=0.06] coordinates {(0.001,0.0600577) (0.002,0.0600569) (0.005,0.0600547) (0.01,0.0600485) (0.02,0.0600393) (0.05,0.0600111) (0.1,0.0599766) (0.2,0.0599222) (0.5,0.0596421) (1,0.0591058) (1,0.0573108) (0.5,0.0576635) (0.2,0.0579961) (0.1,0.0580766) (0.05,0.0581294) (0.02,0.0581621) (0.01,0.0581845) (0.005,0.0581894) (0.002,0.058195) (0.001,0.0581961)} \closedcycle;
  \addplot[color=htMuon, dashed, line width=0.7pt, forget plot] coordinates {(0.001,0.0316082) (1,0.0316082)};
  \addplot[color=htAdam, dash dot, line width=0.7pt, forget plot] coordinates {(0.001,0.0592229) (1,0.0592229)};
  \addplot[color=htSpec, solid, line width=0.8pt] coordinates {(0.001,0.0315874) (0.002,0.0315665) (0.005,0.0315044) (0.01,0.0314018) (0.02,0.0312004) (0.05,0.0306257) (0.1,0.0296808) (0.2,0.0280181) (0.5,0.025317) (1,0.0419274)};
  \addplot[color=htRand, dashed, line width=0.8pt] coordinates {(0.001,0.0316083) (0.002,0.0316084) (0.005,0.0316087) (0.01,0.0316091) (0.02,0.03161) (0.05,0.031612) (0.1,0.0316136) (0.2,0.0316101) (0.5,0.0315317) (1,0.0310874)};
  \addplot[color=htMSOMA, loosely dashed, line width=0.8pt] coordinates {(0.001,0.0316035) (0.002,0.0315988) (0.005,0.0315848) (0.01,0.0315618) (0.02,0.0315168) (0.05,0.0313893) (0.1,0.0311926) (0.2,0.0308285) (0.5,0.0297245) (1,0.02829)};
  \addplot[color=htFSAM, solid, line width=0.8pt] coordinates {(0.001,0.0364283) (0.002,0.0364226) (0.005,0.0364053) (0.01,0.0356306) (0.02,0.0363134) (0.05,0.0360976) (0.1,0.0358751) (0.2,0.0351978) (0.5,0.0430958) (1,0.049394)};
  \addplot[color=htSAM, solid, line width=0.8pt] coordinates {(0.001,0.0364285) (0.002,0.0364229) (0.005,0.0364061) (0.01,0.035307) (0.02,0.0363174) (0.05,0.0361096) (0.1,0.0358894) (0.2,0.0352696) (0.5,0.0436312) (1,0.0471295)};
  \addplot[color=htMSAM, solid, line width=0.8pt] coordinates {(0.001,0.0376341) (0.002,0.0376283) (0.005,0.0376104) (0.01,0.0375795) (0.02,0.0375208) (0.05,0.037361) (0.1,0.0371451) (0.2,0.0367859) (0.5,0.03596) (1,0.03517)};
  \addplot[color=htSOMA, solid, line width=1.1pt] coordinates {(0.001,0.0316016) (0.002,0.0315949) (0.005,0.0315747) (0.01,0.0315405) (0.02,0.0314698) (0.05,0.0312411) (0.1,0.0308235) (0.2,0.0298611) (0.5,0.0335785) (1,0.132991)};
  \addplot[color=htFPSOMA, solid, line width=1.1pt] coordinates {(0.001,0.0315845) (0.002,0.0315607) (0.005,0.0314895) (0.01,0.0313711) (0.02,0.0311354) (0.05,0.0304433) (0.1,0.0293314) (0.2,0.0272774) (0.5,0.0247397) (1,0.0577858)};
  \addplot[color=htPreNS5, densely dotted, line width=1.1pt] coordinates {(0.001,0.0316) (0.002,0.0315917) (0.005,0.0315665) (0.01,0.0315239) (0.02,0.0314356) (0.05,0.031156) (0.1,0.0306603) (0.2,0.0295814) (0.5,0.0325435) (1,0.12808)};
  \addplot[color=htPostNS5, dash dot, line width=1.1pt] coordinates {(0.001,0.0315992) (0.002,0.0315901) (0.005,0.0315628) (0.01,0.0315171) (0.02,0.0314244) (0.05,0.031141) (0.1,0.0306577) (0.2,0.0297506) (0.5,0.0284488) (1,0.0625837)};
  \addplot[color=htSAMAdam, solid, line width=1.1pt] coordinates {(0.001,0.0592216) (0.002,0.0592203) (0.005,0.0592092) (0.01,0.0592009) (0.02,0.0591857) (0.05,0.0591427) (0.1,0.0590833) (0.2,0.0589658) (0.5,0.0586085) (1,0.0582035)};
  \addplot[color=htFSAMAdam, densely dashed, line width=1.4pt] coordinates {(0.001,0.0592221) (0.002,0.0592212) (0.005,0.0592113) (0.01,0.0592052) (0.02,0.0591943) (0.05,0.0591637) (0.1,0.0591235) (0.2,0.0590397) (0.5,0.0587489) (1,0.0583644)};
\end{axis}
\begin{axis}[name=ax3_r0,
    at={(0.8065\textwidth,0in)},
    anchor=south west,
    ht panel, xmode=log, ymode=log,
    ymin=0.0104357, ymax=0.329748,
    ytick={0.03,0.1},
    yticklabels={0.03,0.1},
    title={$\alpha=3.0$}]
  \addplot[draw=none, fill=htSOMA, fill opacity=0.06] coordinates {(0.001,0.0173729) (0.002,0.0173707) (0.005,0.0173639) (0.01,0.0173522) (0.02,0.0173267) (0.05,0.0172464) (0.1,0.0171449) (0.2,0.0173396) (0.5,0.0266044) (1,0.134597) (1,0.129911) (0.5,0.0261324) (0.2,0.015658) (0.1,0.0157622) (0.05,0.0158975) (0.02,0.0159704) (0.01,0.0159966) (0.005,0.0160103) (0.002,0.0160185) (0.001,0.0160213)} \closedcycle;
  \addplot[draw=none, fill=htFPSOMA, fill opacity=0.06] coordinates {(0.001,0.0173604) (0.002,0.0173459) (0.005,0.0173024) (0.01,0.0172306) (0.02,0.0170902) (0.05,0.0166921) (0.1,0.0160833) (0.2,0.0151786) (0.5,0.0152278) (1,0.0529946) (1,0.049704) (0.5,0.0140547) (0.2,0.0139057) (0.1,0.0147791) (0.05,0.015362) (0.02,0.0157492) (0.01,0.0158846) (0.005,0.0159538) (0.002,0.0159958) (0.001,0.0160099)} \closedcycle;
  \addplot[draw=none, fill=htSpec, fill opacity=0.06] coordinates {(0.001,0.0173616) (0.002,0.0173483) (0.005,0.0173084) (0.01,0.017243) (0.02,0.0171159) (0.05,0.0167621) (0.1,0.016255) (0.2,0.0155792) (0.5,0.0149671) (1,0.263752) (1,0.249489) (0.5,0.0139171) (0.2,0.0142269) (0.1,0.0149477) (0.05,0.0154228) (0.02,0.0157705) (0.01,0.0158947) (0.005,0.0159587) (0.002,0.0159977) (0.001,0.0160109)} \closedcycle;
  \addplot[draw=none, fill=htRand, fill opacity=0.06] coordinates {(0.001,0.0173751) (0.002,0.0173751) (0.005,0.0173751) (0.01,0.0173752) (0.02,0.0173753) (0.05,0.017375) (0.1,0.0173731) (0.2,0.0173647) (0.5,0.0173075) (1,0.0170774) (1,0.0156932) (0.5,0.0159095) (0.2,0.0159918) (0.1,0.0160093) (0.05,0.0160173) (0.02,0.0160216) (0.01,0.0160229) (0.005,0.0160235) (0.002,0.0160238) (0.001,0.0160239)} \closedcycle;
  \addplot[draw=none, fill=htPreNS5, fill opacity=0.06] coordinates {(0.001,0.0173708) (0.002,0.0173666) (0.005,0.0173538) (0.01,0.0173324) (0.02,0.0172886) (0.05,0.0171486) (0.1,0.0169295) (0.2,0.0168512) (0.5,0.0262096) (1,0.128712) (1,0.125494) (0.5,0.0251905) (0.2,0.0154168) (0.1,0.0156012) (0.05,0.0158094) (0.02,0.0159384) (0.01,0.0159807) (0.005,0.0160023) (0.002,0.0160153) (0.001,0.0160197)} \closedcycle;
  \addplot[draw=none, fill=htPostNS5, fill opacity=0.06] coordinates {(0.001,0.0173702) (0.002,0.0173654) (0.005,0.0173511) (0.01,0.0173277) (0.02,0.0172834) (0.05,0.0171645) (0.1,0.0169526) (0.2,0.0166384) (0.5,0.0179629) (1,0.0531757) (1,0.0519312) (0.5,0.016416) (0.2,0.0151962) (0.1,0.0155627) (0.05,0.015795) (0.02,0.0159303) (0.01,0.0159765) (0.005,0.0160001) (0.002,0.0160144) (0.001,0.0160192)} \closedcycle;
  \addplot[draw=none, fill=htMSOMA, fill opacity=0.06] coordinates {(0.001,0.0173725) (0.002,0.01737) (0.005,0.0173624) (0.01,0.0173499) (0.02,0.0173253) (0.05,0.0172556) (0.1,0.017152) (0.2,0.0169562) (0.5,0.0163832) (1,0.0155334) (1,0.0139806) (0.5,0.0149865) (0.2,0.0155393) (0.1,0.0157522) (0.05,0.015886) (0.02,0.0159699) (0.01,0.0159973) (0.005,0.0160107) (0.002,0.0160187) (0.001,0.0160214)} \closedcycle;
  \addplot[draw=none, fill=htFSAM, fill opacity=0.06] coordinates {(0.001,0.0162796) (0.002,0.0162771) (0.005,0.0162698) (0.01,0.0162576) (0.02,0.0162334) (0.05,0.0161615) (0.1,0.016029) (0.2,0.0157476) (0.5,0.0151231) (1,0.0141754) (1,0.0129954) (0.5,0.0141491) (0.2,0.0148304) (0.1,0.0150822) (0.05,0.0152157) (0.02,0.0152982) (0.01,0.0153261) (0.005,0.0153401) (0.002,0.0153486) (0.001,0.0153514)} \closedcycle;
  \addplot[draw=none, fill=htSAM, fill opacity=0.06] coordinates {(0.001,0.0162799) (0.002,0.0162778) (0.005,0.0162713) (0.01,0.0162607) (0.02,0.0162394) (0.05,0.0161767) (0.1,0.0160636) (0.2,0.0158267) (0.5,0.0152965) (1,0.0144311) (1,0.0133477) (0.5,0.0143446) (0.2,0.0149091) (0.1,0.0151216) (0.05,0.0152353) (0.02,0.015306) (0.01,0.01533) (0.005,0.0153421) (0.002,0.0153494) (0.001,0.0153518)} \closedcycle;
  \addplot[draw=none, fill=htMSAM, fill opacity=0.06] coordinates {(0.001,0.0172536) (0.002,0.0172528) (0.005,0.0172506) (0.01,0.0172469) (0.02,0.0172395) (0.05,0.0172184) (0.1,0.0171845) (0.2,0.0171163) (0.5,0.0168042) (1,0.0162397) (1,0.0149225) (0.5,0.0154294) (0.2,0.0157262) (0.1,0.0158284) (0.05,0.0158821) (0.02,0.0159157) (0.01,0.0159272) (0.005,0.015933) (0.002,0.0159365) (0.001,0.0159376)} \closedcycle;
  \addplot[draw=none, fill=htSAMAdam, fill opacity=0.06] coordinates {(0.001,0.0373617) (0.002,0.0373607) (0.005,0.0373582) (0.01,0.0373546) (0.02,0.0373513) (0.05,0.0373156) (0.1,0.0372598) (0.2,0.0371733) (0.5,0.0369717) (1,0.0368929) (1,0.0357076) (0.5,0.035515) (0.2,0.0356412) (0.1,0.0357413) (0.05,0.0357949) (0.02,0.0358142) (0.01,0.0358157) (0.005,0.0358204) (0.002,0.0358214) (0.001,0.0358223)} \closedcycle;
  \addplot[draw=none, fill=htFSAMAdam, fill opacity=0.06] coordinates {(0.001,0.0373619) (0.002,0.0373611) (0.005,0.0373592) (0.01,0.0373579) (0.02,0.0373554) (0.05,0.0373254) (0.1,0.0372783) (0.2,0.0372059) (0.5,0.0370101) (1,0.0368915) (1,0.0356428) (0.5,0.0355381) (0.2,0.0356674) (0.1,0.0357575) (0.05,0.0358037) (0.02,0.035816) (0.01,0.0358176) (0.005,0.0358209) (0.002,0.0358217) (0.001,0.0358225)} \closedcycle;
  \addplot[color=htMuon, dashed, line width=0.7pt, forget plot] coordinates {(0.001,0.0165429) (1,0.0165429)};
  \addplot[color=htAdam, dash dot, line width=0.7pt, forget plot] coordinates {(0.001,0.0364153) (1,0.0364153)};
  \addplot[color=htSpec, solid, line width=0.8pt] coordinates {(0.001,0.0165308) (0.002,0.0165188) (0.005,0.0164828) (0.01,0.0164229) (0.02,0.0163034) (0.05,0.0159568) (0.1,0.0154729) (0.2,0.0146907) (0.5,0.0140846) (1,0.256185)};
  \addplot[color=htRand, dashed, line width=0.8pt] coordinates {(0.001,0.0165429) (0.002,0.0165429) (0.005,0.0165431) (0.01,0.0165432) (0.02,0.0165436) (0.05,0.0165442) (0.1,0.016544) (0.2,0.0165391) (0.5,0.0164859) (1,0.016232)};
  \addplot[color=htMSOMA, loosely dashed, line width=0.8pt] coordinates {(0.001,0.0165401) (0.002,0.0165373) (0.005,0.016529) (0.01,0.0165151) (0.02,0.0164875) (0.05,0.016406) (0.1,0.0162735) (0.2,0.0160243) (0.5,0.015414) (1,0.0146505)};
  \addplot[color=htFSAM, solid, line width=0.8pt] coordinates {(0.001,0.0157295) (0.002,0.0157267) (0.005,0.0157183) (0.01,0.0157043) (0.02,0.0156763) (0.05,0.015593) (0.1,0.0154558) (0.2,0.0151867) (0.5,0.014593) (1,0.0134323)};
  \addplot[color=htSAM, solid, line width=0.8pt] coordinates {(0.001,0.0157299) (0.002,0.0157275) (0.005,0.0157202) (0.01,0.0157082) (0.02,0.0156842) (0.05,0.0156126) (0.1,0.0154952) (0.2,0.0152654) (0.5,0.0147438) (1,0.013748)};
  \addplot[color=htMSAM, solid, line width=0.8pt] coordinates {(0.001,0.0166745) (0.002,0.0166733) (0.005,0.0166697) (0.01,0.0166637) (0.02,0.0166517) (0.05,0.0166158) (0.1,0.0165561) (0.2,0.0164345) (0.5,0.0160488) (1,0.0153953)};
  \addplot[color=htSOMA, solid, line width=1.1pt] coordinates {(0.001,0.0165413) (0.002,0.0165397) (0.005,0.0165345) (0.01,0.0165249) (0.02,0.0165027) (0.05,0.0164115) (0.1,0.0162952) (0.2,0.0163726) (0.5,0.0262902) (1,0.132169)};
  \addplot[color=htFPSOMA, solid, line width=1.1pt] coordinates {(0.001,0.016529) (0.002,0.0165152) (0.005,0.0164736) (0.01,0.0164046) (0.02,0.0162678) (0.05,0.015875) (0.1,0.015274) (0.2,0.0142759) (0.5,0.0144695) (1,0.0514629)};
  \addplot[color=htPreNS5, densely dotted, line width=1.1pt] coordinates {(0.001,0.0165391) (0.002,0.0165354) (0.005,0.0165237) (0.01,0.0165034) (0.02,0.0164599) (0.05,0.0163142) (0.1,0.0160937) (0.2,0.0159092) (0.5,0.025793) (1,0.12765)};
  \addplot[color=htPostNS5, dash dot, line width=1.1pt] coordinates {(0.001,0.0165381) (0.002,0.0165332) (0.005,0.0165187) (0.01,0.0164941) (0.02,0.0164448) (0.05,0.0162896) (0.1,0.016028) (0.2,0.0156536) (0.5,0.0171922) (1,0.0524615)};
  \addplot[color=htSAMAdam, solid, line width=1.1pt] coordinates {(0.001,0.0364144) (0.002,0.0364136) (0.005,0.0364112) (0.01,0.0364077) (0.02,0.0364073) (0.05,0.0363768) (0.1,0.0363342) (0.2,0.0362638) (0.5,0.0361067) (1,0.0362158)};
  \addplot[color=htFSAMAdam, densely dashed, line width=1.4pt] coordinates {(0.001,0.0364146) (0.002,0.036414) (0.005,0.0364121) (0.01,0.0364087) (0.02,0.0364109) (0.05,0.0363854) (0.1,0.0363502) (0.2,0.0362897) (0.5,0.0361308) (1,0.0361286)};
\end{axis}
\begin{axis}[hide axis, scale only axis, at={(ax0_r0.north west)}, anchor=south west, yshift=34pt,
    width=0.9300\textwidth, height=1pt, xmin=0, xmax=1, ymin=0, ymax=1,
    legend columns=7, legend cell align=left,
    legend style={at={(0.5,0)}, anchor=south, draw=none, fill=none, font=\scriptsize, inner sep=1pt,
      /tikz/every even column/.append style={column sep=5pt}}]
  \addlegendimage{color=htSOMA, solid, line width=1.1pt}\addlegendentry{\textsc{Soma}}
  \addlegendimage{color=htFPSOMA, solid, line width=1.1pt}\addlegendentry{\textsc{FP{-}Soma}}
  \addlegendimage{color=htMSOMA, loosely dashed, line width=0.8pt}\addlegendentry{\textsc{Msoma}}
  \addlegendimage{color=htPreNS5, densely dotted, line width=1.1pt}\addlegendentry{\textsc{Opspre}}
  \addlegendimage{color=htPostNS5, dash dot, line width=1.1pt}\addlegendentry{\textsc{Opspost}}
  \addlegendimage{color=htSpec, solid, line width=0.8pt}\addlegendentry{\textsc{Ssm}}
  \addlegendimage{color=htRand, dashed, line width=0.8pt}\addlegendentry{\textsc{Rsm}}
  \addlegendimage{color=htSAM, solid, line width=0.8pt}\addlegendentry{\textsc{Sam}}
  \addlegendimage{color=htFSAM, solid, line width=0.8pt}\addlegendentry{\textsc{F{-}Sam}}
  \addlegendimage{color=htMSAM, solid, line width=0.8pt}\addlegendentry{\textsc{Msam}}
  \addlegendimage{color=htSAMAdam, solid, line width=1.1pt}\addlegendentry{\textsc{Saw}}
  \addlegendimage{color=htFSAMAdam, densely dashed, line width=1.4pt}\addlegendentry{\textsc{Fsaw}}
  \addlegendimage{color=htMuon, dashed, line width=0.7pt}\addlegendentry{\textsc{Muon}}
  \addlegendimage{color=htAdam, dash dot, line width=0.7pt}\addlegendentry{\textsc{AdamW}}
\end{axis}
\path (ax0_r0.south west) -- (ax3_r0.south east) coordinate[midway] (xlab);
\node[font=\small, anchor=north] at ([yshift=-14pt]xlab) {perturbation radius $\rho$};
\end{tikzpicture}

%% file: dataforsyntheticplots/mse_aniso_gap.tex
% ---- mse_aniso_gap
\providecolor{muonrow}{RGB}{208,236,214}\providecolor{adamrow}{RGB}{255,214,214}\providecolor{sgdrow}{RGB}{208,226,247}
\begin{table*}[!t]\centering\footnotesize\setlength{\tabcolsep}{5pt}
\caption{Optimality gap: best value over the perturbation radius $\rho$ and the radius $\rho^\star$ where it is attained (median of 8 seeds; Frobenius-ball methods pool the global and per-layer radius modes). Row colour marks the base optimizer family: \muon-based (\colorbox{muonrow}{Green}), \adamw-based (\colorbox{adamrow}{Red}), \sgd-based (\colorbox{sgdrow}{Blue}); baselines have no $\rho^\star$. Best value per $\alpha$ in bold.}\label{tab:mse_aniso_gap}
\vspace{2pt}
\resizebox{\ifdim\width>\textwidth\textwidth\else\width\fi}{!}{%
\begin{tabular}{lcccccccc}\toprule
Algorithm & \multicolumn{2}{c}{$\alpha=1.1$} & \multicolumn{2}{c}{$\alpha=1.6$} & \multicolumn{2}{c}{$\alpha=2.0$} & \multicolumn{2}{c}{$\alpha=3.0$} \\
& gap & $\rho^\star$ & gap & $\rho^\star$ & gap & $\rho^\star$ & gap & $\rho^\star$ \\\midrule
\rowcolor{muonrow} \muon & 0.413 & -- & 0.062 & -- & 0.0316 & -- & 0.0165 & -- \\
\rowcolor{muonrow} \specsammuon & 0.265 & 1 & 0.0473 & 0.5 & 0.0253 & 0.5 & 0.0141 & 0.5 \\
\rowcolor{muonrow} \fpsoma & \textbf{0.252} & 1 & \textbf{0.0457} & 0.5 & \textbf{0.0247} & 0.5 & 0.0143 & 0.2 \\
\rowcolor{muonrow} \soma & 0.263 & 1 & 0.0551 & 0.5 & 0.0299 & 0.2 & 0.0163 & 0.1 \\
\rowcolor{muonrow} \randsammuon & 0.408 & 1 & 0.0609 & 1 & 0.0311 & 1 & 0.0162 & 1 \\
\rowcolor{muonrow} \opsomapre & 0.274 & 1 & 0.0548 & 0.5 & 0.0296 & 0.2 & 0.0159 & 0.2 \\
\rowcolor{muonrow} \opsomapost & 0.287 & 1 & 0.0537 & 0.5 & 0.0284 & 0.5 & 0.0157 & 0.2 \\
\rowcolor{muonrow} \textsc{Msoma} & 0.389 & 1 & 0.0576 & 1 & 0.0283 & 1 & 0.0147 & 1 \\
\rowcolor{adamrow} \adamw & 0.559 & -- & 0.109 & -- & 0.0592 & -- & 0.0364 & -- \\
\rowcolor{adamrow} \samadamw & 0.554 & 1 & 0.107 & 1 & 0.0582 & 1 & 0.0361 & 0.5 \\
\rowcolor{adamrow} \fsamadamw & 0.556 & 1 & 0.108 & 1 & 0.0584 & 1 & 0.0361 & 1 \\
\rowcolor{sgdrow} \sgd & 4.64 & -- & 0.136 & -- & 0.0364 & -- & 0.0157 & -- \\
\rowcolor{sgdrow} \sam & 3.52 & 1 & 0.111 & 1 & 0.0353 & 0.2 & 0.0137 & 1 \\
\rowcolor{sgdrow} \fsam & 3.48 & 1 & 0.113 & 1 & 0.0352 & 0.2 & \textbf{0.0134} & 1 \\
\rowcolor{sgdrow} \textsc{Msam} & 5.69 & 1 & 0.134 & 0.5 & 0.0352 & 1 & 0.0154 & 1 \\
\bottomrule\end{tabular}}\end{table*}

%% file: tikz_for_heavytailed/appendex_isotropic/val_acc_vs_rho.tikz
%% Auto-generated by make_heavytailed_figures.py  --  do not edit by hand, re-run the script.
%% Values: line = median over seeds, band = interquartile range (25th-75th percentile).
%% Preamble needed:  \usepackage{pgfplots}  \pgfplotsset{compat=1.16}   (nothing else)
%% Width follows \textwidth, fonts follow the document (so the legend can never overflow).

\definecolor{htSOMA}{HTML}{0B5D1E}
\definecolor{htFPSOMA}{HTML}{33CC33}
\definecolor{htSpec}{HTML}{00B3A4}
\definecolor{htRand}{HTML}{808000}
\definecolor{htFSAM}{HTML}{2F80D8}
\definecolor{htSAM}{HTML}{0A3A8A}
\definecolor{htPreNS5}{HTML}{9ACD32}
\definecolor{htPostNS5}{HTML}{004D40}
\definecolor{htSAMAdam}{HTML}{8E0F1C}
\definecolor{htFSAMAdam}{HTML}{F4623A}
\definecolor{htMSAM}{HTML}{7DB9F2}
\definecolor{htMSOMA}{HTML}{8FD694}
\definecolor{htMuon}{HTML}{111111}
\definecolor{htAdam}{HTML}{8A8A8A}
\pgfplotsset{
  ht panel/.style={
    scale only axis, width=0.1920\textwidth, height=1.35in,
    xmin=0.00075, xmax=1.35,
    xtick={0.001,0.01,0.1,1}, xticklabels={$10^{-3}$,$10^{-2}$,$10^{-1}$,$10^{0}$},
    minor xtick={0.002,0.003,0.004,0.005,0.006,0.007,0.008,0.009,0.02,0.03,0.04,0.05,0.06,0.07,0.08,0.09,0.2,0.3,0.4,0.5,0.6,0.7,0.8,0.9},
    xmajorgrids=true, ymajorgrids=true,
    grid style={line width=0.25pt, draw=black!10},
    axis line style={line width=0.4pt, black!75},
    tick style={line width=0.4pt, black!75}, major tick length=3pt, minor tick length=1.5pt,
    tick label style={font=\footnotesize, /pgf/number format/assume math mode=true},
    label style={font=\small}, title style={font=\small, yshift=-3pt},
    ylabel style={yshift=-1pt, align=center},
    every tick label/.append style={inner sep=1pt},
    clip=true,
  }}

\begin{tikzpicture}
\begin{axis}[name=ax0_r0,
    at={(0.0820\textwidth,0in)},
    anchor=south west,
    ht panel, xmode=log,
    ymin=33.957, ymax=41.4336,
    ytick={34,36,38,40},
    yticklabels={34,36,38,40},
    title={$\alpha=1.1$},
    ylabel={validation accuracy}]
  \addplot[draw=none, fill=htSOMA, fill opacity=0.06] coordinates {(0.001,36.8164) (0.002,36.6699) (0.005,36.7188) (0.01,36.7676) (0.02,36.6211) (0.05,37.207) (0.1,37.207) (0.2,37.793) (0.5,38.2324) (1,39.1113) (1,37.3047) (0.5,36.7676) (0.2,35.4004) (0.1,35.6445) (0.05,34.9121) (0.02,34.7168) (0.01,34.7168) (0.005,34.7168) (0.002,34.5703) (0.001,34.5215)} \closedcycle;
  \addplot[draw=none, fill=htFPSOMA, fill opacity=0.06] coordinates {(0.001,36.6699) (0.002,37.0117) (0.005,36.8164) (0.01,36.8164) (0.02,36.9629) (0.05,37.207) (0.1,37.1582) (0.2,37.2559) (0.5,37.4023) (1,38.2324) (1,37.5) (0.5,36.5234) (0.2,35.5957) (0.1,35.1562) (0.05,35.0098) (0.02,34.668) (0.01,34.8633) (0.005,34.7656) (0.002,34.8145) (0.001,34.6191)} \closedcycle;
  \addplot[draw=none, fill=htSpec, fill opacity=0.06] coordinates {(0.001,36.8164) (0.002,36.9141) (0.005,36.6211) (0.01,36.9629) (0.02,36.8652) (0.05,37.0117) (0.1,37.2559) (0.2,37.4512) (0.5,37.207) (1,37.3047) (1,35.9375) (0.5,36.3281) (0.2,35.1562) (0.1,35.3516) (0.05,34.7656) (0.02,34.9121) (0.01,34.8145) (0.005,34.7656) (0.002,34.668) (0.001,34.6191)} \closedcycle;
  \addplot[draw=none, fill=htRand, fill opacity=0.06] coordinates {(0.001,36.7676) (0.002,36.8164) (0.005,36.7676) (0.01,36.8164) (0.02,36.9629) (0.05,36.9629) (0.1,36.9141) (0.2,36.7676) (0.5,36.6211) (1,36.4258) (1,34.668) (0.5,34.7168) (0.2,34.7168) (0.1,34.5215) (0.05,34.5703) (0.02,34.5215) (0.01,34.5215) (0.005,34.668) (0.002,34.668) (0.001,34.5215)} \closedcycle;
  \addplot[draw=none, fill=htPreNS5, fill opacity=0.06] coordinates {(0.001,36.8164) (0.002,36.6699) (0.005,36.6699) (0.01,36.9629) (0.02,36.7676) (0.05,36.9141) (0.1,36.9141) (0.2,37.3535) (0.5,37.5977) (1,39.1113) (1,36.7188) (0.5,35.791) (0.2,35.2539) (0.1,34.7656) (0.05,34.7656) (0.02,34.8145) (0.01,34.668) (0.005,34.668) (0.002,34.668) (0.001,34.668)} \closedcycle;
  \addplot[draw=none, fill=htPostNS5, fill opacity=0.06] coordinates {(0.001,36.8164) (0.002,36.7676) (0.005,36.6699) (0.01,36.7188) (0.02,36.6699) (0.05,37.1582) (0.1,36.9629) (0.2,37.3535) (0.5,37.8418) (1,38.5742) (1,37.0117) (0.5,35.9863) (0.2,35.0098) (0.1,35.1074) (0.05,34.9121) (0.02,34.9609) (0.01,34.7168) (0.005,34.7168) (0.002,34.668) (0.001,34.8145)} \closedcycle;
  \addplot[draw=none, fill=htMSOMA, fill opacity=0.06] coordinates {(0.001,36.7676) (0.002,36.7676) (0.005,36.7188) (0.01,36.7188) (0.02,36.7676) (0.05,36.7676) (0.1,36.7676) (0.2,36.9141) (0.5,36.8652) (1,36.9629) (1,35.2539) (0.5,34.8145) (0.2,34.7656) (0.1,34.8633) (0.05,34.7168) (0.02,34.5215) (0.01,34.668) (0.005,34.5215) (0.002,34.9609) (0.001,34.8145)} \closedcycle;
  \addplot[draw=none, fill=htFSAM, fill opacity=0.06] coordinates {(0.001,35.5957) (0.002,36.7676) (0.005,36.377) (0.01,36.8652) (0.02,36.5234) (0.05,36.2793) (0.1,36.1816) (0.2,36.5234) (0.5,37.4512) (1,35.9863) (1,33.9844) (0.5,34.3262) (0.2,33.7891) (0.1,34.3262) (0.05,33.5449) (0.02,32.9102) (0.01,33.6914) (0.005,33.4961) (0.002,33.3008) (0.001,34.375)} \closedcycle;
  \addplot[draw=none, fill=htSAM, fill opacity=0.06] coordinates {(0.001,36.1328) (0.002,36.5723) (0.005,36.2305) (0.01,35.6445) (0.02,36.2793) (0.05,37.207) (0.1,35.5957) (0.2,35.7422) (0.5,36.8164) (1,36.5234) (1,33.8867) (0.5,33.7891) (0.2,32.7148) (0.1,33.5938) (0.05,34.1309) (0.02,33.8867) (0.01,33.9355) (0.005,34.668) (0.002,34.9609) (0.001,34.7168)} \closedcycle;
  \addplot[draw=none, fill=htMSAM, fill opacity=0.06] coordinates {(0.001,36.6699) (0.002,37.207) (0.005,36.7676) (0.01,37.1094) (0.02,37.0605) (0.05,36.8652) (0.1,36.7676) (0.2,36.5723) (0.5,36.084) (1,37.1582) (1,33.7891) (0.5,33.6914) (0.2,33.6426) (0.1,34.1797) (0.05,33.7402) (0.02,34.2773) (0.01,33.7891) (0.005,33.3496) (0.002,35.1562) (0.001,33.252)} \closedcycle;
  \addplot[draw=none, fill=htSAMAdam, fill opacity=0.06] coordinates {(0.001,41.1133) (0.002,41.1133) (0.005,41.1133) (0.01,41.1133) (0.02,41.1133) (0.05,41.1133) (0.1,41.0645) (0.2,41.2109) (0.5,41.2598) (1,41.1133) (1,40.1855) (0.5,40.332) (0.2,40.4297) (0.1,40.332) (0.05,40.332) (0.02,40.2832) (0.01,40.2832) (0.005,40.2832) (0.002,40.2832) (0.001,40.1367)} \closedcycle;
  \addplot[draw=none, fill=htFSAMAdam, fill opacity=0.06] coordinates {(0.001,41.2598) (0.002,41.2598) (0.005,41.1133) (0.01,41.1133) (0.02,41.1133) (0.05,41.1133) (0.1,41.1133) (0.2,41.2109) (0.5,41.2109) (1,41.0156) (1,40.0391) (0.5,40.1855) (0.2,40.2832) (0.1,40.332) (0.05,40.332) (0.02,40.4785) (0.01,40.2832) (0.005,40.1367) (0.002,40.2832) (0.001,40.2832)} \closedcycle;
  \addplot[color=htMuon, dashed, line width=0.7pt, forget plot] coordinates {(0.001,36.1328) (1,36.1328)};
  \addplot[color=htAdam, dash dot, line width=0.7pt, forget plot] coordinates {(0.001,40.625) (1,40.625)};
  \addplot[color=htSpec, solid, line width=0.8pt] coordinates {(0.001,36.1328) (0.002,36.1328) (0.005,36.1328) (0.01,36.0352) (0.02,36.0352) (0.05,36.2305) (0.1,36.0352) (0.2,36.6211) (0.5,36.5234) (1,36.4258)};
  \addplot[color=htRand, dashed, line width=0.8pt] coordinates {(0.001,35.9375) (0.002,35.9375) (0.005,36.0352) (0.01,36.0352) (0.02,36.1328) (0.05,36.3281) (0.1,36.0352) (0.2,36.0352) (0.5,36.1328) (1,36.0352)};
  \addplot[color=htMSOMA, loosely dashed, line width=0.8pt] coordinates {(0.001,36.0352) (0.002,36.0352) (0.005,35.9375) (0.01,36.0352) (0.02,36.0352) (0.05,36.0352) (0.1,36.2305) (0.2,36.2305) (0.5,36.1328) (1,36.0352)};
  \addplot[color=htFSAM, solid, line width=0.8pt] coordinates {(0.001,35.1562) (0.002,35.1562) (0.005,35.1562) (0.01,35.3516) (0.02,34.7656) (0.05,35.3516) (0.1,34.8633) (0.2,35.3516) (0.5,35.8398) (1,35.2539)};
  \addplot[color=htSAM, solid, line width=0.8pt] coordinates {(0.001,35.1562) (0.002,35.1562) (0.005,35.5469) (0.01,35.1562) (0.02,34.4727) (0.05,35.0586) (0.1,34.4727) (0.2,34.4727) (0.5,35.6445) (1,35.6445)};
  \addplot[color=htMSAM, solid, line width=0.8pt] coordinates {(0.001,34.5703) (0.002,36.1328) (0.005,34.4727) (0.01,34.7656) (0.02,35.0586) (0.05,34.7656) (0.1,35.8398) (0.2,36.1328) (0.5,34.4727) (1,35.3516)};
  \addplot[color=htSOMA, solid, line width=1.1pt] coordinates {(0.001,36.0352) (0.002,36.2305) (0.005,36.2305) (0.01,36.4258) (0.02,36.0352) (0.05,36.1328) (0.1,36.0352) (0.2,36.2305) (0.5,37.4023) (1,38.1836)};
  \addplot[color=htFPSOMA, solid, line width=1.1pt] coordinates {(0.001,36.2305) (0.002,36.3281) (0.005,36.3281) (0.01,36.3281) (0.02,36.7188) (0.05,36.1328) (0.1,36.0352) (0.2,36.5234) (0.5,37.1094) (1,37.793)};
  \addplot[color=htPreNS5, densely dotted, line width=1.1pt] coordinates {(0.001,36.0352) (0.002,36.2305) (0.005,36.3281) (0.01,36.3281) (0.02,36.1328) (0.05,36.1328) (0.1,36.1328) (0.2,35.9375) (0.5,37.1094) (1,38.5742)};
  \addplot[color=htPostNS5, dash dot, line width=1.1pt] coordinates {(0.001,36.1328) (0.002,36.2305) (0.005,36.1328) (0.01,36.2305) (0.02,36.1328) (0.05,36.1328) (0.1,35.9375) (0.2,35.9375) (0.5,36.5234) (1,37.8906)};
  \addplot[color=htSAMAdam, solid, line width=1.1pt] coordinates {(0.001,40.5273) (0.002,40.625) (0.005,40.625) (0.01,40.7227) (0.02,40.8203) (0.05,40.8203) (0.1,40.7227) (0.2,40.7227) (0.5,40.8203) (1,40.7227)};
  \addplot[color=htFSAMAdam, densely dashed, line width=1.4pt] coordinates {(0.001,40.5273) (0.002,40.625) (0.005,40.625) (0.01,40.7227) (0.02,40.8203) (0.05,40.8203) (0.1,40.8203) (0.2,40.918) (0.5,40.8203) (1,40.625)};
\end{axis}
\begin{axis}[name=ax1_r0,
    at={(0.3235\textwidth,0in)},
    anchor=south west,
    ht panel, xmode=log,
    ymin=37.0938, ymax=45.8164,
    ytick={38,40,42,44},
    yticklabels={38,40,42,44},
    title={$\alpha=1.6$}]
  \addplot[draw=none, fill=htSOMA, fill opacity=0.06] coordinates {(0.001,43.5547) (0.002,43.8965) (0.005,44.3848) (0.01,43.7988) (0.02,44.2383) (0.05,44.1406) (0.1,43.9941) (0.2,45.3613) (0.5,45.3125) (1,45.3125) (1,44.5312) (0.5,43.3594) (0.2,42.6758) (0.1,42.5293) (0.05,42.2852) (0.02,42.3828) (0.01,42.3828) (0.005,42.1387) (0.002,42.6758) (0.001,42.2852)} \closedcycle;
  \addplot[draw=none, fill=htFPSOMA, fill opacity=0.06] coordinates {(0.001,43.8965) (0.002,43.9941) (0.005,44.1406) (0.01,44.1895) (0.02,43.6035) (0.05,44.5312) (0.1,43.9941) (0.2,44.1895) (0.5,45.6543) (1,45.4102) (1,44.043) (0.5,43.5059) (0.2,42.334) (0.1,42.5781) (0.05,42.5293) (0.02,42.334) (0.01,42.7734) (0.005,42.334) (0.002,41.9434) (0.001,42.3828)} \closedcycle;
  \addplot[draw=none, fill=htSpec, fill opacity=0.06] coordinates {(0.001,44.043) (0.002,44.043) (0.005,43.9941) (0.01,43.9941) (0.02,43.9453) (0.05,44.7754) (0.1,44.1895) (0.2,44.5801) (0.5,44.6777) (1,43.2617) (1,42.041) (0.5,43.2129) (0.2,43.3594) (0.1,42.9199) (0.05,42.3828) (0.02,42.7734) (0.01,42.4805) (0.005,42.2852) (0.002,42.1387) (0.001,42.2363)} \closedcycle;
  \addplot[draw=none, fill=htRand, fill opacity=0.06] coordinates {(0.001,43.8477) (0.002,43.9941) (0.005,44.1895) (0.01,44.1895) (0.02,44.3359) (0.05,43.9941) (0.1,44.1406) (0.2,43.9453) (0.5,43.4082) (1,43.6523) (1,42.9199) (0.5,42.1875) (0.2,42.6758) (0.1,42.1875) (0.05,42.1875) (0.02,42.3828) (0.01,42.2852) (0.005,42.3828) (0.002,42.1387) (0.001,42.1875)} \closedcycle;
  \addplot[draw=none, fill=htPreNS5, fill opacity=0.06] coordinates {(0.001,44.1895) (0.002,44.5801) (0.005,43.6035) (0.01,43.6523) (0.02,43.7988) (0.05,43.9941) (0.1,43.9453) (0.2,44.1406) (0.5,44.3848) (1,45.3613) (1,43.8965) (0.5,42.9199) (0.2,43.1641) (0.1,42.4805) (0.05,42.627) (0.02,42.2852) (0.01,42.0898) (0.005,42.4316) (0.002,42.334) (0.001,41.9922)} \closedcycle;
  \addplot[draw=none, fill=htPostNS5, fill opacity=0.06] coordinates {(0.001,44.0918) (0.002,43.7988) (0.005,43.7012) (0.01,43.9941) (0.02,43.6523) (0.05,43.6035) (0.1,43.8965) (0.2,44.3359) (0.5,44.7266) (1,45.9961) (1,44.3848) (0.5,43.1641) (0.2,42.7734) (0.1,42.5781) (0.05,42.334) (0.02,42.5781) (0.01,42.1875) (0.005,42.1387) (0.002,41.7969) (0.001,41.7969)} \closedcycle;
  \addplot[draw=none, fill=htMSOMA, fill opacity=0.06] coordinates {(0.001,44.043) (0.002,43.8477) (0.005,44.043) (0.01,44.1406) (0.02,43.8965) (0.05,43.7988) (0.1,43.6523) (0.2,43.4082) (0.5,43.9453) (1,43.3105) (1,41.6992) (0.5,42.3828) (0.2,42.1875) (0.1,42.3828) (0.05,42.1387) (0.02,41.7969) (0.01,41.9434) (0.005,41.9922) (0.002,42.1875) (0.001,41.9922)} \closedcycle;
  \addplot[draw=none, fill=htFSAM, fill opacity=0.06] coordinates {(0.001,40.7715) (0.002,39.3555) (0.005,40.332) (0.01,40.8691) (0.02,39.8438) (0.05,39.502) (0.1,39.9902) (0.2,40.2344) (0.5,40.918) (1,39.5508) (1,37.5) (0.5,38.6719) (0.2,38.4277) (0.1,37.1094) (0.05,37.793) (0.02,37.8906) (0.01,37.2559) (0.005,38.6719) (0.002,38.0371) (0.001,38.0371)} \closedcycle;
  \addplot[draw=none, fill=htSAM, fill opacity=0.06] coordinates {(0.001,40.4297) (0.002,39.8438) (0.005,39.4531) (0.01,39.7461) (0.02,39.7461) (0.05,39.7949) (0.1,40.332) (0.2,40.2832) (0.5,40.3809) (1,40.2832) (1,36.7188) (0.5,37.2559) (0.2,38.7695) (0.1,38.1836) (0.05,37.3047) (0.02,37.9883) (0.01,37.0605) (0.005,37.1094) (0.002,37.6465) (0.001,38.0371)} \closedcycle;
  \addplot[draw=none, fill=htMSAM, fill opacity=0.06] coordinates {(0.001,39.0137) (0.002,40.1855) (0.005,40.0879) (0.01,40.0879) (0.02,41.5527) (0.05,39.2578) (0.1,41.6016) (0.2,39.8438) (0.5,39.6973) (1,40.0391) (1,37.3047) (0.5,37.7441) (0.2,36.377) (0.1,37.207) (0.05,37.5488) (0.02,37.2559) (0.01,36.8164) (0.005,37.4512) (0.002,36.084) (0.001,36.9141)} \closedcycle;
  \addplot[draw=none, fill=htSAMAdam, fill opacity=0.06] coordinates {(0.001,43.6035) (0.002,43.6035) (0.005,43.6035) (0.01,43.5547) (0.02,43.6035) (0.05,43.6035) (0.1,43.5547) (0.2,43.9453) (0.5,43.75) (1,43.7988) (1,42.7246) (0.5,42.3828) (0.2,42.1875) (0.1,42.1875) (0.05,42.334) (0.02,42.1875) (0.01,42.1875) (0.005,42.1875) (0.002,42.334) (0.001,42.334)} \closedcycle;
  \addplot[draw=none, fill=htFSAMAdam, fill opacity=0.06] coordinates {(0.001,43.6035) (0.002,43.6035) (0.005,43.5547) (0.01,43.4082) (0.02,43.3594) (0.05,43.5547) (0.1,43.6035) (0.2,43.9453) (0.5,43.7988) (1,43.6035) (1,42.5781) (0.5,42.334) (0.2,42.334) (0.1,42.3828) (0.05,42.3828) (0.02,42.334) (0.01,42.1875) (0.005,42.1875) (0.002,42.1875) (0.001,42.334)} \closedcycle;
  \addplot[color=htMuon, dashed, line width=0.7pt, forget plot] coordinates {(0.001,43.2617) (1,43.2617)};
  \addplot[color=htAdam, dash dot, line width=0.7pt, forget plot] coordinates {(0.001,42.9688) (1,42.9688)};
  \addplot[color=htSpec, solid, line width=0.8pt] coordinates {(0.001,42.9688) (0.002,43.0664) (0.005,43.0664) (0.01,43.0664) (0.02,43.3594) (0.05,43.5547) (0.1,43.75) (0.2,43.5547) (0.5,44.2383) (1,42.5781)};
  \addplot[color=htRand, dashed, line width=0.8pt] coordinates {(0.001,42.5781) (0.002,42.7734) (0.005,42.5781) (0.01,42.9688) (0.02,42.8711) (0.05,42.8711) (0.1,43.5547) (0.2,43.1641) (0.5,43.1641) (1,43.2617)};
  \addplot[color=htMSOMA, loosely dashed, line width=0.8pt] coordinates {(0.001,43.3594) (0.002,43.457) (0.005,42.8711) (0.01,42.8711) (0.02,42.7734) (0.05,42.8711) (0.1,43.457) (0.2,42.8711) (0.5,42.5781) (1,42.7734)};
  \addplot[color=htFSAM, solid, line width=0.8pt] coordinates {(0.001,39.3555) (0.002,38.6719) (0.005,39.5508) (0.01,37.8906) (0.02,39.4531) (0.05,38.6719) (0.1,38.0859) (0.2,39.1602) (0.5,40.0391) (1,38.4766)};
  \addplot[color=htSAM, solid, line width=0.8pt] coordinates {(0.001,39.2578) (0.002,38.9648) (0.005,37.6953) (0.01,38.7695) (0.02,38.4766) (0.05,38.8672) (0.1,39.5508) (0.2,39.3555) (0.5,39.2578) (1,38.4766)};
  \addplot[color=htMSAM, solid, line width=0.8pt] coordinates {(0.001,38.2812) (0.002,38.3789) (0.005,39.1602) (0.01,37.9883) (0.02,38.9648) (0.05,38.4766) (0.1,39.2578) (0.2,39.1602) (0.5,39.1602) (1,38.1836)};
  \addplot[color=htSOMA, solid, line width=1.1pt] coordinates {(0.001,42.7734) (0.002,43.3594) (0.005,42.9688) (0.01,43.1641) (0.02,43.1641) (0.05,43.2617) (0.1,43.1641) (0.2,43.6523) (0.5,44.4336) (1,44.8242)};
  \addplot[color=htFPSOMA, solid, line width=1.1pt] coordinates {(0.001,43.1641) (0.002,42.8711) (0.005,43.1641) (0.01,43.457) (0.02,42.9688) (0.05,43.3594) (0.1,43.1641) (0.2,43.0664) (0.5,44.2383) (1,44.7266)};
  \addplot[color=htPreNS5, densely dotted, line width=1.1pt] coordinates {(0.001,42.7734) (0.002,43.3594) (0.005,43.5547) (0.01,43.0664) (0.02,43.3594) (0.05,43.5547) (0.1,42.6758) (0.2,43.5547) (0.5,43.5547) (1,44.2383)};
  \addplot[color=htPostNS5, dash dot, line width=1.1pt] coordinates {(0.001,42.7734) (0.002,42.6758) (0.005,43.0664) (0.01,42.7734) (0.02,43.2617) (0.05,43.2617) (0.1,42.9688) (0.2,43.0664) (0.5,43.8477) (1,45.2148)};
  \addplot[color=htSAMAdam, solid, line width=1.1pt] coordinates {(0.001,43.0664) (0.002,42.9688) (0.005,43.0664) (0.01,43.1641) (0.02,43.0664) (0.05,43.1641) (0.1,43.3594) (0.2,43.457) (0.5,43.3594) (1,43.2617)};
  \addplot[color=htFSAMAdam, densely dashed, line width=1.4pt] coordinates {(0.001,42.9688) (0.002,42.9688) (0.005,42.9688) (0.01,43.1641) (0.02,43.1641) (0.05,43.3594) (0.1,43.0664) (0.2,42.9688) (0.5,43.1641) (1,43.0664)};
\end{axis}
\begin{axis}[name=ax2_r0,
    at={(0.5650\textwidth,0in)},
    anchor=south west,
    ht panel, xmode=log,
    ymin=37.2852, ymax=47.4805,
    ytick={38,40,42,44,46},
    yticklabels={38,40,42,44,46},
    title={$\alpha=2.0$}]
  \addplot[draw=none, fill=htSOMA, fill opacity=0.06] coordinates {(0.001,45.3125) (0.002,45.2148) (0.005,45.0195) (0.01,45.8496) (0.02,45.8984) (0.05,45.752) (0.1,45.8496) (0.2,46.1426) (0.5,47.7051) (1,46.3379) (1,45.4102) (0.5,45.752) (0.2,45.1172) (0.1,43.8477) (0.05,43.7988) (0.02,44.6777) (0.01,44.2383) (0.005,44.2383) (0.002,44.0918) (0.001,44.3359)} \closedcycle;
  \addplot[draw=none, fill=htFPSOMA, fill opacity=0.06] coordinates {(0.001,45.5078) (0.002,45.3613) (0.005,45.166) (0.01,46.0449) (0.02,45.3613) (0.05,45.7031) (0.1,45.3125) (0.2,45.8008) (0.5,47.2656) (1,43.7012) (1,43.1641) (0.5,45.7031) (0.2,43.6523) (0.1,44.2383) (0.05,44.0918) (0.02,44.5312) (0.01,44.3359) (0.005,44.5312) (0.002,44.0918) (0.001,44.043)} \closedcycle;
  \addplot[draw=none, fill=htSpec, fill opacity=0.06] coordinates {(0.001,45.3613) (0.002,44.9219) (0.005,46.1426) (0.01,45.5078) (0.02,45.5566) (0.05,45.752) (0.1,45.5078) (0.2,45.3125) (0.5,46.2402) (1,39.8926) (1,38.3789) (0.5,44.9707) (0.2,44.2383) (0.1,44.043) (0.05,43.75) (0.02,43.7988) (0.01,44.4824) (0.005,44.2871) (0.002,44.2871) (0.001,44.5801)} \closedcycle;
  \addplot[draw=none, fill=htRand, fill opacity=0.06] coordinates {(0.001,45.6055) (0.002,45.5566) (0.005,45.3613) (0.01,45.1172) (0.02,45.1172) (0.05,45.8008) (0.1,45.8984) (0.2,45.8984) (0.5,45.9473) (1,45.5078) (1,44.1895) (0.5,43.9941) (0.2,44.2871) (0.1,44.3359) (0.05,44.0918) (0.02,44.2871) (0.01,44.3359) (0.005,44.5312) (0.002,44.2383) (0.001,44.4336)} \closedcycle;
  \addplot[draw=none, fill=htPreNS5, fill opacity=0.06] coordinates {(0.001,45.1172) (0.002,45.3125) (0.005,45.166) (0.01,45.5078) (0.02,45.5566) (0.05,45.1172) (0.1,45.6055) (0.2,45.5566) (0.5,47.168) (1,46.2402) (1,45.0684) (0.5,44.6777) (0.2,43.7012) (0.1,44.0918) (0.05,43.9941) (0.02,44.3848) (0.01,44.3359) (0.005,43.75) (0.002,44.1895) (0.001,44.2871)} \closedcycle;
  \addplot[draw=none, fill=htPostNS5, fill opacity=0.06] coordinates {(0.001,45.3125) (0.002,45.3125) (0.005,45.8008) (0.01,45.7031) (0.02,45.3125) (0.05,45.752) (0.1,45.7031) (0.2,46.0449) (0.5,46.2891) (1,48.4375) (1,45.7031) (0.5,45.0684) (0.2,43.7988) (0.1,44.043) (0.05,43.7012) (0.02,44.3848) (0.01,44.1406) (0.005,44.1406) (0.002,44.4824) (0.001,44.4336)} \closedcycle;
  \addplot[draw=none, fill=htMSOMA, fill opacity=0.06] coordinates {(0.001,45.3125) (0.002,45.7031) (0.005,45.5566) (0.01,45.3125) (0.02,45.5566) (0.05,45.752) (0.1,45.6055) (0.2,45.7031) (0.5,45.3613) (1,44.9219) (1,43.0176) (0.5,43.1152) (0.2,44.1895) (0.1,44.1895) (0.05,44.5801) (0.02,44.5312) (0.01,44.1406) (0.005,43.8477) (0.002,44.4824) (0.001,44.1406)} \closedcycle;
  \addplot[draw=none, fill=htFSAM, fill opacity=0.06] coordinates {(0.001,39.5508) (0.002,39.8926) (0.005,40.918) (0.01,40.0391) (0.02,40.0879) (0.05,40.7227) (0.1,40.0391) (0.2,41.0156) (0.5,40.4785) (1,40.5762) (1,38.623) (0.5,37.9883) (0.2,37.9883) (0.1,38.2812) (0.05,38.0371) (0.02,37.793) (0.01,38.0371) (0.005,38.6719) (0.002,38.8184) (0.001,38.4766)} \closedcycle;
  \addplot[draw=none, fill=htSAM, fill opacity=0.06] coordinates {(0.001,39.6973) (0.002,40.6738) (0.005,40.0879) (0.01,40.4785) (0.02,40.4785) (0.05,40.5273) (0.1,40.332) (0.2,40.1367) (0.5,41.0156) (1,39.8926) (1,37.5977) (0.5,39.209) (0.2,37.5977) (0.1,38.8184) (0.05,38.2812) (0.02,38.0859) (0.01,37.1094) (0.005,38.7695) (0.002,39.0137) (0.001,38.2324)} \closedcycle;
  \addplot[draw=none, fill=htMSAM, fill opacity=0.06] coordinates {(0.001,40.2832) (0.002,39.9414) (0.005,40.0879) (0.01,39.3555) (0.02,39.2578) (0.05,39.3555) (0.1,40.2832) (0.2,40.4785) (0.5,39.5508) (1,40.2344) (1,37.3047) (0.5,37.6953) (0.2,37.793) (0.1,37.8418) (0.05,36.9629) (0.02,37.6465) (0.01,37.1582) (0.005,38.0371) (0.002,37.0605) (0.001,37.6465)} \closedcycle;
  \addplot[draw=none, fill=htSAMAdam, fill opacity=0.06] coordinates {(0.001,45.3125) (0.002,45.1172) (0.005,45.0195) (0.01,45.1172) (0.02,45.1172) (0.05,44.9707) (0.1,44.8242) (0.2,45.1172) (0.5,44.9707) (1,44.7266) (1,43.5059) (0.5,43.8965) (0.2,43.9453) (0.1,43.5547) (0.05,43.6523) (0.02,43.5547) (0.01,43.5547) (0.005,43.8965) (0.002,43.75) (0.001,43.75)} \closedcycle;
  \addplot[draw=none, fill=htFSAMAdam, fill opacity=0.06] coordinates {(0.001,45.3125) (0.002,45.166) (0.005,45.166) (0.01,44.9707) (0.02,45.3125) (0.05,44.9219) (0.1,44.9707) (0.2,44.7266) (0.5,44.9219) (1,44.9219) (1,44.1406) (0.5,43.8965) (0.2,43.5547) (0.1,43.6523) (0.05,43.5059) (0.02,43.6523) (0.01,43.75) (0.005,43.75) (0.002,43.7012) (0.001,43.75)} \closedcycle;
  \addplot[color=htMuon, dashed, line width=0.7pt, forget plot] coordinates {(0.001,45.2148) (1,45.2148)};
  \addplot[color=htAdam, dash dot, line width=0.7pt, forget plot] coordinates {(0.001,44.1406) (1,44.1406)};
  \addplot[color=htSpec, solid, line width=0.8pt] coordinates {(0.001,45.1172) (0.002,44.6289) (0.005,45.5078) (0.01,45.0195) (0.02,45.3125) (0.05,45.2148) (0.1,44.5312) (0.2,44.8242) (0.5,45.2148) (1,39.3555)};
  \addplot[color=htRand, dashed, line width=0.8pt] coordinates {(0.001,45.1172) (0.002,45.1172) (0.005,44.9219) (0.01,44.8242) (0.02,44.5312) (0.05,45.1172) (0.1,44.7266) (0.2,45.2148) (0.5,45.0195) (1,44.9219)};
  \addplot[color=htMSOMA, loosely dashed, line width=0.8pt] coordinates {(0.001,45.0195) (0.002,45.3125) (0.005,45.0195) (0.01,44.7266) (0.02,44.9219) (0.05,45.4102) (0.1,45.1172) (0.2,45.1172) (0.5,44.6289) (1,44.043)};
  \addplot[color=htFSAM, solid, line width=0.8pt] coordinates {(0.001,39.0625) (0.002,39.4531) (0.005,39.2578) (0.01,39.0625) (0.02,39.0625) (0.05,39.3555) (0.1,39.0625) (0.2,38.5742) (0.5,39.5508) (1,39.4531)};
  \addplot[color=htSAM, solid, line width=0.8pt] coordinates {(0.001,39.2578) (0.002,39.6484) (0.005,39.5508) (0.01,38.6719) (0.02,39.0625) (0.05,39.8438) (0.1,39.5508) (0.2,39.0625) (0.5,39.6484) (1,39.1602)};
  \addplot[color=htMSAM, solid, line width=0.8pt] coordinates {(0.001,39.3555) (0.002,38.8672) (0.005,38.7695) (0.01,37.9883) (0.02,38.3789) (0.05,38.7695) (0.1,39.6484) (0.2,38.6719) (0.5,38.2812) (1,38.6719)};
  \addplot[color=htSOMA, solid, line width=1.1pt] coordinates {(0.001,45.0195) (0.002,44.7266) (0.005,44.7266) (0.01,45.1172) (0.02,45.4102) (0.05,45.1172) (0.1,44.5312) (0.2,45.4102) (0.5,46.582) (1,46.0938)};
  \addplot[color=htFPSOMA, solid, line width=1.1pt] coordinates {(0.001,44.7266) (0.002,44.7266) (0.005,44.9219) (0.01,45.1172) (0.02,45.0195) (0.05,45.1172) (0.1,44.8242) (0.2,44.8242) (0.5,46.582) (1,43.2617)};
  \addplot[color=htPreNS5, densely dotted, line width=1.1pt] coordinates {(0.001,44.7266) (0.002,45.1172) (0.005,44.7266) (0.01,45.0195) (0.02,45.0195) (0.05,44.5312) (0.1,44.5312) (0.2,44.8242) (0.5,45.7031) (1,45.7031)};
  \addplot[color=htPostNS5, dash dot, line width=1.1pt] coordinates {(0.001,45.1172) (0.002,44.9219) (0.005,45.0195) (0.01,44.8242) (0.02,44.9219) (0.05,44.9219) (0.1,44.4336) (0.2,44.9219) (0.5,45.8008) (1,46.7773)};
  \addplot[color=htSAMAdam, solid, line width=1.1pt] coordinates {(0.001,44.3359) (0.002,44.5312) (0.005,44.3359) (0.01,44.5312) (0.02,44.4336) (0.05,44.1406) (0.1,44.2383) (0.2,44.3359) (0.5,44.3359) (1,43.9453)};
  \addplot[color=htFSAMAdam, densely dashed, line width=1.4pt] coordinates {(0.001,44.1406) (0.002,44.4336) (0.005,44.3359) (0.01,44.3359) (0.02,44.2383) (0.05,44.6289) (0.1,44.3359) (0.2,44.1406) (0.5,44.4336) (1,44.6289)};
\end{axis}
\begin{axis}[name=ax3_r0,
    at={(0.8065\textwidth,0in)},
    anchor=south west,
    ht panel, xmode=log,
    ymin=28.8281, ymax=49.7852,
    ytick={30,35,40,45},
    yticklabels={30,35,40,45},
    title={$\alpha=3.0$}]
  \addplot[draw=none, fill=htSOMA, fill opacity=0.06] coordinates {(0.001,45.0195) (0.002,44.873) (0.005,45.752) (0.01,45.8008) (0.02,45.6543) (0.05,45.2637) (0.1,44.6777) (0.2,45.5566) (0.5,47.0703) (1,46.3379) (1,45.5078) (0.5,45.9961) (0.2,43.9453) (0.1,43.8965) (0.05,43.3594) (0.02,43.6035) (0.01,43.9941) (0.005,43.6035) (0.002,43.1641) (0.001,43.2617)} \closedcycle;
  \addplot[draw=none, fill=htFPSOMA, fill opacity=0.06] coordinates {(0.001,45.4102) (0.002,45.8984) (0.005,45.8496) (0.01,45.2637) (0.02,45.9473) (0.05,44.7754) (0.1,45.166) (0.2,45.1172) (0.5,46.4844) (1,44.6289) (1,42.9199) (0.5,45.8008) (0.2,44.5801) (0.1,44.1895) (0.05,43.75) (0.02,42.7246) (0.01,43.6523) (0.005,43.4082) (0.002,43.4082) (0.001,42.9688)} \closedcycle;
  \addplot[draw=none, fill=htSpec, fill opacity=0.06] coordinates {(0.001,45.459) (0.002,44.9707) (0.005,45.752) (0.01,45.3613) (0.02,45.5078) (0.05,44.9219) (0.1,46.0938) (0.2,45.459) (0.5,46.9238) (1,31.2012) (1,29.7852) (0.5,45.7031) (0.2,44.9707) (0.1,44.2383) (0.05,43.8965) (0.02,43.5059) (0.01,43.4082) (0.005,43.1641) (0.002,43.457) (0.001,42.9688)} \closedcycle;
  \addplot[draw=none, fill=htRand, fill opacity=0.06] coordinates {(0.001,45.3613) (0.002,45.2637) (0.005,45.0195) (0.01,45.0684) (0.02,45.9473) (0.05,45.3613) (0.1,45.166) (0.2,45.9961) (0.5,45.5566) (1,44.873) (1,43.5547) (0.5,43.1641) (0.2,42.9199) (0.1,43.1152) (0.05,43.1641) (0.02,42.9199) (0.01,42.9688) (0.005,43.1641) (0.002,43.0664) (0.001,43.0664)} \closedcycle;
  \addplot[draw=none, fill=htPreNS5, fill opacity=0.06] coordinates {(0.001,44.9219) (0.002,45.166) (0.005,45.3613) (0.01,45.4102) (0.02,45.6543) (0.05,44.6289) (0.1,45.6055) (0.2,45.8984) (0.5,47.4609) (1,47.2656) (1,45.752) (0.5,46.2891) (0.2,44.1406) (0.1,43.8965) (0.05,43.2129) (0.02,43.9453) (0.01,43.2129) (0.005,43.2617) (0.002,43.3105) (0.001,42.9688)} \closedcycle;
  \addplot[draw=none, fill=htPostNS5, fill opacity=0.06] coordinates {(0.001,45.1172) (0.002,45.3613) (0.005,45.0195) (0.01,45.5078) (0.02,45.8008) (0.05,44.873) (0.1,45.2637) (0.2,45.5078) (0.5,47.0703) (1,49.0234) (1,46.9238) (0.5,45.0195) (0.2,43.5059) (0.1,44.043) (0.05,42.9199) (0.02,43.6523) (0.01,43.3594) (0.005,43.1641) (0.002,43.3594) (0.001,43.0176)} \closedcycle;
  \addplot[draw=none, fill=htMSOMA, fill opacity=0.06] coordinates {(0.001,45.3125) (0.002,45.2637) (0.005,45.1172) (0.01,45.3613) (0.02,45.166) (0.05,45.8008) (0.1,45.5078) (0.2,44.5312) (0.5,45.2148) (1,44.9219) (1,43.2129) (0.5,42.8711) (0.2,42.7246) (0.1,42.7246) (0.05,42.9688) (0.02,43.1641) (0.01,43.3105) (0.005,43.3105) (0.002,43.2129) (0.001,43.0664)} \closedcycle;
  \addplot[draw=none, fill=htFSAM, fill opacity=0.06] coordinates {(0.001,43.3105) (0.002,43.0664) (0.005,43.1641) (0.01,43.1641) (0.02,43.1641) (0.05,42.8711) (0.1,42.7734) (0.2,42.627) (0.5,42.9199) (1,44.043) (1,41.9922) (0.5,40.625) (0.2,40.1367) (0.1,39.5508) (0.05,40.0879) (0.02,40.6738) (0.01,40.625) (0.005,40.1855) (0.002,40.4297) (0.001,39.4531)} \closedcycle;
  \addplot[draw=none, fill=htSAM, fill opacity=0.06] coordinates {(0.001,43.2617) (0.002,42.5781) (0.005,43.1641) (0.01,42.9199) (0.02,42.8223) (0.05,43.0664) (0.1,43.6035) (0.2,42.8223) (0.5,43.3594) (1,43.6035) (1,41.2109) (0.5,40.7715) (0.2,40.1367) (0.1,39.4531) (0.05,40.7715) (0.02,40.332) (0.01,40.7227) (0.005,39.3066) (0.002,40.1367) (0.001,39.9902)} \closedcycle;
  \addplot[draw=none, fill=htMSAM, fill opacity=0.06] coordinates {(0.001,42.4316) (0.002,43.1641) (0.005,43.3105) (0.01,43.1641) (0.02,42.627) (0.05,43.4082) (0.1,42.5781) (0.2,43.5059) (0.5,43.1152) (1,43.4082) (1,39.2578) (0.5,40.5762) (0.2,40.5762) (0.1,40.4297) (0.05,41.1621) (0.02,40.332) (0.01,41.1133) (0.005,40.4297) (0.002,40.3809) (0.001,40.9668)} \closedcycle;
  \addplot[draw=none, fill=htSAMAdam, fill opacity=0.06] coordinates {(0.001,46.3379) (0.002,46.3379) (0.005,46.3379) (0.01,46.4844) (0.02,46.4844) (0.05,46.4844) (0.1,46.0938) (0.2,46.1426) (0.5,45.7031) (1,45.8984) (1,45.3125) (0.5,44.9219) (0.2,45.2148) (0.1,45.3125) (0.05,45.4102) (0.02,45.459) (0.01,45.5078) (0.005,45.459) (0.002,45.2637) (0.001,45.4102)} \closedcycle;
  \addplot[draw=none, fill=htFSAMAdam, fill opacity=0.06] coordinates {(0.001,46.2891) (0.002,46.2891) (0.005,46.3379) (0.01,46.4844) (0.02,46.5332) (0.05,46.5332) (0.1,46.2891) (0.2,46.1426) (0.5,45.8984) (1,46.3379) (1,45.5078) (0.5,44.9219) (0.2,45.0684) (0.1,45.3125) (0.05,45.2148) (0.02,45.5078) (0.01,45.5078) (0.005,45.3125) (0.002,45.6055) (0.001,45.459)} \closedcycle;
  \addplot[color=htMuon, dashed, line width=0.7pt, forget plot] coordinates {(0.001,43.75) (1,43.75)};
  \addplot[color=htAdam, dash dot, line width=0.7pt, forget plot] coordinates {(0.001,45.8008) (1,45.8008)};
  \addplot[color=htSpec, solid, line width=0.8pt] coordinates {(0.001,43.9453) (0.002,44.2383) (0.005,43.6523) (0.01,44.2383) (0.02,44.043) (0.05,44.2383) (0.1,45.2148) (0.2,45.2148) (0.5,46.1914) (1,30.2734)};
  \addplot[color=htRand, dashed, line width=0.8pt] coordinates {(0.001,44.1406) (0.002,44.043) (0.005,43.9453) (0.01,43.8477) (0.02,43.9453) (0.05,44.6289) (0.1,44.2383) (0.2,44.8242) (0.5,44.2383) (1,44.1406)};
  \addplot[color=htMSOMA, loosely dashed, line width=0.8pt] coordinates {(0.001,43.5547) (0.002,43.9453) (0.005,44.2383) (0.01,44.043) (0.02,43.9453) (0.05,44.3359) (0.1,43.5547) (0.2,43.75) (0.5,43.5547) (1,44.043)};
  \addplot[color=htFSAM, solid, line width=0.8pt] coordinates {(0.001,41.6992) (0.002,41.1133) (0.005,41.5039) (0.01,41.4062) (0.02,41.9922) (0.05,41.1133) (0.1,41.9922) (0.2,41.7969) (0.5,41.6016) (1,42.6758)};
  \addplot[color=htSAM, solid, line width=0.8pt] coordinates {(0.001,41.6992) (0.002,41.5039) (0.005,41.6992) (0.01,42.0898) (0.02,41.0156) (0.05,41.6992) (0.1,41.1133) (0.2,41.5039) (0.5,42.0898) (1,41.5039)};
  \addplot[color=htMSAM, solid, line width=0.8pt] coordinates {(0.001,41.6016) (0.002,41.7969) (0.005,41.6016) (0.01,41.8945) (0.02,41.3086) (0.05,42.1875) (0.1,41.6016) (0.2,41.3086) (0.5,41.4062) (1,41.4062)};
  \addplot[color=htSOMA, solid, line width=1.1pt] coordinates {(0.001,44.2383) (0.002,44.2383) (0.005,45.3125) (0.01,44.5312) (0.02,44.8242) (0.05,44.4336) (0.1,44.1406) (0.2,44.8242) (0.5,46.582) (1,45.8984)};
  \addplot[color=htFPSOMA, solid, line width=1.1pt] coordinates {(0.001,44.3359) (0.002,43.75) (0.005,43.9453) (0.01,44.6289) (0.02,44.1406) (0.05,43.9453) (0.1,44.8242) (0.2,44.9219) (0.5,46.0938) (1,43.5547)};
  \addplot[color=htPreNS5, densely dotted, line width=1.1pt] coordinates {(0.001,43.9453) (0.002,43.9453) (0.005,44.5312) (0.01,44.4336) (0.02,44.5312) (0.05,43.8477) (0.1,44.4336) (0.2,45.0195) (0.5,46.7773) (1,46.4844)};
  \addplot[color=htPostNS5, dash dot, line width=1.1pt] coordinates {(0.001,44.1406) (0.002,44.2383) (0.005,43.75) (0.01,43.75) (0.02,44.3359) (0.05,43.75) (0.1,44.4336) (0.2,44.6289) (0.5,46.2891) (1,48.3398)};
  \addplot[color=htSAMAdam, solid, line width=1.1pt] coordinates {(0.001,45.8008) (0.002,45.8008) (0.005,45.8008) (0.01,45.9961) (0.02,45.8984) (0.05,45.9961) (0.1,45.7031) (0.2,45.6055) (0.5,45.3125) (1,45.5078)};
  \addplot[color=htFSAMAdam, densely dashed, line width=1.4pt] coordinates {(0.001,45.7031) (0.002,45.8008) (0.005,45.8008) (0.01,45.8984) (0.02,45.9961) (0.05,45.8984) (0.1,45.8008) (0.2,45.8008) (0.5,45.2148) (1,46.1914)};
\end{axis}
\begin{axis}[hide axis, scale only axis, at={(ax0_r0.north west)}, anchor=south west, yshift=34pt,
    width=0.9300\textwidth, height=1pt, xmin=0, xmax=1, ymin=0, ymax=1,
    legend columns=7, legend cell align=left,
    legend style={at={(0.5,0)}, anchor=south, draw=none, fill=none, font=\scriptsize, inner sep=1pt,
      /tikz/every even column/.append style={column sep=5pt}}]
  \addlegendimage{color=htSOMA, solid, line width=1.1pt}\addlegendentry{\textsc{Soma}}
  \addlegendimage{color=htFPSOMA, solid, line width=1.1pt}\addlegendentry{\textsc{FP{-}Soma}}
  \addlegendimage{color=htMSOMA, loosely dashed, line width=0.8pt}\addlegendentry{\textsc{Msoma}}
  \addlegendimage{color=htPreNS5, densely dotted, line width=1.1pt}\addlegendentry{\textsc{Opspre}}
  \addlegendimage{color=htPostNS5, dash dot, line width=1.1pt}\addlegendentry{\textsc{Opspost}}
  \addlegendimage{color=htSpec, solid, line width=0.8pt}\addlegendentry{\textsc{Ssm}}
  \addlegendimage{color=htRand, dashed, line width=0.8pt}\addlegendentry{\textsc{Rsm}}
  \addlegendimage{color=htSAM, solid, line width=0.8pt}\addlegendentry{\textsc{Sam}}
  \addlegendimage{color=htFSAM, solid, line width=0.8pt}\addlegendentry{\textsc{F{-}Sam}}
  \addlegendimage{color=htMSAM, solid, line width=0.8pt}\addlegendentry{\textsc{Msam}}
  \addlegendimage{color=htSAMAdam, solid, line width=1.1pt}\addlegendentry{\textsc{Saw}}
  \addlegendimage{color=htFSAMAdam, densely dashed, line width=1.4pt}\addlegendentry{\textsc{Fsaw}}
  \addlegendimage{color=htMuon, dashed, line width=0.7pt}\addlegendentry{\textsc{Muon}}
  \addlegendimage{color=htAdam, dash dot, line width=0.7pt}\addlegendentry{\textsc{AdamW}}
\end{axis}
\path (ax0_r0.south west) -- (ax3_r0.south east) coordinate[midway] (xlab);
\node[font=\small, anchor=north] at ([yshift=-14pt]xlab) {perturbation radius $\rho$};
\end{tikzpicture}

%% file: tikz_for_heavytailed/appendex_isotropic/train_loss_vs_rho.tikz
%% Auto-generated by make_heavytailed_figures.py  --  do not edit by hand, re-run the script.
%% Values: line = median over seeds, band = interquartile range (25th-75th percentile).
%% Preamble needed:  \usepackage{pgfplots}  \pgfplotsset{compat=1.16}   (nothing else)
%% Width follows \textwidth, fonts follow the document (so the legend can never overflow).

\definecolor{htSOMA}{HTML}{0B5D1E}
\definecolor{htFPSOMA}{HTML}{33CC33}
\definecolor{htSpec}{HTML}{00B3A4}
\definecolor{htRand}{HTML}{808000}
\definecolor{htFSAM}{HTML}{2F80D8}
\definecolor{htSAM}{HTML}{0A3A8A}
\definecolor{htPreNS5}{HTML}{9ACD32}
\definecolor{htPostNS5}{HTML}{004D40}
\definecolor{htSAMAdam}{HTML}{8E0F1C}
\definecolor{htFSAMAdam}{HTML}{F4623A}
\definecolor{htMSAM}{HTML}{7DB9F2}
\definecolor{htMSOMA}{HTML}{8FD694}
\definecolor{htMuon}{HTML}{111111}
\definecolor{htAdam}{HTML}{8A8A8A}
\pgfplotsset{
  ht panel/.style={
    scale only axis, width=0.1920\textwidth, height=1.35in,
    xmin=0.00075, xmax=1.35,
    xtick={0.001,0.01,0.1,1}, xticklabels={$10^{-3}$,$10^{-2}$,$10^{-1}$,$10^{0}$},
    minor xtick={0.002,0.003,0.004,0.005,0.006,0.007,0.008,0.009,0.02,0.03,0.04,0.05,0.06,0.07,0.08,0.09,0.2,0.3,0.4,0.5,0.6,0.7,0.8,0.9},
    xmajorgrids=true, ymajorgrids=true,
    grid style={line width=0.25pt, draw=black!10},
    axis line style={line width=0.4pt, black!75},
    tick style={line width=0.4pt, black!75}, major tick length=3pt, minor tick length=1.5pt,
    tick label style={font=\footnotesize, /pgf/number format/assume math mode=true},
    label style={font=\small}, title style={font=\small, yshift=-3pt},
    ylabel style={yshift=-1pt, align=center},
    every tick label/.append style={inner sep=1pt},
    clip=true,
  }}

\begin{tikzpicture}
\begin{axis}[name=ax0_r0,
    at={(0.0820\textwidth,0in)},
    anchor=south west,
    ht panel, xmode=log, ymode=log,
    ymin=1.23222, ymax=28.5046,
    ytick={2,5,10,20},
    yticklabels={2,5,10,20},
    title={$\alpha=1.1$},
    ylabel={training loss}]
  \addplot[draw=none, fill=htSOMA, fill opacity=0.06] coordinates {(0.001,1.66818) (0.002,1.66817) (0.005,1.66768) (0.01,1.66775) (0.02,1.66562) (0.05,1.6621) (0.1,1.65663) (0.2,1.64353) (0.5,1.62342) (1,1.66006) (1,1.56744) (0.5,1.52154) (0.2,1.53304) (0.1,1.54038) (0.05,1.54709) (0.02,1.55014) (0.01,1.54933) (0.005,1.54995) (0.002,1.55094) (0.001,1.55099)} \closedcycle;
  \addplot[draw=none, fill=htFPSOMA, fill opacity=0.06] coordinates {(0.001,1.66804) (0.002,1.66761) (0.005,1.66697) (0.01,1.66638) (0.02,1.66421) (0.05,1.66148) (0.1,1.65238) (0.2,1.6439) (0.5,1.64326) (1,1.68983) (1,1.6219) (0.5,1.54786) (0.2,1.53427) (0.1,1.53939) (0.05,1.54668) (0.02,1.54996) (0.01,1.5506) (0.005,1.55015) (0.002,1.55086) (0.001,1.55087)} \closedcycle;
  \addplot[draw=none, fill=htSpec, fill opacity=0.06] coordinates {(0.001,1.668) (0.002,1.66794) (0.005,1.66791) (0.01,1.66685) (0.02,1.66643) (0.05,1.66575) (0.1,1.66189) (0.2,1.66643) (0.5,1.7137) (1,1.82355) (1,1.76501) (0.5,1.6241) (0.2,1.56024) (0.1,1.55078) (0.05,1.55178) (0.02,1.55176) (0.01,1.55029) (0.005,1.5505) (0.002,1.55112) (0.001,1.5507)} \closedcycle;
  \addplot[draw=none, fill=htRand, fill opacity=0.06] coordinates {(0.001,1.66827) (0.002,1.66806) (0.005,1.66833) (0.01,1.6685) (0.02,1.66858) (0.05,1.66897) (0.1,1.66892) (0.2,1.66966) (0.5,1.67308) (1,1.67235) (1,1.55405) (0.5,1.5518) (0.2,1.55104) (0.1,1.55102) (0.05,1.55105) (0.02,1.55096) (0.01,1.55106) (0.005,1.55118) (0.002,1.55087) (0.001,1.55082)} \closedcycle;
  \addplot[draw=none, fill=htPreNS5, fill opacity=0.06] coordinates {(0.001,1.66815) (0.002,1.66833) (0.005,1.66792) (0.01,1.66845) (0.02,1.66735) (0.05,1.66541) (0.1,1.66192) (0.2,1.65519) (0.5,1.64923) (1,1.71209) (1,1.63449) (0.5,1.55524) (0.2,1.54648) (0.1,1.54572) (0.05,1.55008) (0.02,1.55091) (0.01,1.54992) (0.005,1.55046) (0.002,1.55081) (0.001,1.5508)} \closedcycle;
  \addplot[draw=none, fill=htPostNS5, fill opacity=0.06] coordinates {(0.001,1.66825) (0.002,1.66785) (0.005,1.66801) (0.01,1.66788) (0.02,1.66781) (0.05,1.6659) (0.1,1.66342) (0.2,1.6579) (0.5,1.64956) (1,1.66894) (1,1.5756) (0.5,1.54564) (0.2,1.54537) (0.1,1.54753) (0.05,1.55106) (0.02,1.55123) (0.01,1.54964) (0.005,1.55099) (0.002,1.5511) (0.001,1.55088)} \closedcycle;
  \addplot[draw=none, fill=htMSOMA, fill opacity=0.06] coordinates {(0.001,1.66786) (0.002,1.66781) (0.005,1.66783) (0.01,1.66781) (0.02,1.66749) (0.05,1.66664) (0.1,1.66447) (0.2,1.66224) (0.5,1.65298) (1,1.63958) (1,1.52687) (0.5,1.53998) (0.2,1.54646) (0.1,1.5491) (0.05,1.54989) (0.02,1.55077) (0.01,1.55098) (0.005,1.55097) (0.002,1.55097) (0.001,1.55094)} \closedcycle;
  \addplot[draw=none, fill=htFSAM, fill opacity=0.06] coordinates {(0.001,18.6295) (0.002,19.7605) (0.005,22.2438) (0.01,19.6952) (0.02,19.8652) (0.05,17.7892) (0.1,19.7622) (0.2,19.6929) (0.5,19.0011) (1,17.8287) (1,9.14765) (0.5,10.087) (0.2,8.92425) (0.1,11.6739) (0.05,9.67376) (0.02,8.82944) (0.01,10.329) (0.005,11.1732) (0.002,9.51259) (0.001,11.3155)} \closedcycle;
  \addplot[draw=none, fill=htSAM, fill opacity=0.06] coordinates {(0.001,20.8041) (0.002,19.3456) (0.005,20.0027) (0.01,21.3638) (0.02,19.0858) (0.05,21.1338) (0.1,18.6146) (0.2,18.6097) (0.5,20.0753) (1,18.1416) (1,10.1676) (0.5,8.69928) (0.2,8.91266) (0.1,9.44054) (0.05,9.77379) (0.02,9.60741) (0.01,9.10521) (0.005,10.1895) (0.002,9.64792) (0.001,10.4062)} \closedcycle;
  \addplot[draw=none, fill=htMSAM, fill opacity=0.06] coordinates {(0.001,27.0242) (0.002,26.9842) (0.005,24.1743) (0.01,25.4059) (0.02,23.3482) (0.05,22.356) (0.1,24.4412) (0.2,23.5266) (0.5,29.2687) (1,29.8619) (1,13.1643) (0.5,12.4357) (0.2,13.9345) (0.1,11.5845) (0.05,14.4757) (0.02,15.0476) (0.01,16.6656) (0.005,12.895) (0.002,12.2924) (0.001,12.8541)} \closedcycle;
  \addplot[draw=none, fill=htSAMAdam, fill opacity=0.06] coordinates {(0.001,1.72734) (0.002,1.72732) (0.005,1.72737) (0.01,1.72729) (0.02,1.72732) (0.05,1.72804) (0.1,1.72848) (0.2,1.72917) (0.5,1.73166) (1,1.73711) (1,1.71383) (0.5,1.70591) (0.2,1.70283) (0.1,1.70283) (0.05,1.70259) (0.02,1.70258) (0.01,1.70252) (0.005,1.70243) (0.002,1.70239) (0.001,1.70235)} \closedcycle;
  \addplot[draw=none, fill=htFSAMAdam, fill opacity=0.06] coordinates {(0.001,1.72736) (0.002,1.72735) (0.005,1.72727) (0.01,1.72715) (0.02,1.72733) (0.05,1.72776) (0.1,1.72783) (0.2,1.72779) (0.5,1.72865) (1,1.73122) (1,1.70883) (0.5,1.7043) (0.2,1.70274) (0.1,1.70231) (0.05,1.70255) (0.02,1.70254) (0.01,1.70239) (0.005,1.70246) (0.002,1.70234) (0.001,1.70239)} \closedcycle;
  \addplot[color=htMuon, dashed, line width=0.7pt, forget plot] coordinates {(0.001,1.58002) (1,1.58002)};
  \addplot[color=htAdam, dash dot, line width=0.7pt, forget plot] coordinates {(0.001,1.71452) (1,1.71452)};
  \addplot[color=htSpec, solid, line width=0.8pt] coordinates {(0.001,1.58043) (0.002,1.58081) (0.005,1.58057) (0.01,1.58062) (0.02,1.58042) (0.05,1.58053) (0.1,1.58279) (0.2,1.59675) (0.5,1.66849) (1,1.80761)};
  \addplot[color=htRand, dashed, line width=0.8pt] coordinates {(0.001,1.57997) (0.002,1.58) (0.005,1.57972) (0.01,1.57985) (0.02,1.58006) (0.05,1.58008) (0.1,1.58054) (0.2,1.58) (0.5,1.58259) (1,1.58415)};
  \addplot[color=htMSOMA, loosely dashed, line width=0.8pt] coordinates {(0.001,1.58017) (0.002,1.58022) (0.005,1.57986) (0.01,1.57983) (0.02,1.57959) (0.05,1.57929) (0.1,1.57816) (0.2,1.57563) (0.5,1.56823) (1,1.55784)};
  \addplot[color=htFSAM, solid, line width=0.8pt] coordinates {(0.001,14.1744) (0.002,13.8223) (0.005,15.0271) (0.01,15.6181) (0.02,13.7824) (0.05,12.0159) (0.1,15.1682) (0.2,14.3503) (0.5,13.1865) (1,11.6938)};
  \addplot[color=htSAM, solid, line width=0.8pt] coordinates {(0.001,14.0911) (0.002,14.5826) (0.005,14.8795) (0.01,13.7945) (0.02,13.8611) (0.05,14.1047) (0.1,12.6986) (0.2,10.6911) (0.5,16.3482) (1,13.9477)};
  \addplot[color=htMSAM, solid, line width=0.8pt] coordinates {(0.001,21.0787) (0.002,20.1524) (0.005,17.9694) (0.01,20.8752) (0.02,19.8544) (0.05,19.9145) (0.1,20.0056) (0.2,20.8349) (0.5,18.4165) (1,22.5737)};
  \addplot[color=htSOMA, solid, line width=1.1pt] coordinates {(0.001,1.57992) (0.002,1.5801) (0.005,1.57961) (0.01,1.57967) (0.02,1.5783) (0.05,1.57571) (0.1,1.57208) (0.2,1.56128) (0.5,1.55596) (1,1.61984)};
  \addplot[color=htFPSOMA, solid, line width=1.1pt] coordinates {(0.001,1.58008) (0.002,1.57995) (0.005,1.57977) (0.01,1.57951) (0.02,1.57792) (0.05,1.57508) (0.1,1.56966) (0.2,1.56667) (0.5,1.57814) (1,1.65519)};
  \addplot[color=htPreNS5, densely dotted, line width=1.1pt] coordinates {(0.001,1.57986) (0.002,1.58039) (0.005,1.58027) (0.01,1.58022) (0.02,1.57994) (0.05,1.57925) (0.1,1.57935) (0.2,1.57838) (0.5,1.59129) (1,1.67732)};
  \addplot[color=htPostNS5, dash dot, line width=1.1pt] coordinates {(0.001,1.58006) (0.002,1.58043) (0.005,1.57979) (0.01,1.57997) (0.02,1.57951) (0.05,1.57868) (0.1,1.57882) (0.2,1.57629) (0.5,1.5733) (1,1.61488)};
  \addplot[color=htSAMAdam, solid, line width=1.1pt] coordinates {(0.001,1.71452) (0.002,1.71447) (0.005,1.71462) (0.01,1.7146) (0.02,1.71462) (0.05,1.71512) (0.1,1.71502) (0.2,1.71591) (0.5,1.71806) (1,1.72449)};
  \addplot[color=htFSAMAdam, densely dashed, line width=1.4pt] coordinates {(0.001,1.71452) (0.002,1.71452) (0.005,1.71457) (0.01,1.71451) (0.02,1.71469) (0.05,1.7149) (0.1,1.71469) (0.2,1.71461) (0.5,1.71587) (1,1.7213)};
\end{axis}
\begin{axis}[name=ax1_r0,
    at={(0.3235\textwidth,0in)},
    anchor=south west,
    ht panel, xmode=log, ymode=log,
    ymin=0.00534433, ymax=2.29316,
    ytick={0.01,0.1,1},
    yticklabels={0.01,0.1,1},
    title={$\alpha=1.6$}]
  \addplot[draw=none, fill=htSOMA, fill opacity=0.06] coordinates {(0.001,0.512234) (0.002,0.511795) (0.005,0.511574) (0.01,0.509849) (0.02,0.511285) (0.05,0.512011) (0.1,0.515998) (0.2,0.527) (0.5,0.612851) (1,1.00664) (1,0.981674) (0.5,0.592958) (0.2,0.488911) (0.1,0.478149) (0.05,0.472605) (0.02,0.471384) (0.01,0.469801) (0.005,0.469048) (0.002,0.468039) (0.001,0.469769)} \closedcycle;
  \addplot[draw=none, fill=htFPSOMA, fill opacity=0.06] coordinates {(0.001,0.511949) (0.002,0.512167) (0.005,0.510992) (0.01,0.512429) (0.02,0.515436) (0.05,0.520949) (0.1,0.528706) (0.2,0.565327) (0.5,0.758697) (1,1.28362) (1,1.2485) (0.5,0.714529) (0.2,0.519033) (0.1,0.485128) (0.05,0.47531) (0.02,0.473927) (0.01,0.471199) (0.005,0.469299) (0.002,0.467807) (0.001,0.467899)} \closedcycle;
  \addplot[draw=none, fill=htSpec, fill opacity=0.06] coordinates {(0.001,0.511118) (0.002,0.51181) (0.005,0.509758) (0.01,0.511754) (0.02,0.517673) (0.05,0.526682) (0.1,0.544432) (0.2,0.59832) (0.5,0.858644) (1,1.52607) (1,1.49378) (0.5,0.815108) (0.2,0.553024) (0.1,0.499311) (0.05,0.481805) (0.02,0.475303) (0.01,0.472918) (0.005,0.469514) (0.002,0.468179) (0.001,0.468578)} \closedcycle;
  \addplot[draw=none, fill=htRand, fill opacity=0.06] coordinates {(0.001,0.512552) (0.002,0.512313) (0.005,0.511536) (0.01,0.510953) (0.02,0.511299) (0.05,0.511201) (0.1,0.510268) (0.2,0.511046) (0.5,0.515679) (1,0.517885) (1,0.471899) (0.5,0.468381) (0.2,0.470007) (0.1,0.467009) (0.05,0.467328) (0.02,0.468893) (0.01,0.467533) (0.005,0.467246) (0.002,0.468208) (0.001,0.467956)} \closedcycle;
  \addplot[draw=none, fill=htPreNS5, fill opacity=0.06] coordinates {(0.001,0.51015) (0.002,0.510238) (0.005,0.509242) (0.01,0.509826) (0.02,0.511817) (0.05,0.513847) (0.1,0.518367) (0.2,0.526828) (0.5,0.608252) (1,0.956816) (1,0.908361) (0.5,0.578975) (0.2,0.485052) (0.1,0.475662) (0.05,0.469666) (0.02,0.46934) (0.01,0.470327) (0.005,0.470105) (0.002,0.470326) (0.001,0.467327)} \closedcycle;
  \addplot[draw=none, fill=htPostNS5, fill opacity=0.06] coordinates {(0.001,0.510985) (0.002,0.511412) (0.005,0.511378) (0.01,0.510957) (0.02,0.513471) (0.05,0.512824) (0.1,0.513845) (0.2,0.516098) (0.5,0.570504) (1,0.788822) (1,0.761659) (0.5,0.532074) (0.2,0.475867) (0.1,0.472353) (0.05,0.469714) (0.02,0.470101) (0.01,0.469365) (0.005,0.469254) (0.002,0.468116) (0.001,0.469717)} \closedcycle;
  \addplot[draw=none, fill=htMSOMA, fill opacity=0.06] coordinates {(0.001,0.513017) (0.002,0.511516) (0.005,0.512208) (0.01,0.511692) (0.02,0.50979) (0.05,0.509553) (0.1,0.510213) (0.2,0.50565) (0.5,0.494713) (1,0.479753) (1,0.425035) (0.5,0.440426) (0.2,0.458018) (0.1,0.463904) (0.05,0.466246) (0.02,0.466476) (0.01,0.468747) (0.005,0.468636) (0.002,0.468245) (0.001,0.468329)} \closedcycle;
  \addplot[draw=none, fill=htFSAM, fill opacity=0.06] coordinates {(0.001,0.0263637) (0.002,0.0141983) (0.005,0.0332845) (0.01,0.0348631) (0.02,0.0241282) (0.05,0.0273678) (0.1,0.033975) (0.2,0.0327093) (0.5,0.0210822) (1,0.0356428) (1,0.018296) (0.5,0.00915571) (0.2,0.00525296) (0.1,0.00707429) (0.05,0.00475598) (0.02,0.00243049) (0.01,0.00309176) (0.005,0.00802115) (0.002,0.00247) (0.001,0.00323065)} \closedcycle;
  \addplot[draw=none, fill=htSAM, fill opacity=0.06] coordinates {(0.001,0.0234151) (0.002,0.0275627) (0.005,0.0125338) (0.01,0.0276115) (0.02,0.0205578) (0.05,0.0194221) (0.1,0.035325) (0.2,0.0204864) (0.5,0.0303859) (1,0.0627348) (1,0.0170781) (0.5,0.0167536) (0.2,0.00244269) (0.1,0.0127609) (0.05,0.00336118) (0.02,0.00369947) (0.01,0.00706287) (0.005,0.00469177) (0.002,0.00737315) (0.001,0.00641784)} \closedcycle;
  \addplot[draw=none, fill=htMSAM, fill opacity=0.06] coordinates {(0.001,0.0205769) (0.002,0.0239433) (0.005,0.0260176) (0.01,0.0315598) (0.02,0.025933) (0.05,0.0161784) (0.1,0.042706) (0.2,0.0418296) (0.5,0.0240601) (1,0.0306721) (1,0.0150456) (0.5,0.00799723) (0.2,0.00708152) (0.1,0.0102923) (0.05,0.00776148) (0.02,0.00461664) (0.01,0.00875817) (0.005,0.0159392) (0.002,0.00345436) (0.001,0.012252)} \closedcycle;
  \addplot[draw=none, fill=htSAMAdam, fill opacity=0.06] coordinates {(0.001,1.18151) (0.002,1.18146) (0.005,1.18136) (0.01,1.18154) (0.02,1.18174) (0.05,1.18231) (0.1,1.18426) (0.2,1.18698) (0.5,1.19747) (1,1.24102) (1,1.20608) (0.5,1.18209) (0.2,1.16777) (0.1,1.16367) (0.05,1.16285) (0.02,1.16097) (0.01,1.16083) (0.005,1.16066) (0.002,1.16063) (0.001,1.16062)} \closedcycle;
  \addplot[draw=none, fill=htFSAMAdam, fill opacity=0.06] coordinates {(0.001,1.1815) (0.002,1.18145) (0.005,1.18126) (0.01,1.18133) (0.02,1.18167) (0.05,1.18265) (0.1,1.1837) (0.2,1.18507) (0.5,1.19179) (1,1.21961) (1,1.19539) (0.5,1.17647) (0.2,1.1653) (0.1,1.16291) (0.05,1.1619) (0.02,1.16094) (0.01,1.16061) (0.005,1.16064) (0.002,1.16065) (0.001,1.16061)} \closedcycle;
  \addplot[color=htMuon, dashed, line width=0.7pt, forget plot] coordinates {(0.001,0.481691) (1,0.481691)};
  \addplot[color=htAdam, dash dot, line width=0.7pt, forget plot] coordinates {(0.001,1.17217) (1,1.17217)};
  \addplot[color=htSpec, solid, line width=0.8pt] coordinates {(0.001,0.482089) (0.002,0.483082) (0.005,0.48282) (0.01,0.485479) (0.02,0.487399) (0.05,0.493697) (0.1,0.519246) (0.2,0.570916) (0.5,0.852351) (1,1.5179)};
  \addplot[color=htRand, dashed, line width=0.8pt] coordinates {(0.001,0.482008) (0.002,0.482596) (0.005,0.48203) (0.01,0.482423) (0.02,0.482236) (0.05,0.481534) (0.1,0.483139) (0.2,0.485664) (0.5,0.483182) (1,0.485931)};
  \addplot[color=htMSOMA, loosely dashed, line width=0.8pt] coordinates {(0.001,0.481739) (0.002,0.482099) (0.005,0.481132) (0.01,0.479195) (0.02,0.480507) (0.05,0.478013) (0.1,0.477486) (0.2,0.472945) (0.5,0.461217) (1,0.445841)};
  \addplot[color=htFSAM, solid, line width=0.8pt] coordinates {(0.001,0.00839691) (0.002,0.00863778) (0.005,0.017654) (0.01,0.0144423) (0.02,0.0146255) (0.05,0.0135943) (0.1,0.0258487) (0.2,0.0114496) (0.5,0.0140091) (1,0.0234814)};
  \addplot[color=htSAM, solid, line width=0.8pt] coordinates {(0.001,0.00962116) (0.002,0.0129291) (0.005,0.00927378) (0.01,0.015674) (0.02,0.00850781) (0.05,0.0102833) (0.1,0.0226077) (0.2,0.00807392) (0.5,0.0214171) (1,0.045277)};
  \addplot[color=htMSAM, solid, line width=0.8pt] coordinates {(0.001,0.0150743) (0.002,0.0179781) (0.005,0.021057) (0.01,0.0166914) (0.02,0.0139644) (0.05,0.00956259) (0.1,0.024143) (0.2,0.0155195) (0.5,0.0144715) (1,0.0246303)};
  \addplot[color=htSOMA, solid, line width=1.1pt] coordinates {(0.001,0.482837) (0.002,0.482758) (0.005,0.484486) (0.01,0.484442) (0.02,0.483508) (0.05,0.485534) (0.1,0.489961) (0.2,0.501103) (0.5,0.603501) (1,0.998586)};
  \addplot[color=htFPSOMA, solid, line width=1.1pt] coordinates {(0.001,0.48142) (0.002,0.482763) (0.005,0.481554) (0.01,0.482555) (0.02,0.484123) (0.05,0.486966) (0.1,0.503143) (0.2,0.538441) (0.5,0.754552) (1,1.26895)};
  \addplot[color=htPreNS5, densely dotted, line width=1.1pt] coordinates {(0.001,0.481813) (0.002,0.482198) (0.005,0.48369) (0.01,0.482414) (0.02,0.481282) (0.05,0.482544) (0.1,0.4856) (0.2,0.495966) (0.5,0.581704) (1,0.937146)};
  \addplot[color=htPostNS5, dash dot, line width=1.1pt] coordinates {(0.001,0.481851) (0.002,0.481431) (0.005,0.483171) (0.01,0.482997) (0.02,0.482466) (0.05,0.482173) (0.1,0.483297) (0.2,0.491277) (0.5,0.542725) (1,0.777903)};
  \addplot[color=htSAMAdam, solid, line width=1.1pt] coordinates {(0.001,1.17242) (0.002,1.1724) (0.005,1.17283) (0.01,1.17311) (0.02,1.17363) (0.05,1.17408) (0.1,1.17532) (0.2,1.17857) (0.5,1.19498) (1,1.21886)};
  \addplot[color=htFSAMAdam, densely dashed, line width=1.4pt] coordinates {(0.001,1.17252) (0.002,1.17244) (0.005,1.1728) (0.01,1.17303) (0.02,1.17343) (0.05,1.1732) (0.1,1.17419) (0.2,1.176) (0.5,1.18799) (1,1.20524)};
\end{axis}
\begin{axis}[name=ax2_r0,
    at={(0.5650\textwidth,0in)},
    anchor=south west,
    ht panel, xmode=log, ymode=log,
    ymin=0.00535546, ymax=2.55782,
    ytick={0.01,0.1,1},
    yticklabels={0.01,0.1,1},
    title={$\alpha=2.0$}]
  \addplot[draw=none, fill=htSOMA, fill opacity=0.06] coordinates {(0.001,0.240982) (0.002,0.241685) (0.005,0.240036) (0.01,0.240623) (0.02,0.240598) (0.05,0.241758) (0.1,0.244212) (0.2,0.257314) (0.5,0.404796) (1,0.865576) (1,0.839769) (0.5,0.372981) (0.2,0.244155) (0.1,0.234731) (0.05,0.234023) (0.02,0.230613) (0.01,0.232689) (0.005,0.23185) (0.002,0.231839) (0.001,0.232752)} \closedcycle;
  \addplot[draw=none, fill=htFPSOMA, fill opacity=0.06] coordinates {(0.001,0.240254) (0.002,0.241281) (0.005,0.24118) (0.01,0.240802) (0.02,0.242692) (0.05,0.245454) (0.1,0.257788) (0.2,0.297206) (0.5,0.564862) (1,1.34459) (1,1.31182) (0.5,0.532019) (0.2,0.284809) (0.1,0.244456) (0.05,0.234904) (0.02,0.232516) (0.01,0.233494) (0.005,0.234078) (0.002,0.232568) (0.001,0.232138)} \closedcycle;
  \addplot[draw=none, fill=htSpec, fill opacity=0.06] coordinates {(0.001,0.239925) (0.002,0.241024) (0.005,0.242813) (0.01,0.243166) (0.02,0.246377) (0.05,0.252227) (0.1,0.266438) (0.2,0.325237) (0.5,0.674679) (1,1.69276) (1,1.66343) (0.5,0.642444) (0.2,0.313367) (0.1,0.256968) (0.05,0.242437) (0.02,0.23698) (0.01,0.234943) (0.005,0.233167) (0.002,0.233821) (0.001,0.231723)} \closedcycle;
  \addplot[draw=none, fill=htRand, fill opacity=0.06] coordinates {(0.001,0.239358) (0.002,0.240341) (0.005,0.240602) (0.01,0.24226) (0.02,0.240807) (0.05,0.240292) (0.1,0.240016) (0.2,0.240733) (0.5,0.236646) (1,0.241262) (1,0.229341) (0.5,0.228243) (0.2,0.232093) (0.1,0.233822) (0.05,0.233092) (0.02,0.232909) (0.01,0.23176) (0.005,0.231618) (0.002,0.232156) (0.001,0.231441)} \closedcycle;
  \addplot[draw=none, fill=htPreNS5, fill opacity=0.06] coordinates {(0.001,0.241987) (0.002,0.240633) (0.005,0.24134) (0.01,0.239886) (0.02,0.24019) (0.05,0.239511) (0.1,0.239728) (0.2,0.249986) (0.5,0.364008) (1,0.801174) (1,0.775165) (0.5,0.342123) (0.2,0.236916) (0.1,0.230009) (0.05,0.231064) (0.02,0.231097) (0.01,0.231047) (0.005,0.231395) (0.002,0.233188) (0.001,0.232478)} \closedcycle;
  \addplot[draw=none, fill=htPostNS5, fill opacity=0.06] coordinates {(0.001,0.241213) (0.002,0.240474) (0.005,0.241348) (0.01,0.239362) (0.02,0.238323) (0.05,0.24054) (0.1,0.238174) (0.2,0.239627) (0.5,0.291709) (1,0.544607) (1,0.522677) (0.5,0.279103) (0.2,0.230894) (0.1,0.229385) (0.05,0.229558) (0.02,0.231892) (0.01,0.232161) (0.005,0.233645) (0.002,0.231526) (0.001,0.233266)} \closedcycle;
  \addplot[draw=none, fill=htMSOMA, fill opacity=0.06] coordinates {(0.001,0.240157) (0.002,0.240098) (0.005,0.240575) (0.01,0.241368) (0.02,0.239786) (0.05,0.239135) (0.1,0.236415) (0.2,0.231187) (0.5,0.222306) (1,0.208712) (1,0.195318) (0.5,0.210387) (0.2,0.223259) (0.1,0.226652) (0.05,0.230017) (0.02,0.230706) (0.01,0.231831) (0.005,0.232398) (0.002,0.232696) (0.001,0.231782)} \closedcycle;
  \addplot[draw=none, fill=htFSAM, fill opacity=0.06] coordinates {(0.001,0.0106229) (0.002,0.0145217) (0.005,0.0112089) (0.01,0.0111706) (0.02,0.0115103) (0.05,0.00957768) (0.1,0.0124218) (0.2,0.0107345) (0.5,0.0118917) (1,0.0251472) (1,0.0166237) (0.5,0.0084587) (0.2,0.00688647) (0.1,0.00657964) (0.05,0.00716162) (0.02,0.00731434) (0.01,0.00719583) (0.005,0.00755245) (0.002,0.00737819) (0.001,0.00779629)} \closedcycle;
  \addplot[draw=none, fill=htSAM, fill opacity=0.06] coordinates {(0.001,0.0156956) (0.002,0.0116813) (0.005,0.00981808) (0.01,0.0098356) (0.02,0.0113461) (0.05,0.0124503) (0.1,0.0107509) (0.2,0.0112835) (0.5,0.0127391) (1,0.0284082) (1,0.0165801) (0.5,0.00818313) (0.2,0.00769195) (0.1,0.0066486) (0.05,0.00694561) (0.02,0.00674378) (0.01,0.00702374) (0.005,0.00693186) (0.002,0.0076194) (0.001,0.00844026)} \closedcycle;
  \addplot[draw=none, fill=htMSAM, fill opacity=0.06] coordinates {(0.001,0.0179216) (0.002,0.0166366) (0.005,0.0159823) (0.01,0.016229) (0.02,0.0155448) (0.05,0.016745) (0.1,0.0171068) (0.2,0.0193446) (0.5,0.0230317) (1,0.0190052) (1,0.00976886) (0.5,0.0123684) (0.2,0.0118591) (0.1,0.0113649) (0.05,0.0104705) (0.02,0.0100073) (0.01,0.0113629) (0.005,0.0110267) (0.002,0.0113668) (0.001,0.0110447)} \closedcycle;
  \addplot[draw=none, fill=htSAMAdam, fill opacity=0.06] coordinates {(0.001,0.983136) (0.002,0.982971) (0.005,0.983375) (0.01,0.983743) (0.02,0.985077) (0.05,0.986271) (0.1,0.990966) (0.2,0.99801) (0.5,1.02639) (1,1.10847) (1,1.04724) (0.5,1.00738) (0.2,0.983215) (0.1,0.976672) (0.05,0.97265) (0.02,0.970867) (0.01,0.970682) (0.005,0.970452) (0.002,0.970184) (0.001,0.970172)} \closedcycle;
  \addplot[draw=none, fill=htFSAMAdam, fill opacity=0.06] coordinates {(0.001,0.983138) (0.002,0.983016) (0.005,0.983383) (0.01,0.983711) (0.02,0.984599) (0.05,0.985746) (0.1,0.989931) (0.2,0.994296) (0.5,1.01229) (1,1.08071) (1,1.02934) (0.5,0.999225) (0.2,0.980061) (0.1,0.974964) (0.05,0.972048) (0.02,0.97053) (0.01,0.970355) (0.005,0.970341) (0.002,0.970158) (0.001,0.970268)} \closedcycle;
  \addplot[color=htMuon, dashed, line width=0.7pt, forget plot] coordinates {(0.001,0.232812) (1,0.232812)};
  \addplot[color=htAdam, dash dot, line width=0.7pt, forget plot] coordinates {(0.001,0.973999) (1,0.973999)};
  \addplot[color=htSpec, solid, line width=0.8pt] coordinates {(0.001,0.23581) (0.002,0.237146) (0.005,0.236113) (0.01,0.236546) (0.02,0.239111) (0.05,0.247341) (0.1,0.260404) (0.2,0.316284) (0.5,0.651849) (1,1.68199)};
  \addplot[color=htRand, dashed, line width=0.8pt] coordinates {(0.001,0.232906) (0.002,0.233384) (0.005,0.233441) (0.01,0.233473) (0.02,0.237323) (0.05,0.235819) (0.1,0.235835) (0.2,0.233273) (0.5,0.232788) (1,0.235966)};
  \addplot[color=htMSOMA, loosely dashed, line width=0.8pt] coordinates {(0.001,0.233262) (0.002,0.23333) (0.005,0.234245) (0.01,0.235339) (0.02,0.235052) (0.05,0.232761) (0.1,0.23101) (0.2,0.224714) (0.5,0.217627) (1,0.202537)};
  \addplot[color=htFSAM, solid, line width=0.8pt] coordinates {(0.001,0.00985623) (0.002,0.00951544) (0.005,0.00898232) (0.01,0.00903769) (0.02,0.00885697) (0.05,0.00814414) (0.1,0.00867081) (0.2,0.00897108) (0.5,0.0105665) (1,0.0183757)};
  \addplot[color=htSAM, solid, line width=0.8pt] coordinates {(0.001,0.00989313) (0.002,0.00860804) (0.005,0.00878095) (0.01,0.00872513) (0.02,0.00916434) (0.05,0.00892066) (0.1,0.00873828) (0.2,0.0087599) (0.5,0.0101553) (1,0.0232806)};
  \addplot[color=htMSAM, solid, line width=0.8pt] coordinates {(0.001,0.0125688) (0.002,0.012276) (0.005,0.0114632) (0.01,0.0124032) (0.02,0.0113721) (0.05,0.0118776) (0.1,0.0129194) (0.2,0.0124876) (0.5,0.0133165) (1,0.0127657)};
  \addplot[color=htSOMA, solid, line width=1.1pt] coordinates {(0.001,0.233554) (0.002,0.236337) (0.005,0.235587) (0.01,0.235202) (0.02,0.235157) (0.05,0.235246) (0.1,0.238281) (0.2,0.249801) (0.5,0.381174) (1,0.857423)};
  \addplot[color=htFPSOMA, solid, line width=1.1pt] coordinates {(0.001,0.236478) (0.002,0.237064) (0.005,0.236383) (0.01,0.235425) (0.02,0.233627) (0.05,0.239077) (0.1,0.250814) (0.2,0.287162) (0.5,0.539283) (1,1.32877)};
  \addplot[color=htPreNS5, densely dotted, line width=1.1pt] coordinates {(0.001,0.237153) (0.002,0.234926) (0.005,0.233482) (0.01,0.235439) (0.02,0.234022) (0.05,0.233375) (0.1,0.234588) (0.2,0.241574) (0.5,0.347795) (1,0.782736)};
  \addplot[color=htPostNS5, dash dot, line width=1.1pt] coordinates {(0.001,0.237098) (0.002,0.23514) (0.005,0.235894) (0.01,0.235236) (0.02,0.23395) (0.05,0.232743) (0.1,0.232574) (0.2,0.234413) (0.5,0.279997) (1,0.529857)};
  \addplot[color=htSAMAdam, solid, line width=1.1pt] coordinates {(0.001,0.974103) (0.002,0.974041) (0.005,0.974049) (0.01,0.974524) (0.02,0.975014) (0.05,0.978443) (0.1,0.982213) (0.2,0.991637) (0.5,1.01524) (1,1.08262)};
  \addplot[color=htFSAMAdam, densely dashed, line width=1.4pt] coordinates {(0.001,0.974141) (0.002,0.97392) (0.005,0.974255) (0.01,0.974312) (0.02,0.974452) (0.05,0.978005) (0.1,0.981219) (0.2,0.987561) (0.5,1.00623) (1,1.05644)};
\end{axis}
\begin{axis}[name=ax3_r0,
    at={(0.8065\textwidth,0in)},
    anchor=south west,
    ht panel, xmode=log, ymode=log,
    ymin=0.00720041, ymax=2.98533,
    ytick={0.1,1},
    yticklabels={0.1,1},
    title={$\alpha=3.0$}]
  \addplot[draw=none, fill=htSOMA, fill opacity=0.06] coordinates {(0.001,0.11576) (0.002,0.115964) (0.005,0.114897) (0.01,0.11483) (0.02,0.115287) (0.05,0.116325) (0.1,0.120344) (0.2,0.131294) (0.5,0.264473) (1,0.832321) (1,0.795293) (0.5,0.241249) (0.2,0.124358) (0.1,0.11304) (0.05,0.113555) (0.02,0.112481) (0.01,0.111408) (0.005,0.110989) (0.002,0.112358) (0.001,0.112071)} \closedcycle;
  \addplot[draw=none, fill=htFPSOMA, fill opacity=0.06] coordinates {(0.001,0.115964) (0.002,0.116099) (0.005,0.114942) (0.01,0.115945) (0.02,0.114192) (0.05,0.115257) (0.1,0.124934) (0.2,0.163236) (0.5,0.484413) (1,1.56936) (1,1.5358) (0.5,0.457059) (0.2,0.154987) (0.1,0.121938) (0.05,0.112728) (0.02,0.110663) (0.01,0.112251) (0.005,0.111579) (0.002,0.112103) (0.001,0.111592)} \closedcycle;
  \addplot[draw=none, fill=htSpec, fill opacity=0.06] coordinates {(0.001,0.115713) (0.002,0.117185) (0.005,0.114968) (0.01,0.115864) (0.02,0.115274) (0.05,0.118512) (0.1,0.13471) (0.2,0.186057) (0.5,0.614181) (1,2.00447) (1,1.97242) (0.5,0.591153) (0.2,0.181054) (0.1,0.12743) (0.05,0.116417) (0.02,0.112501) (0.01,0.111606) (0.005,0.112973) (0.002,0.113047) (0.001,0.111426)} \closedcycle;
  \addplot[draw=none, fill=htRand, fill opacity=0.06] coordinates {(0.001,0.115689) (0.002,0.116266) (0.005,0.115982) (0.01,0.115109) (0.02,0.11514) (0.05,0.114867) (0.1,0.11533) (0.2,0.115015) (0.5,0.115971) (1,0.115263) (1,0.111457) (0.5,0.109259) (0.2,0.111431) (0.1,0.112012) (0.05,0.112258) (0.02,0.111734) (0.01,0.111845) (0.005,0.111783) (0.002,0.112199) (0.001,0.11156)} \closedcycle;
  \addplot[draw=none, fill=htPreNS5, fill opacity=0.06] coordinates {(0.001,0.115476) (0.002,0.116252) (0.005,0.116077) (0.01,0.116731) (0.02,0.115304) (0.05,0.11605) (0.1,0.117772) (0.2,0.127194) (0.5,0.223465) (1,0.733485) (1,0.707796) (0.5,0.209618) (0.2,0.119744) (0.1,0.112652) (0.05,0.112666) (0.02,0.111479) (0.01,0.112056) (0.005,0.11105) (0.002,0.111013) (0.001,0.112107)} \closedcycle;
  \addplot[draw=none, fill=htPostNS5, fill opacity=0.06] coordinates {(0.001,0.115357) (0.002,0.115908) (0.005,0.115609) (0.01,0.115628) (0.02,0.115893) (0.05,0.115473) (0.1,0.113231) (0.2,0.114603) (0.5,0.153322) (1,0.383569) (1,0.362572) (0.5,0.148504) (0.2,0.109982) (0.1,0.111239) (0.05,0.111958) (0.02,0.110927) (0.01,0.111655) (0.005,0.111727) (0.002,0.112128) (0.001,0.111938)} \closedcycle;
  \addplot[draw=none, fill=htMSOMA, fill opacity=0.06] coordinates {(0.001,0.115871) (0.002,0.115968) (0.005,0.115032) (0.01,0.114665) (0.02,0.114812) (0.05,0.113732) (0.1,0.111521) (0.2,0.110886) (0.5,0.102305) (1,0.0957148) (1,0.0877093) (0.5,0.0961781) (0.2,0.102927) (0.1,0.106706) (0.05,0.109374) (0.02,0.111635) (0.01,0.112008) (0.005,0.111979) (0.002,0.111742) (0.001,0.112106)} \closedcycle;
  \addplot[draw=none, fill=htFSAM, fill opacity=0.06] coordinates {(0.001,0.012947) (0.002,0.0136105) (0.005,0.0133592) (0.01,0.0129816) (0.02,0.0133913) (0.05,0.0125967) (0.1,0.012153) (0.2,0.0116656) (0.5,0.0124708) (1,0.0263863) (1,0.0157437) (0.5,0.0109722) (0.2,0.0103984) (0.1,0.0106819) (0.05,0.0109704) (0.02,0.0115431) (0.01,0.0118) (0.005,0.0114781) (0.002,0.011818) (0.001,0.0118551)} \closedcycle;
  \addplot[draw=none, fill=htSAM, fill opacity=0.06] coordinates {(0.001,0.0130581) (0.002,0.013459) (0.005,0.013522) (0.01,0.0131678) (0.02,0.0132579) (0.05,0.0123365) (0.1,0.0124056) (0.2,0.0116556) (0.5,0.0126639) (1,0.0265026) (1,0.0159869) (0.5,0.0113223) (0.2,0.0101814) (0.1,0.0105085) (0.05,0.0107584) (0.02,0.0116112) (0.01,0.0115784) (0.005,0.0118046) (0.002,0.0115501) (0.001,0.0120539)} \closedcycle;
  \addplot[draw=none, fill=htMSAM, fill opacity=0.06] coordinates {(0.001,0.0152354) (0.002,0.015357) (0.005,0.0154536) (0.01,0.0154125) (0.02,0.0156442) (0.05,0.0159442) (0.1,0.0162053) (0.2,0.0163122) (0.5,0.0158633) (1,0.0150123) (1,0.0130322) (0.5,0.0148062) (0.2,0.0143985) (0.1,0.0142335) (0.05,0.0138549) (0.02,0.0137968) (0.01,0.0137619) (0.005,0.0137297) (0.002,0.0131623) (0.001,0.0133426)} \closedcycle;
  \addplot[draw=none, fill=htSAMAdam, fill opacity=0.06] coordinates {(0.001,0.82543) (0.002,0.825529) (0.005,0.825555) (0.01,0.825997) (0.02,0.827739) (0.05,0.831149) (0.1,0.836848) (0.2,0.850192) (0.5,0.905923) (1,1.04709) (1,0.954533) (0.5,0.869618) (0.2,0.832675) (0.1,0.820906) (0.05,0.815987) (0.02,0.813316) (0.01,0.812503) (0.005,0.81192) (0.002,0.811745) (0.001,0.811936)} \closedcycle;
  \addplot[draw=none, fill=htFSAMAdam, fill opacity=0.06] coordinates {(0.001,0.825427) (0.002,0.825497) (0.005,0.825896) (0.01,0.825762) (0.02,0.826845) (0.05,0.829246) (0.1,0.833911) (0.2,0.843517) (0.5,0.888771) (1,1.00926) (1,0.926699) (0.5,0.859573) (0.2,0.829287) (0.1,0.820187) (0.05,0.814522) (0.02,0.812371) (0.01,0.812262) (0.005,0.811756) (0.002,0.81191) (0.001,0.811934)} \closedcycle;
  \addplot[color=htMuon, dashed, line width=0.7pt, forget plot] coordinates {(0.001,0.113683) (1,0.113683)};
  \addplot[color=htAdam, dash dot, line width=0.7pt, forget plot] coordinates {(0.001,0.816596) (1,0.816596)};
  \addplot[color=htSpec, solid, line width=0.8pt] coordinates {(0.001,0.113277) (0.002,0.114479) (0.005,0.114132) (0.01,0.113266) (0.02,0.113552) (0.05,0.117825) (0.1,0.131845) (0.2,0.183994) (0.5,0.602469) (1,1.98023)};
  \addplot[color=htRand, dashed, line width=0.8pt] coordinates {(0.001,0.113957) (0.002,0.114197) (0.005,0.114426) (0.01,0.113897) (0.02,0.112963) (0.05,0.113839) (0.1,0.113871) (0.2,0.112547) (0.5,0.112754) (1,0.11348)};
  \addplot[color=htMSOMA, loosely dashed, line width=0.8pt] coordinates {(0.001,0.11325) (0.002,0.113072) (0.005,0.113025) (0.01,0.113885) (0.02,0.112935) (0.05,0.111117) (0.1,0.109481) (0.2,0.107093) (0.5,0.0987706) (1,0.0911111)};
  \addplot[color=htFSAM, solid, line width=0.8pt] coordinates {(0.001,0.0126703) (0.002,0.0126962) (0.005,0.0129123) (0.01,0.0123726) (0.02,0.0124584) (0.05,0.0116018) (0.1,0.0114163) (0.2,0.011071) (0.5,0.0116485) (1,0.0196711)};
  \addplot[color=htSAM, solid, line width=0.8pt] coordinates {(0.001,0.0127684) (0.002,0.0128756) (0.005,0.0128707) (0.01,0.01264) (0.02,0.0127974) (0.05,0.0116868) (0.1,0.0115524) (0.2,0.0108551) (0.5,0.011745) (1,0.0216226)};
  \addplot[color=htMSAM, solid, line width=0.8pt] coordinates {(0.001,0.0142892) (0.002,0.0144769) (0.005,0.0143404) (0.01,0.0145558) (0.02,0.0147665) (0.05,0.015096) (0.1,0.0149902) (0.2,0.0150914) (0.5,0.0155484) (1,0.0138095)};
  \addplot[color=htSOMA, solid, line width=1.1pt] coordinates {(0.001,0.114129) (0.002,0.113183) (0.005,0.113525) (0.01,0.112784) (0.02,0.113737) (0.05,0.114181) (0.1,0.116819) (0.2,0.127586) (0.5,0.248203) (1,0.807998)};
  \addplot[color=htFPSOMA, solid, line width=1.1pt] coordinates {(0.001,0.113121) (0.002,0.11388) (0.005,0.113994) (0.01,0.113609) (0.02,0.111567) (0.05,0.113957) (0.1,0.122826) (0.2,0.15926) (0.5,0.468781) (1,1.54739)};
  \addplot[color=htPreNS5, densely dotted, line width=1.1pt] coordinates {(0.001,0.11408) (0.002,0.113076) (0.005,0.112953) (0.01,0.113927) (0.02,0.113496) (0.05,0.113533) (0.1,0.113857) (0.2,0.122684) (0.5,0.217382) (1,0.725691)};
  \addplot[color=htPostNS5, dash dot, line width=1.1pt] coordinates {(0.001,0.113959) (0.002,0.113952) (0.005,0.112727) (0.01,0.113467) (0.02,0.112778) (0.05,0.114092) (0.1,0.112068) (0.2,0.112208) (0.5,0.150253) (1,0.372596)};
  \addplot[color=htSAMAdam, solid, line width=1.1pt] coordinates {(0.001,0.816751) (0.002,0.817329) (0.005,0.81738) (0.01,0.818654) (0.02,0.819837) (0.05,0.823475) (0.1,0.82886) (0.2,0.841604) (0.5,0.892444) (1,1.00086)};
  \addplot[color=htFSAMAdam, densely dashed, line width=1.4pt] coordinates {(0.001,0.816806) (0.002,0.817223) (0.005,0.816876) (0.01,0.818269) (0.02,0.819037) (0.05,0.822069) (0.1,0.825042) (0.2,0.835828) (0.5,0.877623) (1,0.967798)};
\end{axis}
\path (ax0_r0.south west) -- (ax3_r0.south east) coordinate[midway] (xlab);
\node[font=\small, anchor=north] at ([yshift=-14pt]xlab) {perturbation radius $\rho$};
\end{tikzpicture}

%% file: dataforsyntheticplots/nn_iso_val_acc.tex
% ---- nn_iso_val_acc
\providecolor{muonrow}{RGB}{208,236,214}\providecolor{adamrow}{RGB}{255,214,214}\providecolor{sgdrow}{RGB}{208,226,247}
\begin{table*}[!t]\centering\footnotesize\setlength{\tabcolsep}{5pt}
\caption{Validation accuracy (\%): best value over the perturbation radius $\rho$ and the radius $\rho^\star$ where it is attained (median of 8 seeds; Frobenius-ball methods pool the global and per-layer radius modes). Row colour marks the base optimizer family: \muon-based (\colorbox{muonrow}{Green}), \adamw-based (\colorbox{adamrow}{Red}), \sgd-based (\colorbox{sgdrow}{Blue}); baselines have no $\rho^\star$. Best value per $\alpha$ in bold.}\label{tab:nn_iso_val_acc}
\vspace{2pt}
\resizebox{\ifdim\width>\textwidth\textwidth\else\width\fi}{!}{%
\begin{tabular}{lcccccccc}\toprule
Algorithm & \multicolumn{2}{c}{$\alpha=1.1$} & \multicolumn{2}{c}{$\alpha=1.6$} & \multicolumn{2}{c}{$\alpha=2.0$} & \multicolumn{2}{c}{$\alpha=3.0$} \\
& best acc. & $\rho^\star$ & best acc. & $\rho^\star$ & best acc. & $\rho^\star$ & best acc. & $\rho^\star$ \\\midrule
\rowcolor{muonrow} \muon & 36.1 & -- & 43.3 & -- & 45.2 & -- & 43.8 & -- \\
\rowcolor{muonrow} \specsammuon & 36.6 & 0.2 & 44.2 & 0.5 & 45.5 & 0.005 & 46.2 & 0.5 \\
\rowcolor{muonrow} \fpsoma & 37.8 & 1 & 44.7 & 1 & 46.6 & 0.5 & 46.1 & 0.5 \\
\rowcolor{muonrow} \soma & 38.2 & 1 & 44.8 & 1 & 46.6 & 0.5 & 46.6 & 0.5 \\
\rowcolor{muonrow} \randsammuon & 36.3 & 0.05 & 43.6 & 0.1 & 45.2 & 0.2 & 44.8 & 0.2 \\
\rowcolor{muonrow} \opsomapre & 38.6 & 1 & 44.2 & 1 & 45.7 & 0.5 & 46.8 & 0.5 \\
\rowcolor{muonrow} \opsomapost & 37.9 & 1 & \textbf{45.2} & 1 & \textbf{46.8} & 1 & \textbf{48.3} & 1 \\
\rowcolor{muonrow} \textsc{Msoma} & 36.2 & 0.1 & 43.5 & 0.002 & 45.4 & 0.05 & 44.3 & 0.05 \\
\rowcolor{adamrow} \adamw & 40.6 & -- & 43 & -- & 44.1 & -- & 45.8 & -- \\
\rowcolor{adamrow} \samadamw & 40.8 & 0.02 & 43.5 & 0.2 & 44.5 & 0.002 & 46 & 0.01 \\
\rowcolor{adamrow} \fsamadamw & \textbf{40.9} & 0.2 & 43.4 & 0.05 & 44.6 & 0.05 & 46.2 & 1 \\
\rowcolor{sgdrow} \sgd & 34.2 & -- & 38.7 & -- & 39.3 & -- & 41.2 & -- \\
\rowcolor{sgdrow} \sam & 35.6 & 0.5 & 39.6 & 0.1 & 39.8 & 0.05 & 42.1 & 0.01 \\
\rowcolor{sgdrow} \fsam & 35.8 & 0.5 & 40 & 0.5 & 39.6 & 0.5 & 42.7 & 1 \\
\rowcolor{sgdrow} \textsc{Msam} & 36.1 & 0.002 & 39.3 & 0.1 & 39.6 & 0.1 & 42.2 & 0.05 \\
\bottomrule\end{tabular}}\end{table*}

%% file: dataforsyntheticplots/nn_iso_train_loss.tex
% ---- nn_iso_train_loss
\providecolor{muonrow}{RGB}{208,236,214}\providecolor{adamrow}{RGB}{255,214,214}\providecolor{sgdrow}{RGB}{208,226,247}
\begin{table*}[!t]\centering\footnotesize\setlength{\tabcolsep}{5pt}
\caption{Training loss: best value over the perturbation radius $\rho$ and the radius $\rho^\star$ where it is attained (median of 8 seeds; Frobenius-ball methods pool the global and per-layer radius modes). Row colour marks the base optimizer family: \muon-based (\colorbox{muonrow}{Green}), \adamw-based (\colorbox{adamrow}{Red}), \sgd-based (\colorbox{sgdrow}{Blue}); baselines have no $\rho^\star$. Best value per $\alpha$ in bold.}\label{tab:nn_iso_train_loss}
\vspace{2pt}
\resizebox{\ifdim\width>\textwidth\textwidth\else\width\fi}{!}{%
\begin{tabular}{lcccccccc}\toprule
Algorithm & \multicolumn{2}{c}{$\alpha=1.1$} & \multicolumn{2}{c}{$\alpha=1.6$} & \multicolumn{2}{c}{$\alpha=2.0$} & \multicolumn{2}{c}{$\alpha=3.0$} \\
& best loss & $\rho^\star$ & best loss & $\rho^\star$ & best loss & $\rho^\star$ & best loss & $\rho^\star$ \\\midrule
\rowcolor{muonrow} \muon & 1.58 & -- & 0.482 & -- & 0.233 & -- & 0.114 & -- \\
\rowcolor{muonrow} \specsammuon & 1.58 & 0.02 & 0.482 & 0.001 & 0.236 & 0.001 & 0.113 & 0.01 \\
\rowcolor{muonrow} \fpsoma & 1.57 & 0.2 & 0.481 & 0.001 & 0.234 & 0.02 & 0.112 & 0.02 \\
\rowcolor{muonrow} \soma & \textbf{1.56} & 0.5 & 0.483 & 0.002 & 0.234 & 0.001 & 0.113 & 0.01 \\
\rowcolor{muonrow} \randsammuon & 1.58 & 0.005 & 0.482 & 0.05 & 0.233 & 0.5 & 0.113 & 0.2 \\
\rowcolor{muonrow} \opsomapre & 1.58 & 0.2 & 0.481 & 0.02 & 0.233 & 0.05 & 0.113 & 0.005 \\
\rowcolor{muonrow} \opsomapost & 1.57 & 0.5 & 0.481 & 0.002 & 0.233 & 0.1 & 0.112 & 0.1 \\
\rowcolor{muonrow} \textsc{Msoma} & \textbf{1.56} & 1 & 0.446 & 1 & 0.203 & 1 & 0.0911 & 1 \\
\rowcolor{adamrow} \adamw & 1.71 & -- & 1.17 & -- & 0.974 & -- & 0.817 & -- \\
\rowcolor{adamrow} \samadamw & 1.71 & 0.002 & 1.17 & 0.002 & 0.974 & 0.002 & 0.817 & 0.001 \\
\rowcolor{adamrow} \fsamadamw & 1.71 & 0.01 & 1.17 & 0.002 & 0.974 & 0.002 & 0.817 & 0.001 \\
\rowcolor{sgdrow} \sgd & 14.9 & -- & 0.0111 & -- & 0.00943 & -- & 0.0125 & -- \\
\rowcolor{sgdrow} \sam & 10.7 & 0.2 & \textbf{0.00807} & 0.2 & 0.00861 & 0.002 & \textbf{0.0109} & 0.2 \\
\rowcolor{sgdrow} \fsam & 11.7 & 1 & 0.0084 & 0.001 & \textbf{0.00814} & 0.05 & 0.0111 & 0.2 \\
\rowcolor{sgdrow} \textsc{Msam} & 18 & 0.005 & 0.00956 & 0.05 & 0.0114 & 0.02 & 0.0138 & 1 \\
\bottomrule\end{tabular}}\end{table*}

%% file: tikz_for_heavytailed/appendex_isotropic/gap_vs_rho.tikz
%% Auto-generated by make_heavytailed_figures.py  --  do not edit by hand, re-run the script.
%% Values: line = median over seeds, band = interquartile range (25th-75th percentile).
%% Preamble needed:  \usepackage{pgfplots}  \pgfplotsset{compat=1.16}   (nothing else)
%% Width follows \textwidth, fonts follow the document (so the legend can never overflow).

\definecolor{htSOMA}{HTML}{0B5D1E}
\definecolor{htFPSOMA}{HTML}{33CC33}
\definecolor{htSpec}{HTML}{00B3A4}
\definecolor{htRand}{HTML}{808000}
\definecolor{htFSAM}{HTML}{2F80D8}
\definecolor{htSAM}{HTML}{0A3A8A}
\definecolor{htPreNS5}{HTML}{9ACD32}
\definecolor{htPostNS5}{HTML}{004D40}
\definecolor{htSAMAdam}{HTML}{8E0F1C}
\definecolor{htFSAMAdam}{HTML}{F4623A}
\definecolor{htMSAM}{HTML}{7DB9F2}
\definecolor{htMSOMA}{HTML}{8FD694}
\definecolor{htMuon}{HTML}{111111}
\definecolor{htAdam}{HTML}{8A8A8A}
\pgfplotsset{
  ht panel/.style={
    scale only axis, width=0.1920\textwidth, height=1.35in,
    xmin=0.00075, xmax=1.35,
    xtick={0.001,0.01,0.1,1}, xticklabels={$10^{-3}$,$10^{-2}$,$10^{-1}$,$10^{0}$},
    minor xtick={0.002,0.003,0.004,0.005,0.006,0.007,0.008,0.009,0.02,0.03,0.04,0.05,0.06,0.07,0.08,0.09,0.2,0.3,0.4,0.5,0.6,0.7,0.8,0.9},
    xmajorgrids=true, ymajorgrids=true,
    grid style={line width=0.25pt, draw=black!10},
    axis line style={line width=0.4pt, black!75},
    tick style={line width=0.4pt, black!75}, major tick length=3pt, minor tick length=1.5pt,
    tick label style={font=\footnotesize, /pgf/number format/assume math mode=true},
    label style={font=\small}, title style={font=\small, yshift=-3pt},
    ylabel style={yshift=-1pt, align=center},
    every tick label/.append style={inner sep=1pt},
    clip=true,
  }}

\begin{tikzpicture}
\begin{axis}[name=ax0_r0,
    at={(0.0820\textwidth,0in)},
    anchor=south west,
    ht panel, xmode=log, ymode=log,
    ymin=0.290631, ymax=12.0123,
    ytick={1,3},
    yticklabels={1,3},
    title={$\alpha=1.1$},
    ylabel={optimality gap}]
  \addplot[draw=none, fill=htSOMA, fill opacity=0.06] coordinates {(0.001,1.04605) (0.002,1.0458) (0.005,1.04507) (0.01,1.04382) (0.02,1.0413) (0.05,1.0334) (0.1,1.01919) (0.2,0.987107) (0.5,0.8684) (1,0.662158) (1,0.594834) (0.5,0.787907) (0.2,0.901887) (0.1,0.931671) (0.05,0.944607) (0.02,0.951695) (0.01,0.953942) (0.005,0.955043) (0.002,0.955697) (0.001,0.955914)} \closedcycle;
  \addplot[draw=none, fill=htFPSOMA, fill opacity=0.06] coordinates {(0.001,1.04605) (0.002,1.04581) (0.005,1.04508) (0.01,1.04386) (0.02,1.04137) (0.05,1.03359) (0.1,1.01959) (0.2,0.98812) (0.5,0.874859) (1,0.684356) (1,0.612159) (0.5,0.791946) (0.2,0.901961) (0.1,0.931661) (0.05,0.944605) (0.02,0.951696) (0.01,0.953943) (0.005,0.955044) (0.002,0.955697) (0.001,0.955914)} \closedcycle;
  \addplot[draw=none, fill=htSpec, fill opacity=0.06] coordinates {(0.001,1.04591) (0.002,1.04552) (0.005,1.04436) (0.01,1.04243) (0.02,1.03856) (0.05,1.02688) (0.1,1.00721) (0.2,0.967302) (0.5,0.84805) (1,0.684648) (1,0.612853) (0.5,0.764371) (0.2,0.881357) (0.1,0.919579) (0.05,0.938106) (0.02,0.948986) (0.01,0.952569) (0.005,0.954352) (0.002,0.955419) (0.001,0.955775)} \closedcycle;
  \addplot[draw=none, fill=htRand, fill opacity=0.06] coordinates {(0.001,1.04629) (0.002,1.04629) (0.005,1.04629) (0.01,1.0463) (0.02,1.0463) (0.05,1.04629) (0.1,1.04624) (0.2,1.04601) (0.5,1.04432) (1,1.03818) (1,0.944106) (0.5,0.953069) (0.2,0.955628) (0.1,0.956) (0.05,0.956095) (0.02,0.956123) (0.01,0.956128) (0.005,0.956129) (0.002,0.956129) (0.001,0.95613)} \closedcycle;
  \addplot[draw=none, fill=htPreNS5, fill opacity=0.06] coordinates {(0.001,1.04605) (0.002,1.0458) (0.005,1.04506) (0.01,1.04382) (0.02,1.04128) (0.05,1.0333) (0.1,1.01887) (0.2,0.986073) (0.5,0.864904) (1,0.661654) (1,0.59595) (0.5,0.785459) (0.2,0.901148) (0.1,0.931487) (0.05,0.944573) (0.02,0.951698) (0.01,0.953946) (0.005,0.955046) (0.002,0.955698) (0.001,0.955914)} \closedcycle;
  \addplot[draw=none, fill=htPostNS5, fill opacity=0.06] coordinates {(0.001,1.04612) (0.002,1.04595) (0.005,1.04543) (0.01,1.04455) (0.02,1.04278) (0.05,1.03725) (0.1,1.02739) (0.2,1.00532) (0.5,0.922554) (1,0.754098) (1,0.676644) (0.5,0.837893) (0.2,0.918032) (0.1,0.938832) (0.05,0.947938) (0.02,0.952965) (0.01,0.954566) (0.005,0.955353) (0.002,0.95582) (0.001,0.955975)} \closedcycle;
  \addplot[draw=none, fill=htMSOMA, fill opacity=0.06] coordinates {(0.001,1.04628) (0.002,1.04627) (0.005,1.04623) (0.01,1.04617) (0.02,1.04604) (0.05,1.04565) (0.1,1.04497) (0.2,1.04349) (0.5,1.03826) (1,1.02727) (1,0.931454) (0.5,0.945639) (0.2,0.952449) (0.1,0.954379) (0.05,0.955277) (0.02,0.955794) (0.01,0.955963) (0.005,0.956046) (0.002,0.956096) (0.001,0.956113)} \closedcycle;
  \addplot[draw=none, fill=htFSAM, fill opacity=0.06] coordinates {(0.001,7.72839) (0.002,7.73304) (0.005,7.72269) (0.01,7.71787) (0.02,7.71799) (0.05,7.68931) (0.1,7.64608) (0.2,7.57823) (0.5,7.39683) (1,7.00727) (1,6.18385) (0.5,6.5945) (0.2,6.86924) (0.1,6.95579) (0.05,7.02078) (0.02,7.06002) (0.01,7.07445) (0.005,7.07925) (0.002,7.08372) (0.001,7.07909)} \closedcycle;
  \addplot[draw=none, fill=htSAM, fill opacity=0.06] coordinates {(0.001,7.72525) (0.002,7.72719) (0.005,7.72713) (0.01,7.71861) (0.02,7.70748) (0.05,7.68112) (0.1,7.63448) (0.2,7.54318) (0.5,7.36683) (1,6.96487) (1,6.26026) (0.5,6.64385) (0.2,6.87188) (0.1,6.95276) (0.05,7.02226) (0.02,7.06066) (0.01,7.07538) (0.005,7.07764) (0.002,7.08053) (0.001,7.08398)} \closedcycle;
  \addplot[draw=none, fill=htMSAM, fill opacity=0.06] coordinates {(0.001,9.28061) (0.002,9.2801) (0.005,9.27858) (0.01,9.27605) (0.02,9.27101) (0.05,9.25606) (0.1,9.24503) (0.2,9.24483) (0.5,9.24463) (1,9.23439) (1,7.22652) (0.5,7.38792) (0.2,7.44218) (0.1,7.45706) (0.05,7.46486) (0.02,7.46968) (0.01,7.47131) (0.005,7.47213) (0.002,7.47262) (0.001,7.47278)} \closedcycle;
  \addplot[draw=none, fill=htSAMAdam, fill opacity=0.06] coordinates {(0.001,0.388016) (0.002,0.388015) (0.005,0.388014) (0.01,0.388039) (0.02,0.388038) (0.05,0.388048) (0.1,0.388034) (0.2,0.388005) (0.5,0.387681) (1,0.386276) (1,0.373256) (0.5,0.375025) (0.2,0.375741) (0.1,0.375867) (0.05,0.375915) (0.02,0.375959) (0.01,0.375972) (0.005,0.375975) (0.002,0.375976) (0.001,0.375976)} \closedcycle;
  \addplot[draw=none, fill=htFSAMAdam, fill opacity=0.06] coordinates {(0.001,0.388017) (0.002,0.388016) (0.005,0.388015) (0.01,0.388041) (0.02,0.388043) (0.05,0.388061) (0.1,0.388076) (0.2,0.388066) (0.5,0.387876) (1,0.386817) (1,0.373742) (0.5,0.375214) (0.2,0.375792) (0.1,0.375896) (0.05,0.375929) (0.02,0.375964) (0.01,0.375974) (0.005,0.375976) (0.002,0.375976) (0.001,0.375976)} \closedcycle;
  \addplot[color=htMuon, dashed, line width=0.7pt, forget plot] coordinates {(0.001,1.01458) (1,1.01458)};
  \addplot[color=htAdam, dash dot, line width=0.7pt, forget plot] coordinates {(0.001,0.382231) (1,0.382231)};
  \addplot[color=htSpec, solid, line width=0.8pt] coordinates {(0.001,1.01418) (0.002,1.01379) (0.005,1.0126) (0.01,1.01062) (0.02,1.00665) (0.05,0.994708) (0.1,0.974637) (0.2,0.934023) (0.5,0.81307) (1,0.647714)};
  \addplot[color=htRand, dashed, line width=0.8pt] coordinates {(0.001,1.01458) (0.002,1.01458) (0.005,1.01458) (0.01,1.01458) (0.02,1.01458) (0.05,1.01456) (0.1,1.01449) (0.2,1.01421) (0.5,1.01215) (1,1.00479)};
  \addplot[color=htMSOMA, loosely dashed, line width=0.8pt] coordinates {(0.001,1.01456) (0.002,1.01454) (0.005,1.01449) (0.01,1.0144) (0.02,1.01421) (0.05,1.01365) (0.1,1.0127) (0.2,1.01069) (0.5,1.00395) (1,0.990634)};
  \addplot[color=htFSAM, solid, line width=0.8pt] coordinates {(0.001,7.43463) (0.002,7.43293) (0.005,7.42955) (0.01,7.42453) (0.02,7.42004) (0.05,7.41108) (0.1,7.38769) (0.2,7.33148) (0.5,7.07731) (1,6.58245)};
  \addplot[color=htSAM, solid, line width=0.8pt] coordinates {(0.001,7.43746) (0.002,7.43038) (0.005,7.43106) (0.01,7.4244) (0.02,7.42112) (0.05,7.40154) (0.1,7.36643) (0.2,7.2944) (0.5,7.05007) (1,6.54172)};
  \addplot[color=htMSAM, solid, line width=0.8pt] coordinates {(0.001,9.178) (0.002,9.17803) (0.005,9.17814) (0.01,9.17833) (0.02,9.1787) (0.05,9.17984) (0.1,9.16724) (0.2,9.13578) (0.5,9.0566) (1,8.97623)};
  \addplot[color=htSOMA, solid, line width=1.1pt] coordinates {(0.001,1.01432) (0.002,1.01407) (0.005,1.0133) (0.01,1.012) (0.02,1.00938) (0.05,1.00119) (0.1,0.986531) (0.2,0.953695) (0.5,0.83294) (1,0.629683)};
  \addplot[color=htFPSOMA, solid, line width=1.1pt] coordinates {(0.001,1.01434) (0.002,1.01409) (0.005,1.01336) (0.01,1.01214) (0.02,1.00964) (0.05,1.00184) (0.1,0.987807) (0.2,0.956286) (0.5,0.842819) (1,0.653618)};
  \addplot[color=htPreNS5, densely dotted, line width=1.1pt] coordinates {(0.001,1.01432) (0.002,1.01407) (0.005,1.0133) (0.01,1.01201) (0.02,1.00939) (0.05,1.00116) (0.1,0.986346) (0.2,0.952917) (0.5,0.829548) (1,0.628405)};
  \addplot[color=htPostNS5, dash dot, line width=1.1pt] coordinates {(0.001,1.0144) (0.002,1.01422) (0.005,1.01368) (0.01,1.01278) (0.02,1.01096) (0.05,1.00527) (0.1,0.995137) (0.2,0.972494) (0.5,0.887781) (1,0.716575)};
  \addplot[color=htSAMAdam, solid, line width=1.1pt] coordinates {(0.001,0.382231) (0.002,0.38223) (0.005,0.382231) (0.01,0.382221) (0.02,0.382225) (0.05,0.382188) (0.1,0.382196) (0.2,0.382142) (0.5,0.381709) (1,0.380398)};
  \addplot[color=htFSAMAdam, densely dashed, line width=1.4pt] coordinates {(0.001,0.38223) (0.002,0.382229) (0.005,0.38223) (0.01,0.382218) (0.02,0.382219) (0.05,0.382176) (0.1,0.382171) (0.2,0.382094) (0.5,0.381617) (1,0.380304)};
\end{axis}
\begin{axis}[name=ax1_r0,
    at={(0.3235\textwidth,0in)},
    anchor=south west,
    ht panel, xmode=log, ymode=log,
    ymin=0.11685, ymax=1.23769,
    ytick={0.2,0.5,1},
    yticklabels={0.2,0.5,1},
    title={$\alpha=1.6$}]
  \addplot[draw=none, fill=htSOMA, fill opacity=0.06] coordinates {(0.001,0.234101) (0.002,0.234063) (0.005,0.233943) (0.01,0.233732) (0.02,0.23327) (0.05,0.231566) (0.1,0.227725) (0.2,0.216876) (0.5,0.177753) (1,0.184118) (1,0.177653) (0.5,0.170592) (0.2,0.212251) (0.1,0.223036) (0.05,0.226846) (0.02,0.228498) (0.01,0.228938) (0.005,0.229137) (0.002,0.22925) (0.001,0.229286)} \closedcycle;
  \addplot[draw=none, fill=htFPSOMA, fill opacity=0.06] coordinates {(0.001,0.234089) (0.002,0.234038) (0.005,0.233881) (0.01,0.23361) (0.02,0.233027) (0.05,0.230974) (0.1,0.226615) (0.2,0.21514) (0.5,0.174678) (1,0.148874) (1,0.142889) (0.5,0.169049) (0.2,0.21024) (0.1,0.221759) (0.05,0.226144) (0.02,0.228206) (0.01,0.22879) (0.005,0.229063) (0.002,0.22922) (0.001,0.229271)} \closedcycle;
  \addplot[draw=none, fill=htSpec, fill opacity=0.06] coordinates {(0.001,0.234068) (0.002,0.233996) (0.005,0.233777) (0.01,0.233402) (0.02,0.23262) (0.05,0.230023) (0.1,0.224956) (0.2,0.212906) (0.5,0.175049) (1,0.143775) (1,0.135555) (0.5,0.168631) (0.2,0.208184) (0.1,0.220194) (0.05,0.225245) (0.02,0.227823) (0.01,0.228595) (0.005,0.228964) (0.002,0.22918) (0.001,0.229251)} \closedcycle;
  \addplot[draw=none, fill=htRand, fill opacity=0.06] coordinates {(0.001,0.234139) (0.002,0.234139) (0.005,0.234138) (0.01,0.234135) (0.02,0.23413) (0.05,0.234106) (0.1,0.234039) (0.2,0.23381) (0.5,0.232375) (1,0.227306) (1,0.221029) (0.5,0.227175) (0.2,0.228973) (0.1,0.229241) (0.05,0.229305) (0.02,0.229321) (0.01,0.229322) (0.005,0.229322) (0.002,0.229322) (0.001,0.229322)} \closedcycle;
  \addplot[draw=none, fill=htPreNS5, fill opacity=0.06] coordinates {(0.001,0.234105) (0.002,0.23407) (0.005,0.233962) (0.01,0.233771) (0.02,0.233348) (0.05,0.231762) (0.1,0.228113) (0.2,0.21761) (0.5,0.178372) (1,0.183321) (1,0.174213) (0.5,0.172027) (0.2,0.213174) (0.1,0.223554) (0.05,0.227107) (0.02,0.228602) (0.01,0.22899) (0.005,0.229163) (0.002,0.22926) (0.001,0.229291)} \closedcycle;
  \addplot[draw=none, fill=htPostNS5, fill opacity=0.06] coordinates {(0.001,0.234109) (0.002,0.234078) (0.005,0.233982) (0.01,0.233814) (0.02,0.233449) (0.05,0.232126) (0.1,0.22919) (0.2,0.220899) (0.5,0.186052) (1,0.163446) (1,0.15459) (0.5,0.180358) (0.2,0.216368) (0.1,0.224579) (0.05,0.22743) (0.02,0.228683) (0.01,0.229022) (0.005,0.229177) (0.002,0.229265) (0.001,0.229294)} \closedcycle;
  \addplot[draw=none, fill=htMSOMA, fill opacity=0.06] coordinates {(0.001,0.234141) (0.002,0.234143) (0.005,0.234147) (0.01,0.234153) (0.02,0.234166) (0.05,0.234193) (0.1,0.234208) (0.2,0.234133) (0.5,0.233052) (1,0.229808) (1,0.224152) (0.5,0.228213) (0.2,0.229277) (0.1,0.229345) (0.05,0.229345) (0.02,0.229334) (0.01,0.229328) (0.005,0.229325) (0.002,0.229323) (0.001,0.229323)} \closedcycle;
  \addplot[draw=none, fill=htFSAM, fill opacity=0.06] coordinates {(0.001,0.91038) (0.002,0.912436) (0.005,0.910472) (0.01,0.913643) (0.02,0.911778) (0.05,0.912001) (0.1,0.898406) (0.2,0.883704) (0.5,0.839569) (1,0.784064) (1,0.712044) (0.5,0.778604) (0.2,0.82358) (0.1,0.835754) (0.05,0.852676) (0.02,0.841318) (0.01,0.857753) (0.005,0.855353) (0.002,0.859019) (0.001,0.855437)} \closedcycle;
  \addplot[draw=none, fill=htSAM, fill opacity=0.06] coordinates {(0.001,0.910213) (0.002,0.911874) (0.005,0.912) (0.01,0.914011) (0.02,0.909594) (0.05,0.906095) (0.1,0.896137) (0.2,0.881207) (0.5,0.865786) (1,0.787978) (1,0.707005) (0.5,0.774408) (0.2,0.820954) (0.1,0.846767) (0.05,0.846505) (0.02,0.854285) (0.01,0.857323) (0.005,0.851145) (0.002,0.859179) (0.001,0.852678)} \closedcycle;
  \addplot[draw=none, fill=htMSAM, fill opacity=0.06] coordinates {(0.001,1.02559) (0.002,1.02567) (0.005,1.02593) (0.01,1.02636) (0.02,1.0272) (0.05,1.02955) (0.1,1.03305) (0.2,1.03894) (0.5,1.04914) (1,1.04975) (1,1.00861) (0.5,1.00402) (0.2,0.991148) (0.1,0.983559) (0.05,0.978798) (0.02,0.975608) (0.01,0.974483) (0.005,0.97391) (0.002,0.973562) (0.001,0.973446)} \closedcycle;
  \addplot[draw=none, fill=htSAMAdam, fill opacity=0.06] coordinates {(0.001,0.191847) (0.002,0.191845) (0.005,0.191841) (0.01,0.191834) (0.02,0.19182) (0.05,0.191765) (0.1,0.191658) (0.2,0.19142) (0.5,0.190428) (1,0.187486) (1,0.183259) (0.5,0.186767) (0.2,0.188208) (0.1,0.188544) (0.05,0.188656) (0.02,0.188718) (0.01,0.18875) (0.005,0.188758) (0.002,0.188767) (0.001,0.188769)} \closedcycle;
  \addplot[draw=none, fill=htFSAMAdam, fill opacity=0.06] coordinates {(0.001,0.191848) (0.002,0.191848) (0.005,0.191849) (0.01,0.191847) (0.02,0.191844) (0.05,0.191826) (0.1,0.191779) (0.2,0.191657) (0.5,0.190992) (1,0.188773) (1,0.184728) (0.5,0.187595) (0.2,0.188563) (0.1,0.188725) (0.05,0.188747) (0.02,0.188754) (0.01,0.188768) (0.005,0.188768) (0.002,0.18877) (0.001,0.188771)} \closedcycle;
  \addplot[color=htMuon, dashed, line width=0.7pt, forget plot] coordinates {(0.001,0.231438) (1,0.231438)};
  \addplot[color=htAdam, dash dot, line width=0.7pt, forget plot] coordinates {(0.001,0.190444) (1,0.190444)};
  \addplot[color=htSpec, solid, line width=0.8pt] coordinates {(0.001,0.231363) (0.002,0.231288) (0.005,0.23106) (0.01,0.230672) (0.02,0.229863) (0.05,0.227192) (0.1,0.222033) (0.2,0.209872) (0.5,0.17333) (1,0.14064)};
  \addplot[color=htRand, dashed, line width=0.8pt] coordinates {(0.001,0.231438) (0.002,0.231437) (0.005,0.231437) (0.01,0.231436) (0.02,0.231434) (0.05,0.231417) (0.1,0.231355) (0.2,0.231107) (0.5,0.229401) (1,0.224032)};
  \addplot[color=htMSOMA, loosely dashed, line width=0.8pt] coordinates {(0.001,0.231439) (0.002,0.23144) (0.005,0.231443) (0.01,0.231447) (0.02,0.231455) (0.05,0.231474) (0.1,0.231492) (0.2,0.231458) (0.5,0.230643) (1,0.226817)};
  \addplot[color=htFSAM, solid, line width=0.8pt] coordinates {(0.001,0.869861) (0.002,0.881149) (0.005,0.871199) (0.01,0.874054) (0.02,0.876519) (0.05,0.866013) (0.1,0.862371) (0.2,0.844312) (0.5,0.807743) (1,0.752656)};
  \addplot[color=htSAM, solid, line width=0.8pt] coordinates {(0.001,0.87619) (0.002,0.880571) (0.005,0.878768) (0.01,0.877434) (0.02,0.875209) (0.05,0.870121) (0.1,0.865189) (0.2,0.846338) (0.5,0.795754) (1,0.735458)};
  \addplot[color=htMSAM, solid, line width=0.8pt] coordinates {(0.001,1.00454) (0.002,1.00464) (0.005,1.00493) (0.01,1.00541) (0.02,1.00633) (0.05,1.00891) (0.1,1.01267) (0.2,1.01864) (0.5,1.02833) (1,1.02811)};
  \addplot[color=htSOMA, solid, line width=1.1pt] coordinates {(0.001,0.231395) (0.002,0.231352) (0.005,0.231218) (0.01,0.230985) (0.02,0.230478) (0.05,0.228644) (0.1,0.224617) (0.2,0.213635) (0.5,0.175774) (1,0.179914)};
  \addplot[color=htFPSOMA, solid, line width=1.1pt] coordinates {(0.001,0.231385) (0.002,0.231332) (0.005,0.231169) (0.01,0.230887) (0.02,0.230286) (0.05,0.228182) (0.1,0.223763) (0.2,0.212264) (0.5,0.172883) (1,0.146914)};
  \addplot[color=htPreNS5, densely dotted, line width=1.1pt] coordinates {(0.001,0.2314) (0.002,0.231362) (0.005,0.231243) (0.01,0.231036) (0.02,0.23058) (0.05,0.228897) (0.1,0.225105) (0.2,0.214475) (0.5,0.176539) (1,0.181215)};
  \addplot[color=htPostNS5, dash dot, line width=1.1pt] coordinates {(0.001,0.231403) (0.002,0.231368) (0.005,0.231261) (0.01,0.231075) (0.02,0.230674) (0.05,0.229246) (0.1,0.226162) (0.2,0.217689) (0.5,0.18473) (1,0.161745)};
  \addplot[color=htSAMAdam, solid, line width=1.1pt] coordinates {(0.001,0.190442) (0.002,0.190441) (0.005,0.190435) (0.01,0.190426) (0.02,0.190398) (0.05,0.19033) (0.1,0.190186) (0.2,0.189914) (0.5,0.188685) (1,0.185788)};
  \addplot[color=htFSAMAdam, densely dashed, line width=1.4pt] coordinates {(0.001,0.190444) (0.002,0.190444) (0.005,0.190444) (0.01,0.190443) (0.02,0.190433) (0.05,0.190417) (0.1,0.190358) (0.2,0.190255) (0.5,0.189502) (1,0.186987)};
\end{axis}
\begin{axis}[name=ax2_r0,
    at={(0.5650\textwidth,0in)},
    anchor=south west,
    ht panel, xmode=log, ymode=log,
    ymin=0.064529, ymax=0.265234,
    ytick={0.1,0.2},
    yticklabels={0.1,0.2},
    title={$\alpha=2.0$}]
  \addplot[draw=none, fill=htSOMA, fill opacity=0.06] coordinates {(0.001,0.114942) (0.002,0.114926) (0.005,0.114873) (0.01,0.114776) (0.02,0.114554) (0.05,0.113658) (0.1,0.111545) (0.2,0.105557) (0.5,0.0885348) (1,0.148435) (1,0.147319) (0.5,0.0848849) (0.2,0.101506) (0.1,0.107267) (0.05,0.109366) (0.02,0.110262) (0.01,0.110493) (0.005,0.110596) (0.002,0.110653) (0.001,0.110671)} \closedcycle;
  \addplot[draw=none, fill=htFPSOMA, fill opacity=0.06] coordinates {(0.001,0.114936) (0.002,0.114912) (0.005,0.114838) (0.01,0.114708) (0.02,0.114417) (0.05,0.113322) (0.1,0.110917) (0.2,0.104491) (0.5,0.0843786) (1,0.0877489) (1,0.085089) (0.5,0.081373) (0.2,0.100658) (0.1,0.106793) (0.05,0.109113) (0.02,0.110157) (0.01,0.11044) (0.005,0.110569) (0.002,0.110642) (0.001,0.110666)} \closedcycle;
  \addplot[draw=none, fill=htSpec, fill opacity=0.06] coordinates {(0.001,0.114927) (0.002,0.114895) (0.005,0.114795) (0.01,0.114623) (0.02,0.114253) (0.05,0.113005) (0.1,0.110432) (0.2,0.103907) (0.5,0.0846026) (1,0.0751496) (1,0.0726269) (0.5,0.0820262) (0.2,0.100216) (0.1,0.106348) (0.05,0.108839) (0.02,0.110038) (0.01,0.110379) (0.005,0.110538) (0.002,0.11063) (0.001,0.110659)} \closedcycle;
  \addplot[draw=none, fill=htRand, fill opacity=0.06] coordinates {(0.001,0.114959) (0.002,0.114959) (0.005,0.114959) (0.01,0.114959) (0.02,0.114957) (0.05,0.114947) (0.1,0.114912) (0.2,0.114771) (0.5,0.114013) (1,0.111442) (1,0.106682) (0.5,0.109241) (0.2,0.110458) (0.1,0.110635) (0.05,0.110678) (0.02,0.110688) (0.01,0.110689) (0.005,0.110689) (0.002,0.110689) (0.001,0.110689)} \closedcycle;
  \addplot[draw=none, fill=htPreNS5, fill opacity=0.06] coordinates {(0.001,0.114944) (0.002,0.11493) (0.005,0.114883) (0.01,0.114797) (0.02,0.114595) (0.05,0.113759) (0.1,0.111668) (0.2,0.105744) (0.5,0.0878658) (1,0.14694) (1,0.142756) (0.5,0.0843102) (0.2,0.10167) (0.1,0.10736) (0.05,0.109417) (0.02,0.110283) (0.01,0.110504) (0.005,0.110601) (0.002,0.110655) (0.001,0.110672)} \closedcycle;
  \addplot[draw=none, fill=htPostNS5, fill opacity=0.06] coordinates {(0.001,0.114947) (0.002,0.114936) (0.005,0.114899) (0.01,0.114832) (0.02,0.114681) (0.05,0.114078) (0.1,0.11259) (0.2,0.108212) (0.5,0.0915471) (1,0.097031) (1,0.0933439) (0.5,0.0881065) (0.2,0.103931) (0.1,0.108179) (0.05,0.109692) (0.02,0.110356) (0.01,0.110534) (0.005,0.110614) (0.002,0.11066) (0.001,0.110675)} \closedcycle;
  \addplot[draw=none, fill=htMSOMA, fill opacity=0.06] coordinates {(0.001,0.114961) (0.002,0.114962) (0.005,0.114967) (0.01,0.114975) (0.02,0.11499) (0.05,0.115032) (0.1,0.11508) (0.2,0.115077) (0.5,0.11436) (1,0.112398) (1,0.107694) (0.5,0.109409) (0.2,0.110481) (0.1,0.110631) (0.05,0.110671) (0.02,0.110685) (0.01,0.110687) (0.005,0.110688) (0.002,0.110689) (0.001,0.110689)} \closedcycle;
  \addplot[draw=none, fill=htFSAM, fill opacity=0.06] coordinates {(0.001,0.219815) (0.002,0.219844) (0.005,0.219926) (0.01,0.220046) (0.02,0.220205) (0.05,0.219225) (0.1,0.219098) (0.2,0.215465) (0.5,0.2646) (1,0.282371) (1,0.216664) (0.5,0.18425) (0.2,0.18773) (0.1,0.188402) (0.05,0.190939) (0.02,0.192518) (0.01,0.193114) (0.005,0.193444) (0.002,0.193654) (0.001,0.193727)} \closedcycle;
  \addplot[draw=none, fill=htSAM, fill opacity=0.06] coordinates {(0.001,0.219817) (0.002,0.219848) (0.005,0.219936) (0.01,0.220066) (0.02,0.220233) (0.05,0.219377) (0.1,0.21947) (0.2,0.215793) (0.5,0.287397) (1,0.284567) (1,0.210829) (0.5,0.184904) (0.2,0.188018) (0.1,0.188683) (0.05,0.191091) (0.02,0.192583) (0.01,0.193148) (0.005,0.193461) (0.002,0.193662) (0.001,0.19373)} \closedcycle;
  \addplot[draw=none, fill=htMSAM, fill opacity=0.06] coordinates {(0.001,0.237783) (0.002,0.237813) (0.005,0.2379) (0.01,0.238041) (0.02,0.23831) (0.05,0.23902) (0.1,0.239965) (0.2,0.241264) (0.5,0.242327) (1,0.239687) (1,0.192496) (0.5,0.197616) (0.2,0.198482) (0.1,0.198324) (0.05,0.198171) (0.02,0.198054) (0.01,0.198011) (0.005,0.197988) (0.002,0.197974) (0.001,0.19797)} \closedcycle;
  \addplot[draw=none, fill=htSAMAdam, fill opacity=0.06] coordinates {(0.001,0.141324) (0.002,0.14132) (0.005,0.141316) (0.01,0.14131) (0.02,0.141297) (0.05,0.141238) (0.1,0.14114) (0.2,0.1409) (0.5,0.139847) (1,0.137277) (1,0.133985) (0.5,0.137129) (0.2,0.138491) (0.1,0.138825) (0.05,0.138959) (0.02,0.139041) (0.01,0.139063) (0.005,0.139078) (0.002,0.139083) (0.001,0.139085)} \closedcycle;
  \addplot[draw=none, fill=htFSAMAdam, fill opacity=0.06] coordinates {(0.001,0.141325) (0.002,0.141324) (0.005,0.141325) (0.01,0.141328) (0.02,0.141344) (0.05,0.141333) (0.1,0.141324) (0.2,0.141261) (0.5,0.140688) (1,0.13876) (1,0.135621) (0.5,0.138127) (0.2,0.138902) (0.1,0.139035) (0.05,0.139067) (0.02,0.139078) (0.01,0.139085) (0.005,0.139089) (0.002,0.139087) (0.001,0.139087)} \closedcycle;
  \addplot[color=htMuon, dashed, line width=0.7pt, forget plot] coordinates {(0.001,0.114465) (1,0.114465)};
  \addplot[color=htAdam, dash dot, line width=0.7pt, forget plot] coordinates {(0.001,0.13972) (1,0.13972)};
  \addplot[color=htSpec, solid, line width=0.8pt] coordinates {(0.001,0.114435) (0.002,0.114405) (0.005,0.114311) (0.01,0.114148) (0.02,0.113798) (0.05,0.11252) (0.1,0.10981) (0.2,0.103038) (0.5,0.0827894) (1,0.0732812)};
  \addplot[color=htRand, dashed, line width=0.8pt] coordinates {(0.001,0.114465) (0.002,0.114465) (0.005,0.114466) (0.01,0.114466) (0.02,0.114465) (0.05,0.114457) (0.1,0.114422) (0.2,0.11426) (0.5,0.113043) (1,0.108217)};
  \addplot[color=htMSOMA, loosely dashed, line width=0.8pt] coordinates {(0.001,0.114465) (0.002,0.114464) (0.005,0.114463) (0.01,0.114461) (0.02,0.114457) (0.05,0.114441) (0.1,0.114402) (0.2,0.114272) (0.5,0.113419) (1,0.110104)};
  \addplot[color=htFSAM, solid, line width=0.8pt] coordinates {(0.001,0.209433) (0.002,0.209417) (0.005,0.209364) (0.01,0.209262) (0.02,0.20902) (0.05,0.208206) (0.1,0.207087) (0.2,0.203312) (0.5,0.217194) (1,0.233556)};
  \addplot[color=htSAM, solid, line width=0.8pt] coordinates {(0.001,0.209435) (0.002,0.209421) (0.005,0.209373) (0.01,0.20928) (0.02,0.209056) (0.05,0.208298) (0.1,0.207267) (0.2,0.204354) (0.5,0.232364) (1,0.232512)};
  \addplot[color=htMSAM, solid, line width=0.8pt] coordinates {(0.001,0.223167) (0.002,0.223181) (0.005,0.22322) (0.01,0.223285) (0.02,0.223409) (0.05,0.223751) (0.1,0.224226) (0.2,0.224881) (0.5,0.225152) (1,0.22145)};
  \addplot[color=htSOMA, solid, line width=1.1pt] coordinates {(0.001,0.114449) (0.002,0.114433) (0.005,0.114382) (0.01,0.11429) (0.02,0.114078) (0.05,0.113219) (0.1,0.111048) (0.2,0.104681) (0.5,0.0859645) (1,0.147939)};
  \addplot[color=htFPSOMA, solid, line width=1.1pt] coordinates {(0.001,0.114441) (0.002,0.114417) (0.005,0.114342) (0.01,0.114211) (0.02,0.11392) (0.05,0.112844) (0.1,0.110365) (0.2,0.103583) (0.5,0.0822998) (1,0.0868701)};
  \addplot[color=htPreNS5, densely dotted, line width=1.1pt] coordinates {(0.001,0.114449) (0.002,0.114433) (0.005,0.114383) (0.01,0.114292) (0.02,0.114081) (0.05,0.113235) (0.1,0.111131) (0.2,0.10474) (0.5,0.0854558) (1,0.145889)};
  \addplot[color=htPostNS5, dash dot, line width=1.1pt] coordinates {(0.001,0.114452) (0.002,0.114439) (0.005,0.114397) (0.01,0.114323) (0.02,0.114158) (0.05,0.113524) (0.1,0.112023) (0.2,0.107481) (0.5,0.0896489) (1,0.0937082)};
  \addplot[color=htSAMAdam, solid, line width=1.1pt] coordinates {(0.001,0.139718) (0.002,0.139716) (0.005,0.139711) (0.01,0.139696) (0.02,0.139678) (0.05,0.139608) (0.1,0.139497) (0.2,0.139173) (0.5,0.138021) (1,0.135392)};
  \addplot[color=htFSAMAdam, densely dashed, line width=1.4pt] coordinates {(0.001,0.13972) (0.002,0.13972) (0.005,0.139721) (0.01,0.139717) (0.02,0.139718) (0.05,0.13971) (0.1,0.139693) (0.2,0.139571) (0.5,0.138826) (1,0.136963)};
\end{axis}
\begin{axis}[name=ax3_r0,
    at={(0.8065\textwidth,0in)},
    anchor=south west,
    ht panel, xmode=log, ymode=log,
    ymin=0.0353948, ymax=0.217327,
    ytick={0.05,0.1},
    yticklabels={0.05,0.1},
    title={$\alpha=3.0$}]
  \addplot[draw=none, fill=htSOMA, fill opacity=0.06] coordinates {(0.001,0.0557473) (0.002,0.0557413) (0.005,0.0557221) (0.01,0.0556855) (0.02,0.0555957) (0.05,0.0551987) (0.1,0.0541636) (0.2,0.0511955) (0.5,0.0491005) (1,0.133054) (1,0.131262) (0.5,0.0472011) (0.2,0.0502525) (0.1,0.0531687) (0.05,0.054127) (0.02,0.054503) (0.01,0.0545893) (0.005,0.0546248) (0.002,0.0546436) (0.001,0.0546495)} \closedcycle;
  \addplot[draw=none, fill=htFPSOMA, fill opacity=0.06] coordinates {(0.001,0.0557434) (0.002,0.0557335) (0.005,0.0557025) (0.01,0.0556467) (0.02,0.0555189) (0.05,0.0550083) (0.1,0.0537988) (0.2,0.050489) (0.5,0.0419513) (1,0.0598078) (1,0.0576184) (0.5,0.0408173) (0.2,0.0496531) (0.1,0.0528655) (0.05,0.0539686) (0.02,0.0544364) (0.01,0.0545555) (0.005,0.0546078) (0.002,0.0546368) (0.001,0.0546461)} \closedcycle;
  \addplot[draw=none, fill=htSpec, fill opacity=0.06] coordinates {(0.001,0.0557408) (0.002,0.0557284) (0.005,0.05569) (0.01,0.0556219) (0.02,0.0554688) (0.05,0.0548934) (0.1,0.053619) (0.2,0.0504673) (0.5,0.0418568) (1,0.197482) (1,0.140417) (0.5,0.0407487) (0.2,0.0496751) (0.1,0.0527415) (0.05,0.0538779) (0.02,0.054395) (0.01,0.0545341) (0.005,0.0545969) (0.002,0.0546324) (0.001,0.0546439)} \closedcycle;
  \addplot[draw=none, fill=htRand, fill opacity=0.06] coordinates {(0.001,0.055753) (0.002,0.055753) (0.005,0.055753) (0.01,0.0557529) (0.02,0.0557522) (0.05,0.0557471) (0.1,0.055728) (0.2,0.0556509) (0.5,0.0551244) (1,0.053356) (1,0.0515469) (0.5,0.0538428) (0.2,0.0545332) (0.1,0.0546298) (0.05,0.0546514) (0.02,0.0546554) (0.01,0.0546555) (0.005,0.0546554) (0.002,0.0546553) (0.001,0.0546552)} \closedcycle;
  \addplot[draw=none, fill=htPreNS5, fill opacity=0.06] coordinates {(0.001,0.0557475) (0.002,0.0557418) (0.005,0.0557233) (0.01,0.055688) (0.02,0.0556004) (0.05,0.0552075) (0.1,0.0541684) (0.2,0.0511087) (0.5,0.048623) (1,0.13074) (1,0.12729) (0.5,0.0468455) (0.2,0.0500909) (0.1,0.0531292) (0.05,0.0541162) (0.02,0.0544989) (0.01,0.0545872) (0.005,0.0546237) (0.002,0.0546432) (0.001,0.0546493)} \closedcycle;
  \addplot[draw=none, fill=htPostNS5, fill opacity=0.06] coordinates {(0.001,0.0557479) (0.002,0.0557426) (0.005,0.0557259) (0.01,0.0556957) (0.02,0.0556258) (0.05,0.0553405) (0.1,0.0546281) (0.2,0.0524789) (0.5,0.0455118) (1,0.0653721) (1,0.0646476) (0.5,0.0440795) (0.2,0.051474) (0.1,0.0535719) (0.05,0.0542493) (0.02,0.0545267) (0.01,0.0545967) (0.005,0.0546274) (0.002,0.0546444) (0.001,0.0546498)} \closedcycle;
  \addplot[draw=none, fill=htMSOMA, fill opacity=0.06] coordinates {(0.001,0.055753) (0.002,0.055753) (0.005,0.0557527) (0.01,0.0557521) (0.02,0.05575) (0.05,0.0557384) (0.1,0.0557061) (0.2,0.0555938) (0.5,0.0549828) (1,0.0532991) (1,0.0514681) (0.5,0.0536746) (0.2,0.054499) (0.1,0.0545971) (0.05,0.0546307) (0.02,0.0546465) (0.01,0.054651) (0.005,0.0546531) (0.002,0.0546544) (0.001,0.0546548)} \closedcycle;
  \addplot[draw=none, fill=htFSAM, fill opacity=0.06] coordinates {(0.001,0.0574557) (0.002,0.0574534) (0.005,0.0574465) (0.01,0.057435) (0.02,0.0574115) (0.05,0.0573382) (0.1,0.0572063) (0.2,0.0569079) (0.5,0.0554657) (1,0.0532108) (1,0.0509108) (0.5,0.0539805) (0.2,0.0552556) (0.1,0.0556787) (0.05,0.0558638) (0.02,0.0559673) (0.01,0.0560002) (0.005,0.0560164) (0.002,0.056026) (0.001,0.0560291)} \closedcycle;
  \addplot[draw=none, fill=htSAM, fill opacity=0.06] coordinates {(0.001,0.0574564) (0.002,0.0574549) (0.005,0.0574503) (0.01,0.0574425) (0.02,0.0574266) (0.05,0.0573759) (0.1,0.0572819) (0.2,0.0570597) (0.5,0.0559092) (1,0.0538813) (1,0.0518706) (0.5,0.0543071) (0.2,0.0554506) (0.1,0.055776) (0.05,0.0559122) (0.02,0.0559867) (0.01,0.0560099) (0.005,0.0560212) (0.002,0.0560279) (0.001,0.0560301)} \closedcycle;
  \addplot[draw=none, fill=htMSAM, fill opacity=0.06] coordinates {(0.001,0.0594697) (0.002,0.0594688) (0.005,0.0594659) (0.01,0.0594611) (0.02,0.0594511) (0.05,0.0594181) (0.1,0.059352) (0.2,0.0591803) (0.5,0.0583777) (1,0.0563878) (1,0.054444) (0.5,0.0568048) (0.2,0.0577613) (0.1,0.0579781) (0.05,0.0580544) (0.02,0.0580895) (0.01,0.0580994) (0.005,0.058104) (0.002,0.0581067) (0.001,0.0581075)} \closedcycle;
  \addplot[draw=none, fill=htSAMAdam, fill opacity=0.06] coordinates {(0.001,0.099562) (0.002,0.0995617) (0.005,0.0995582) (0.01,0.0995567) (0.02,0.0995404) (0.05,0.0995112) (0.1,0.0994694) (0.2,0.0993136) (0.5,0.0984351) (1,0.0965126) (1,0.0949808) (0.5,0.0967659) (0.2,0.0975843) (0.1,0.0977215) (0.05,0.0978067) (0.02,0.0978462) (0.01,0.0978618) (0.005,0.097866) (0.002,0.0978687) (0.001,0.0978669)} \closedcycle;
  \addplot[draw=none, fill=htFSAMAdam, fill opacity=0.06] coordinates {(0.001,0.0995643) (0.002,0.0995653) (0.005,0.0995699) (0.01,0.0995798) (0.02,0.0995818) (0.05,0.0996136) (0.1,0.0996767) (0.2,0.0997143) (0.5,0.0995262) (1,0.098184) (1,0.0964797) (0.5,0.0978359) (0.2,0.0979963) (0.1,0.097948) (0.05,0.0979233) (0.02,0.0978958) (0.01,0.0978842) (0.005,0.0978797) (0.002,0.0978728) (0.001,0.0978691)} \closedcycle;
  \addplot[color=htMuon, dashed, line width=0.7pt, forget plot] coordinates {(0.001,0.0551559) (1,0.0551559)};
  \addplot[color=htAdam, dash dot, line width=0.7pt, forget plot] coordinates {(0.001,0.0986366) (1,0.0986366)};
  \addplot[color=htSpec, solid, line width=0.8pt] coordinates {(0.001,0.0551445) (0.002,0.0551329) (0.005,0.0550969) (0.01,0.055033) (0.02,0.0548903) (0.05,0.0543519) (0.1,0.0531406) (0.2,0.0500373) (0.5,0.0411984) (1,0.186712)};
  \addplot[color=htRand, dashed, line width=0.8pt] coordinates {(0.001,0.0551559) (0.002,0.0551559) (0.005,0.0551559) (0.01,0.0551557) (0.02,0.0551549) (0.05,0.0551485) (0.1,0.055125) (0.2,0.0550307) (0.5,0.0543675) (1,0.0527377)};
  \addplot[color=htMSOMA, loosely dashed, line width=0.8pt] coordinates {(0.001,0.0551546) (0.002,0.0551533) (0.005,0.0551494) (0.01,0.0551429) (0.02,0.0551295) (0.05,0.0550871) (0.1,0.0550084) (0.2,0.0548385) (0.5,0.0542919) (1,0.0525688)};
  \addplot[color=htFSAM, solid, line width=0.8pt] coordinates {(0.001,0.0568636) (0.002,0.056861) (0.005,0.0568532) (0.01,0.05684) (0.02,0.056813) (0.05,0.0567285) (0.1,0.0565752) (0.2,0.056225) (0.5,0.0548461) (1,0.0516882)};
  \addplot[color=htSAM, solid, line width=0.8pt] coordinates {(0.001,0.0568645) (0.002,0.0568627) (0.005,0.0568573) (0.01,0.0568482) (0.02,0.0568295) (0.05,0.0567697) (0.1,0.0566577) (0.2,0.0563905) (0.5,0.0552957) (1,0.0526004)};
  \addplot[color=htMSAM, solid, line width=0.8pt] coordinates {(0.001,0.0591052) (0.002,0.0591042) (0.005,0.0591012) (0.01,0.059096) (0.02,0.0590851) (0.05,0.0590485) (0.1,0.0589745) (0.2,0.0587791) (0.5,0.0576729) (1,0.0552488)};
  \addplot[color=htSOMA, solid, line width=1.1pt] coordinates {(0.001,0.0551501) (0.002,0.0551442) (0.005,0.055125) (0.01,0.0550886) (0.02,0.0549993) (0.05,0.0546051) (0.1,0.053581) (0.2,0.050752) (0.5,0.0486467) (1,0.131516)};
  \addplot[color=htFPSOMA, solid, line width=1.1pt] coordinates {(0.001,0.0551457) (0.002,0.0551354) (0.005,0.0551032) (0.01,0.0550452) (0.02,0.0549135) (0.05,0.0544003) (0.1,0.0531996) (0.2,0.0500998) (0.5,0.0412527) (1,0.0591116)};
  \addplot[color=htPreNS5, densely dotted, line width=1.1pt] coordinates {(0.001,0.0551491) (0.002,0.0551422) (0.005,0.05512) (0.01,0.0550786) (0.02,0.0549797) (0.05,0.0545568) (0.1,0.05348) (0.2,0.0504831) (0.5,0.0479229) (1,0.128912)};
  \addplot[color=htPostNS5, dash dot, line width=1.1pt] coordinates {(0.001,0.0551497) (0.002,0.0551433) (0.005,0.0551236) (0.01,0.0550882) (0.02,0.0550083) (0.05,0.0546972) (0.1,0.0539557) (0.2,0.0518186) (0.5,0.0450044) (1,0.0648291)};
  \addplot[color=htSAMAdam, solid, line width=1.1pt] coordinates {(0.001,0.0986362) (0.002,0.0986358) (0.005,0.0986323) (0.01,0.0986283) (0.02,0.0986206) (0.05,0.0985853) (0.1,0.0985239) (0.2,0.0983621) (0.5,0.0975731) (1,0.0958095)};
  \addplot[color=htFSAMAdam, densely dashed, line width=1.4pt] coordinates {(0.001,0.0986385) (0.002,0.0986402) (0.005,0.0986433) (0.01,0.0986502) (0.02,0.0986597) (0.05,0.0986912) (0.1,0.0987483) (0.2,0.0987883) (0.5,0.0985304) (1,0.0973049)};
\end{axis}
\begin{axis}[hide axis, scale only axis, at={(ax0_r0.north west)}, anchor=south west, yshift=34pt,
    width=0.9300\textwidth, height=1pt, xmin=0, xmax=1, ymin=0, ymax=1,
    legend columns=7, legend cell align=left,
    legend style={at={(0.5,0)}, anchor=south, draw=none, fill=none, font=\scriptsize, inner sep=1pt,
      /tikz/every even column/.append style={column sep=5pt}}]
  \addlegendimage{color=htSOMA, solid, line width=1.1pt}\addlegendentry{\textsc{Soma}}
  \addlegendimage{color=htFPSOMA, solid, line width=1.1pt}\addlegendentry{\textsc{FP{-}Soma}}
  \addlegendimage{color=htMSOMA, loosely dashed, line width=0.8pt}\addlegendentry{\textsc{Msoma}}
  \addlegendimage{color=htPreNS5, densely dotted, line width=1.1pt}\addlegendentry{\textsc{Opspre}}
  \addlegendimage{color=htPostNS5, dash dot, line width=1.1pt}\addlegendentry{\textsc{Opspost}}
  \addlegendimage{color=htSpec, solid, line width=0.8pt}\addlegendentry{\textsc{Ssm}}
  \addlegendimage{color=htRand, dashed, line width=0.8pt}\addlegendentry{\textsc{Rsm}}
  \addlegendimage{color=htSAM, solid, line width=0.8pt}\addlegendentry{\textsc{Sam}}
  \addlegendimage{color=htFSAM, solid, line width=0.8pt}\addlegendentry{\textsc{F{-}Sam}}
  \addlegendimage{color=htMSAM, solid, line width=0.8pt}\addlegendentry{\textsc{Msam}}
  \addlegendimage{color=htSAMAdam, solid, line width=1.1pt}\addlegendentry{\textsc{Saw}}
  \addlegendimage{color=htFSAMAdam, densely dashed, line width=1.4pt}\addlegendentry{\textsc{Fsaw}}
  \addlegendimage{color=htMuon, dashed, line width=0.7pt}\addlegendentry{\textsc{Muon}}
  \addlegendimage{color=htAdam, dash dot, line width=0.7pt}\addlegendentry{\textsc{AdamW}}
\end{axis}
\path (ax0_r0.south west) -- (ax3_r0.south east) coordinate[midway] (xlab);
\node[font=\small, anchor=north] at ([yshift=-14pt]xlab) {perturbation radius $\rho$};
\end{tikzpicture}

%% file: dataforsyntheticplots/mse_iso_gap.tex
% ---- mse_iso_gap
\providecolor{muonrow}{RGB}{208,236,214}\providecolor{adamrow}{RGB}{255,214,214}\providecolor{sgdrow}{RGB}{208,226,247}
\begin{table*}[!t]\centering\footnotesize\setlength{\tabcolsep}{5pt}
\caption{Optimality gap: best value over the perturbation radius $\rho$ and the radius $\rho^\star$ where it is attained (median of 8 seeds; Frobenius-ball methods pool the global and per-layer radius modes). Row colour marks the base optimizer family: \muon-based (\colorbox{muonrow}{Green}), \adamw-based (\colorbox{adamrow}{Red}), \sgd-based (\colorbox{sgdrow}{Blue}); baselines have no $\rho^\star$. Best value per $\alpha$ in bold.}\label{tab:mse_iso_gap}
\vspace{2pt}
\resizebox{\ifdim\width>\textwidth\textwidth\else\width\fi}{!}{%
\begin{tabular}{lcccccccc}\toprule
Algorithm & \multicolumn{2}{c}{$\alpha=1.1$} & \multicolumn{2}{c}{$\alpha=1.6$} & \multicolumn{2}{c}{$\alpha=2.0$} & \multicolumn{2}{c}{$\alpha=3.0$} \\
& gap & $\rho^\star$ & gap & $\rho^\star$ & gap & $\rho^\star$ & gap & $\rho^\star$ \\\midrule
\rowcolor{muonrow} \muon & 1.01 & -- & 0.231 & -- & 0.114 & -- & 0.0552 & -- \\
\rowcolor{muonrow} \specsammuon & 0.648 & 1 & \textbf{0.141} & 1 & \textbf{0.0733} & 1 & \textbf{0.0412} & 0.5 \\
\rowcolor{muonrow} \fpsoma & 0.654 & 1 & 0.147 & 1 & 0.0823 & 0.5 & 0.0413 & 0.5 \\
\rowcolor{muonrow} \soma & 0.63 & 1 & 0.176 & 0.5 & 0.086 & 0.5 & 0.0486 & 0.5 \\
\rowcolor{muonrow} \randsammuon & 1 & 1 & 0.224 & 1 & 0.108 & 1 & 0.0527 & 1 \\
\rowcolor{muonrow} \opsomapre & 0.628 & 1 & 0.177 & 0.5 & 0.0855 & 0.5 & 0.0479 & 0.5 \\
\rowcolor{muonrow} \opsomapost & 0.717 & 1 & 0.162 & 1 & 0.0896 & 0.5 & 0.045 & 0.5 \\
\rowcolor{muonrow} \msoma & 0.991 & 1 & 0.227 & 1 & 0.11 & 1 & 0.0526 & 1 \\
\rowcolor{adamrow} \adamw & 0.382 & -- & 0.19 & -- & 0.14 & -- & 0.0986 & -- \\
\rowcolor{adamrow} \samadamw & \textbf{0.38} & 1 & 0.186 & 1 & 0.135 & 1 & 0.0958 & 1 \\
\rowcolor{adamrow} \fsamadamw & \textbf{0.38} & 1 & 0.187 & 1 & 0.137 & 1 & 0.0973 & 1 \\
\rowcolor{sgdrow} \sgd & 7.43 & -- & 0.866 & -- & 0.209 & -- & 0.0569 & -- \\
\rowcolor{sgdrow} \sam & 6.54 & 1 & 0.735 & 1 & 0.204 & 0.2 & 0.0526 & 1 \\
\rowcolor{sgdrow} \fsam & 6.58 & 1 & 0.753 & 1 & 0.203 & 0.2 & 0.0517 & 1 \\
\rowcolor{sgdrow} \textsc{Msam} & 8.98 & 1 & 1 & 0.001 & 0.221 & 1 & 0.0552 & 1 \\
\bottomrule\end{tabular}}\end{table*}

%% file: sections/llmexperiment.tex
For language model experiments, we used a single server with $8$ NVIDIA H200 GPUs ($141$\,GB memory each), $2$\,TB of host memory and Linux, using PyTorch~$2.13$ (CUDA~$13.0$) with Hugging Face Transformers~$4.57$ and Python~$3.13$. Each training run used one GPU, and several runs (with different optimizers or seeds) were executed in parallel on different GPUs on the shared server.

\input{paper_tables/table_pretraining}
Table~\ref{tab:pretraining} gives the final training loss and validation perplexity for pre-training on nanoGPT and Pythia-70M and the final validation accuracy on CIFAR-10, which Section~\ref{sec:exp-llm-pretraining} discusses.

\subsection{Catastrophic forgetting.}\label{sec:catastrophic-forgetting}
\textbf{Setup.} When a pretrained model is fine-tuned on new data, it usually becomes worse on the data it was originally trained on. We measure how strongly this happens for models pretrained with different optimizers. For each pretrained model, we record its loss on held-out FineWeb-Edu text, the data it was pretrained on, before ($L_{\mathrm{before}}$) and after ($L_{\mathrm{after}}$) fine-tuning on a new dataset; the larger the increase, the more the model has forgotten. We use four new datasets: Python code (codeparrot, \citealp{tunstall2022natural}), mathematics question answering (stackmathqa, \citealp{zhang2024stackmathqa}), music text (musicpile, \citealp{yuan2024chatmusician}) and instruction-following chat (tulu3, \citealp{lambert2024tulu}).

\textbf{Implementation.} All runs start from the same model, pretrained with \adamw, which we then fine-tune with each optimizer under the same schedule and number of steps. Differences in forgetting therefore come only from the fine-tuning optimizer. Since training longer on the new data also increases forgetting, we compare the optimizers at the same fine-tuning loss on each dataset. All results are averaged over three seeds.

\input{paper_figures_v2/figure_forgetting}
\textbf{Results.} Figures~\ref{fig:forgetting-nanogpt} and~\ref{fig:forgetting-pythia70m} show, for each model, the increase in FineWeb-Edu loss after fine-tuning on each dataset, together with its average. Forgetting depends strongly on the fine-tuning dataset: for almost every optimizer, it is largest after fine-tuning on code (codeparrot) or mathematics (stackmathqa) and smallest after fine-tuning on music (musicpile). It also depends strongly on the fine-tuning optimizer. The \sgd-based runs forget the least, but they start from a clearly higher loss ($^\dagger$), so they are not comparable. Among the comparable runs, fine-tuning with the \muon-based methods causes much more forgetting than fine-tuning with \adamw, about $1.6\times$ as much on nanoGPT and $2.1\times$ on Pythia-70M. Part of this gap is likely due to optimizer mismatch: the base model was pretrained with \adamw, so fine-tuning it with \adamw keeps the optimizer unchanged, whereas fine-tuning it with \muon switches the optimizer between pretraining and fine-tuning, which is known to erode acquired knowledge \citep{qu2026can}. Within the \muon family, an informed perturbation reduces forgetting on nanoGPT: compared with \muon, \fpsoma lowers the average increase by $16\%$, \specsammuon by $11\%$, \opsomapre by $8\%$ and \soma by $4\%$, whereas the random perturbation increases it slightly. Almost all of this reduction comes from code, where the increase drops from $0.37$ for \muon to $0.24$ for \specsammuon and \fpsoma. On Pythia-70M, all \muon-based methods forget about equally (within $1\%$ of \muon).

\input{paper_tables/table_forgetting_nanogpt}
\input{paper_tables/table_forgetting_pythia}

Tables~\ref{tab:forgetting-nanogpt} and~\ref{tab:forgetting-pythia70m} report $L_{\mathrm{before}}$, $L_{\mathrm{after}}$ on each dataset, their average and the smoothed relative increase $\Delta\%_{\mathrm{adj}}$ (lower means less forgetting). They confirm the same picture: among the comparable runs, fine-tuning with \adamw ends with the lowest average $L_{\mathrm{after}}$ ($3.304$ on nanoGPT and $3.811$ on Pythia-70M), while the \muon-based methods end higher ($3.37$--$3.42$ and $4.10$--$4.11$). Within the \muon family on nanoGPT, the relative increase falls from $4.20\%$ for \muon to $3.50\%$ for \fpsoma. Overall, sharpness-aware perturbations reduce the forgetting caused by \muon fine-tuning only moderately, and mainly with a fresh gradient, so they do not close the gap to \adamw, which matches the optimizer used in pretraining.

\input{paper_tables/table_runtime_pretraining}
\input{paper_tables/table_runtime_finetuning}

\paragraph{Runtime.}
Tables~\ref{tab:runtime-pretraining} and~\ref{tab:runtime-finetuning} report the wall-clock time of each optimizer in pretraining (seconds per epoch of 100 steps) and fine-tuning (seconds per run of 64 steps). We compare on nanoGPT, the largest model, where the gradient computation dominates the cost; for Pythia-70M and the CIFAR-10 CNN, the timings are affected by other jobs sharing the GPUs and by data loading, and do not reflect the number of gradient evaluations. The cost is driven by the number of forward--backward passes per step. The methods that need a second pass are the most expensive: on nanoGPT, \sam and \fsam take about $1.7\times$ the time of \sgd per epoch, and \specsammuon and \fpsoma take $2.4$--$2.9\times$ the time of \muon ($642$--$773$\,s against $265$\,s); in fine-tuning, they take $2.3$--$2.5\times$ as long. \soma needs only one pass per step and adds just $3\%$ to the time of \muon per epoch ($272$\,s), and \opsomapre, \opsomapost, \msoma and \randsammuon are similarly close to \muon.

\subsection{Hyperparameter Details}
\label{app:hyperparameters}
\paragraph{Hyperparameter search.}
For each optimizer, we select the learning rate $\eta$ and, for the sharpness-aware methods, the perturbation radius $\rho$ by grid search. Each candidate is trained on a single seed (seed~$0$) with a short proxy run, and we select the hyperparameter with the best final validation loss (validation accuracy for the CNN) at the end of the run. The grids used in each experiment are given below.

\paragraph{Grid for pretraining and PTQ.}

The initial learning-rate grids on the language models were:
\begin{itemize}
  \item[] For \muon-based methods: $\{6\times10^{-4},2\times10^{-3},6\times10^{-3},2\times10^{-2},6\times10^{-2}\}$,
  \item[] For \adamw-based methods: $\{2,6\}\times10^{-5},\ 2\times10^{-4},\ 6\times10^{-4},\ 2\times10^{-3}$,
  \item[] For \sgd-based methods: $\{10^{-3},3\times10^{-3},10^{-2},3\times10^{-2},10^{-1},3\times10^{-1}\}$.
\end{itemize}
  On CNN,  they were:
  \begin{itemize}
  \item[] For \muon-based methods: $\{3\times10^{-3},10^{-2},3\times10^{-2},10^{-1}\}$,
  \item[] For \adamw-based methods: $\{3\times10^{-4},10^{-3},3\times10^{-3},10^{-2}\}$,
  \item[] For \sgd-based methods: $\{10^{-2},3\times10^{-2},10^{-1},3\times10^{-1}\}$.
\end{itemize}
If the best value lay at the edge of a grid, the grid was extended by a factor of $\sqrt{10}$ in that direction, at most twice. The perturbation radius was searched on $\{3\times10^{-3},10^{-2},3\times10^{-2},10^{-1},3\times10^{-1}\}$ with the same extension rule; the sharpness-aware variants inherit the learning rate, momentum, weight decay and auxiliary learning rate of their base optimizer. 

For \adamw{} and \muon{} on the language models, we ran a second, finer search that varied one setting at a time around the values above: the learning rate ($\{\tfrac14,\tfrac12,\tfrac{1}{\sqrt2},1,\sqrt2,2,4\}\times$ the current value), the weight decay ($0$, $0.01$, $0.1$, $0.3$) and, for \muon{} only, the momentum ($0.9$, $0.95$, $0.98$) and the auxiliary \adamw{} learning rate ($0.01\times$ to $1\times$ the Muon learning rate). We then combined the best value of each setting and checked nearby learning rates ($\times\tfrac{1}{\sqrt2}$ and $\times\sqrt2$) and, for \muon{}, nearby auxiliary rates ($\times\tfrac12$ and $\times2$) around the combination.

\paragraph{Grid for fine-tuning.}
The learning rate of each optimizer was calibrated on a grid of $19$ log-spaced values from $10^{-8}$ to $10$ (two per decade), on the \adamw-pretrained model with seed~$42$ and the codeparrot and tulu3 datasets, and the value with the lowest mean held-out fine-tuning loss was chosen as the centre of the grid. Fine-tuning then used $11$ log-spaced learning rates, from the centre divided by $100$ to the centre multiplied by $10$, for every dataset, pretrained model, and seed. Among the runs that reach the common target fine-tuning loss, the run with the lowest FineWeb-Edu loss is reported.

\textbf{Selected Hyperparameters for Pre-training.} 
Table~\ref{tab:hyperparams} lists the final selected learning rate (lr), weight decay (wd) and perturbation radius $\rho$ of every run for the pre-training experiments. The \sgd{} family uses Nesterov momentum $0.9$ on all three datasets. The \muon family uses momentum $0.9$ on nanoGPT and $0.95$ on Pythia-70M and the CNN, five Newton--Schulz iterations, and an auxiliary AdamW for the embeddings, the output head and all one-dimensional parameters ($\beta=(0.9,0.95)$, $\epsilon=10^{-10}$, weight decay $0.1$ on the language models and $5\times10^{-4}$ on the CNN, learning rate $4.8\times10^{-4}$ on nanoGPT, $6\times10^{-6}$ on Pythia-70M and $1.2\times10^{-3}$ on the CNN). \adamw uses $\beta=(0.9,0.95)$ and $\epsilon=10^{-8}$, and the \fsam-type methods use $\lambda=0.9$ and $\sigma=1$. Learning rates and $\rho$ were selected on a single seed with short proxy runs, not on the three evaluation seeds.

\begin{table}[t]
  \centering
  \caption{Selected learning rate, weight decay and perturbation radius $\rho$ of the continued-pretraining experiments.}
  \label{tab:hyperparams}
  \footnotesize\setlength{\tabcolsep}{4pt}
  \renewcommand{\arraystretch}{1.15}
  \resizebox{\columnwidth}{!}{%
  \begin{tabular}{|l|ccc|ccc|ccc|}
    \hline
    & \multicolumn{3}{c|}{nanoGPT} & \multicolumn{3}{c|}{Pythia-70M} & \multicolumn{3}{c|}{CIFAR-10 CNN} \\
    \cline{2-10}
    Optimizer & lr & wd & $\rho$ & lr & wd & $\rho$ & lr & wd & $\rho$ \\
    \hline
    \muon & $4.24\!\times\!10^{-3}$ & $0.01$ & -- & $8.48\!\times\!10^{-4}$ & $0.3$ & -- & $0.03$ & $0$ & -- \\
    \specsammuon & $4.24\!\times\!10^{-3}$ & $0.01$ & $0.01$ & $8.48\!\times\!10^{-4}$ & $0.3$ & $3\!\times\!10^{-4}$ & $0.03$ & $0$ & $0.01$ \\
    \fpsoma & $4.24\!\times\!10^{-3}$ & $0.01$ & $0.01$ & $8.48\!\times\!10^{-4}$ & $0.3$ & $3\!\times\!10^{-4}$ & $0.03$ & $0$ & $0.01$ \\
    \soma & $4.24\!\times\!10^{-3}$ & $0.01$ & $3\!\times\!10^{-4}$ & $8.48\!\times\!10^{-4}$ & $0.3$ & $3\!\times\!10^{-4}$ & $0.03$ & $0$ & $3\!\times\!10^{-3}$ \\
    \randsammuon & $4.24\!\times\!10^{-3}$ & $0.01$ & $0.01$ & $8.48\!\times\!10^{-4}$ & $0.3$ & $9.5\!\times\!10^{-4}$ & $0.03$ & $0$ & $0.1$ \\
    \opsomapre & $4.24\!\times\!10^{-3}$ & $0.01$ & $3\!\times\!10^{-3}$ & $8.48\!\times\!10^{-4}$ & $0.3$ & $3\!\times\!10^{-4}$ & $0.03$ & $0$ & $3\!\times\!10^{-3}$ \\
    \opsomapost & $4.24\!\times\!10^{-3}$ & $0.01$ & $3\!\times\!10^{-3}$ & $8.48\!\times\!10^{-4}$ & $0.3$ & $3\!\times\!10^{-4}$ & $0.03$ & $0$ & $3\!\times\!10^{-3}$ \\
    \msoma & $4.24\!\times\!10^{-3}$ & $0.01$ & $0.01$ & $8.48\!\times\!10^{-4}$ & $0.3$ & $0.01$ & $0.03$ & $0$ & $0.01$ \\
    \hline
    \adamw & $4.24\!\times\!10^{-4}$ & $0$ & -- & $6\!\times\!10^{-5}$\textsuperscript{a} & $0.1$ & -- & $0.0316$ & $5\!\times\!10^{-4}$ & -- \\
    \saw & $4.24\!\times\!10^{-4}$ & $0$ & $3\!\times\!10^{-3}$ & $8.48\!\times\!10^{-5}$ & $0.1$ & $3\!\times\!10^{-4}$ & $0.0316$ & $5\!\times\!10^{-4}$ & $3\!\times\!10^{-3}$ \\
    \fsaw & $4.24\!\times\!10^{-4}$ & $0$ & $9.5\!\times\!10^{-4}$ & $8.48\!\times\!10^{-5}$ & $0.1$ & $3\!\times\!10^{-4}$ & $0.0316$ & $5\!\times\!10^{-4}$ & $9.5\!\times\!10^{-4}$ \\
    \hline
    \sgd & $0.3$ & $0$ & -- & $3\!\times\!10^{-3}$ & $0$ & -- & $0.3$ & $5\!\times\!10^{-4}$ & -- \\
    \sam & $0.3$ & $0$ & $3\!\times\!10^{-3}$ & $3\!\times\!10^{-3}$ & $0$ & $3\!\times\!10^{-4}$ & $0.3$ & $5\!\times\!10^{-4}$ & $3\!\times\!10^{-3}$ \\
    \fsam & $0.3$ & $0$ & $9.5\!\times\!10^{-4}$ & $3\!\times\!10^{-3}$ & $0$ & $3\!\times\!10^{-4}$ & $0.3$ & $5\!\times\!10^{-4}$ & $9.5\!\times\!10^{-4}$ \\
    \msam & $0.3$ & $0$ & $3\!\times\!10^{-3}$ & $3\!\times\!10^{-3}$ & $0$ & $3\!\times\!10^{-4}$ & $0.3$ & $5\!\times\!10^{-4}$ & $3\!\times\!10^{-3}$ \\
    \hline
  \end{tabular}}
\end{table}

\textbf{Selected Hyperparameters for Fine-Tuning.} Every fine-tuning run takes $64$ optimizer steps with a batch of $64$ sequences of $512$ tokens (two micro-batches of $32$), a linear warm-up over the first $10\%$ of the steps followed by a cosine decay to zero, and gradient-norm clipping at $1.0$. Table~\ref{tab:ft-hyperparams} lists the learning rate at the centre of the search grid and the perturbation radius $\rho$. The \sgd{} family uses Nesterov momentum $0.9$ and weight decay $0$. The \muon family uses momentum $0.95$, weight decay $0$, five Newton--Schulz iterations, and an auxiliary AdamW (learning rate equal to the Muon learning rate divided by $25$, weight decay $0.1$) for the embeddings, the output head and the one-dimensional parameters. \adamw uses weight decay $0.1$. The \fsam-type methods use $\lambda=0.9$ and $\sigma=1$.

\begin{table}[t]
  \centering
  \caption{Selected learning rate and perturbation radius $\rho$ for fine-tuning.}
  \label{tab:ft-hyperparams}
  \footnotesize\setlength{\tabcolsep}{5pt}
  \renewcommand{\arraystretch}{1.15}
  \begin{tabular}{|l|cc|cc|}
    \hline
    & \multicolumn{2}{c|}{nanoGPT} & \multicolumn{2}{c|}{Pythia-70M} \\
    \cline{2-5}
    Optimizer & lr & $\rho$ & lr & $\rho$ \\
    \hline
    \muon & $10^{-2}$ & -- & $10^{-3}$ & -- \\
    \specsammuon & $10^{-2}$ & $10^{-2}$ & $10^{-3}$ & $3\!\times\!10^{-4}$ \\
    \fpsoma & $10^{-2}$ & $10^{-2}$ & $10^{-3}$ & $3\!\times\!10^{-4}$ \\
    \soma & $10^{-2}$ & $3\!\times\!10^{-4}$ & $10^{-3}$ & $3\!\times\!10^{-4}$ \\
    \randsammuon & $10^{-2}$ & $10^{-2}$ & $10^{-3}$ & $9.5\!\times\!10^{-4}$ \\
    \opsomapre & $10^{-2}$ & $3\!\times\!10^{-3}$ & $10^{-3}$ & $3\!\times\!10^{-4}$ \\
    \opsomapost & $10^{-2}$ & $3\!\times\!10^{-3}$ & $10^{-3}$ & $3\!\times\!10^{-4}$ \\
    \hline
    \adamw & $10^{-3}$ & -- & $3.16\!\times\!10^{-5}$ & -- \\
    \hline
    \sgd & $0.316$ & -- & $0.0316$ & -- \\
    \sam & $0.316$ & $3\!\times\!10^{-3}$ & $0.0316$ & $3\!\times\!10^{-4}$ \\
    \fsam & $0.316$ & $9.5\!\times\!10^{-4}$ & $0.0316$ & $3\!\times\!10^{-4}$ \\
    \hline
  \end{tabular}
\end{table}

%% file: paper_tables/table_pretraining.tex
\makeatletter
\@ifundefined{rowcolor}{\PackageError{table}{Row colours need the colortbl package}{Add \string\usepackage{colortbl} (after xcolor) to the preamble.}}{}
\makeatother
\definecolor{adamrow}{RGB}{255,214,214}   % AdamW family (red)
\definecolor{sgdrow}{RGB}{208,226,247}    % SGD family   (blue)
\definecolor{muonrow}{RGB}{208,236,214}   % Muon family  (green)

\begin{table*}[!t]
  \centering
  \caption{Continued-pretraining results of the fifteen optimizers after the last epoch (FineWeb-Edu NanoGPT 15, FineWeb-Edu Pythia 14,
  CIFAR-10 Airbench CNN 7 epochs). Lower is better for train loss
  and validation perplexity (Val ppl), higher for validation accuracy (Val acc); the best validation value of each dataset is in bold.
  Mean $\pm$ std over 3 seeds. Row colour marks the base optimizer family: \muon-based (\colorbox{muonrow}{Green}), \adamw-based (\colorbox{adamrow}{Red}), \sgd-based (\colorbox{sgdrow}{Blue}).}
  \label{tab:pretraining}
  \vspace{2pt}
  \resizebox{\ifdim\width>\textwidth\textwidth\else\width\fi}{!}{%
  \footnotesize\setlength{\tabcolsep}{5pt}
  \begin{tabular}{lcccccc}
    \toprule
    & \multicolumn{2}{c}{FineWeb-Edu NanoGPT} & \multicolumn{2}{c}{FineWeb-Edu Pythia} & \multicolumn{2}{c}{CIFAR-10 Airbench CNN} \\
    \cmidrule(lr){2-3}\cmidrule(lr){4-5}\cmidrule(lr){6-7}
    Optimizer & Train loss & Val ppl & Train loss & Val ppl & Train loss & Val acc (\%) \\
    \midrule
    \rowcolor{muonrow} \muon & 3.303{\scriptsize$\pm$0.004} & 24.41{\scriptsize$\pm$0.02} & 3.558{\scriptsize$\pm$0.003} & \textbf{35.72}{\scriptsize$\pm$0.03} & 0.266{\scriptsize$\pm$0.002} & 87.27{\scriptsize$\pm$0.08} \\
    \rowcolor{muonrow} \specsammuon & 3.305{\scriptsize$\pm$0.004} & 24.55{\scriptsize$\pm$0.03} & 3.559{\scriptsize$\pm$0.003} & 35.78{\scriptsize$\pm$0.03} & 0.271{\scriptsize$\pm$0.003} & 87.63{\scriptsize$\pm$0.34} \\
    \rowcolor{muonrow} \fpsoma & 3.305{\scriptsize$\pm$0.004} & 24.54{\scriptsize$\pm$0.03} & 3.560{\scriptsize$\pm$0.003} & 35.79{\scriptsize$\pm$0.03} & 0.270{\scriptsize$\pm$0.001} & 87.37{\scriptsize$\pm$0.44} \\
    \rowcolor{muonrow} \soma & 3.303{\scriptsize$\pm$0.004} & 24.40{\scriptsize$\pm$0.01} & 3.562{\scriptsize$\pm$0.003} & 35.79{\scriptsize$\pm$0.03} & 0.269{\scriptsize$\pm$0.002} & 87.51{\scriptsize$\pm$0.27} \\
    \rowcolor{muonrow} \randsammuon & 3.303{\scriptsize$\pm$0.004} & 24.41{\scriptsize$\pm$0.02} & 3.558{\scriptsize$\pm$0.003} & 35.74{\scriptsize$\pm$0.03} & 0.266{\scriptsize$\pm$0.004} & 87.22{\scriptsize$\pm$0.07} \\
    \rowcolor{muonrow} \opsomapre & 3.305{\scriptsize$\pm$0.004} & 24.42{\scriptsize$\pm$0.01} & 3.560{\scriptsize$\pm$0.003} & 35.75{\scriptsize$\pm$0.03} & 0.266{\scriptsize$\pm$0.003} & 87.14{\scriptsize$\pm$0.13} \\
    \rowcolor{muonrow} \opsomapost & 3.304{\scriptsize$\pm$0.004} & 24.41{\scriptsize$\pm$0.01} & 3.559{\scriptsize$\pm$0.003} & \textbf{35.72}{\scriptsize$\pm$0.03} & 0.266{\scriptsize$\pm$0.002} & 87.14{\scriptsize$\pm$0.05} \\
    \rowcolor{muonrow} \msoma & 3.303{\scriptsize$\pm$0.004} & \textbf{24.39}{\scriptsize$\pm$0.02} & 3.743{\scriptsize$\pm$0.003} & 37.41{\scriptsize$\pm$0.05} & 0.265{\scriptsize$\pm$0.002} & 87.39{\scriptsize$\pm$0.37} \\
    \rowcolor{adamrow} \adamw & 3.311{\scriptsize$\pm$0.004} & 24.60{\scriptsize$\pm$0.02} & 3.574{\scriptsize$\pm$0.003} & 36.20{\scriptsize$\pm$0.02} & 0.316{\scriptsize$\pm$0.002} & 87.87{\scriptsize$\pm$0.53} \\
    \rowcolor{adamrow} \saw & 3.310{\scriptsize$\pm$0.004} & 24.59{\scriptsize$\pm$0.03} & 3.570{\scriptsize$\pm$0.003} & 35.96{\scriptsize$\pm$0.05} & 0.317{\scriptsize$\pm$0.008} & \textbf{87.91}{\scriptsize$\pm$0.50} \\
    \rowcolor{adamrow} \fsaw & 3.310{\scriptsize$\pm$0.004} & 24.60{\scriptsize$\pm$0.03} & 3.570{\scriptsize$\pm$0.003} & 35.95{\scriptsize$\pm$0.04} & 0.318{\scriptsize$\pm$0.004} & 87.70{\scriptsize$\pm$0.35} \\
    \rowcolor{sgdrow} \sgd & 3.361{\scriptsize$\pm$0.004} & 25.75{\scriptsize$\pm$0.02} & 3.664{\scriptsize$\pm$0.002} & 40.03{\scriptsize$\pm$0.01} & 0.461{\scriptsize$\pm$0.003} & 84.66{\scriptsize$\pm$0.40} \\
    \rowcolor{sgdrow} \sam & 3.361{\scriptsize$\pm$0.004} & 25.75{\scriptsize$\pm$0.01} & 3.666{\scriptsize$\pm$0.002} & 40.09{\scriptsize$\pm$0.01} & 0.462{\scriptsize$\pm$0.002} & 84.72{\scriptsize$\pm$0.41} \\
    \rowcolor{sgdrow} \fsam & 3.361{\scriptsize$\pm$0.004} & 25.76{\scriptsize$\pm$0.01} & 3.666{\scriptsize$\pm$0.002} & 40.09{\scriptsize$\pm$0.01} & 0.459{\scriptsize$\pm$0.002} & 84.73{\scriptsize$\pm$0.29} \\
    \rowcolor{sgdrow} \msam & 3.359{\scriptsize$\pm$0.004} & 25.71{\scriptsize$\pm$0.01} & 3.657{\scriptsize$\pm$0.002} & 39.73{\scriptsize$\pm$0.01} & 0.467{\scriptsize$\pm$0.002} & 84.03{\scriptsize$\pm$0.58} \\
    \bottomrule
  \end{tabular}}
\end{table*}

%% file: paper_figures_v2/figure_forgetting.tex
% AUTO-GENERATED: one full-width figure per model (needs common/preamble.tex in the preamble)
\begin{figure*}[!t]
  \centering
  \resizebox{\ifdim\width>\textwidth\textwidth\else\width\fi}{!}{\input{\curveroot/forgetting_nanogpt.tex}}
  \caption{Catastrophic forgetting of nanoGPT (GPT-2 small) models pretrained with each optimizer, after fine-tuning on four datasets. Bars: rise $\Delta$PT of the held-out FineWeb-Edu loss at a common fine-tuning loss (lower is better); diamond: average over datasets; error bars: standard deviation over 3 pretraining seeds. $^\dagger$ starting model clearly weaker than AdamW's (not comparable).}
  \label{fig:forgetting-nanogpt}
\end{figure*}
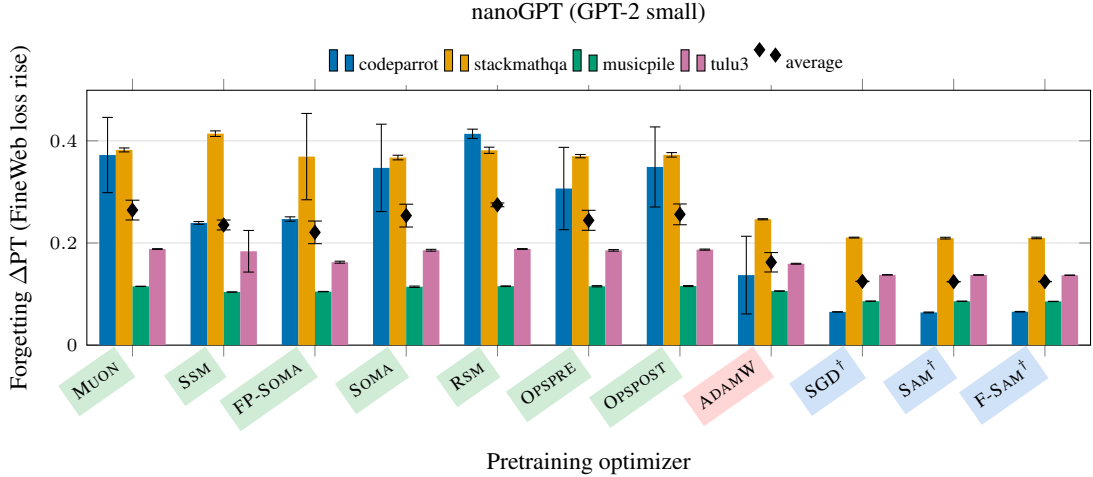

\begin{figure*}[!t]
  \centering
  \resizebox{\ifdim\width>\textwidth\textwidth\else\width\fi}{!}{\input{\curveroot/forgetting_pythia70m.tex}}
  \caption{Catastrophic forgetting of Pythia-70M models pretrained with each optimizer, after fine-tuning on four datasets. Bars: rise $\Delta$PT of the held-out FineWeb-Edu loss at a common fine-tuning loss (lower is better); diamond: average over datasets; error bars: standard deviation over 3 pretraining seeds. $^\dagger$ starting model clearly weaker than AdamW's (not comparable).}
  \label{fig:forgetting-pythia70m}
\end{figure*}
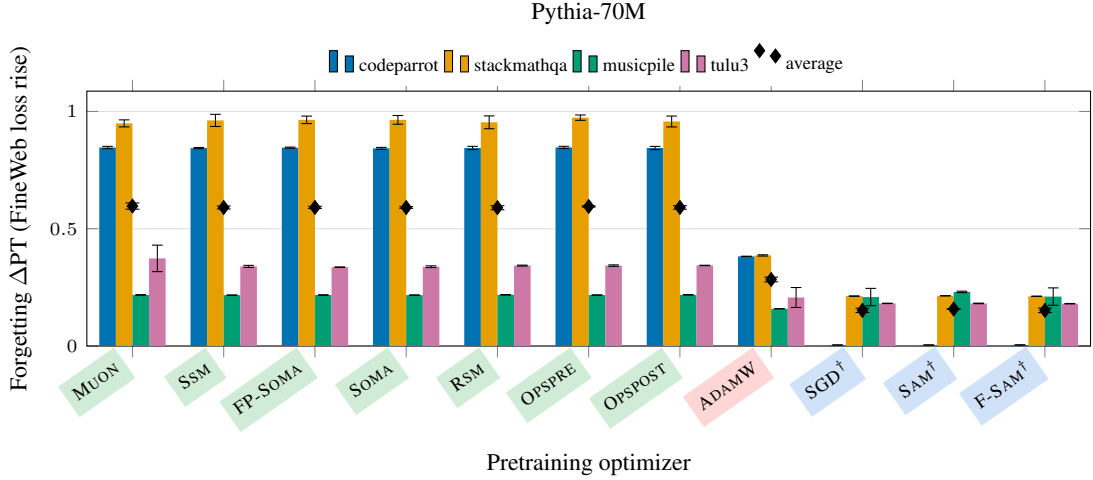

%% file: paper_figures_v2/forgetting_nanogpt.tex
\providecolor{muonrow}{RGB}{208,236,214}\providecolor{adamrow}{RGB}{255,214,214}\providecolor{sgdrow}{RGB}{208,226,247}
\providecommand{\fgM}[1]{\colorbox{muonrow}{\strut #1}}
\providecommand{\fgA}[1]{\colorbox{adamrow}{\strut #1}}
\providecommand{\fgS}[1]{\colorbox{sgdrow}{\strut #1}}
\definecolor{fg-codeparrot}{HTML}{0072B2}
\definecolor{fg-stackmathqa}{HTML}{E69F00}
\definecolor{fg-musicpile}{HTML}{009E73}
\definecolor{fg-tulu3}{HTML}{CC79A7}
\pgfplotstableread[col sep=comma]{\curveroot/data/forgetting_nanogpt.csv}\forgetTabNanoGPT
% bar width computed from the axis width (0.9\textwidth): 4 bars fill 92% of a category, so there is no extra space between the groups
\pgfmathsetlengthmacro{\fgbw}{0.65*0.9*\textwidth/11/4}
\pgfmathsetlengthmacro{\fgsA}{-1.5*\fgbw}\pgfmathsetlengthmacro{\fgsB}{-0.5*\fgbw}
\pgfmathsetlengthmacro{\fgsC}{0.5*\fgbw}\pgfmathsetlengthmacro{\fgsD}{1.5*\fgbw}
\begin{tikzpicture}
\begin{axis}[
  ybar, width=0.9\textwidth, height=0.30\textwidth, bar width=\fgbw, ymin=0, xmin=-0.5, xmax=10.5, enlarge x limits=false,
  ymajorgrids, grid style={gray!25}, xtick=data, xticklabels={\fgM{\muon},\fgM{\specsammuon},\fgM{\fpsoma},\fgM{\soma},\fgM{\randsammuon},\fgM{\opsomapre},\fgM{\opsomapost},\fgA{\adamw},\fgS{\sgd{}$^\dagger$},\fgS{\sam{}$^\dagger$},\fgS{\fsam{}$^\dagger$}},
  x tick label style={rotate=35, anchor=east, font=\scriptsize}, tick label style={font=\scriptsize}, label style={font=\footnotesize},
  title={nanoGPT (GPT-2 small)}, title style={font=\footnotesize, yshift=0.55cm}, xlabel={Pretraining optimizer},
  ylabel={Forgetting $\Delta$PT (FineWeb loss rise)}, scaled y ticks=false,
  legend style={font=\scriptsize, at={(0.5,1.02)}, anchor=south, legend columns=5, draw=none},
  error bars/y dir=both, error bars/y explicit, error bars/error bar style={line width=0.4pt}]
\addplot[bar width=\fgbw, bar shift=\fgsA, fill=fg-codeparrot, draw=none, mark=none] table[col sep=comma, x=idx, y=codeparrot, y error=codeparrot_sd]{\curveroot/data/forgetting_nanogpt.csv};
\addlegendentry{codeparrot}
\addplot[bar width=\fgbw, bar shift=\fgsB, fill=fg-stackmathqa, draw=none, mark=none] table[col sep=comma, x=idx, y=stackmathqa, y error=stackmathqa_sd]{\curveroot/data/forgetting_nanogpt.csv};
\addlegendentry{stackmathqa}
\addplot[bar width=\fgbw, bar shift=\fgsC, fill=fg-musicpile, draw=none, mark=none] table[col sep=comma, x=idx, y=musicpile, y error=musicpile_sd]{\curveroot/data/forgetting_nanogpt.csv};
\addlegendentry{musicpile}
\addplot[bar width=\fgbw, bar shift=\fgsD, fill=fg-tulu3, draw=none, mark=none] table[col sep=comma, x=idx, y=tulu3, y error=tulu3_sd]{\curveroot/data/forgetting_nanogpt.csv};
\addlegendentry{tulu3}
\addplot[only marks, bar shift=0pt, mark=diamond*, mark size=2.6pt, color=black, fill=black] table[col sep=comma, x=idx, y=avg, y error=avg_sd]{\curveroot/data/forgetting_nanogpt.csv};
\addlegendentry{average}
\end{axis}
\end{tikzpicture}

%% file: paper_figures_v2/forgetting_pythia70m.tex
\providecolor{muonrow}{RGB}{208,236,214}\providecolor{adamrow}{RGB}{255,214,214}\providecolor{sgdrow}{RGB}{208,226,247}
\providecommand{\fgM}[1]{\colorbox{muonrow}{\strut #1}}
\providecommand{\fgA}[1]{\colorbox{adamrow}{\strut #1}}
\providecommand{\fgS}[1]{\colorbox{sgdrow}{\strut #1}}
\definecolor{fg-codeparrot}{HTML}{0072B2}
\definecolor{fg-stackmathqa}{HTML}{E69F00}
\definecolor{fg-musicpile}{HTML}{009E73}
\definecolor{fg-tulu3}{HTML}{CC79A7}
\pgfplotstableread[col sep=comma]{\curveroot/data/forgetting_pythia70m.csv}\forgetTabPythia
% bar width computed from the axis width (0.9\textwidth): 4 bars fill 92% of a category, so there is no extra space between the groups
\pgfmathsetlengthmacro{\fgbw}{0.65*0.9*\textwidth/11/4}
\pgfmathsetlengthmacro{\fgsA}{-1.5*\fgbw}\pgfmathsetlengthmacro{\fgsB}{-0.5*\fgbw}
\pgfmathsetlengthmacro{\fgsC}{0.5*\fgbw}\pgfmathsetlengthmacro{\fgsD}{1.5*\fgbw}
\begin{tikzpicture}
\begin{axis}[
  ybar, width=0.9\textwidth, height=0.30\textwidth, bar width=\fgbw, ymin=0, xmin=-0.5, xmax=10.5, enlarge x limits=false,
  ymajorgrids, grid style={gray!25}, xtick=data, xticklabels={\fgM{\muon},\fgM{\specsammuon},\fgM{\fpsoma},\fgM{\soma},\fgM{\randsammuon},\fgM{\opsomapre},\fgM{\opsomapost},\fgA{\adamw},\fgS{\sgd{}$^\dagger$},\fgS{\sam{}$^\dagger$},\fgS{\fsam{}$^\dagger$}},
  x tick label style={rotate=35, anchor=east, font=\scriptsize}, tick label style={font=\scriptsize}, label style={font=\footnotesize},
  title={Pythia-70M}, title style={font=\footnotesize, yshift=0.55cm}, xlabel={Pretraining optimizer},
  ylabel={Forgetting $\Delta$PT (FineWeb loss rise)}, scaled y ticks=false,
  legend style={font=\scriptsize, at={(0.5,1.02)}, anchor=south, legend columns=5, draw=none},
  error bars/y dir=both, error bars/y explicit, error bars/error bar style={line width=0.4pt}]
\addplot[bar width=\fgbw, bar shift=\fgsA, fill=fg-codeparrot, draw=none, mark=none] table[col sep=comma, x=idx, y=codeparrot, y error=codeparrot_sd]{\curveroot/data/forgetting_pythia70m.csv};
\addlegendentry{codeparrot}
\addplot[bar width=\fgbw, bar shift=\fgsB, fill=fg-stackmathqa, draw=none, mark=none] table[col sep=comma, x=idx, y=stackmathqa, y error=stackmathqa_sd]{\curveroot/data/forgetting_pythia70m.csv};
\addlegendentry{stackmathqa}
\addplot[bar width=\fgbw, bar shift=\fgsC, fill=fg-musicpile, draw=none, mark=none] table[col sep=comma, x=idx, y=musicpile, y error=musicpile_sd]{\curveroot/data/forgetting_pythia70m.csv};
\addlegendentry{musicpile}
\addplot[bar width=\fgbw, bar shift=\fgsD, fill=fg-tulu3, draw=none, mark=none] table[col sep=comma, x=idx, y=tulu3, y error=tulu3_sd]{\curveroot/data/forgetting_pythia70m.csv};
\addlegendentry{tulu3}
\addplot[only marks, bar shift=0pt, mark=diamond*, mark size=2.6pt, color=black, fill=black] table[col sep=comma, x=idx, y=avg, y error=avg_sd]{\curveroot/data/forgetting_pythia70m.csv};
\addlegendentry{average}
\end{axis}
\end{tikzpicture}

%% file: paper_tables/table_forgetting_nanogpt.tex
% Forgetting table, nanoGPT (GPT-2 small). Values are hard-coded; needs \usepackage{booktabs} and \usepackage{graphicx}.
\begin{table*}[!t]
  \centering
  \caption{Catastrophic Forgetting of nanoGPT (GPT-2 small) models pretrained with different optimizers: held-out FineWeb-Edu loss before ($L_{\mathrm{before}}$) and after
  ($L_{\mathrm{after}}$) fine-tuning on four datasets (lower is better). $\Delta\%_{\mathrm{adj}}$ is the smoothed relative increase of the average
  $L_{\mathrm{after}}$. Mean $\pm$ std over 3 seeds. Row colour marks the base optimizer family: \muon-based (\colorbox{muonrow}{Green}), \adamw-based (\colorbox{adamrow}{Red}), \sgd-based (\colorbox{sgdrow}{Blue}).}
  \label{tab:forgetting-nanogpt}
  \vspace{2pt}
  \resizebox{\ifdim\width>\textwidth\textwidth\else\width\fi}{!}{%
  \footnotesize\setlength{\tabcolsep}{4.5pt}
  \begin{tabular}{lccccccc}
    \toprule
    & & \multicolumn{4}{c}{$L_{\mathrm{after}}$ by fine-tuning dataset} & & \\
    \cmidrule(lr){3-6}
    Optimizer & $L_{\mathrm{before}}$ & codeparrot & stackmathqa & musicpile & tulu3 & Average & $\Delta\%_{\mathrm{adj}}$ \\
    \midrule
    \rowcolor{muonrow} \muon & 3.144{\scriptsize$\pm$0.000} & 3.516{\scriptsize$\pm$0.074} & 3.526{\scriptsize$\pm$0.004} & 3.259{\scriptsize$\pm$0.000} & 3.332{\scriptsize$\pm$0.001} & 3.409{\scriptsize$\pm$0.020} & 4.20{\scriptsize$\pm$0.31} \\
    \rowcolor{muonrow} \specsammuon & 3.147{\scriptsize$\pm$0.000} & 3.386{\scriptsize$\pm$0.003} & 3.561{\scriptsize$\pm$0.006} & 3.251{\scriptsize$\pm$0.001} & 3.331{\scriptsize$\pm$0.040} & 3.382{\scriptsize$\pm$0.009} & 3.73{\scriptsize$\pm$0.15} \\
   \rowcolor{muonrow} \fpsoma & 3.146{\scriptsize$\pm$0.000} & 3.393{\scriptsize$\pm$0.004} & 3.515{\scriptsize$\pm$0.084} & 3.251{\scriptsize$\pm$0.001} & 3.308{\scriptsize$\pm$0.002} & 3.367{\scriptsize$\pm$0.022} & 3.50{\scriptsize$\pm$0.35} \\
    \rowcolor{muonrow} \soma & 3.144{\scriptsize$\pm$0.001} & 3.491{\scriptsize$\pm$0.086} & 3.511{\scriptsize$\pm$0.005} & 3.258{\scriptsize$\pm$0.002} & 3.330{\scriptsize$\pm$0.001} & 3.397{\scriptsize$\pm$0.023} & 4.03{\scriptsize$\pm$0.35} \\
    \rowcolor{muonrow} \randsammuon & 3.144{\scriptsize$\pm$0.001} & 3.558{\scriptsize$\pm$0.009} & 3.526{\scriptsize$\pm$0.007} & 3.260{\scriptsize$\pm$0.000} & 3.333{\scriptsize$\pm$0.001} & 3.419{\scriptsize$\pm$0.004} & 4.36{\scriptsize$\pm$0.06} \\
    \rowcolor{muonrow} \opsomapre & 3.144{\scriptsize$\pm$0.001} & 3.451{\scriptsize$\pm$0.081} & 3.514{\scriptsize$\pm$0.003} & 3.259{\scriptsize$\pm$0.001} & 3.330{\scriptsize$\pm$0.001} & 3.388{\scriptsize$\pm$0.020} & 3.88{\scriptsize$\pm$0.31} \\
    \rowcolor{muonrow} \opsomapost & 3.144{\scriptsize$\pm$0.001} & 3.493{\scriptsize$\pm$0.079} & 3.516{\scriptsize$\pm$0.005} & 3.259{\scriptsize$\pm$0.001} & 3.331{\scriptsize$\pm$0.001} & 3.400{\scriptsize$\pm$0.021} & 4.06{\scriptsize$\pm$0.32} \\
    \rowcolor{adamrow} \adamw & 3.142{\scriptsize$\pm$0.000} & 3.279{\scriptsize$\pm$0.076} & \textbf{3.388}{\scriptsize$\pm$0.001} & \textbf{3.248}{\scriptsize$\pm$0.001} & \textbf{3.301}{\scriptsize$\pm$0.000} & \textbf{3.304}{\scriptsize$\pm$0.019} & 2.58{\scriptsize$\pm$0.30} \\
    \rowcolor{sgdrow} \sgd & 3.185{\scriptsize$\pm$0.000} & 3.250{\scriptsize$\pm$0.000} & 3.396{\scriptsize$\pm$0.001} & 3.271{\scriptsize$\pm$0.001} & 3.323{\scriptsize$\pm$0.000} & 3.310{\scriptsize$\pm$0.001} & 1.97{\scriptsize$\pm$0.01} \\
    \rowcolor{sgdrow} \sam & 3.185{\scriptsize$\pm$0.000} & \textbf{3.249}{\scriptsize$\pm$0.001} & 3.395{\scriptsize$\pm$0.002} & 3.271{\scriptsize$\pm$0.001} & 3.323{\scriptsize$\pm$0.000} & 3.310{\scriptsize$\pm$0.000} & \textbf{1.96}{\scriptsize$\pm$0.01} \\
    \rowcolor{sgdrow} \fsam & 3.185{\scriptsize$\pm$0.000} & 3.251{\scriptsize$\pm$0.000} & 3.395{\scriptsize$\pm$0.002} & 3.271{\scriptsize$\pm$0.000} & 3.322{\scriptsize$\pm$0.000} & 3.310{\scriptsize$\pm$0.000} & \textbf{1.96}{\scriptsize$\pm$0.01} \\
    \bottomrule
  \end{tabular}}
\end{table*}

%% file: paper_tables/table_forgetting_pythia.tex
% Forgetting table, Pythia-70M. Values are hard-coded; needs \usepackage{booktabs} and \usepackage{graphicx}.
\begin{table*}[!t]
  \centering
  \caption{Catastrophic Forgetting of Pythia-70M models pretrained with different optimizers: held-out FineWeb-Edu loss before ($L_{\mathrm{before}}$) and after
  ($L_{\mathrm{after}}$) fine-tuning on four datasets (lower is better). $\Delta\%_{\mathrm{adj}}$ is the smoothed relative increase of the average
  $L_{\mathrm{after}}$. Mean $\pm$ std over 3 seeds. Row colour marks the base optimizer family: \muon-based (\colorbox{muonrow}{Green}), \adamw-based (\colorbox{adamrow}{Red}), \sgd-based (\colorbox{sgdrow}{Blue}).}
  \label{tab:forgetting-pythia70m}
  \vspace{2pt}
  \resizebox{\ifdim\width>\textwidth\textwidth\else\width\fi}{!}{%
  \footnotesize\setlength{\tabcolsep}{4.5pt}
  \begin{tabular}{lccccccc}
    \toprule
    & & \multicolumn{4}{c}{$L_{\mathrm{after}}$ by fine-tuning dataset} & & \\
    \cmidrule(lr){3-6}
    Optimizer & $L_{\mathrm{before}}$ & codeparrot & stackmathqa & musicpile & tulu3 & Average & $\Delta\%_{\mathrm{adj}}$ \\
    \midrule
   \rowcolor{muonrow} \muon & 3.515{\scriptsize$\pm$0.000} & 4.361{\scriptsize$\pm$0.005} & 4.464{\scriptsize$\pm$0.015} & 3.732{\scriptsize$\pm$0.001} & 3.888{\scriptsize$\pm$0.056} & 4.111{\scriptsize$\pm$0.013} & 8.45{\scriptsize$\pm$0.19} \\
    \rowcolor{muonrow} \specsammuon & 3.516{\scriptsize$\pm$0.001} & 4.360{\scriptsize$\pm$0.002} & 4.478{\scriptsize$\pm$0.026} & 3.733{\scriptsize$\pm$0.001} & 3.856{\scriptsize$\pm$0.004} & 4.107{\scriptsize$\pm$0.006} & 8.36{\scriptsize$\pm$0.09} \\
    \rowcolor{muonrow} \fpsoma & 3.516{\scriptsize$\pm$0.001} & 4.361{\scriptsize$\pm$0.002} & 4.480{\scriptsize$\pm$0.016} & 3.733{\scriptsize$\pm$0.001} & 3.852{\scriptsize$\pm$0.001} & 4.107{\scriptsize$\pm$0.004} & 8.36{\scriptsize$\pm$0.07} \\
     \rowcolor{muonrow} \soma & 3.517{\scriptsize$\pm$0.000} & 4.360{\scriptsize$\pm$0.004} & 4.481{\scriptsize$\pm$0.018} & 3.734{\scriptsize$\pm$0.001} & 3.855{\scriptsize$\pm$0.004} & 4.107{\scriptsize$\pm$0.003} & 8.36{\scriptsize$\pm$0.05} \\
     \rowcolor{muonrow} \randsammuon & 3.514{\scriptsize$\pm$0.000} & 4.358{\scriptsize$\pm$0.007} & 4.468{\scriptsize$\pm$0.027} & 3.732{\scriptsize$\pm$0.001} & 3.857{\scriptsize$\pm$0.002} & 4.104{\scriptsize$\pm$0.008} & 8.35{\scriptsize$\pm$0.12} \\
     \rowcolor{muonrow} \opsomapre & 3.516{\scriptsize$\pm$0.000} & 4.362{\scriptsize$\pm$0.005} & 4.489{\scriptsize$\pm$0.012} & 3.733{\scriptsize$\pm$0.000} & 3.859{\scriptsize$\pm$0.003} & 4.111{\scriptsize$\pm$0.003} & 8.42{\scriptsize$\pm$0.04} \\
    \rowcolor{muonrow} \opsomapost & 3.515{\scriptsize$\pm$0.000} & 4.359{\scriptsize$\pm$0.007} & 4.472{\scriptsize$\pm$0.023} & 3.733{\scriptsize$\pm$0.001} & 3.858{\scriptsize$\pm$0.001} & 4.106{\scriptsize$\pm$0.007} & 8.36{\scriptsize$\pm$0.11} \\
   \rowcolor{adamrow} \adamw & 3.528{\scriptsize$\pm$0.001} & 3.910{\scriptsize$\pm$0.001} & 3.914{\scriptsize$\pm$0.002} & \textbf{3.686}{\scriptsize$\pm$0.000} & \textbf{3.734}{\scriptsize$\pm$0.043} & 3.811{\scriptsize$\pm$0.011} & 4.01{\scriptsize$\pm$0.14} \\
   \rowcolor{sgdrow} \sgd & 3.631{\scriptsize$\pm$0.000} & \textbf{3.636}{\scriptsize$\pm$0.000} & \textbf{3.844}{\scriptsize$\pm$0.001} & 3.840{\scriptsize$\pm$0.037} & 3.813{\scriptsize$\pm$0.001} & \textbf{3.783}{\scriptsize$\pm$0.009} & \textbf{2.12}{\scriptsize$\pm$0.13} \\
   \rowcolor{sgdrow} \sam & 3.633{\scriptsize$\pm$0.000} & 3.637{\scriptsize$\pm$0.000} & 3.847{\scriptsize$\pm$0.000} & 3.864{\scriptsize$\pm$0.003} & 3.814{\scriptsize$\pm$0.002} & 3.791{\scriptsize$\pm$0.001} & 2.20{\scriptsize$\pm$0.01} \\
   \rowcolor{sgdrow} \fsam & 3.633{\scriptsize$\pm$0.000} & 3.637{\scriptsize$\pm$0.000} & 3.845{\scriptsize$\pm$0.001} & 3.844{\scriptsize$\pm$0.037} & 3.813{\scriptsize$\pm$0.001} & 3.785{\scriptsize$\pm$0.009} & \textbf{2.12}{\scriptsize$\pm$0.13} \\
    \bottomrule
  \end{tabular}}
\end{table*}

%% file: paper_tables/table_runtime_pretraining.tex
% Pretraining runtime (from the existing logs); row colour = base optimizer family. Values are hard-coded.
% Needs in the PREAMBLE:   \usepackage{booktabs}   \usepackage{graphicx}   \usepackage{xcolor}   \usepackage{colortbl}   and the paper's name macros (\muon, \soma, ...)
\makeatletter
\@ifundefined{rowcolor}{\PackageError{table}{Row colours need the colortbl package}{Add \string\usepackage{colortbl} (after xcolor) to the preamble.}}{}
\makeatother
\definecolor{adamrow}{RGB}{255,214,214}   % AdamW family (red)
\definecolor{sgdrow}{RGB}{208,226,247}    % SGD family   (blue)
\definecolor{muonrow}{RGB}{208,236,214}   % Muon family  (green)

\begin{table*}[!t]
  \centering
  \caption{Pretraining runtime of the optimizers: wall-clock seconds per epoch (100 optimizer steps). Row colour marks the base optimizer family: \muon-based (\colorbox{muonrow}{Green}), \adamw-based (\colorbox{adamrow}{Red}), \sgd-based (\colorbox{sgdrow}{Blue}).}
  \label{tab:runtime-pretraining}
  \vspace{2pt}
  \resizebox{\ifdim\width>\textwidth\textwidth\else\width\fi}{!}{%
  \footnotesize\setlength{\tabcolsep}{5pt}
  \begin{tabular}{lccc}
    \toprule
    Optimizer & FineWeb-Edu NanoGPT & FineWeb-Edu Pythia & CIFAR-10 Airbench CNN \\
    & s / epoch & s / epoch & s / epoch \\
    \midrule
    \rowcolor{muonrow} \muon & \textbf{265.2}{\scriptsize$\pm$3.6} & 287.2{\scriptsize$\pm$2.0} & 23.6{\scriptsize$\pm$0.3} \\
    \rowcolor{muonrow} \specsammuon & 772.5{\scriptsize$\pm$6.7} & 236.8{\scriptsize$\pm$1.9} & 23.9{\scriptsize$\pm$0.4} \\
    \rowcolor{muonrow} \fpsoma & 641.8{\scriptsize$\pm$129.0} & 240.8{\scriptsize$\pm$4.2} & 23.9{\scriptsize$\pm$0.6} \\
    \rowcolor{muonrow} \soma & 272.1{\scriptsize$\pm$1.8} & 294.4{\scriptsize$\pm$4.8} & 23.2{\scriptsize$\pm$2.1} \\
    \rowcolor{muonrow} \randsammuon & 267.2{\scriptsize$\pm$1.1} & 298.2{\scriptsize$\pm$2.6} & 23.1{\scriptsize$\pm$0.5} \\
    \rowcolor{muonrow} \opsomapre & 278.2{\scriptsize$\pm$0.6} & 297.1{\scriptsize$\pm$0.3} & 22.5{\scriptsize$\pm$1.4} \\
    \rowcolor{muonrow} \opsomapost & 300.1{\scriptsize$\pm$3.3} & 164.5{\scriptsize$\pm$0.8} & 13.6{\scriptsize$\pm$1.1} \\
    \rowcolor{muonrow} \msoma & 267.5{\scriptsize$\pm$1.7} & 281.7{\scriptsize$\pm$1.2} & \textbf{21.9}{\scriptsize$\pm$0.9} \\
    \rowcolor{adamrow} \adamw & 269.2{\scriptsize$\pm$5.6} & \textbf{185.5}{\scriptsize$\pm$6.4} & 22.5{\scriptsize$\pm$0.8} \\
    \rowcolor{adamrow} \saw & 842.8{\scriptsize$\pm$24.2} & 278.6{\scriptsize$\pm$38.0} & 22.8{\scriptsize$\pm$0.5} \\
    \rowcolor{adamrow} \fsaw & 637.4{\scriptsize$\pm$215.3} & 309.8{\scriptsize$\pm$2.8} & 22.9{\scriptsize$\pm$1.4} \\
    \rowcolor{sgdrow} \sgd & 276.3{\scriptsize$\pm$16.5} & 285.4{\scriptsize$\pm$3.0} & 22.9{\scriptsize$\pm$0.2} \\
    \rowcolor{sgdrow} \sam & 456.4{\scriptsize$\pm$3.0} & 237.3{\scriptsize$\pm$3.3} & 22.7{\scriptsize$\pm$1.9} \\
    \rowcolor{sgdrow} \fsam & 458.8{\scriptsize$\pm$0.7} & 236.7{\scriptsize$\pm$4.1} & 23.5{\scriptsize$\pm$0.8} \\
    \rowcolor{sgdrow} \msam & 333.1{\scriptsize$\pm$90.9} & 287.0{\scriptsize$\pm$3.0} & 22.2{\scriptsize$\pm$1.0} \\
    \bottomrule
  \end{tabular}}
\end{table*}

%% file: paper_tables/table_runtime_finetuning.tex
% Fine-tuning runtime (from the existing logs); row colour = base optimizer family. Values are hard-coded.
% Needs in the PREAMBLE:   \usepackage{booktabs}   \usepackage{graphicx}   \usepackage{xcolor}   \usepackage{colortbl}   and the paper's name macros (\muon, \soma, ...)
\makeatletter
\@ifundefined{rowcolor}{\PackageError{table}{Row colours need the colortbl package}{Add \string\usepackage{colortbl} (after xcolor) to the preamble.}}{}
\makeatother
\definecolor{adamrow}{RGB}{255,214,214}   % AdamW family (red)
\definecolor{sgdrow}{RGB}{208,226,247}    % SGD family   (blue)
\definecolor{muonrow}{RGB}{208,236,214}   % Muon family  (green)

\begin{table*}[!t]
  \centering
  \caption{Fine-tuning runtime of the optimizers: wall-clock seconds per fine-tuning run, taken from the fixed-base fine-tuning study (per job the median over its successful runs, then mean $\pm$ std over 3 pretrained bases $\times$ 3 seeds). Runs shared the GPUs with other jobs, so only large differences are meaningful. Row colour marks the base optimizer family: \muon-based (\colorbox{muonrow}{Green}), \adamw-based (\colorbox{adamrow}{Red}), \sgd-based (\colorbox{sgdrow}{Blue}).}
  \label{tab:runtime-finetuning}
  \vspace{2pt}
  \resizebox{\ifdim\width>\textwidth\textwidth\else\width\fi}{!}{%
  \footnotesize\setlength{\tabcolsep}{5pt}
  \begin{tabular}{lcc}
    \toprule
    Optimizer & FineWeb-Edu NanoGPT & FineWeb-Edu Pythia \\
    & s / run & s / run \\
    \midrule
    \rowcolor{muonrow} \muon & 48.6{\scriptsize$\pm$18.2} & 25.4{\scriptsize$\pm$7.0} \\
    \rowcolor{muonrow} \specsammuon & 82.0{\scriptsize$\pm$33.7} & 58.6{\scriptsize$\pm$21.4} \\
    \rowcolor{muonrow} \fpsoma & 84.2{\scriptsize$\pm$49.9} & 50.9{\scriptsize$\pm$20.2} \\
    \rowcolor{muonrow} \soma & 36.0{\scriptsize$\pm$2.5} & 27.8{\scriptsize$\pm$6.8} \\
    \rowcolor{muonrow} \randsammuon & \textbf{31.2}{\scriptsize$\pm$0.1} & 25.7{\scriptsize$\pm$5.8} \\
    \rowcolor{muonrow} \opsomapre & 42.0{\scriptsize$\pm$16.0} & 26.7{\scriptsize$\pm$5.7} \\
    \rowcolor{muonrow} \opsomapost & 34.8{\scriptsize$\pm$0.0} & 22.0{\scriptsize$\pm$0.1} \\
    \rowcolor{adamrow} \adamw & 33.7{\scriptsize$\pm$11.6} & 25.0{\scriptsize$\pm$5.6} \\
    \rowcolor{sgdrow} \sgd & 36.2{\scriptsize$\pm$13.2} & \textbf{21.1}{\scriptsize$\pm$0.2} \\
    \rowcolor{sgdrow} \sam & 84.3{\scriptsize$\pm$51.9} & 41.2{\scriptsize$\pm$6.6} \\
    \rowcolor{sgdrow} \fsam & 79.9{\scriptsize$\pm$53.5} & 44.3{\scriptsize$\pm$12.7} \\
    \bottomrule
  \end{tabular}}
\end{table*}